\documentclass[10pt]{article}
\usepackage[margin=1in]{geometry}
\usepackage[T1]{fontenc}
\usepackage[english]{babel}
\usepackage[round]{natbib}
\usepackage[dvipsnames]{xcolor}
\usepackage{amsmath, amssymb, amsthm}
\usepackage{mathtools}
\mathtoolsset{showonlyrefs}
\usepackage{bbm}
\usepackage{blkarray}
\usepackage{cancel}
\usepackage{enumitem}
\usepackage{float}
\usepackage{bm}
\usepackage{graphicx}
\usepackage{mathrsfs}
\usepackage{ragged2e}
\usepackage{thmtools}
\usepackage{tikz}
\usepackage[Symbolsmallscale]{upgreek} 
\usepackage{caption}
\usepackage{subcaption}
\usepackage[pagebackref]{hyperref}
\hypersetup{citecolor=purple, colorlinks=true, linkcolor=purple}
\usepackage{fancyhdr}
\usepackage{algorithm}
\usepackage{algpseudocode}
\usepackage{comment}
\usepackage{subfiles}

\DeclareMathOperator*{\argmax}{arg\,max}
\DeclareMathOperator*{\argmin}{arg\,min}

\DeclareUnicodeCharacter{2161}{II} 

\newtheorem{theorem}{Theorem}[section]
\newtheorem{proposition}[theorem]{Proposition}
\newtheorem{lemma}[theorem]{Lemma}
\newtheorem{corollary}[theorem]{Corollary}
\newtheorem{definition}[theorem]{Definition}
\newtheorem{assumption}[theorem]{Assumption}
\newtheorem{remark}[theorem]{Remark}

\title{Wasserstein Gradient Flows for Batch Bayesian Optimal Experimental Design}
\author{
Louis Sharrock\thanks{Department of Statistical Science, University College London. \texttt{l.sharrock@ucl.ac.uk}}
}
\date{}

\begin{document}

\maketitle

\begin{abstract}
Bayesian optimal experimental design (BOED) provides a powerful, decision-theoretic framework for selecting experiments so as to maximise the expected utility of the data to be collected. In practice, however, its applicability can be limited by the difficulty of optimising the chosen utility. The expected information gain (EIG), for example, is often high-dimensional and strongly non-convex. This challenge is particularly acute in the batch setting, where multiple experiments are to be designed simultaneously. In this paper, we introduce a new approach to batch EIG-based BOED via a probabilistic lifting of the original optimisation problem to the space of probability measures. In particular, we propose to optimise an entropic regularisation of the expected utility over the space of design measures. Under mild conditions, we show that this objective admits a unique minimiser, which can be explicitly characterised in the form of a Gibbs distribution. The resulting design law can be used directly as a randomised batch-design policy, or as a computational relaxation from which a deterministic batch is extracted. To obtain scalable approximations when the batch size is large, we then consider two tractable restrictions of the full batch distribution: a mean-field family, and an i.i.d. product family. For the i.i.d. objective, and formally for its mean-field extension, we derive the corresponding Wasserstein gradient flow, characterise its long-time behaviour, and obtain particle-based algorithms via space-time discretisations. We also introduce doubly stochastic variants that combine interacting particle updates with Monte Carlo estimators of the EIG gradient. Finally, we illustrate the performance of the proposed methods in several numerical experiments, demonstrating their ability to explore multimodal optimisation landscapes and obtain high-utility batches in challenging examples.
\end{abstract}

\section{Introduction}
Bayesian optimal experimental design (BOED) provides a principled decision-theoretic framework for selecting experiments so as to maximise the expected inferential value of the data to be collected \citep{chalonerverdinelli1995review, ryan2016review}. In its classical formulation, BOED chooses a design variable $\xi\in\Xi\subseteq\mathbb{R}^d$ in order to maximise an expected utility. Among the most widely used utilities is the expected information gain (EIG), which is equivalent to the mutual information between the parameter $\theta\in\Theta\subseteq\mathbb{R}^p$ and the prospective observation $y\in\mathcal{Y}\subseteq\mathbb{R}^q$ under the prior predictive model \citep{lindley1956information, chalonerverdinelli1995review}. BOED is increasingly central in domains where experiments are expensive, slow, or ethically constrained. Examples include Bayesian adaptive clinical trials (e.g., response-adaptive randomisation and interim decision rules), where one seeks to learn efficiently while maintaining safety and power \citep{giovagnoli2021bayesian}; sensor placement and data acquisition for large-scale inverse problems governed by PDEs, where measurements must be deployed under severe budget constraints \citep{alexanderian2021review}; and systems biology and pharmacometrics, where non-linear dynamical models are calibrated from limited, noisy time-series data and design choices strongly determine identifiability \citep{kreutz2009systems}.

Despite its conceptual appeal, EIG-based BOED is computationally challenging. Evaluating $\mathrm{EIG}(\xi)$ and its gradient entails nested expectations over $(\theta,y)$ that are rarely available in closed form. As a result, practical BOED algorithms often hinge on Monte Carlo approximations and their refinements (e.g.\ multilevel methods), whose bias--variance-cost trade-offs are subtle in nested settings \citep{rainforth2018nesting, goda2020mlmc, huan2024acta}. Moreover, even when the design space $\Xi$ is low- or moderate-dimensional, the expected-utility landscape is typically multimodal and strongly non-convex. These issues are amplified in the batch setting, where one must choose $\xi_{1:m} = (\xi_1,\dots,\xi_m)\in\Xi^m$ experiments simultaneously: the ambient dimension grows to $md$, and the utility landscape becomes increasingly complex due to interactions between design points. These difficulties have motivated a broad spectrum of methods, including simulation-based ``design by sampling'' schemes targeting Gibbs-type design distributions \citep{muller2005simulation, amzal2006bayesian}, surrogate-assisted and stochastic-approximation approaches for high-dimensional designs \citep{huan2013simulation, overstallwoods2017ace}, and more recent variational and amortised estimators of the EIG, based on tractable lower bounds on mutual information \citep{barberagakov2003information, foster2019variational, foster2020unified, foster2021dad} as well as neural ratio or mutual information estimation for implicit or likelihood-free models \citep{belghazi2018mutual, kleinegesse2020bayesian, kleinegesse2021sequential}.

In this paper, we introduce a new, distributional formulation of EIG-based batch BOED by lifting the design variable from a point design ${\xi}_{1:m} = (\xi_1,\dots,\xi_m)$ to a design measure $\nu_m\in \mathcal{P}(\Xi^m)$. In particular, we propose to minimise the free energy functional
\begin{equation}
\label{eq:intro-free-energy}
\mathcal F_{m}^{\lambda}(\nu_m)
\;=\;
-\int_{\Xi^m} \mathrm{EIG}_m(\xi_{1:m})\,\nu_m(\mathrm d\xi_{1:m})
\;+\;
\lambda_m\,\mathrm{KL}(\nu_m\|\rho_m),
\end{equation}
where $\rho_m\in\mathcal{P}(\Xi^m)$ is a reference measure, and $\lambda_m>0$ is a regularisation parameter which plays the role of a temperature: large values of $\lambda_m$ favour exploration, while small values concentrate the design measure around batches that achieve a high value of the EIG. Under an exponential-integrability condition ensuring the normaliser is finite, the objective $\mathcal F_m^\lambda$ is strictly convex in $\nu_m$, and admits a unique minimiser with explicit Gibbs form
\begin{equation}
\label{eq:intro-gibbs}
\frac{\mathrm d\nu_m^{\lambda,\star}}{\mathrm d\rho_m}(\xi_{1:m})
\;=\;
\frac{1}{Z_m^\lambda}\exp\!\left(\frac{\mathrm{EIG}_m(\xi_{1:m})}{\lambda_m}\right),
\qquad
Z_m^\lambda \coloneqq \int_{\Xi^m}\exp\!\left(\frac{\mathrm{EIG}_m(\xi_{1:m})}{\lambda_m}\right)\rho_m(\mathrm d\xi_{1:m}).
\end{equation}
This distributional viewpoint is illustrated in Figure~\ref{fig:1}: we lift optimisation over a design point (Fig.~\ref{fig:1a}) to optimisation over a design distribution (Fig.~\ref{fig:1b}), and incorporate an entropic regularisation that renders the variational problem well-posed and tunably concentrated (Fig.~\ref{fig:1c}). While this perspective is new to BOED, it has previously proved fruitful in other fields \citep[e.g.,][]{wild2023rigorous}. 

\begin{figure}[t!]
    \centering
    \begin{subfigure}[b]{0.28\textwidth}
        \centering
        \includegraphics[width=\linewidth]{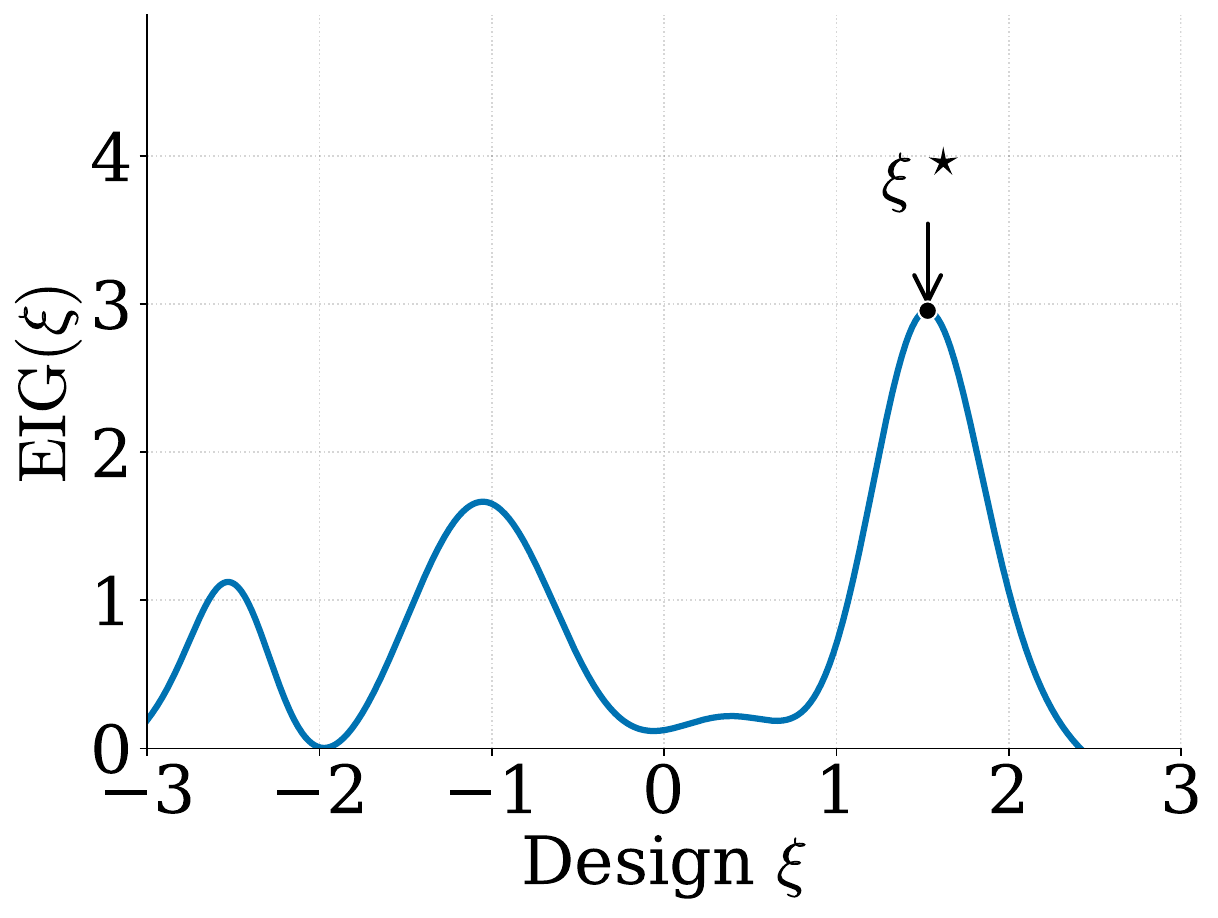}
        \caption{\textbf{Pointwise Optimisation}. Find $\xi^{*}=\argmax F(\xi)$, where $F(\xi)=\mathrm{EIG}(\xi)$.}
        \label{fig:1a}
    \end{subfigure}
    \hfill 
    \begin{subfigure}[b]{0.28\textwidth}
        \centering
        \includegraphics[width=\linewidth]{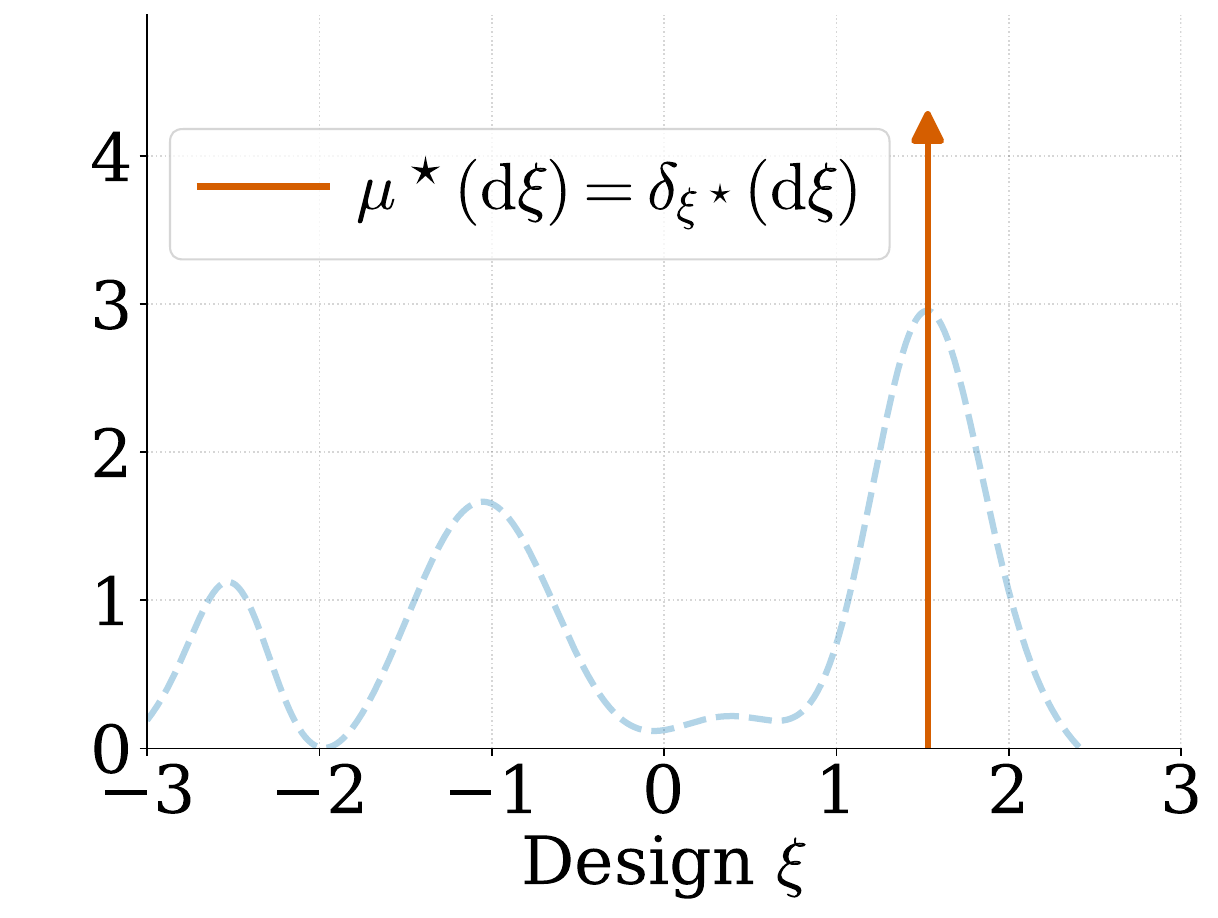}
        \caption{\textbf{Probabilistic Lifting}. Find $\mu^{*}\in\argmin \mathcal{F}(\mu)$, where $\mathcal F(\mu)=-\int_{\Xi} F(\xi)\mu(\mathrm{d}\xi)$.}
        \label{fig:1b}
    \end{subfigure}
    \hfill 
    \begin{subfigure}[b]{0.28\textwidth}
        \centering
        \includegraphics[width=\linewidth]{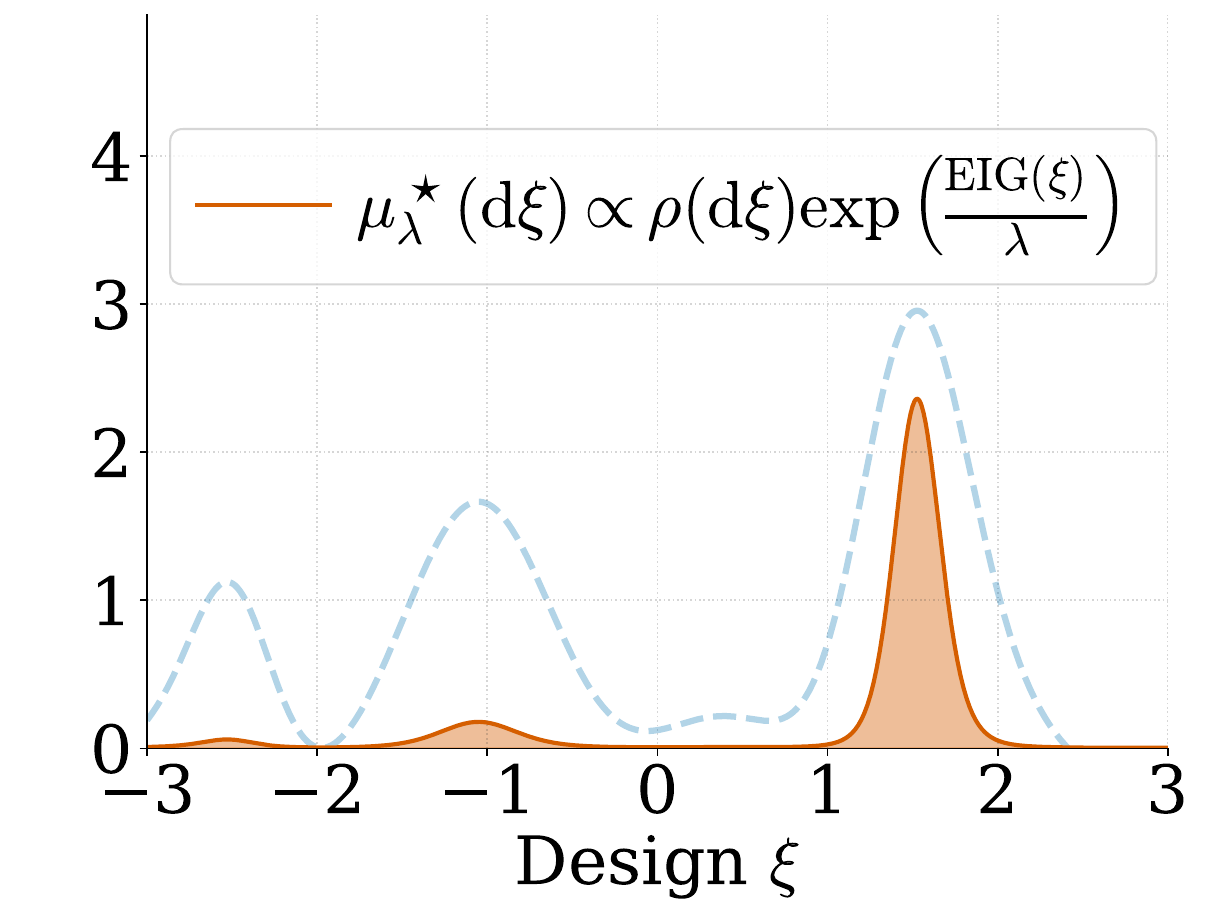}
        \caption{\textbf{Entropic Regularisation}. Find $\mu_{\lambda}^{*}=\argmin\mathcal{F}_{\lambda}(\mu)$, where $\mathcal{F}_{\lambda}(\mu) = \mathcal{F}(\mu) + \lambda \mathrm{KL}(\mu\|\rho)$.}
        \label{fig:1c}
    \end{subfigure}
    \caption{\textbf{Bayesian optimal experimental design as an optimisation problem over the space of probability measures}. We lift the original optimisation problem over a \emph{design point} $\xi\in\Xi$ (Fig. \ref{fig:1a}) to an optimisation problem over a \emph{design distribution} $\mu\in\mathcal{P}(\Xi)$ (Fig. \ref{fig:1b}), before incorporating an entropic regulariser to ensure that this optimisation problem is strictly convex, and thus admits a unique optimum (Fig.~\ref{fig:1c}).}
    \label{fig:1}
\end{figure}

Reformulating BOED as an optimisation problem over $\nu_m\in\mathcal{P}(\Xi^m)$ is meaningful in two distinct ways. First, in a decision-theoretic framework, it is valid to treat \emph{randomised designs} as decisions. Thus, one can deploy the learned design law directly by sampling $\xi_{1:m}\sim\nu_m^{\lambda,\star}$. In this case, it should be noted that $\nu_m^{\lambda,\star}$ is optimal with respect to the \emph{regularised} criterion, and not for the unregularised expected utility alone. Accordingly, $\lambda_m$ may be interpreted either as an explicit preference for exploration or as a computational relaxation parameter whose influence disappears in the zero-temperature limit. Second, in many BOED applications, one ultimately requires a deterministic batch $\hat{\xi}_{1:m}$. In this case, $\nu_m^{\lambda,\star}$ is best viewed as a computational surrogate that supports a principled extraction step, e.g., a best-of-$n$ (BoN) approach which samples $n$ candidate batches from $\nu_m^{\lambda,\star}$, and selects the one with the largest (estimated) EIG.

In practice, directly optimising the entropy-regularised objective over $\mathcal P(\Xi^m)$ is challenging when $m$ is (very) large. We thus study two tractable restrictions of the batch law. The first is a mean-field product family, which allows independent but non-identical coordinates, $\nu_m=\mu_1\otimes\cdots\otimes\mu_m$, with each $\mu_i\in\mathcal P(\Xi)$. The second is an i.i.d.\ family enforcing exchangeability, $\nu_m=\mu^{\otimes m}$, for a single $\mu\in\mathcal P(\Xi)$. These restrictions reduce the computational cost of the optimisation problem, while retaining enough flexibility to represent diverse batches, particularly when combined with an appropriate extraction step. They may also be viewed as structured variational approximations to the joint Gibbs law, and lead to explicit stationary conditions in the form of fixed-point (or self-consistency) equations, which clarify exactly how the full batch dependence is being approximated.

For the i.i.d.\ objective and, formally, for the coordinate-wise product objective and the joint objective, we identify the corresponding Wasserstein gradient flow (WGF). In the joint convex case, its unique equilibrium is the Gibbs minimiser of the free energy. In the mean-field and i.i.d.\ cases, the resulting non-linear flows have equilibria characterised by the corresponding self-consistency equations. These gradient flows are non-linear in the sense of McKean \citep[e.g.,][]{sznitman1991topics,meleard1996asymptotic,malrieu2001logarithmic}, but can be approximated via a (stochastic) interacting particle system (IPS). In practice, since $\nabla \mathrm{EIG}_m$ is itself often intractable, this in fact leads to doubly stochastic algorithms, which combine particle updates with Monte Carlo estimators of $\nabla \mathrm{EIG}_m$, enabling scalable implementation in nested-expectation settings. This structure supports a theoretical analysis in which the overall error separates into finite-particle effects (i.e., propagation of chaos), time discretisation error, and stochastic-gradient error.

\paragraph{Contributions} Our main contributions are summarised below.
\begin{itemize}
    \item We formulate batch EIG-based BOED as an entropy-regularised variational optimisation problem over $\mathcal P(\Xi^{m})$, and establish existence and uniqueness of the optimal design law. 
    \item We introduce two scalable approximations to the batch design law, namely,  a mean-field product family $\nu_m=\mu_1\otimes\cdots\otimes \mu_m$, and a homogeneous i.i.d.\ product family $\nu_m = \mu^{\otimes m}$, and derive the corresponding fixed point equations satisfied by their global minimisers.
    \item For the i.i.d.\ design-law objective and, formally, for the coordinate-wise product objective, we derive the associated WGFs, identify them with non-linear (i.e., McKean--Vlasov) Fokker--Planck PDEs, and obtain the corresponding mean-field SDEs.
    \item We obtain scalable particle-based algorithms as space-time discretisations of these dynamics, as well as doubly stochastic variants that accommodate (nested) Monte Carlo estimators of the intractable gradient of the EIG. 
    \item For the i.i.d.\ objective under an unbiased gradient oracle, we provide a finite-horizon error decomposition which separates the effects of finite particle number, time-discretisation, and stochastic-approximation.
    \item We illustrate the application of our proposed approach in several numerical experiments, demonstrating its efficacy in multimodal and non-convex settings. 
\end{itemize}

\subsection{Related Work}
\label{sec:related-work}

Bayesian optimal experimental design (BOED) has a long history, with classical roots in Bayesian decision theory and (approximate) optimal design, and a large modern literature driven by the computational demands of non-linear, high-dimensional, and simulation-based models. In this section, we position the current work relative to several distinct themes. Broader overviews of BOED can be found in, e.g., \citet{chalonerverdinelli1995review,ryan2016review,huan2024acta,rainforth2023modern}.

\paragraph{Approximate design theory and optimisation over design measures} Optimisation over design measures is classical in approximate optimal design theory, where a design is represented by a probability measure over candidate points and optimality criteria are optimised over a convex set; see \citet{kiefer1959optimum,kieferwolfowitz1960equivalence,fedorov1972theory,pukelsheim2006optimal,atkinsondonevtobias2007sas}. Our measure theoretic viewpoint is philosophically aligned with this tradition, but the specific setting is different. In particular, classical approximate-design literature typically focuses on criteria derived from (linearised) information matrices and often enjoys structure (e.g., convexity, duality) that is absent from BOED with EIG \citep{chalonerverdinelli1995review, ryan2016review}.

\paragraph{Simulation-based optimal design and design by sampling}
Somewhat closer in spirit to our work is a series of papers which replace direct optimisation of the expected utility over the design space by sampling from an augmented design distribution that concentrates in high-utility regions \citep[e.g.,][]{clyde1995exploring,bielza1999decision,muller1999bayesstat6,muller2004inhomogeneous,muller2005simulation,amzal2006bayesian}. In this line of work, the design variable is treated as a random variable, and various sampling methods are used to explore the induced ``utility landscape'' \citep{clyde1995exploring,bielza1999decision,muller2004inhomogeneous,kuck2006smc,amzal2006bayesian}. Our approach can, in some sense, be viewed as a particular instantiation of this general framework for the EIG: we also propose to sample from a particular design distribution, namely, the optimiser of the entropy-regularised version of the expected utility, which concentrates in high-utility regions. On the other hand, our perspective is rather different from the classical one: we view this distribution explicitly as the optimiser of an entropy-regularised functional over the space of probability measures. This shift in viewpoint is not merely philosophical; rather, it has algorithmic and theoretic consequences. While the sampling-based design literature typically focuses on MCMC, SMC, or annealing-based schemes \citep{muller2004inhomogeneous,muller2005simulation,kuck2006smc,amzal2006bayesian}, we instead construct WGFs which converge to the target distribution in the long-time limit \citep[e.g.,][]{ambrosio2008gradientflows}. Our approach leads naturally to scalable particle-based algorithms via a space-time discretisation, which in turn provides a direct route to a principled and modular theoretical analysis. It also accommodates scalable product-measure restrictions for large batch sizes.

\paragraph{Wasserstein gradient flows in optimal design} The study of WGFs dates back to the seminal work of \citet{jordan1998variational}; see also \cite{otto2001geometry,ambrosio2008gradientflows} for other classical references. In the context of experimental design, several recent papers have begun to explore gradient-flow formulations for {classical} optimal design criteria directly in measure space, both on finite candidate sets and on continuous design spaces \citep[e.g.,][]{piazzon2022logdetgf,jin2024oedgf,shi2026eoptimalgf,jin2024nonlinearoedgf}. In particular, \citet{shi2026eoptimalgf} develops WGFs for E-optimal experimental designs in regression models, while \citet{jin2024oedgf,jin2024nonlinearoedgf} consider A- and D- optimal designs in linear and non-linear settings, respectively. Even more recently, concurrent work by \citet{makinen2026batch} introduces a WGF-based approach to batch Bayesian A-optimal design in linear inverse problems via a design-measure relaxation, including a practical regularisation scheme that guarantees convergence to separated point designs. Adjacent ideas also appear in batch Bayesian optimisation, where convex acquisition functionals over probability measures lead naturally to particle gradient flows \citep{crovini2022batchpgf}. Our work is complementary: we focus on EIG-based BOED, whose nested-expectation structure is qualitatively different from classical matrix-based criteria; introduce an entropy-regularised objective that, for the full joint batch law, admits a unique Gibbs optimiser; and develop scalable mean-field and i.i.d.\ design-law parameterisations, together with (doubly stochastic) interacting-particle algorithms amenable to nested Monte Carlo gradient estimation.

\paragraph{Particle-based and diffusion-based experimental design}
Recent BOED methods also use particle or diffusion machinery, but without adopting a measure-valued design variable. In particular, \citet{iollo2024pasoa} combine stochastic optimisation with tempered SMC for sequential EIG-based design, while \citet{iollo2025contrastive} use diffusion-based samplers together with a pooled-posterior construction and a new EIG-gradient representation. These methods are complementary to ours: they target pointwise sequential design rules, rather than optimisation over a design law. Another adjacent contribution is \citet{helin2025wasserstein}, who replace the KL-based EIG utility by expected Wasserstein information criteria. In contrast, we retain the EIG utility and alter the optimisation geometry.

\paragraph{EIG estimation and gradient estimation}
The nested-expectation structure of EIG is central to BOED computation and leads to subtle bias--variance trade-offs \citep[e.g.,][]{rainforth2018nesting}. Several approaches have been proposed to address this, including multilevel and de-biasing ideas \citep[e.g.,][]{goda2020mlmc}, alternative EIG gradient representations \citep{ao2024eiggrad}, transport or density-approximation approaches \citep[e.g.,][]{li2024eigdensity}, and variational or ratio-estimation strategies \citep[e.g.,][]{barberagakov2003information,foster2019variational,foster2020unified,foster2021dad,kleinegesse2020bayesian,huan2024acta,rainforth2023modern}. Our methodology is modular at the level of the \emph{inner} stochastic approximation. The particle updates require estimates of $\nabla_{\xi}\mathrm{EIG}(\xi)$ (or its batch analogue), which can be instantiated using nested Monte Carlo, multilevel or de-biased estimators, Laplace or variational approximations, transport-map density surrogates, or likelihood-free ratio or mutual information estimators, provided appropriate moment bounds and, where relevant, bias control conditions hold.

\subsection{Paper Organisation}
The remainder of this paper is organised as follows. In Section~\ref{sec:background}, we introduce notation and define the problem setup. In Section~\ref{sec:method}, we introduce our main methodology: an entropy-regularised lifting of batch EIG maximisation to an optimisation problem over design laws on $\Xi^m$, together with mean-field and i.i.d.\  product-measure restrictions for scalability. We then derive the associated WGFs and their interacting-particle discretisations, and state our main theoretical guarantees. In Section~\ref{sec:numerics}, we provide numerical results which demonstrate the efficacy of our proposed approach. Finally, in Section~\ref{sec:conclusions}, we present some concluding remarks.

\section{Background and Problem Setup}
\label{sec:background}

\paragraph{Model and Notation}
We adopt the following notation. The experimental design is characterised by a continuous parameter $\xi\in\Xi\subseteq\mathbb{R}^d$. We aim to choose the design so as to maximise the expected information gain (EIG) about a parameter of interest $\theta\in\Theta\subseteq\mathbb{R}^p$, given prospective data $y\in\mathcal{Y}\subseteq\mathbb{R}^q$. Let $\pi(\theta)$ denote the prior, $\pi_{\xi}(y\mid\theta)$ the likelihood, and $\pi_{\xi}(\theta\mid y)\propto \pi(\theta)\pi_{\xi}(y\mid\theta)$ the posterior.\footnote{For convenience, we work with densities with respect to Lebesgue measure; the definitions extend verbatim to general dominating measures.}

\paragraph{The Expected Information Gain} The expected information gain (EIG) is the expected Kullback-Leibler (KL) divergence between the posterior and the prior, viz
\begin{equation}
    \mathrm{EIG}(\xi)
    :=
    \int \mathrm{KL}\big(\pi_\xi(\cdot\mid y)\,\|\,\pi\big)\, \pi_\xi(y)\,\mathrm dy
    =
    \iint \log \left[\frac{\pi_{\xi}(\theta\mid y)}{\pi(\theta)}\right]\,
    \pi_{\xi}(\theta\mid y)\,\pi_{\xi}(y)\,
    \mathrm{d}\theta \, \mathrm{d}y,
\end{equation}
where $\pi_{\xi}(y):=\int \pi_\xi(y\mid\theta)\pi(\theta)\mathrm{d}\theta$ denotes the prior predictive (evidence). Equivalently, the EIG is the mutual information between $\theta$ and $y$ under the joint $\pi_{\xi}(\theta,y):=\pi(\theta)\pi_{\xi}(y|\theta)$, namely, 
\begin{align}
\mathrm{EIG}(\xi)=\iint \log\left[\frac{\pi_\xi(y\mid \theta)}{\pi_\xi(y)}\right]\,\pi_\xi(y\mid \theta)\,\pi(\theta)\,\mathrm{d}\theta\,\mathrm{d}y = \mathbb E_{\theta\sim\pi}\,
\mathbb E_{y\sim \pi_\xi(\cdot\mid\theta)}
\Big[\log \pi_\xi(y\mid\theta)-\log \pi_\xi(y)\Big]. \label{eq:eig-mi}
\end{align}
This form is often the most useful for computation, since it makes explicit that $\mathrm{EIG}(\xi)$ is an expected log-likelihood ratio between $\pi_\xi(y\mid \theta)$ and the prior-predictive $\pi_\xi(y)$.

\paragraph{Bayesian Optimal Experimental Design} The (point) Bayesian optimal design is then defined as the (possibly non-unique) solution of
\begin{equation}
    \xi^{*} \in \argmax_{\xi\in\Xi} \mathrm{EIG}(\xi).
\end{equation}

\subsection{Batch Design}
\label{sec:batch}
Let $m\in\mathbb N$ denote the number of experiments to be performed without adaptation, and write $\xi_{1:m}:=(\xi_1,\dots,\xi_m)\in \Xi^m$ and $y_{1:m}:= (y_1,\dots,y_m)\in \mathcal Y^m$. We assume a standard conditionally independent observation model, so that $\pi_{{\xi}_{1:m}}(y_{1:m}\mid \theta)=\prod_{b=1}^m \pi_{\xi_b}(y_b\mid \theta)$ and $\smash{\pi_{\xi_{1:m}}(\theta\mid y_{1:m}) \propto \pi(\theta)\prod_{b=1}^m \pi_{\xi_b}(y_b\mid\theta)}$. The batch EIG is then given by
\begin{align}
\label{eq:eig-batch-kl}
\mathrm{EIG}_m(\xi_{1:m})
&:=
\iint \log \left[\frac{\pi_{\xi_{1:m}}(\theta\mid y_{1:m})}{\pi(\theta)}\right]\,
    \pi_{\xi_{1:m}}(\theta\mid y_{1:m})\,\pi_{\xi_{1:m}}(y_{1:m})\,
    \mathrm{d}\theta \, \mathrm d y_{1:m} \\
&:=
    \iint
    \log\left[
        \frac{\pi_{\xi_{1:m}}(y_{1:m}\mid \theta)}{\pi_{\xi_{1:m}}(y_{1:m})}
    \right]\,
    \pi_{\xi_{1:m}}(y_{1:m}\mid \theta)\,\pi(\theta)\,
    \mathrm{d}\theta\,\mathrm{d}y_{1:m}, \label{eq:eig-batch}
\end{align}
where, similar to before, $\pi_{\xi_{1:m}}(y_{1:m})=\int \pi_{\xi_{1:m}}(y_{1:m}\mid \theta)\,\pi(\theta)\,\mathrm{d}\theta$. The Bayesian optimal batch design is then given by any solution of 
\begin{equation}
{\xi}_{1:m}^\star \in \argmax_{{\xi}_{1:m}\in \Xi^m}\ \mathrm{EIG}_m({\xi}_{1:m}). \label{eq:eig-batch-opt}
\end{equation}

\section{Methodology}
\label{sec:method}
A fundamental challenge in EIG-based BOED is that the resulting optimisation problem is typically non-convex and often ill-conditioned, even for moderate-dimensional design spaces $\Xi$. This difficulty is intrinsic to the nested structure of the EIG, and particularly acute in the batch setting. In particular, the EIG consists of an expectation of a log-marginal likelihood term (i.e., a log-sum or a log-integral), evaluated under the corresponding prior predictive $\pi_{\xi_{1:m}}(y_{1:m})$. In practice, this commonly yields multimodal landscapes with many local optima, rendering global convergence guarantees for first-order methods unrealistic outside special cases. Moreover, in most models the gradient $\nabla_\xi \mathrm{EIG}(\xi)$ is only available via nested Monte Carlo (or related) estimators, so optimisation must contend simultaneously with non-convexity and stochastic (often biased) gradient information \citep[e.g.,][]{rainforth2018nesting}.

\subsection{A Distributional Objective via Entropic Regularisation}
\label{sec:method:objective}
To mitigate these issues, we propose to lift the optimisation variable from a point design ${\xi}_{1:m} = (\xi_1,\dots,\xi_m)$ to a design measure $\nu_m \in \mathcal{P}(\Xi^m)$, and introduce an entropic regularisation. The resulting formulation replaces a non-convex point optimisation problem by a strictly convex optimisation problem on the space of probability measures; naturally accommodates exploration through randomness in the design; and provides a canonical route to scalable particle algorithms via WGFs.

Let $\nu_m\in\mathcal P(\Xi^m)$ be a batch design measure, and write $G(\xi_{1:m}):=\mathrm{EIG}_m(\xi_{1:m})$ for the (deterministic) batch EIG associated with a fixed design vector $\xi_{1:m}\in\Xi^m$. We can then define the expected batch utility under $\nu_m$ as
\begin{equation}
\mathcal J_m^{\mathrm{joint}}(\nu_m)
:=
\int_{\Xi^m} G(\xi_{1:m})\,\nu_m(\mathrm d\xi_{1:m}).
\label{eq:Jm-joint-method}
\end{equation}
For now, we will not impose any restrictions on the space of probability measures over which we optimise. The measure-valued design problem is thus 
\begin{equation}
    \nu_m^{*}\in \argmax_{\nu_m\in\mathcal{P}(\Xi^m)} \mathcal{J}_m^{\mathrm{joint}}(\nu_m). \label{eq:measure-opt}
\end{equation}
Under the assumption that $G(\xi_{1:m}):=\mathrm{EIG}_m(\xi_{1:m})$ attains its maximum on $\Xi^m$, the measure-valued optimisation problem in \eqref{eq:measure-opt} is a value-preserving relaxation of the pointwise optimisation problem in \eqref{eq:eig-batch-opt}. In particular, $\max_{\nu_m\in\mathcal{P}(\Xi^m)} \mathcal{J}_m^{\mathrm{joint}}(\nu_m)=\max_{\xi_{1:m}\in\Xi^m} G(\xi_{1:m})$, and every $\nu_m\in\mathcal{P}(\Xi^m)$ supported on the set $\argmax_{\xi_{1:m}\in\Xi^m} G(\xi_{1:m})$ is optimal (see Lemma~\ref{lem:lift-value-preserving}, Appendix~\ref{sec:joint-batch-basic-results}). Consequently, the lifted problem is {highly non-identifiable}: it admits an entire simplex of global maximisers, and provides no intrinsic mechanism to select between them. 

From an algorithmic perspective, this degeneracy is problematic. First, the objective $\nu_m\mapsto\mathcal J_m^{\mathrm{joint}}(\nu_m)$ is linear (indeed, affine), and thus is not strictly convex (nor strictly concave). This means that measure-valued gradient-based dynamics on $\mathcal{P}(\Xi^m)$ are not stabilised by curvature, and may drift along flat directions. Second, in multimodal landscapes, the global maximisers are often accompanied by many local maximisers separated by energy barriers. In practice, the result is that particle-based approximations of the gradient-flow dynamics may become trapped in basins of attraction determined by the initialisation (see Proposition~\ref{prop:joint-wgf-basins},  Appendix~\ref{sec:joint-batch-basic-results}).

In this context, we now introduce an entropic regularisation, which renders the variational problem {strictly convex} and yields a {unique} optimiser with explicit Gibbs form. Let $\rho_m\in\mathcal{P}_{2,\mathrm{ac}}(\Xi^m)$ be a reference probability measure on $\Xi^m$, with density $\rho_m\propto e^{-V_m(\xi)}$ for some confining potential $V_{m}:(\mathbb{R}^d)^m\rightarrow \mathbb{R}$.\footnote{We write $\mathcal{P}_{2,\mathrm{ac}}(\Xi^m) = \mathcal{P}_2(\Xi^m)\cap\mathcal{P}_{\mathrm{ac}}(\Xi^m)$, where $\mathcal{P}_2(\Xi^m) = \{\nu_m\in\mathcal{P}(\Xi^m):\int \|\xi_{1:m}\|^2 \nu_m(\mathrm{d}\xi_{1:m})<\infty\}$ denotes the space of measures over $\Xi^m$ with finite second moment, and $\mathcal{P}_{\mathrm{ac}}(\Xi^m)$ the space of measures over $\Xi^m$ which are absolutely continuous w.r.t. the Lebesgue measure.} We then define a regularised version of our objective function as 
\begin{equation}
    \mathcal{F}_m^{\lambda,\mathrm{joint}}(\nu_m)
    =
    - \mathcal{J}_m^{\mathrm{joint}}(\nu_m) + \lambda_m \mathrm{KL}(\nu_m\| \rho_m).
    \label{eq:Fm-joint}
\end{equation}
where $\lambda_m>0$ denotes a regularisation parameter, and  $\mathrm{KL}(\cdot\|\cdot)$ denotes the Kullback--Leibler divergence, defined by $\mathrm{KL}(\nu\|\rho) = \int \log [\frac{\mathrm{d}\nu}{\mathrm{d}\rho}] \nu(\mathrm{d}\xi)$ if $\nu\ll \rho$, and $+\infty$ otherwise. The parameter $\lambda_m$ plays the role of a temperature: as $\lambda_m\downarrow 0$, the optimiser concentrates on high-utility regions, while larger $\lambda_m$ yields more exploratory designs. This temperature interpretation can be made rigorous: as $\lambda_m\downarrow 0$, the joint Gibbs optimiser $\nu_m^{\lambda,\star}$ (see below) concentrates on the set of global maximisers of the EIG; in particular, if the batch maximiser is unique, then $\nu_m^{\lambda,\star}\Rightarrow \delta_{\xi_{1:m}^\star}$ (see Theorem~\ref{thm:joint-zero-temp}, Appendix~\ref{sec:zero-temp-joint}). In any case, the entropy-regularised batch design problem is then given by
\begin{equation}
    \nu_m^{\lambda,\star} \in \argmin_{\nu\in\mathcal{P}(\Xi^m)} \mathcal{F}_m^{\lambda,\mathrm{joint}}(\nu).
    \label{eq:nu-joint-opt}
\end{equation}
 Under a mild integrability assumption, the functional $\nu\mapsto \mathcal{F}_{m}^{\lambda,\mathrm{joint}}(\nu)$ is proper and strictly convex on its effective domain $\{\nu_m:\mathrm{KL}(\nu_m\|\rho_m)<\infty\}$. Moreover, it admits a unique minimiser $\nu_m^{\lambda,\star}$, given by the Gibbs change of measure
\begin{equation}
    \nu_m^{\lambda,\star}(\mathrm{d}\xi_{1:m})
    =
    \frac{1}{Z_{m}^{\lambda}}
    \exp \left(\frac{1}{\lambda_m}G(\xi_{1:m})\right)\, \rho_m(\mathrm{d}\xi_{1:m}),
    \label{eq:nu-joint}
\end{equation}
where $\smash{Z}_m^{\lambda}$ is the normalising constant $\smash{Z_{m}^{\lambda}=\int \exp(\frac{1}{\lambda_m}G(\xi_{1:m}))\rho_m(\mathrm{d}\xi_{1:m})}$ (see Proposition~\ref{prop:joint-gibbs}, Appendix~\ref{sec:joint-batch-basic-results}). This follows from standard variational arguments \citep[e.g.,][]{donsker1975asymptotic}. In particular, using the Gibbs variational principle, the batch objective can be rewritten in the form
\begin{equation}
    \mathcal{F}_m^{\lambda,\mathrm{joint}}(\nu)
    =
    \lambda_m \mathrm{KL}(\nu\| \nu_m^{\lambda,\star}) -\lambda_m \log Z_m^\lambda.
    \label{eq:Fm-joint-rewrite}
\end{equation}
Thus, restricting $\nu_m$ to a tractable family (e.g., product measures) is precisely a reverse-KL variational approximation to the Gibbs law $\nu_m^{\lambda,\star}$.

\subsubsection{Mean-field restriction to independent designs}
\label{sec:method:batch-std-mf}
While the entropy-regularised batch design problem yields an explicit Gibbs solution $\nu_m^{\lambda,\star}$ on $\Xi^m\subseteq(\mathbb{R}^d)^m$, sampling from this joint law may be computationally prohibitive when the batch size $m$, and hence the ambient dimension $md$, is moderate or large. In this context, we now consider tractable approximations to the full batch law $\nu_m\in\mathcal P(\Xi^m)$. We begin by considering the restriction to the \emph{mean-field} family of \emph{product} measures, viz
\begin{equation}
\mathcal P_{\mathrm{mf}}(\Xi^m)
:=
\Big\{
\mu_1\otimes\cdots\otimes\mu_m:\ \mu_b\in\mathcal P(\Xi)\ \text{for }b=1,\dots,m
\Big\}.
\label{eq:P-mf}
\end{equation}
This family enforces independence across batch coordinates but allows non-identical marginals, so different batch elements may specialise to different regions of $\Xi$ while retaining tractability. Explicitly, the mean-field restriction of the entropy-regularised batch design objective in \eqref{eq:Fm-joint} is given by
\begin{align}
\mathcal F_{m}^{\lambda,\mathrm{mf}}(\mu_1,\dots,\mu_m)
&:=
\mathcal F_m^{\lambda,\mathrm{joint}}\!\left(\mu_1\otimes\cdots\otimes\mu_m\right) \\
&=
-\!\int_{\Xi^m} G(\xi_{1:m})\,(\otimes_{b=1}^m \mu_b)(\mathrm d\xi_{1:m})
+\lambda_m\,\mathrm{KL}\!\left(\otimes_{b=1}^m \mu_b\,\middle\|\,\rho_m\right),
\label{eq:F-mf-def}
\end{align}
where $\lambda_m$ once again plays the role of a temperature. As $\lambda_m\downarrow 0$, the optimal mean-field value approaches the joint optimum $G_m^\star:= \max_{\xi_{1:m}\in\Xi^m} G(\xi_{1:m})$. Under an additional isolation condition on a maximising batch $\xi_{1:m}^\star$, the coordinate marginals concentrate on the corresponding coordinates $(\xi_1^\star,\dots,\xi_m^\star)$ (see Proposition~\ref{prop:mf-zero-temp}, Appendix~\ref{sec:zero-temp-mf}).

Under the assumption that the reference distribution factorises as $\rho_m=\rho^{\otimes m}$, for some $\rho\in\mathcal{P}(\Xi)$, the entropic regulariser decomposes as $\mathrm{KL}(\mu_1\otimes\cdots\otimes\mu_m \|\rho^{\otimes m}) =\sum_{b=1}^m \mathrm{KL}(\mu_b\|\rho)$, while the expected utility $\mathbb E_{\xi_{1:m}\sim \mu_1\otimes\cdots\otimes\mu_m}[G(\xi_{1:m})]$ retains the coupling induced by the batch EIG. In any case, the mean-field restriction of the entropy-regularised batch design problem in \eqref{eq:nu-joint-opt} can be written as 
\begin{equation}
(\mu_1^{\lambda,\star},\dots,\mu_m^{\lambda,\star}) 
\in 
\argmin_{(\mu_1,\dots,\mu_m)\in\mathcal P(\Xi)^m}
\mathcal F_m^{\lambda,\mathrm{mf}}\!\left(\mu_1,\dots,\mu_m\right).
\label{eq:mf-opt}
\end{equation}

\noindent Unlike before, the mean-field objective is generally not jointly convex in $(\mu_1,\dots,\mu_m)$, as the expected utility couples batch coordinates. This being said, provided $\rho_m=\rho^{\otimes m}$ factorises, the objective is strictly convex in each coordinate $\mu_b$, conditional on the other coordinates $\mu_{-b}$, since the utility is linear in $\mu_b$ and the KL is strictly convex. We can also characterise its global minimisers. In particular, assuming that $\rho_m =\rho^{\otimes m}$, one can show that each marginal satisfies a self-consistency equation of the form (see Proposition~\ref{prop:mf-coordinate-gibbs}, Appendix~\ref{sec:mean-field-basic-results})
\begin{equation}
\mu_b^{\lambda,\star}(\mathrm d\xi)
=
\frac{1}{Z_{b}^{\lambda}}
\exp\!\left(\frac{1}{\lambda_m}\,\Phi_{b}(\xi;\mu^{\lambda,\star}_{-b})\right)\rho(\mathrm d\xi),
\qquad b=1,\dots,m,
\label{eq:mf-fixedpoint-b}
\end{equation}
where $Z_b^{\lambda}<\infty$ denotes the normalisation constant, $\smash{\mu_{-b}:=\otimes_{j\neq b}\mu_j}$ denotes the product of all marginals except the $b^{\text{th}}$, and $\smash{\Phi_b(\xi;\mu_{-b})}$ denotes the expected batch utility when the $b^{\text{th}}$ coordinate is fixed at $\xi$ and the remaining coordinates are drawn from their current marginals.
\begin{equation}
\Phi_b(\xi;\mu_{-b})
=\int_{\Xi^{m-1}} G(\xi_1,\dots,\xi_{b-1},\xi,\xi_{b+1},\dots,\xi_m)\,\mu_{-b}(\mathrm{d}\xi_{-b}).
\label{eq:Phi-b}
\end{equation}

\paragraph{Why prefer the mean-field approximation?}
While the joint, entropy-regularised batch problem admits an explicit Gibbs solution $\nu_m^{\lambda,\star}$ on $\Xi^m$, sampling from this joint law becomes rapidly impractical as $m$ grows, as it requires simulation in the ambient space $\Xi^m\subseteq (\mathbb R^d)^m$. For example, simulating the overdamped Langevin diffusion w.r.t. $\nu_m^{\lambda,\star}$ requires evaluating or estimating all partial gradients $\{\nabla_{\xi_b}G(\xi_{1:m})\}_{b=1}^m$ at each step. Thus, each MCMC step entails $m$ evaluations of a typically expensive (e.g., nested Monte Carlo) gradient oracle. In addition, issues with slow mixing or stability are typically exacerbated in the higher-dimensional state space. The mean-field restriction $\nu_m=\mu_1\otimes\cdots\otimes\mu_m$ provides a tractable compromise: it replaces sampling in $\Xi^m$ with sampling from a set of $m$ coupled marginal laws on $\Xi\subseteq\mathbb{R}^d$, retaining the essential coupling through the utility term while keeping the entropic regularisation separable when $\rho_m=\rho^{\otimes m}$. Practically, the mean-field formulation also permits within-batch specialisation: the marginals $\mu_b$ may concentrate on different high-utility regions, encouraging diversity without introducing explicit repulsive potentials.

\subsubsection{Restriction to i.i.d.\ product designs}
\label{sec:method:batch-mf}
We can further restrict the standard mean-field variational family $\mathcal P_{\mathrm{mf}}(\Xi^m)$ in \eqref{eq:P-mf} by considering i.i.d.\ product laws, namely, 
\begin{equation}
\mathcal P_{\mathrm{iid}}(\Xi^m)
:=
\{\mu^{\otimes m}:\mu\in\mathcal P(\Xi)\}.
\end{equation}
In this case, we will optimise for a single design law $\mu\in\mathcal P(\Xi)$, and generate a batch $\xi_{1:m}$ by i.i.d.\ draws $\xi_{1:m}\sim\mu^{\otimes m}$. Suppose we define the expected batch utility induced by $\mu$ as
\begin{equation}
\mathcal J_m(\mu)
:=
\int_{\Xi^m} G(\xi_{1:m})\,\mu^{\otimes m}(\mathrm d\xi_{1:m})
=
\mathbb E_{\xi_{1:m}\sim \mu^{\otimes m}}\big[G(\xi_{1:m})\big].
\label{eq:Jm-def}
\end{equation}
Once again, we will introduce an entropic regularisation. Let $\rho\in\mathcal{P}_{2,\mathrm{ac}}(\Xi)$ be a reference measure on $\Xi$, with density $\rho(\xi)\propto e^{-V(\xi)}$ for some confining potential $V:\mathbb{R}^d\to\mathbb{R}$. We can then define
\begin{equation}
    \mathcal{F}_{m}^{\lambda}(\mu):= -\mathcal{J}_m(\mu) + \lambda \mathrm{KL}(\mu\|\rho),
    \label{eq:f-m}
\end{equation}
where, similar to before, $\lambda>0$ is a regularisation parameter that controls the exploration-concentration trade-off. In this case, convergence as $\lambda\downarrow 0$ to the \emph{true} joint optimum requires an additional structural assumption, namely that a globally optimal batch lies on the diagonal (see Proposition~\ref{prop:iid-zero-temp}, Section~\ref{sec:zero-temp-iid}). Returning to \eqref{eq:f-m}, the corresponding measure-valued design problem is then given by
\begin{equation}
    \mu_m^{\lambda,\star} \in \argmin_{\mu\in\mathcal{P}(\Xi)} \mathcal{F}_m^{\lambda}(\mu),
    \qquad \xi_{1:m}\sim (\mu_m^{\lambda,\star})^{\otimes m}.
    \label{eq:mu-opt}
\end{equation}
This formulation decouples the optimisation problem over a design space from sampling an actual batch of designs. While the batch elements are independent given the learned design law, the objective $\mu\mapsto\mathcal J_m(\mu)$ remains non-linear. Thus, the induced optimisation problem still encodes interactions between designs through the batch EIG.

Similar to the general mean-field case, this objective is not convex. However, under suitable regularity conditions (see Theorem \ref{thm:fixedpoint}, Appendix \ref{sec:exist-unique-gibbs}), any minimiser of $\mathcal{F}_m^{\lambda}$ satisfies the self-consistency equation
\begin{equation}
    \mu_m^{\lambda,\star}(\mathrm{d}\xi)
    =
    \frac{1}{Z_m^{\lambda}}
    \exp\left(\frac{m}{\lambda} \Phi_m(\xi;\mu_m^{\lambda,\star}) \right)
    \rho(\mathrm{d}\xi),
    \label{eq:mf-fixed-point}
\end{equation}
where $\smash{Z_m^{\lambda}}<\infty$ denotes the normalisation constant, and $\Phi_m(\xi;\mu)$ denotes the expected batch utility when one design is fixed at $\xi$, and the remaining designs are sampled i.i.d.\ from $\mu$:
\begin{equation}
    \Phi_{m}(\xi;\mu):= \int_{\Xi^{m-1}}G(\xi,\xi_{2:m}) \mu^{\otimes (m-1)}(\mathrm{d}\xi_{2:m}).
    \label{eq:conditional-batch-eig-v0}
\end{equation}
Clearly, the i.i.d.\ restriction is a special case of the original, unconstrained, batch design problem, as well as the mean-field approximation. In fact, if $\rho_m = \rho^{\otimes m}$, and $\lambda_m =\frac{\lambda}{m}$, then the original batch objective in \eqref{eq:Fm-joint} reduces to the objective in \eqref{eq:f-m}, viz
\begin{align}
    \mathcal{F}_m^{\lambda,\mathrm{joint}}(\mu^{\otimes m})
    &=
    - \mathcal{J}_m^{\mathrm{joint}}(\mu^{\otimes m})
    + \frac{\lambda}{m}\mathrm{KL}(\mu^{\otimes m}\| \rho^{\otimes m})
    \label{eq:f-m-joint-simp-1} \\[1mm]
    &=
    -\mathcal{J}_m(\mu) + \lambda \mathrm{KL}(\mu\|\rho)
    =: \mathcal{F}_{m}^{\lambda}(\mu).
    \label{eq:f-m-joint-simp-2}
\end{align}

\paragraph{Encouraging diversity via repulsive interactions.}
A limitation of the i.i.d.\ batch restriction is that it does not prevent duplicates and need not explicitly encourage within-batch diversity. In many applications, a high-quality batch should balance informativeness and diversity, covering complementary regions of the design space. We incorporate this desideratum directly at the level of the design law by adding a repulsive interaction term to the mean-field objective. Let $r:\mathbb{R}^d\to \mathbb{R}\cup\{+\infty\}$ be a symmetric repulsive interaction potential. We can then define the normalised pairwise repulsion as
\begin{equation}
    R(\xi_{1:m})
    =
    \frac{1}{2m(m-1)}
    \sum_{1\le i\neq j\le m} r(\xi_i-\xi_j).
    \label{eq:batch-repulsion-R}
\end{equation}
By incorporating this additional repulsion term into the original joint (i.e., batch) free energy functional, we obtain
\begin{equation}
    \mathcal{F}_m^{\lambda,\mathrm{joint},\mathrm{rep}}(\nu)
    =
    - \mathbb{E}_{\nu}\!\left[G(\xi_{1:m})\right]
    + \eta\, \mathbb{E}_{\nu}\!\left[R(\xi_{1:m})\right]
    + \lambda_m \mathrm{KL}(\nu\|\rho_m),
    \label{eq:f-m-joint-rep-fixed}
\end{equation}
where $\eta\ge 0$ is a parameter which tunes the strength of the diversity penalty: large values of $\eta$ encourage significant diversity, while $\eta=0$ recovers the original joint objective. Arguing as before (see Proposition \ref{prop:joint-gibbs}, Appendix \ref{sec:joint-batch-basic-results}), the minimiser of this objective functional admits an explicit Gibbs form, namely, 
\begin{equation}
    \nu_{m}^{\lambda,\eta,\star}(\mathrm{d}\xi_{1:m})
    =
    \frac{1}{Z_m^{\lambda,\eta}}
    \exp\!\left(\frac{1}{\lambda_m}\left(G(\xi_{1:m})-\eta R(\xi_{1:m})\right)\right)
    \rho_m(\mathrm{d}\xi_{1:m}),
\end{equation}
with normalisation constant $Z_m^{\lambda,\eta}<\infty$. By restricting this objective to product laws, we can obtain a free energy functional for the design law $\mu\in\mathcal P(\Xi)$ which explicitly incorporates diversity. In particular, substituting $\nu_m=\mu^{\otimes m}$, $\rho_m=\rho^{\otimes m}$, and $\lambda_m = \frac{\lambda}{m}$, we have that
\begin{align}
    \mathcal{F}_m^{\lambda,\mathrm{rep}}(\mu)
    &:=
    \mathcal{F}_m^{\lambda,\mathrm{joint},\mathrm{rep}}(\mu^{\otimes m})
    \\
    &=
    -\mathcal{J}_m(\mu)
    + \eta\,\mathcal{R}(\mu)
    + \lambda \mathrm{KL}(\mu\|\rho),
    \label{eq:F-rep-mu-moved}
\end{align}
where $\smash{\mathcal{R}(\mu):=
    \frac{1}{2}\int_{\Xi}\int_{\Xi} r(\xi-\chi)\,\mu(\mathrm d\xi)\,\mu(\mathrm d\chi)}$
induces repulsion between particles in the mean-field approximation, thereby discouraging collapse of the design law onto a small set of atoms.

\paragraph{Why prefer the i.i.d.\ design-law formulation?}
The standard mean-field family $\nu_m=\mu_1\otimes\cdots\otimes\mu_m$ provides a tractable surrogate for the joint Gibbs batch law, but also introduces $m$ coupled marginal laws, each of which must be separately approximated. The i.i.d.\ restriction $\nu_m=\mu^{\otimes m}$ further simplifies the parameterisation to a \emph{single} design law on $\Xi$, yielding an exchangeable random batch by construction, and reducing the computational overhead.  While the resulting optimum $\mu_m^{\lambda,\star}\in\mathcal P(\Xi)$ does not admit an explicit Gibbs representation unless $m=1$, it can be sampled from efficiently, even when the batch size $m$ is large. In particular, each particle update only requires estimates of $\nabla_{\xi_1} G(\xi_{1:m})$, evaluated on a small number of randomly sampled tuples, rather than all coordinate gradients $\{\nabla_{\xi_b}G(\xi_{1:m})\}_{b=1}^m$. The trade-off is expressiveness: unlike the mean-field designs, i.i.d.\ designs cannot allocate distinct marginals to different batch positions. Thus, explicit diversity terms (or an appropriate extraction mechanism; see Appendix~\ref{sec:bon}) may be necessary in practice.

\subsection{Optimising the i.i.d.\ Objective via Wasserstein Gradient Flows}
\label{sec:method:optimisation}
We now develop gradient-based methods for optimising the i.i.d.\ design-law free energy (with repulsion) defined in \eqref{eq:F-rep-mu-moved}, namely
\begin{equation}
\mathcal F_m^{\lambda,\mathrm{rep}}(\mu)=-\mathcal J_m(\mu)+\eta\mathcal{R}(\mu) + \lambda\,\mathrm{KL}(\mu\|\rho),
\end{equation}
where, for convenience, we recall that
\begin{equation}
\mathcal J_m(\mu)= \int_{\Xi^m}\!\! G(\xi_{1:m})\,\mu^{\otimes m}(\mathrm d\xi_{1:m}), 
\qquad
\mathcal{R}(\mu):=
\frac{1}{2}\int_{\Xi}\int_{\Xi} r(\xi-\chi)\,\mu(\mathrm d\xi)\,\mu(\mathrm d\chi).
\end{equation}

\noindent The minimiser of this objective does not admit a closed form solution, but rather can be characterised implicitly by a fixed-point equation; see \eqref{eq:mf-fixed-point} for the $\eta=0$ case. Accordingly, one cannot directly apply an off-the-shelf MCMC scheme targeting a known static density. Instead, we will optimise this free energy by simulating a WGF whose stationary solutions satisfy the relevant fixed point equation.

\begin{remark}
\label{remark:mean-field-wgf}
The minima of the mean-field objective $\mu_{1:m}\mapsto \mathcal F_{m}^{\lambda,\mathrm{mf}}(\mu_{1:m})$, defined in \eqref{eq:F-mf-def}, also lack an explicit Gibbs characterisation. They can be computed using a similar approach to the one developed in this section. In this case, rather than considering a single WGF which converges to $\mu_m^{\lambda,\star}\in\argmin \mathcal{F}_{m}^{\lambda}(\mu)$, one would consider an ensemble of coordinate-wise WGFs whose stationary laws coincide with $(\mu_1^{\lambda,\star},\dots,\mu_m^{\lambda,\star})\in\argmin \mathcal{F}_{m}^{\lambda,\mathrm{mf}}(\mu_1,\dots,\mu_m)$ \citep[e.g.,][]{yao2022mean,tran2023particle,lacker2023independent}. These WGFs will rely on the coordinate-wise conditional utilities $\Phi_{b}(\cdot;\mu_{-b})$, which serve as the analogues of the conditional utility $\Phi_m(\cdot;\mu)$ which appears in the i.i.d.\ setting. Since this extension is notationally inconvenient but otherwise direct, we here develop the algorithms and analysis in detail only for the i.i.d.\ case. 
\end{remark}

\begin{remark}
\label{remark:joint-wgf}
The minimiser of the joint objective $\nu\mapsto \mathcal F_m^{\lambda,\mathrm{joint}}(\nu)$, defined in \eqref{eq:Fm-joint}, does have an explicit Gibbs characterisation, unlike the i.i.d.\ design objectives $\mu\mapsto \mathcal{F}_m^{\lambda}(\mu)$ or  $\mu\mapsto \mathcal{F}_m^{\lambda,\mathrm{rep}}(\mu)$,  or the mean-field objective $\mu_{1:m}\mapsto \mathcal{F}_m^{\lambda,\mathrm{mf}}(\mu_{1:m})$; see \eqref{eq:nu-joint}. In principle, it can therefore be sampled by any standard MCMC method over the batch space $\Xi^m$. Amongst the various choices, the most natural analogue of the algorithm developed in this section is (stochastic gradient) Langevin dynamics on the batch space $\Xi^m$. Indeed, this is precisely the WGF of the joint free energy $\mathcal F_m^{\lambda,\mathrm{joint}}$ over $\mathcal{P}_{2}(\Xi^m)$ \citep[see, e.g.,][]{jordan1998variational}.
\end{remark}

The remainder of this section proceeds as follows. We first compute the first variation of $\mathcal F_m^{\lambda,\mathrm{rep}}$ (Section~\ref{sec:method:firstvariation}), which in turn allows us to obtain its Wasserstein ($\mathsf{W}_2$) gradient (Section~\ref{sec:method:wassgrad}). We then derive the associated WGF, which can be represented as a McKean--Vlasov Fokker--Planck PDE (Section~\ref{sec:method:wgf}). We next show how to approximate these non-linear dynamics via a space-time discretisation, which yields an interacting particle system (Section~\ref{sec:method:ips}). Finally, to obtain a method which remains scalable when $m$ is large, and $\nabla_1\mathrm{EIG}_m$ is only available through (nested) Monte Carlo or related estimators, we introduce a doubly stochastic IPS that combines tuple subsampling for the interaction term with stochastic gradient estimation (Section~\ref{sec:method:doubly-stochastic-ips}).

In the interest of readability, the presentation in this section will remain formal; detailed theoretical results are deferred to the appendices.

\subsubsection{The First Variation}
\label{sec:method:firstvariation}
Let $\mu,\nu\in\mathcal{P}(\Xi)$, and consider the mixture path $\mu_\varepsilon := (1-\varepsilon)\mu+\varepsilon \nu$ for $\varepsilon\in[0,1]$. We can then compute, via standard calculations, the G\^ateaux derivative (see Lemma~\ref{lem:deltaJ}, Appendix~\ref{sec:first-var-theory}) 
\begin{equation}
    \left.\frac{\mathrm{d}}{\mathrm{d}\varepsilon} \mathcal{J}_m(\mu_\varepsilon)\right|_{\varepsilon=0} = m \int_{\Xi} \Phi_{m}(\xi;\mu)(\nu-\mu)(\mathrm{d}\xi),
    \label{eq:Jm-gatederiv}
\end{equation}
where $\Phi_m(\xi;\mu)$ is the expected batch utility defined in Section~\ref{sec:method:batch-mf}, cf. \eqref{eq:conditional-batch-eig-v0}.\footnote{This representation is specific to the conditionally independent observation model in Section~\ref{sec:batch}. In particular, this means that $G(\xi_{1:m})=\mathrm{EIG}_m(\xi_{1:m})$ is invariant under permutations of $(\xi_1,\dots,\xi_m)$, and thus symmetric in its $m$ arguments. This symmetry is what permits the factor $m$ and the single conditional utility $\Phi_m(\xi;\mu)$ in the first-variation formula above. For a general non-symmetric batch utility one would instead obtain a sum of coordinate-wise conditional utilities.} It follows, in particular, that the first variation of $\mathcal{J}_m$ admits the unique (up to an additive constant) pointwise representation (see Lemma~\ref{lem:deltaJ}, Appendix~\ref{sec:first-var-theory}) 
\begin{equation}
    \frac{\delta \mathcal{J}_m(\mu)}{\delta \mu}(\xi) = m\Phi_m(\xi;\mu).
    \label{eq:first-variation-Jm}
\end{equation}
Meanwhile, for the additional repulsion term, standard results \cite[e.g.,][Section 10.4.5]{ambrosio2008gradientflows} yield (see Lemma~\ref{lem:deltaR}, Appendix~\ref{sec:first-var-theory})
\begin{equation}
\frac{\delta \mathcal R(\mu)}{\delta \mu}(\xi)
=
\Psi_r(\xi;\mu), \qquad \Psi_r(\xi;\mu)
:=
\int_{\Xi} r(\xi-\chi)\,\mu(\mathrm d\chi).
\label{eq:Psi-r-def}
\end{equation}
Combining these two displays with standard results for the KL divergence,\footnote{In particular, we recall that the first variation of the KL divergence is given by $\smash{\frac{\delta}{\delta\mu} \mathrm{KL}(\mu\|\rho)(\xi) =  \log(\frac{\mathrm{d}\mu}{\mathrm{d}\rho}(\xi))+ 1}$ for $\mu\ll \rho$ \citep[e.g.,][Lemma 10.4.1]{ambrosio2008gradientflows}.} it follows that the first variation of $\mathcal{F}_m^{\lambda,\mathrm{rep}}$ is given (up to an additive constant) by (see Corollary~\ref{cor:deltaF}, Appendix~\ref{sec:first-var-theory})
\begin{equation}
    \frac{\delta \mathcal{F}_m^{\lambda,\mathrm{rep}}(\mu)}{\delta \mu}(\xi) = -m\Phi_m(\xi;\mu) + \eta \Psi_{r}(\xi;\mu) + \lambda \left(\log \left[\frac{\mathrm{d}\mu}{\mathrm{d}\rho}(\xi)\right]+ 1\right).
    \label{eq:first-variation-Fm}
\end{equation}

\subsubsection{The Wasserstein Gradient}
\label{sec:method:wassgrad}
We will work on $\mathcal{P}_2(\Xi)$ equipped with the ${W}_2$ geometry. Suppose that $\Xi\subseteq\mathbb{R}^d$ is open, and that $\xi \mapsto \Phi_m(\xi;\mu)$ is differentiable. Then, differentiating under the integral sign, we have (see Lemma~\ref{lem:Phi-diff-Lip}, Appendix~\ref{sec:wgf})
\begin{equation}
\nabla_\xi \Phi_m(\xi;\mu)
=
\int_{\Xi^{m-1}} \nabla_1 G(\xi,\xi_{2:m})\,\mu^{\otimes(m-1)}(\mathrm d\xi_{2:m}).
\label{eq:grad-Phi}
\end{equation}
Using classical results \citep[e.g.,][]{ambrosio2008gradientflows}, it follows under mild regularity conditions that the Wasserstein gradient of $\mathcal{J}_m$ at $\mu$ is given by the vector field
\begin{equation}
\nabla_{\mathsf{W}_2}\mathcal{J}_m(\mu)(\xi)
=
\nabla_\xi \frac{\delta \mathcal{J}_m}{\delta \mu}(\mu)(\xi)
=
m\,\nabla_\xi \Phi_m(\xi;\mu).
\label{eq:wass-grad}
\end{equation}
Similarly, if \(r\) is differentiable and differentiation under the integral sign is justified, then we have that $\nabla_\xi \Psi_r(\xi;\mu)
=
\int_{\Xi}\nabla r(\xi-\chi)\,\mu(\mathrm d\chi)$, and consequently that (see Lemma~\ref{lem:Psi-diff-Lip}, Appendix~\ref{sec:wgf})
\begin{equation}
\nabla_{\mathsf W_2}\mathcal{R}(\mu)(\xi) 
= \nabla_\xi \frac{\delta \mathcal{R}}{\delta \mu}(\mu)(\xi)
= \nabla_{\xi}\Psi_{r}(\xi;\mu).
\end{equation}
Thus, whenever the indicated derivatives are well-defined, the Wasserstein gradient of the regularised objective function $\mathcal{F}_m^{\lambda,\mathrm{rep}}$ is given by (see Proposition~\ref{prop:EDI-smooth}, Appendix \ref{sec:wgf})
    \begin{equation}
\nabla_{\mathsf{W}_2}\mathcal{F}_m^{\lambda,\mathrm{rep}}(\mu)(\xi)
=
\nabla_\xi \frac{\delta \mathcal F_m^{\lambda,\mathrm{rep}}}{\delta\mu}(\mu)(\xi)
=
- m\,\nabla_\xi \Phi_m(\xi;\mu)
+ \eta \nabla_{\xi}\Psi_{r}(\xi;\mu) 
+\lambda \nabla_{\xi} \log \left(\frac{\mathrm{d}\mu}{\mathrm{d}\rho}(\xi)\right).
\label{eq:wass-grad-regularised}
\end{equation}

\subsubsection{The Wasserstein Gradient Flow}
\label{sec:method:wgf}
The WGF corresponds to the steepest-descent dynamics in $(\mathcal{P}_{2,\mathrm{ac}}(\Xi),\mathsf{W}_2)$. It is defined as the weak solution $\smash{\mu:[0,\infty)\rightarrow\mathcal{P}_2(\Xi})$ of the continuity equation \citep[e.g.,][Chapter 11]{ambrosio2008gradientflows}
\begin{equation}
    \frac{\partial \mu_t}{\partial t} + \nabla \cdot(v_t\mu_t) = 0, \qquad v_t = -\nabla_{\xi}\frac{\delta \mathcal{F}_m^{\lambda,\mathrm{rep}}}{\delta \mu}(\mu_t)(\xi).
    \label{eq:wgf}
\end{equation}
In our case, substituting the Wasserstein gradient from above, and recalling that $\rho(\xi)\propto e^{-V(\xi)}$ for some confining potential $V(\xi)$, we can rewrite the WGF as 
\begin{equation}
    \frac{\partial \mu_t}{\partial t} = -\mathrm{div}\left(\mu_t\left(m \nabla \Phi_m(\cdot;\mu_t) - \eta\nabla \Psi_{r}(\cdot;\mu_t) + \lambda \nabla \log \rho\right)\right) + \lambda \Delta \mu_t. \label{eq:mv-pde}
\end{equation}
This is a {non-linear} or {McKean--Vlasov} Fokker-Planck equation, the {non-linearity} arising due to the dependence of $\Phi_{m}(\cdot;\mu_t)$ on the current distribution. The Fokker-Planck equation admits a corresponding probabilistic (or \emph{Lagrangian}) representation as a {mean-field} or {McKean--Vlasov} SDE, given by 
\begin{equation}
\mathrm d\xi_t
=
\left(
m\,\nabla \Phi_m(\xi_t;\mu_t)
-\eta\,\nabla \Psi_{r}(\xi_t;\mu_t)
+\lambda\,\nabla \log\rho(\xi_t)
\right)\mathrm dt
+
\sqrt{2\lambda}\,\mathrm dw_t,
\label{eq:mv-sde}
\end{equation}
where $\mu_t=\mathrm{Law}(\xi_t)$ and $(w_t)_{t\ge 0}$ is a standard $\mathbb R^d$-valued Brownian motion. In particular, under mild regularity conditions, $\mu_t=\mathrm{Law}(\xi_t)$ is the solution of \eqref{eq:mv-pde} (see Theorem~\ref{thm:MV-wellposed}, Appendix~\ref{sec:wgf}). 

\paragraph{Long-time behaviour}
It is well known that every minimiser of $\mathcal{F}_m^{\lambda,\mathrm{rep}}$ is a stationary solution of the WGF defined in \eqref{eq:wgf}, \eqref{eq:mv-pde}, or \eqref{eq:mv-sde} (see Corollary~\ref{cor:stationary-minimiser}, Appendix~\ref{sec:wgf}). Conversely, under certain regularity assumptions, the stationarity condition reduces to the fixed-point (i.e., self-consistency) equation
\begin{equation}
\mu_m^{\lambda,\star}(\mathrm d\xi)
=
\frac{1}{Z_m^{\lambda,\eta}}
\exp\!\left(\frac{1}{\lambda}(m\Phi_m(\xi;\mu_m^{\lambda,\star}) - \eta \Psi_{r}(\xi;\mu_m^{\lambda,\star}))\right)\rho(\mathrm d\xi).
\label{eq:fixed-point-stationary}
\end{equation}
The WGF enjoys a number of other long-time properties. For example, along sufficiently regular solutions, the free energy $\mu_t\mapsto \mathcal F_m^{\lambda,\mathrm{rep}}(\mu_t)$ dissipates monotonically in time, so that any limit point of the WGF must be stationary \citep[e.g.,][]{ambrosio2008gradientflows}. 

Under additional assumptions (e.g., geodesic $\alpha$-convexity of $\mathcal F_m^{\lambda,\mathrm{rep}}$), one can establish uniqueness of, and exponential convergence to, the minimiser $\mu_m^{\lambda,\star}$ \citep{mccann1997convexity,ambrosio2008gradientflows,villani2009optimal}. For the McKean–Vlasov SDE in \eqref{eq:mv-sde}, a more direct coupling argument yields exponential contractivity under a slightly stronger dissipativity condition (see, e.g., Lemma~\ref{lem:MV-contractivity}, Appendix~\ref{sec:wgf}; Appendix~\ref{sec:evi}).  In the classical non-linear setting, complementary sufficient conditions based on strong confinement, together with a sufficiently small Lipschitz mean-field interaction yield explicit exponential convergence rates \citep{malrieu2001logarithmic,carrillo2006contractions,bolley2010trend}. Alternatively, one can establish quantitative convergence rates via functional inequalities: in particular, a logarithmic Sobolev inequality (LSI) at equilibrium yields exponential decay of relative entropy \citep{bakry2014analysis}, which implies Wasserstein convergence via standard transport inequalities (e.g.\ Talagrand's $T_2$), and the implication $\mathrm{LSI}\Rightarrow T_2$ due to Otto and Villani \citep{talagrand1996transport,ottoVillani2000talagrand}. 

It is worth noting that the convergence results obtained in our current analysis require a strong-confinement regime, namely that the entropic regularisation dominates the curvature of the interaction terms; see Assumption~\ref{ass:strong} in Appendix~\ref{sec:assumptions}. This condition is somewhat conservative, and may not hold in low-temperature, highly multimodal settings. Accordingly, our quantitative long-time guarantees should be interpreted as stability results for a regularised regime, rather than as a complete global theory for the non-convex operating regime explored in the numerical section.

\subsubsection{The Interacting Particle System}
\label{sec:method:ips}
The McKean--Vlasov SDE in \eqref{eq:mv-sde} cannot be simulated directly, since its drift depends on the unknown distribution $\smash{\mu_t = \mathrm{Law}(\xi_t)}$. One approach is to approximate the mean-field SDE using an {interacting particle system} (IPS), viz 
\begin{equation}
    \mathrm{d}\xi_t^{i,N} = \left(m\nabla\Phi_m(\xi_t^{i,N}; \mu_t^N) - \eta\nabla\Psi_r(\xi_t^{i,N}; \mu_t^N) +\lambda \nabla \log \rho(\xi_t^{i,N})\right)\mathrm{d}t + \sqrt{2\lambda}\mathrm{d}w_t^{i,N}, \qquad i\in[N],
    \label{eq:ips-sde}
\end{equation}
where $\smash{\mu_t^N:=\frac{1}{N}\sum_{j=1}^N \delta_{\xi_t^{j,N}}}$ denotes the empirical distribution of the particles, and $\smash{(w_t^{i,N})_{t\geq 0}^{i\in[N]}}$ are a collection of independent $\mathbb{R}^d$-valued Brownian motions. Under certain conditions, the empirical measure $\smash{\mu_t^N\rightarrow\mu_t}$ as $\smash{N\rightarrow\infty}$, a phenomenon known as the {propagation of chaos} (see Theorems~\ref{thm:PoC}, \ref{thm:PoC-uniform}; Appendix \ref{sec:poc}).

In order to obtain an implementable algorithm, we will also need to discretise \eqref{eq:ips-sde} in time. Let $\mu_n^N:=\frac1N\sum_{j=1}^N\delta_{\xi_{n}^{j,N}}$. Then, applying an Euler-Maruyama discretisation, we arrive at
\begin{equation}
\xi_{n+1}^{i,N}
=
\xi_{n}^{i,N}
+
\gamma_n\left(
m\,\nabla\Phi_m\big(\xi_{n}^{i,N};\mu_n^N\big)
- \eta\,\nabla\Psi_r(\xi_n^{i,N}; \mu_n^N)
+ \lambda\,\nabla\log\rho\big(\xi_{n}^{i,N}\big)
\right)
+
\sqrt{2\lambda\gamma_n}\,Z_n^{i},
\label{eq:ips-em}
\end{equation}
where $\smash{(\gamma_n)_{n\geq 0}}$ denotes the {step-size schedule}, and $\smash{(Z_n^{i})_{n\ge 0}^{i\in[N]}}$ is a collection of i.i.d.\ standard normal random variables in $\mathbb R^d$. Classical results imply that, on any finite time horizon, the continuous-time interpolation of \eqref{eq:ips-em} converges to the solution of the continuous-time IPS dynamics in \eqref{eq:ips-sde} as $\max_n \gamma_n\to 0$  \citep[e.g.,][]{kloedenplaten1992,higham2001sde,chen2024euler} (see Theorem~\ref{thm:Euler}, Appendix~\ref{sec:Euler}). In our numerics, we will generally use a constant step size, although adaptive choices are also possible \citep[e.g.,][]{sharrock2025tuning}.

\begin{figure}[t!]
    \centering
    \includegraphics[width=\linewidth]{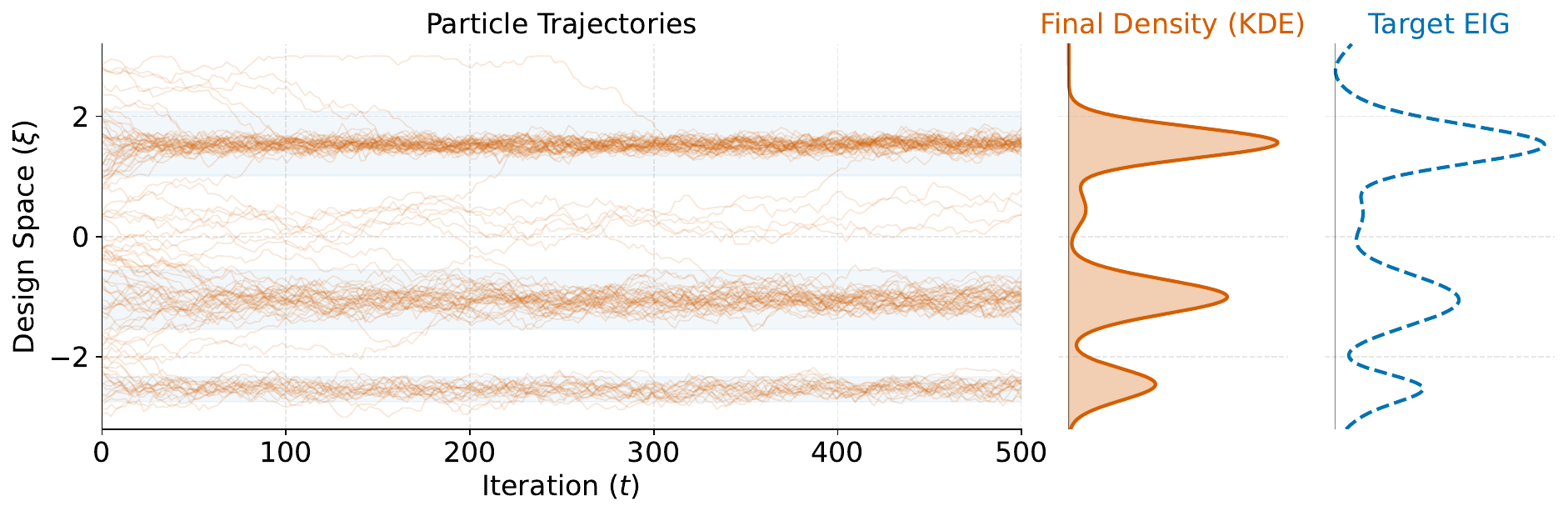}
    \caption{\textbf{The interacting particle system (IPS)}. We plot the trajectories of $N=100$ particles over $T=500$ iterations (orange), the kernel density estimate of the final particle distribution (orange), and the target expected information gain (EIG) (blue dashed).}
    \label{fig:wgf_toy}
\end{figure}

\subsubsection{The (Doubly) Stochastic Interacting Particle System}
\label{sec:method:doubly-stochastic-ips}
The IPS update defined in the previous section is still not directly implementable for two reasons. The first is interaction cost. Even if $\xi_{1:m}\mapsto G(\xi_{1:m})$ and $\xi_{1:m}\mapsto \nabla_1 G(\xi_{1:m})$ could be evaluated exactly, computing the gradient
\begin{equation}
\nabla \Phi_m(\xi;\mu_n^N)
=
\int_{\Xi^{m-1}} \nabla_1 G(\xi,\xi_{2:m})\,(\mu_n^N)^{\otimes(m-1)}(\mathrm d\xi_{2:m})
\end{equation}
requires summing over all $(m-1)$ tuples of particles, at a cost of $\mathcal{O}(N^{m-1})$ per particle. This is prohibitive for $m\geq 3$. We therefore approximate this integral by Monte Carlo over random index tuples: conditional on $(\xi_n^{j,N})_{j=1}^N$, draw $K$ i.i.d.\ tuples
$(I_{2,k},\dots,I_{m,k})\sim \mathrm{Unif}([N]^{m-1})$ and estimate
\begin{equation}
\widehat{\nabla \Phi}_m(\xi_n^{i,N};\mu_n^N)
:=
\frac{1}{K}\sum_{k=1}^K
\nabla_1 G \Big(
\xi_n^{i,N},\,\xi_n^{I_{2,k},N},\,\dots,\,\xi_n^{I_{m,k},N}
\Big).
\label{eq:phi-hat}
\end{equation}
This is conditionally unbiased for $\nabla \Phi_m(\xi_n^{i,N};\mu_n^N)$, and reduces the interaction cost to $\mathcal{O}(K)$ evaluations of $\nabla_1 G$ per particle per iteration. We also use an analogous Monte Carlo estimator for the repulsion term. In particular, conditional on the current particles, draw $K_{\mathrm{rep}}$ i.i.d.\ indices
$J_{n,\ell}^i\sim \mathrm{Unif}([N])$, independently across $i$ and $\ell$, and define
\begin{equation}
\widehat{\nabla \Psi}_r(\xi_n^{i,N};\mu_n^N)
:=
\frac{1}{K_{\mathrm{rep}}}\sum_{\ell=1}^{K_{\mathrm{rep}}}
\nabla r\big(\xi_n^{i,N}-\xi_n^{J_{n,\ell}^i,N}\big).
\label{eq:psi-grad-est}
\end{equation}
This estimator is conditionally unbiased for $\nabla\Psi_r(\xi_n^{i,N};\mu_n^N)$ and reduces the repulsion cost from $\mathcal O(N)$ to $\mathcal O(K_{\mathrm{rep}})$ per particle per iteration. 

The second reason is intractable gradients: in general, $\nabla_{1} G(\xi_{1:m})=\nabla_{1}\mathrm{EIG}_m(\xi_{1:m})$ is defined via a (nested) expectation, and does not admit a closed form. Thus, $\nabla_1 G(\xi_{1:m})$ is typically approximated via a (nested) Monte Carlo estimator. Accordingly, we assume access to a (possibly biased) stochastic gradient oracle
$\smash{\widehat{\nabla_1 G}(\xi_{1:m};U)}$, based on auxiliary randomness $U$, such that
\begin{equation}
\mathbb{E}\big[\widehat{\nabla_1 G}(\xi_{1:m};U)\big]
=
\nabla_1 G(\xi_{1:m}) + b(\xi_{1:m}),
\qquad
\mathbb{E}\big[\|\widehat{\nabla_1 G}(\xi_{1:m};U)\|^2\big]<\infty,
\label{eq:stoch-gradG}
\end{equation}
where $b(\xi_{1:m})$ denotes the (algorithm-dependent) bias. 
We can then replace each occurrence of $\nabla_1 G$ in the interaction estimator above by the estimator $\smash{\widehat{\nabla_1 G}}$. 

\begin{remark}
In practice, this estimator is typically constructed via Monte Carlo simulation under the joint model $\smash{(\theta,y_{1:m}) \sim \pi(\theta)\,\pi_{\xi_{1:m}}(\cdot\mid\theta)}$, together with an approximation of an evidence ratio or a posterior expectation. This yields a (typically biased) nested Monte Carlo estimator \citep[e.g.,][]{rainforth2018nesting}. Alternatives include multi-level Monte Carlo, variational approximations, SMC-based estimators, or differentiable density-ratio or mutual-information estimators in likelihood-free settings. The outer IPS mechanism is agnostic to the choice of inner approximation, provided appropriate moment bounds and, where relevant, bias-control conditions hold.
\end{remark}

Combining \eqref{eq:phi-hat} and  \eqref{eq:stoch-gradG}, we can estimate the interaction drift via tuple subsampling, together with an inner Monte Carlo estimator. In particular, conditional on the current IPS $(\xi_n^{j,N})_{j=1}^N$, draw $K$ i.i.d.\ index tuples $(I_{2,k},\dots,I_{m,k})\sim \mathrm{Unif}([N]^{m-1})$ for $k=1,\dots,K$, and let $(U_{n,k}^i)$ be i.i.d.\ auxiliary random variables, independent of the tuple draws and the Gaussian noises. We then define
\begin{equation}
\widehat{\nabla \Phi}_m(\xi_n^{i,N};\mu_n^N)
:=
\frac{1}{K}\sum_{k=1}^K
\widehat{\nabla_1 G}(
\xi_n^{i,N},\,\xi_n^{I_{2,k},N},\,\dots,\,\xi_n^{I_{m,k},N};\,U_{n,k}^{i}
).
\label{eq:phi-grad-est}
\end{equation}
The parameter $K$ controls the variance of the interaction estimator and, in practice, can be taken very small (e.g., $K=1$) for scalability. 

Finally, replacing $\nabla\Phi_m(\xi_n^{i,N};\mu_n^N)$ by $\widehat{\nabla\Phi}_m(\xi_n^{i,N};\mu_n^N)$ and $\nabla\Psi_r(\xi_n^{i,N};\mu_n^N)$ by $\widehat{\nabla\Psi}_r(\xi_n^{i,N};\mu_n^N)$ in the Euler--Maruyama discretisation, cf. \eqref{eq:ips-em}, yields a fully implementable, doubly stochastic algorithm, viz
\begin{equation}
\xi_{n+1}^{i,N}
=
\xi_n^{i,N}
+
\gamma_n\left(
m\,\widehat{\nabla\Phi}_m(\xi_n^{i,N};\mu_n^N)
- \eta\,\widehat{\nabla\Psi}_r(\xi_n^{i,N}; \mu_n^N)
+
\lambda\,\nabla\log\rho(\xi_n^{i,N})
\right)
+
\sqrt{2\lambda\gamma_n}\,Z_n^{i},
\label{eq:algorithm-stochastic}
\end{equation}
where $\mu_n^N=\frac1N\sum_{j=1}^N\delta_{\xi_n^{j,N}}$ and $(Z_n^i)_{n\geq 0}^{i\in[N]}$ are i.i.d.\ standard Gaussians in $\mathbb{R}^d$. This update is \emph{doubly stochastic}: it uses Monte Carlo both to approximate the mean-field interaction (via the tuples) and to approximate the intractable gradient $\nabla_1 G$ (via the auxiliary variables $U_{n,k}^i$).

\begin{remark} 
If the gradient estimator is unbiased, this algorithm can be analysed using stochastic approximation techniques under suitable stability conditions. Indeed, our end-to-end error bounds in the appendices are proved for this case (see Theorem~\ref{thm:stoch}, Appendix~\ref{sec:stoch} and Theorem~\ref{thm:end2end}, Appendix~\ref{sec:end2end}). In several of our numerical experiments (see Section~\ref{sec:numerics}), we instead use a biased, fixed-budget nested Monte Carlo estimator. These experiments should therefore be interpreted as empirical evaluations of the practical, biased-gradient extension of the method, rather than as direct numerical confirmations of the unbiased-oracle theory. Establishing analogous non-asymptotic guarantees for biased inner estimators remains an important open problem.
\end{remark}

\begin{remark}
The analysis above is stated for absolutely continuous laws on open subsets of $\mathbb{R}^d$. Several numerical examples in Section~\ref{sec:numerics} impose box constraints, periodicity, or ordering/minimum-gap constraints. In those cases, we implement projected, wrapped, or repaired variants of the particle updates. These should be viewed as practical approximations of the idealised unconstrained dynamics developed here, rather than as direct discretisations covered by the present theory. A rigorous treatment on constrained domains or manifolds (e.g., via reflected diffusions, projected Wasserstein flows, or suitable reparameterisations) is left to future work.
\end{remark}

\subsection{Summary}
\label{sec:implemented-algorithms}

In Section~\ref{sec:method:objective}, we introduced four entropy-regularised objectives: the joint batch objective on $\Xi^m$, a mean-field product approximation with coordinate-wise marginals, an i.i.d.\ product approximation, and an i.i.d.\ approximation with explicit repulsion. In Section~\ref{sec:method:optimisation}, we showed in detail how to minimise one of these objectives (the i.i.d.\ product approximation) using a doubly stochastic approximation of the space-time discretisation of the corresponding WGF. The other three objectives give rise to analogous optimisation schemes, with differences arising from the choice of variational family and, consequently, the form of the interaction drift. Below, we provide explicit formulations for each of these schemes.

\paragraph{WGF (Joint).}
This method evolves a full batch $\boldsymbol\xi_n^r=(\xi_{1,n}^r,\dots,\xi_{m,n}^r)\in\Xi^m$ directly in the batch space, and corresponds to i.i.d.\ copies of stochastic gradient Langevin dynamics w.r.t. the solution of the joint entropy-regularised batch objective $\mathcal F_m^{\lambda,\mathrm{joint}}$; see Remark~\ref{remark:joint-wgf}. Its update is
\begin{equation}
\label{eq:joint_batch_langevin}
\boldsymbol{\xi}^{r}_{n+1}
=
\boldsymbol{\xi}^{r}_n
+
\gamma_n\Big(
\widehat{\nabla G}(\boldsymbol{\xi}^{r}_n)
+
\lambda_m \,\nabla \log \rho_m(\boldsymbol{\xi}^{r}_n)
\Big)
+
\sqrt{2\lambda_m\gamma_n}\,Z^{r}_n,
\end{equation}
where $\widehat{\nabla G}$ denotes a stochastic gradient estimator of the full batch utility $G=\mathrm{EIG}_m$, and $Z_n^r\sim\mathcal N(0,I_{md})$. The deterministic drift ascends the batch EIG, while the Gaussian perturbation induces exploration at temperature $\lambda_m$.

\paragraph{WGF (MF).}
Under the mean-field restriction $\nu_m=\mu_1\otimes\cdots\otimes\mu_m$, we evolve $m$ coupled particle systems, one for each coordinate marginal; see Remark~\ref{remark:mean-field-wgf}. For $b\in\{1,\dots,m\}$ and $i\in[N_{\mathrm{mf}}]$, the update reads
\begin{equation}
\label{eq:mf_ips_update_main}
\xi_{b,n+1}^{i,N_{\mathrm{mf}}}
=
\xi_{b,n}^{i,N_{\mathrm{mf}}}
+
\gamma_n\Big(
\widehat{\nabla \Phi}_b(\xi_{b,n}^{i,N_{\mathrm{mf}}};\mu_{-b,n}^{N_{\mathrm{mf}}})
+
\lambda_m\,\nabla\log\rho(\xi_{b,n}^{i,N_{\mathrm{mf}}})
\Big)
+
\sqrt{2\lambda_m\gamma_n}\,Z_{b,n}^{i},
\end{equation}
where $\smash{\widehat{\nabla \Phi}_b}$ denotes the coordinate-wise analogue of $\smash{\widehat{\nabla \Phi}_m}$, $\smash{\mu_{-b,n}^{N_{\mathrm{mf}
}}}$ denotes the empirical product law of all coordinates except the $b^{\text{th}}$, and $\smash{Z_{b,n}^{i}\sim\mathcal N(0,I_d)}$. This formulation retains within-batch heterogeneity by allowing different marginals to specialise to different regions of the design space.

\paragraph{WGF (MF-IID).}
Under the i.i.d.\ restriction $\nu_m=\mu^{\otimes m}$, we evolve a single particle system for the shared design law $\mu$; see Section~\ref{sec:method:optimisation}. The corresponding update is
\begin{equation}
\label{eq:iid_ips_update_main}
\xi_{n+1}^{i,N}
=
\xi_n^{i,N}
+
\gamma_n\Big(
m\,\widehat{\nabla \Phi}_m(\xi_n^{i,N};\mu_n^N)
+
\lambda\,\nabla\log\rho(\xi_n^{i,N})
\Big)
+
\sqrt{2\lambda\gamma_n}\,Z_n^{i},
\end{equation}
where $\mu_n^N=\frac1N\sum_{j=1}^N\delta_{\xi_n^{j,N}}$ is the empirical measure and $Z_n^i\sim\mathcal N(0,I_d)$. This is the simplest structured approximation and is particularly attractive when the batch size $m$ is large.

\paragraph{WGF (MF-IID-REP).}
To encourage within-batch diversity, we augment the i.i.d.\ formulation with an explicit repulsive interaction term; see Section~\ref{sec:method:optimisation}. The resulting update becomes
\begin{equation}
\label{eq:iid_rep_ips_update_main}
\xi_{n+1}^{i,N}
=
\xi_n^{i,N}
+
\gamma_n\Big(
m\,\widehat{\nabla \Phi}_m(\xi_n^{i,N};\mu_n^N)
- \eta\,\widehat{\nabla\Psi}_r(\xi_n^{i,N}; \mu_n^N)
+
\lambda\,\nabla\log\rho(\xi_n^{i,N})
\Big)
+
\sqrt{2\lambda\gamma_n}\,Z_n^{i},
\end{equation}
where $\eta\geq 0$ controls the strength of repulsion. Setting $\eta=0$ recovers \texttt{WGF (MF-IID)}. The additional repulsive drift discourages particle collapse and promotes more diverse candidate batches.

\begin{table}[t!]
\centering
\footnotesize
\setlength{\tabcolsep}{5pt}
\begin{tabular}{p{1.8cm} p{2.5cm} p{2.5cm} p{2.5cm} p{2.4cm} p{2.6cm}}
\hline
\textbf{Method} & \textbf{Design-law ansatz} & \textbf{State evolved} & \textbf{Main advantage} & \textbf{Main limitation} & \textbf{Computational scaling} \\
\hline
\texttt{WGF (Joint)}
& $\nu_m\in\mathcal P(\Xi^m)$
& full batch $\boldsymbol\xi\in\Xi^m$
& most expressive; directly targets joint batch law
& scales poorly with batch size $m$
& $\mathcal O(R\,C_{\nabla G})$ \newline per iteration \\[5mm]

\texttt{WGF (MF)}
& $\nu_m=\mu_1\otimes{\cdot}\mkern1mu{\cdot}\mkern1mu{\cdot}\otimes\mu_m$
& $m$ particle systems $(\xi_{b}^{i,N})^{i\in[N_{\mathrm{mf}
}]}_{b\in[m]}$
& allows within-batch specialisation via distinct marginals
& must learn $m$ coupled marginals
& $\mathcal O(m N_{\mathrm{mf}
} K\,C_{\nabla_1 G})$ \newline per iteration \\[5mm]

\texttt{WGF (MF-IID)}
& $\nu_m=\mu^{\otimes m}$
& one particle system $(\xi^{i,N})^{i\in[N]}$
& simplest and most scalable approximation
& does not explicitly encourage within-batch diversity
& $\mathcal O(N K\,C_{\nabla_1 G})$ \newline per iteration \\[5mm]

\texttt{WGF (MF-IID-REP)}
& $\nu_m=\mu^{\otimes m}$
& one particle system $(\xi^{i,N})^{i\in[N]}$
& scalable and explicitly promotes diversity
& introduces an additional repulsion hyperparameter
& $\mathcal O (N(K\,C_{\nabla_1 G}+K_{\mathrm{rep}}\,C_{\nabla r}))$ \newline per iteration \\[.3mm]
\hline
\end{tabular}
\caption{\textbf{Summary of the four optimisation schemes considered in this paper.} The methods are ordered from the most expressive but least scalable formulation to the most scalable structured approximation. We write $R$ for the number of joint chains, $N_{\mathrm{mf}}$ for the number of particles per-coordinate for the mean-field approximation, $N$ for the number of particles in the i.i.d.\ approximation, $K$ for the number of sampled partner-tuples used to approximate the interaction term, $K_{\mathrm{rep}}$ for the number of repulsion samples, $C_{\nabla G}$ for the cost of one batch-gradient evaluation, $C_{\nabla_1 G}$ for the cost of one partial-gradient evaluation, and $C_{\nabla r}$ for the cost of one repulsion-gradient evaluation.}
\label{tab:algorithm-summary}
\end{table}

\paragraph{Discussion.}
Together, these four algorithms define a natural progression from the full batch-space formulation to increasingly structured and scalable approximations. The joint method is the most expressive, but also the most computationally demanding, since it evolves directly on $\Xi^m$. The mean-field formulation reduces this burden by evolving $m$ coupled marginal laws on $\Xi$, while still allowing different batch coordinates to specialise to different regions of the design space. The i.i.d.\ formulation goes one step further by learning a single shared design law, thereby providing the greatest scalability, at the cost of reduced expressiveness. Finally, the repulsive i.i.d.\ variant partially restores diversity at the level of the learned design law. In our numerical experiments, we compare all four of these algorithms in order to assess the trade-off between expressiveness, computational tractability, and the ability to discover diverse high-utility batches in examples of practical interest. To ensure a consistent regularisation across all algorithms, we fix $\lambda_m:=\frac{\lambda}{m}$ throughout (see Section~\ref{sec:method:batch-mf}).

\section{Numerical Experiments}
\label{sec:numerics}

We now present numerical experiments to illustrate the performance of our proposed methods. We also include comparisons to natural pointwise optimisation baselines.  We begin with two experiments in the single-design setting $m=1$, where the various structured design-law formulations coincide and the resulting dynamics reduce to Langevin sampling from an entropy-regularised design law on $\Xi$. These experiments therefore primarily isolate the effect of entropic regularisation and injected diffusion on exploration in multimodal landscapes, rather than the structured batch-law approximations that are specific to $m>1$. We then turn to genuine batch-design problems with $m>1$, where the different formulations introduced in Section~\ref{sec:method:objective} no longer coincide. This allows us to compare the full joint, mean-field, i.i.d., and repulsive i.i.d.\ approaches, and to assess the trade-off between expressiveness, scalability, and within-batch diversity. We perform all experiments on a MacBook Pro 16'' (2021) with an Apple M1 Pro chip and 16GB of RAM.

\subsection{1D Benchmark with Multimodal Observation Model}
\label{sec:exp1-toy}

\paragraph{Experimental Details} We first consider a one--dimensional BOED problem with scalar design variable $\xi \in [\xi_{\min},\xi_{\max}] \subset \mathbb{R}$, with batch size $m=1$. The parameter of interest is a binary latent variable \(\theta \in \{-1,+1\}\) with symmetric prior $\pi(\theta=+1)=\pi(\theta=-1)=\tfrac12$.  Given a design \(\xi\) and parameter \(\theta\), observations $y\in\mathbb{R}$ are generated according to the Gaussian likelihood $y \mid \theta,\xi \;\sim\; \mathcal{N}\!\big(\theta\,a(\xi),\,\sigma_y^2\big)$, where $\sigma_y>0$ is fixed and \(a(\xi)\ge 0\) is a design-dependent {sensitivity} (signal amplitude). To induce a non-convex objective with multiple separated optima, we construct \(a(\xi)\) as a positive baseline plus a mixture of localised Gaussian bumps:
\begin{equation}
a(\xi) = d_0 + \sum_{i=1}^4 d_i \,\exp\big(-\tfrac{(\xi-c_i)^2}{0.4}\big),
\label{eq:toy_sensitivity}
\end{equation}
where the amplitudes $(d_0,\dots,d_4)=(0.2,0.4,0.8,1.4,0.9)$, and the centres $c_1,\dots,c_4$ are evenly spaced across $[\xi_{\min},\xi_{\max}]:=[-3.5,3.5]$. This construction yields distinct regions with different sensitivities, resulting in a multimodal EIG landscape. In this case, \(\mathrm{EIG}(\xi)\) admits a one-dimensional integral representation, which we can evaluate accurately via a Gauss--Hermite quadrature scheme on a dense grid $\{\xi_i\}_{i=1}^{n}\subset[\xi_{\min},\xi_{\max}]$. This yields an effectively exact EIG landscape. We can then obtain $\nabla_\xi \mathrm{EIG}(\xi)$ via a finite difference scheme on the same grid, and use linear interpolation during optimisation. In our experiments, we compare gradient ascent (GA) with multiple restarts against i.i.d.\ copies of our proposed WGF. We provide further experimental details in Appendix~\ref{sec:add-experimental-details-1d-benchmark}.

\paragraph{Results.} Our first set of results is shown in Figure \ref{fig:3}. The top row (Fig. \ref{fig:3a} - Fig. \ref{fig:3c}) shows that gradient ascent is strongly basin-dependent in this multimodal landscape.  In particular, Figure \ref{fig:3a} indicates that the empirical distribution of the final designs  exhibits mode collapse, placing substantial mass at the suboptimal local maximiser. Figure \ref{fig:3b} shows that, given a poorly chosen initialisation, trajectories contract rapidly toward the local maximiser. Figure \ref{fig:3c}, which plots the maps $\xi_0\mapsto \xi_T$, further illustrates this point, revealing four attractor regions corresponding to the local and global optima.

\begin{figure}[t!]
\vspace{-2mm}
  \centering
{\small\bfseries Gradient Ascent} \\[1em]
  \begin{subfigure}[b]{0.31\textwidth}
    \centering
    \includegraphics[width=\linewidth]{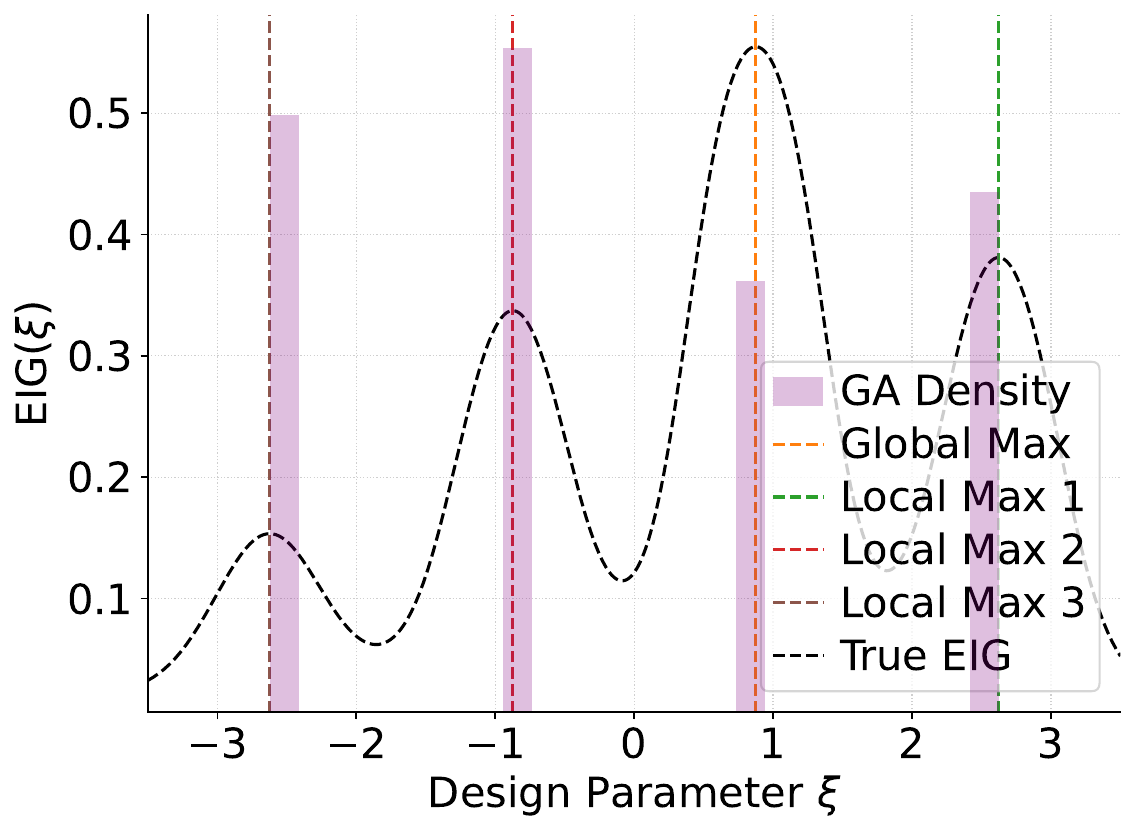}
    \caption{Final Distribution.}
    \label{fig:3a}
  \end{subfigure}
  \hfill
  \begin{subfigure}[b]{0.31\textwidth}
    \centering
    \includegraphics[width=\linewidth]{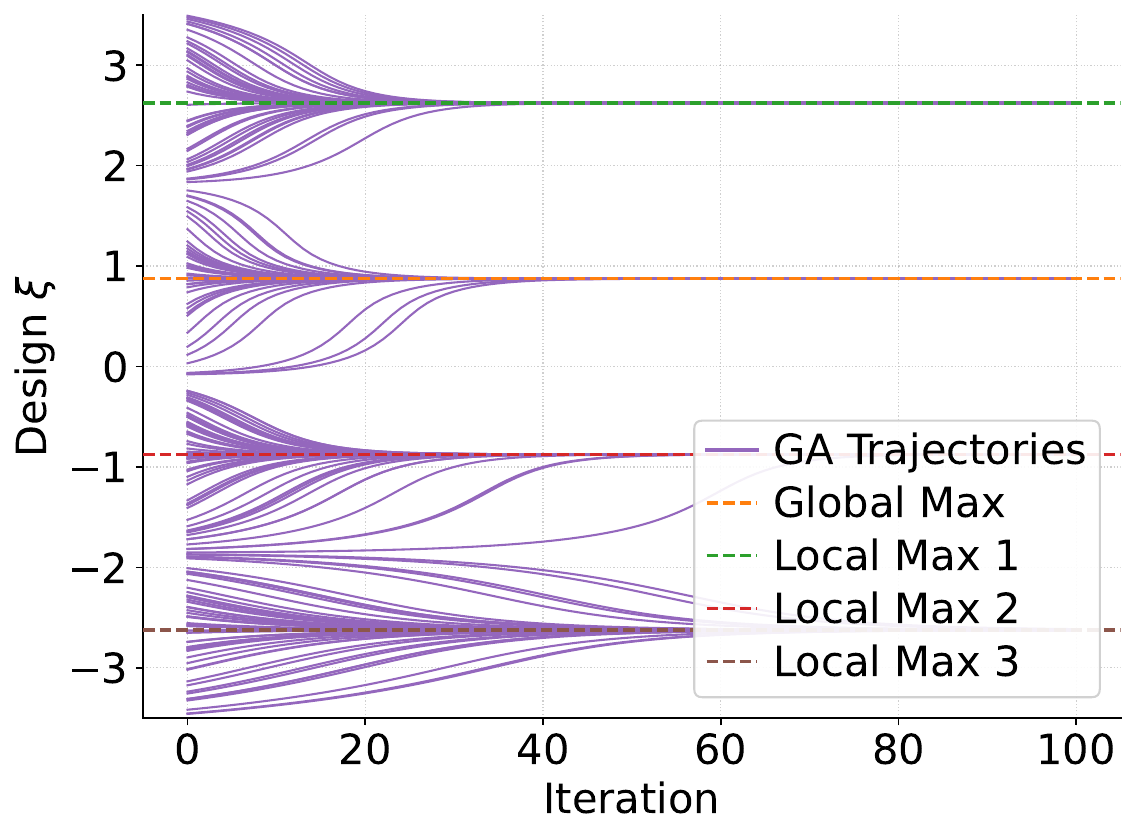}
    \caption{Trajectories  (Uniform Init).}
    \label{fig:3b}
  \end{subfigure}\hfill
  \begin{subfigure}[b]{0.31\textwidth}
    \centering
    \includegraphics[width=\linewidth]{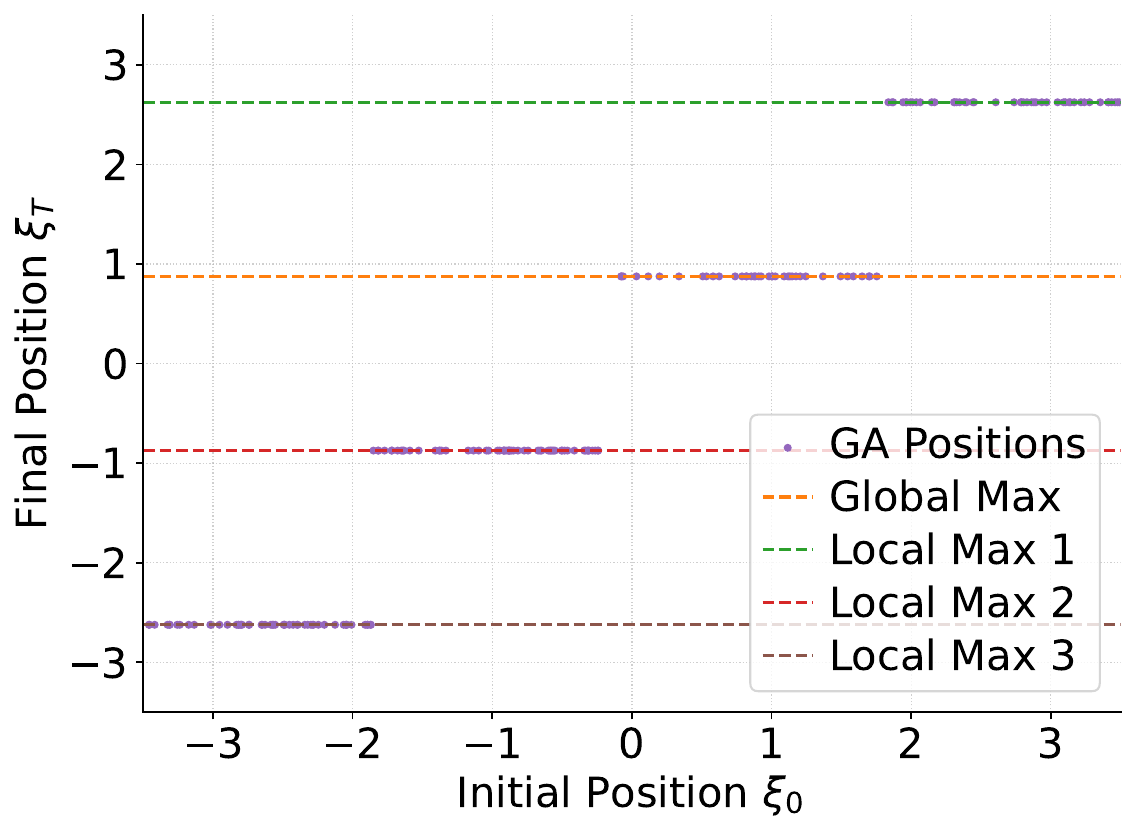}
    \caption{Final vs Initial Positions.}
    \label{fig:3c}
  \end{subfigure}
  ~\\[3mm]
{\small\bfseries Wasserstein Gradient Flow}\\[1em]
  \begin{subfigure}[b]{0.31\textwidth}
    \centering
    \includegraphics[width=\linewidth]{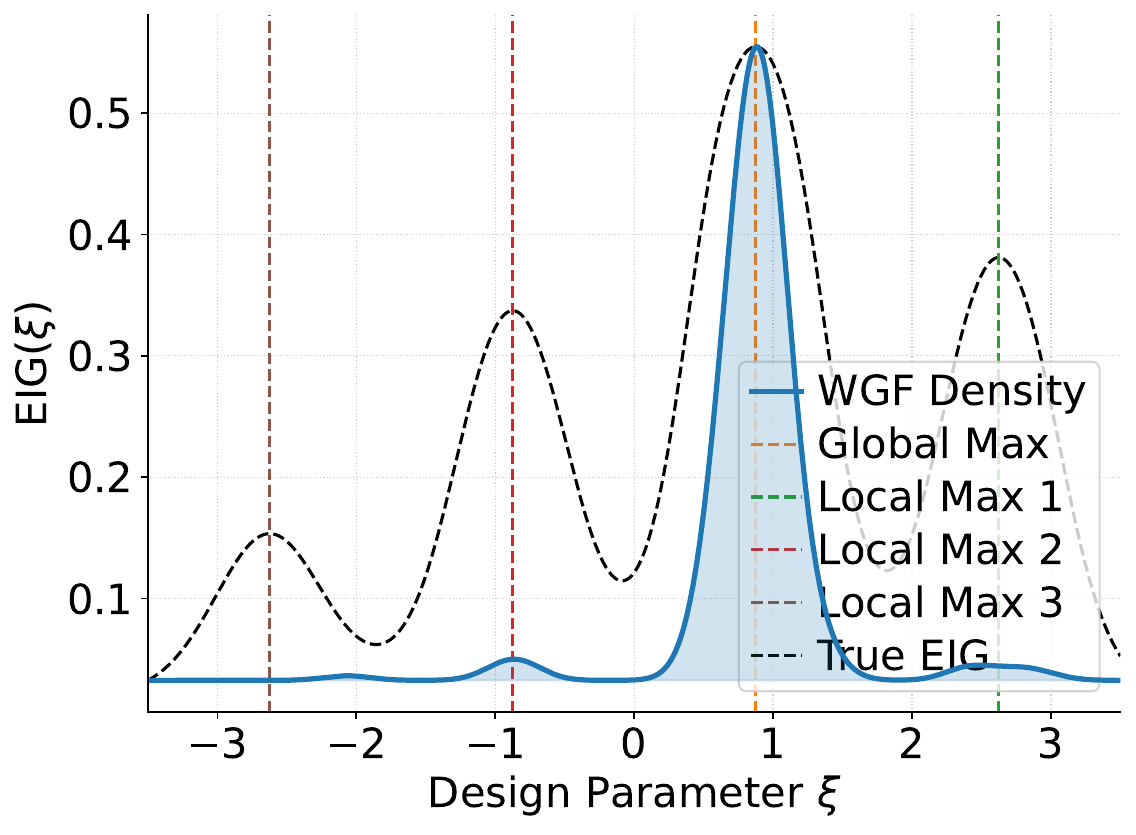}
    \caption{Final Distribution.}
    \label{fig:3d}
  \end{subfigure}
  \hfill
  \begin{subfigure}[b]{0.31\textwidth}
    \centering
    \includegraphics[width=\linewidth]{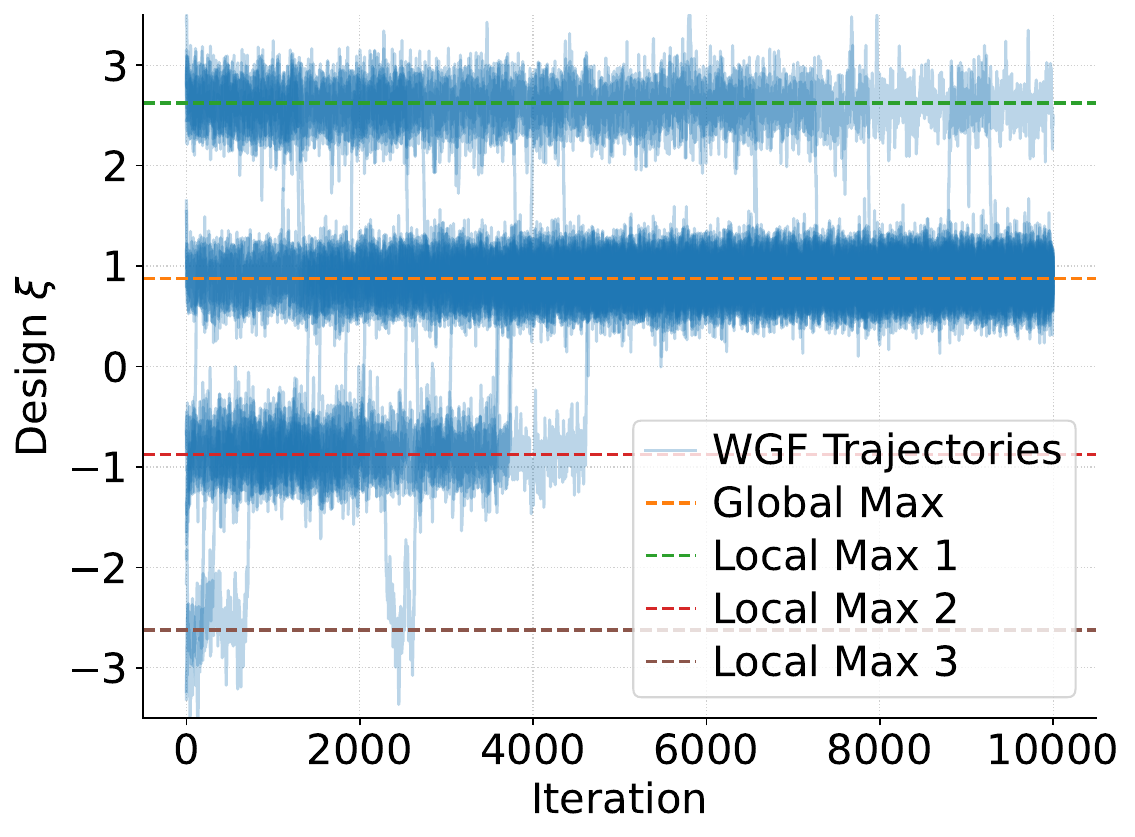}
    \caption{Trajectories (Uniform Init).}
    \label{fig:3e}
  \end{subfigure}\hfill
  \begin{subfigure}[b]{0.31\textwidth}
    \centering
    \includegraphics[width=\linewidth]{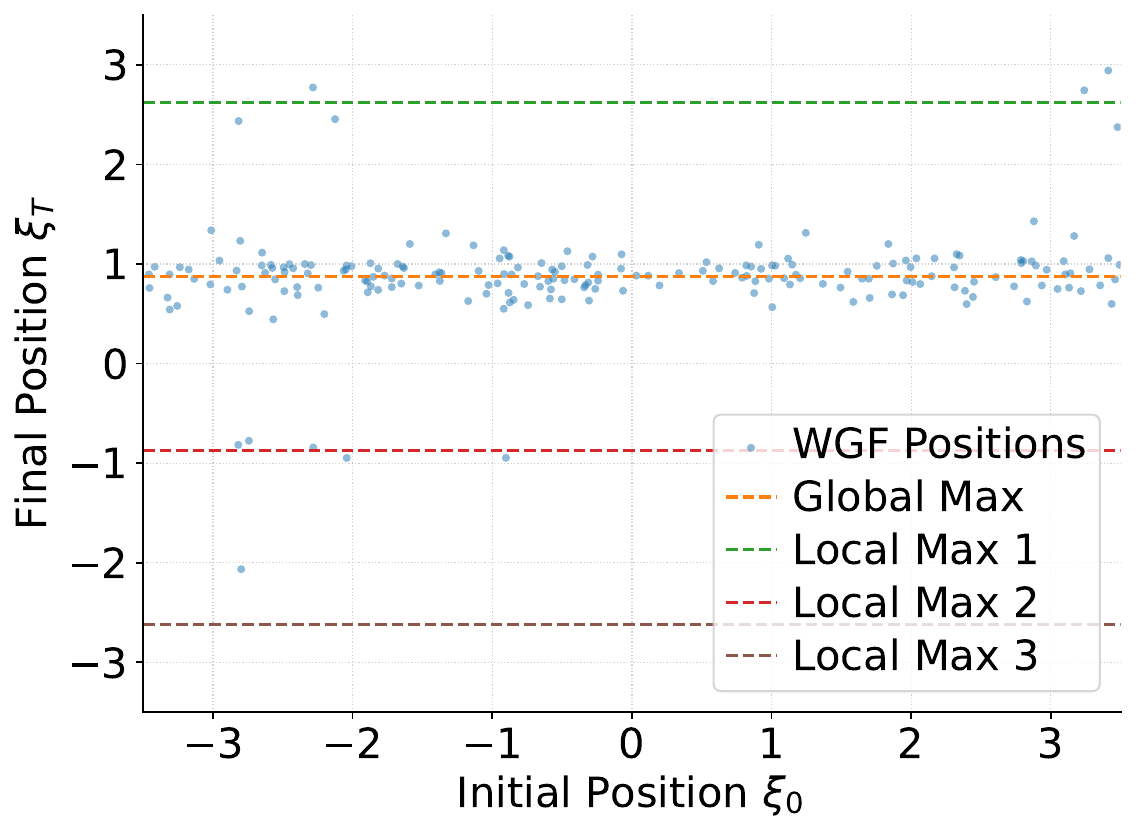}
    \caption{Final vs Initial Positions.}
    \label{fig:3f}
  \end{subfigure}\hfill
  \caption{\textbf{Comparison of pointwise optimisation and distributional optimisation for a one-dimensional experimental design problem.} The top row (Fig. \ref{fig:3a} - Fig. \ref{fig:3c}) shows the results of directly optimising the EIG using GA (purple); the bottom row (Fig. \ref{fig:3d} - Fig. \ref{fig:3f}) shows the results of optimising the entropy-regularised objective using the WGF (blue). To be specific, Fig. \ref{fig:3a} and Fig. \ref{fig:3d} show the empirical distribution of the final designs generated by the two approaches, given a uniform initialisation over the interval $[-3.5,3.5]$. Fig. \ref{fig:3b} and Fig. \ref{fig:3e} show the corresponding trajectories; while Fig. \ref{fig:3c} and Fig. \ref{fig:3f} show the mapping from initial designs $\xi_0$ to final designs $\xi_T$. In this example, gradient ascent converges to the local maximisers associated with its basins of attraction (Fig. \ref{fig:3a} - Fig. \ref{fig:3c}). Conversely, the additional noise allows the WGF to discover the global maximum (Fig. \ref{fig:3d} - Fig. \ref{fig:3f}).}
  \label{fig:3}
  \vspace{-2mm}
\end{figure}

The bottom row (Fig. \ref{fig:3d} - Fig. \ref{fig:3f}) demonstrates how distributional optimisation can mitigate these pathologies. Figure \ref{fig:3d} shows that the WGF results in a final design distribution that concentrates around the global optimum, while retaining sufficient spread to capture residual multimodality. Figure \ref{fig:3e} shows how the WGF maintains exploration through the injected noise, allowing particles initialised near local modes to escape and discover the global maximiser. Figure \ref{fig:3f} confirms this observation, illustrating that the injected noise yields markedly weaker dependence of $\xi_T$ on $\xi_0$, with most particles converging to the global optimum. Together, these results provide evidence that optimising over a design distribution, rather than a single design point, provides a principled mechanism for improved robustness to initialisation and better mode coverage. 

\begin{figure}[t!]
  \centering
  \begin{subfigure}[b]{0.31\textwidth}
    \centering
    \includegraphics[width=\linewidth]{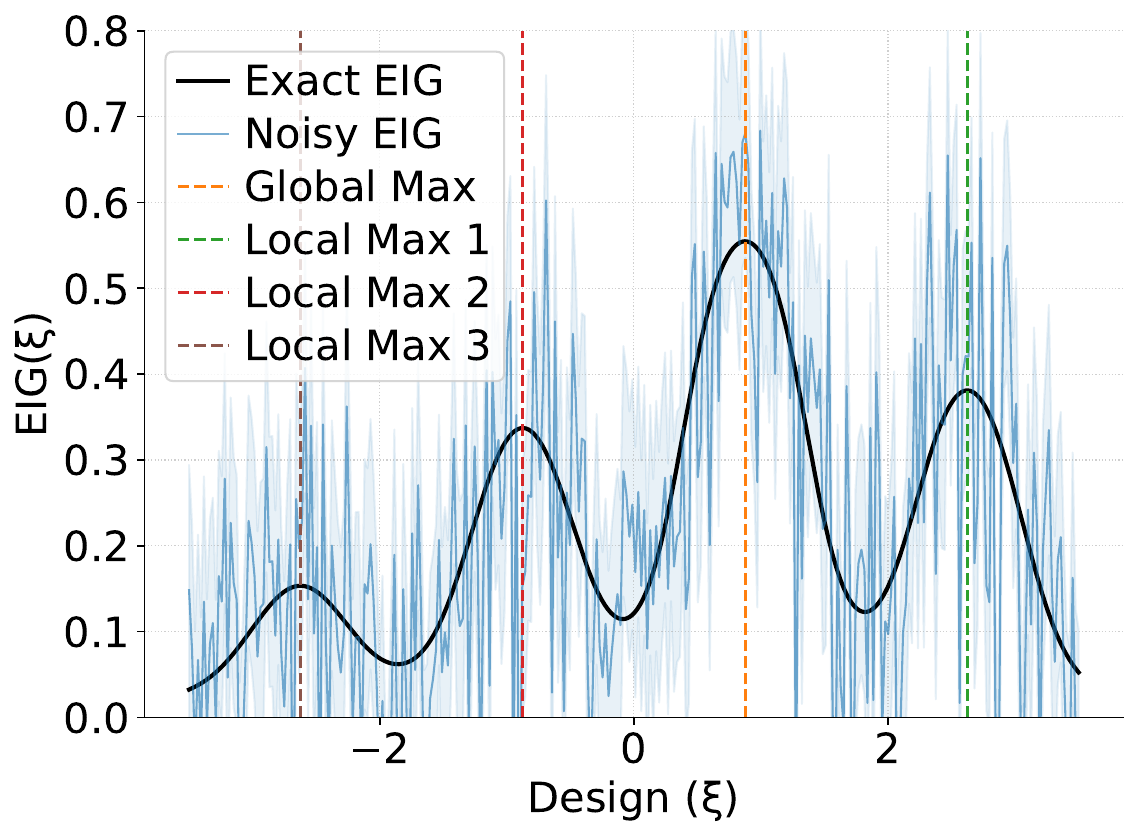}
    \caption{Noisy EIG Landscape.}
    \label{fig:4a}
  \end{subfigure}\hfill
  \begin{subfigure}[b]{0.31\textwidth}
    \centering
    \includegraphics[width=\linewidth]{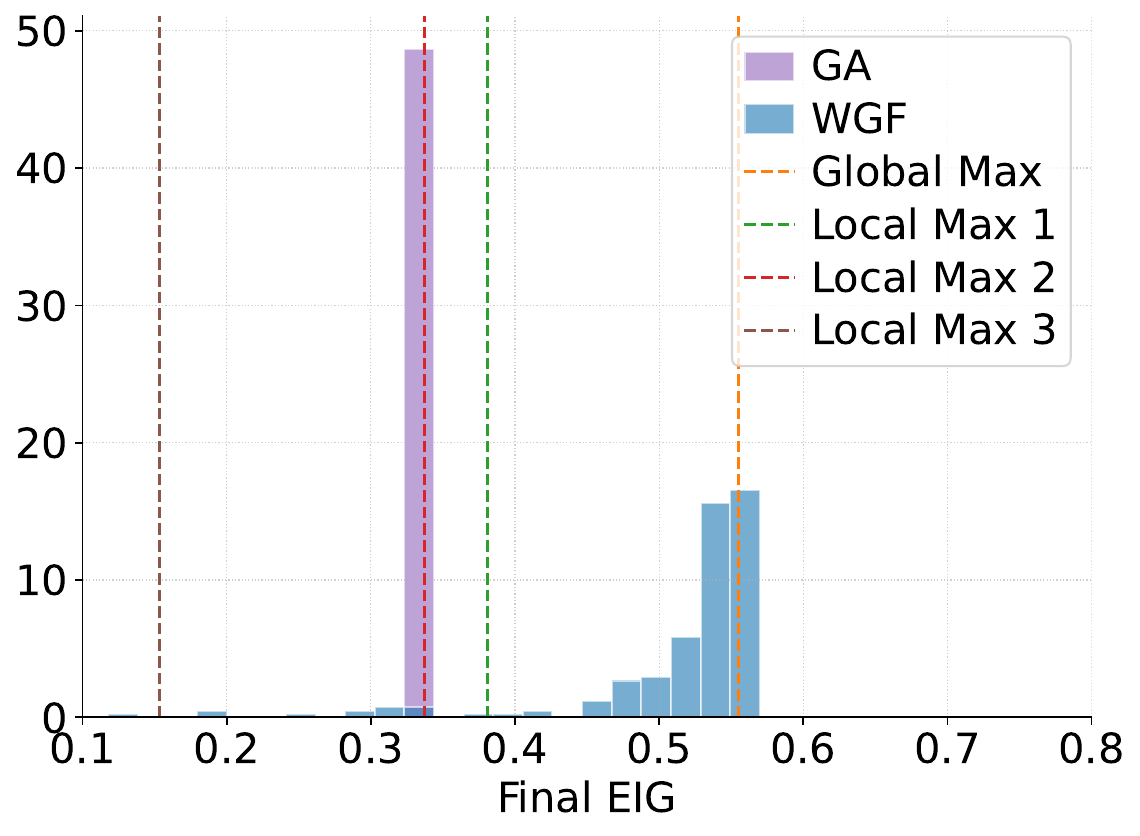}
    \caption{Final Obtained EIG.}
    \label{fig:4b}
  \end{subfigure}\hfill
  \begin{subfigure}[b]{0.31\textwidth}
    \centering
    \includegraphics[width=\linewidth]{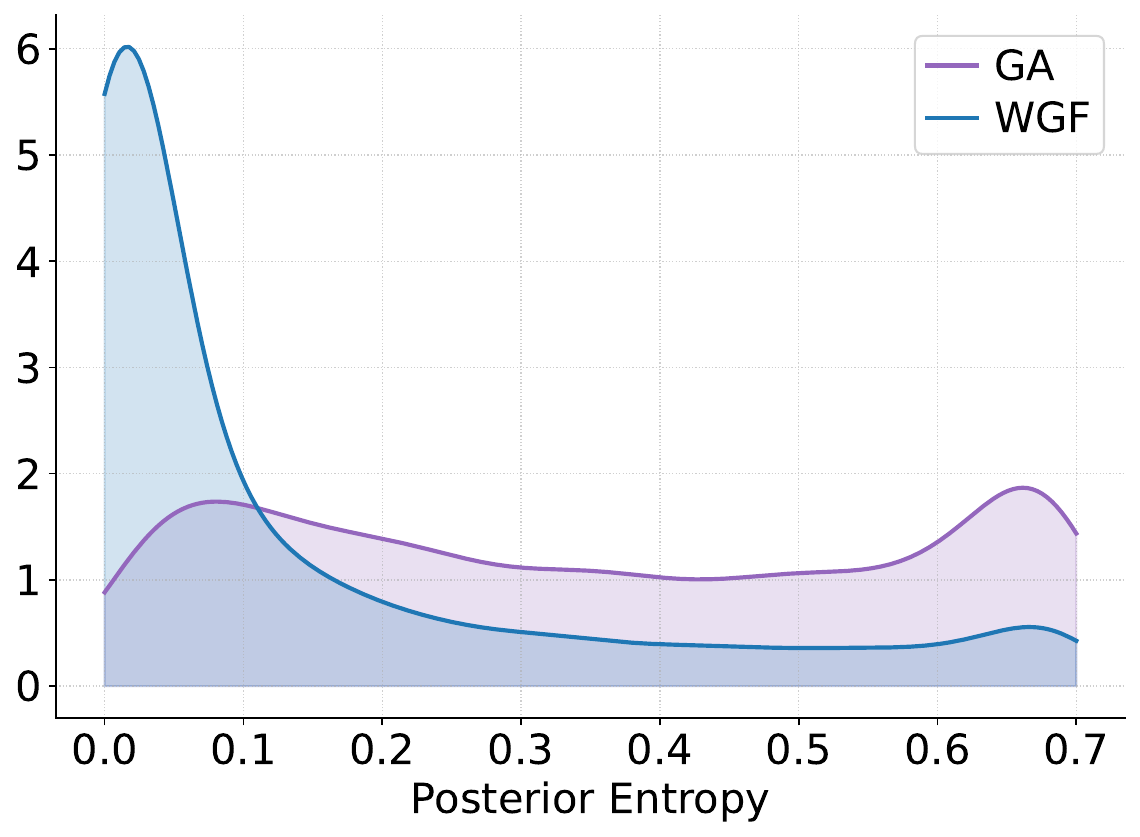}
    \caption{Posterior Entropy.}
    \label{fig:4c}
  \end{subfigure}
  \caption{\textbf{Comparison of stochastic pointwise optimisation and stochastic distributional optimisation for a one-dimensional experimental design problem.} Fig. \ref{fig:4a} displays the stochastic estimate of the EIG landscape. Fig. \ref{fig:4b} reports a histogram of the final EIG values obtained via SGA trajectories (purple) and WGF particles (blue), after initialisation near one of the local maxima. Fig. \ref{fig:4c} illustrates the posterior entropy associated with the ``best'' result obtained via stochastic gradient ascent (purple) or via the WGF (blue), as measured by the EIG, after initialisation at this same local maximum.}
  \label{fig:4}
\vspace{-2mm}
\end{figure}

In Figure \ref{fig:4}, we provide additional results, now only assuming access to a stochastic estimate of the EIG (and its gradient). Figure \ref{fig:4a} overlays the exact EIG (black) with a Monte Carlo estimate (blue), illustrating the high-variance landscape associated with the stochastic estimate of the objective. Meanwhile, Figures \ref{fig:4b} - \ref{fig:4c} compare the two methods over repeated runs, this time assuming a sub-optimal initialisation (i.e., near to one of the local maxima). Specifically, Figure \ref{fig:4b} reports the empirical distribution of the $\mathrm{EIG}$ values achieved by both methods: stochastic gradient ascent concentrates a substantial fraction of runs near the local maximum, whereas the interacting particle system more reliably attains values close to the global maximum. This reflects greater robustness to the initialisation, even in the presence of stochasticity. Figure \ref{fig:4c} provides a ``downstream'' validation of both methods, plotting the distribution of the posterior entropy obtained under designs produced by each method. In this case, lower is better, corresponding to more informative experiments. Consistent with the $\mathrm{EIG}$ outcomes, the WGF yields systematically lower posterior entropies than SGA, demonstrating that distributional optimisation not only improves the nominal objective, but also results in more informative experiments in terms of posterior uncertainty reduction.

\subsection{2D Non-Linear Sensor Placement with Multimodal Priors}
\label{sec:exp52}
\paragraph{Experimental Details.} We next consider a two-dimensional sensor-placement problem, again with $m=1$. In this case the scalar observation is generated according to
\begin{equation}
y = f(\theta,\xi) + \varepsilon, 
\qquad 
\varepsilon \sim \mathcal{N}(0,\sigma_y^2),
\end{equation}
where $\theta \in \mathbb{R}^2$ represents an unknown target location, and $\xi \in \Xi \subset \mathbb{R}^2$ a sensor location. We take $\Xi=[-5,5]^2$, and use the smooth radial response $\smash{f(\theta,\xi) = \exp\!\left(-\frac{\|\theta-\xi\|^2}{2\ell^2}\right)}$, with  $\ell=0.5$ and $\sigma_y=0.1$. The prior is a two-component Gaussian mixture, namely, 
\begin{equation}
\pi(\theta)
=
w\,\mathcal{N}(\theta;\mu_{\mathrm{major}},\sigma_{\mathrm{major}}^2 I_2)
+
(1-w)\,\mathcal{N}(\theta;\mu_{\mathrm{minor}},\sigma_{\mathrm{minor}}^2 I_2),
\end{equation}
with $w=0.6$, $\mu_{\mathrm{major}}=(2.2,0)$, $\mu_{\mathrm{minor}}=(-1.5,0)$, $\sigma_{\mathrm{major}}=0.2$, and $\sigma_{\mathrm{minor}}=0.5$. Even for $m=1$, this construction induces a deliberately non-convex EIG landscape with separated informative regions associated with the two prior modes. In this case, we must approximate the EIG and its gradient. We do so using a nested Monte Carlo estimator: an {outer} loop draws $(\theta,y)$ pairs from $\pi(\theta)p(y\mid \theta,\xi)$, and an {inner} loop approximates the marginal likelihood $p(y\mid \xi)$ using Monte Carlo integration under the prior $\pi(\theta)$. We obtain gradients by differentiating through the estimator, using the analytic derivative of $f(\theta,\xi)$. Similar to the last experiment, we compare SGA with multiple restarts against i.i.d.\ copies of our proposed WGF. Now, rather than reporting the final iterate, we report the best design {visited} in the final portion of the run, as selected via a common best-of-$n_{\mathrm{eval}}$ extraction procedure. Further experimental details are provided in Appendix~\ref{sec:additional-experimental-details-2d-non-linear}.

\paragraph{Results.} In Figure \ref{fig:5}, we display the designs selected by both methods, for three different initialisation regimes. Across all three initialisations, the WGF is able to discover both the local and global maxima of the EIG. This is not true for SGA. In particular, given an initialisation around the local mode, SGA never discovers the global mode (Fig. \ref{fig:5a}). Meanwhile, given an initialisation far from either mode in an uninformative region of the design space, SGA never discovers either mode (Fig. \ref{fig:5c}). 

\begin{figure}[t!]
  \centering
  \begin{subfigure}[b]{0.325\textwidth}
    \centering
    \includegraphics[width=\linewidth]{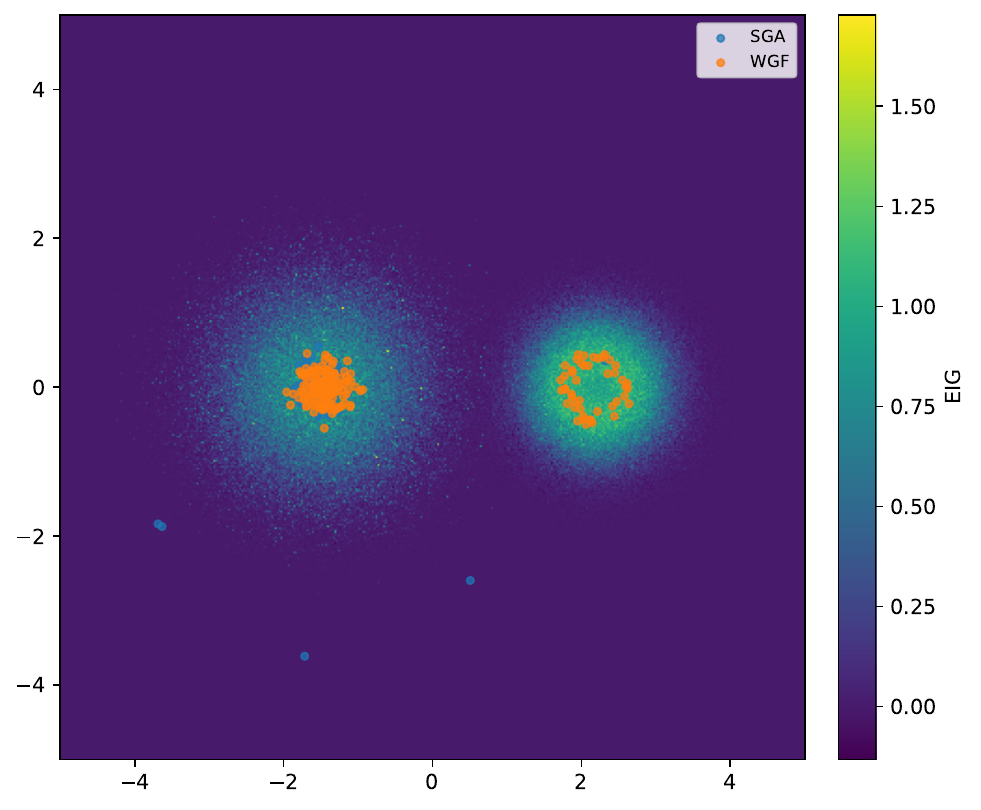}
    \caption{Local Mode Initialisation.}
    \label{fig:5a}
  \end{subfigure}\hfill
  \begin{subfigure}[b]{0.325\textwidth}
    \centering
    \includegraphics[width=\linewidth]{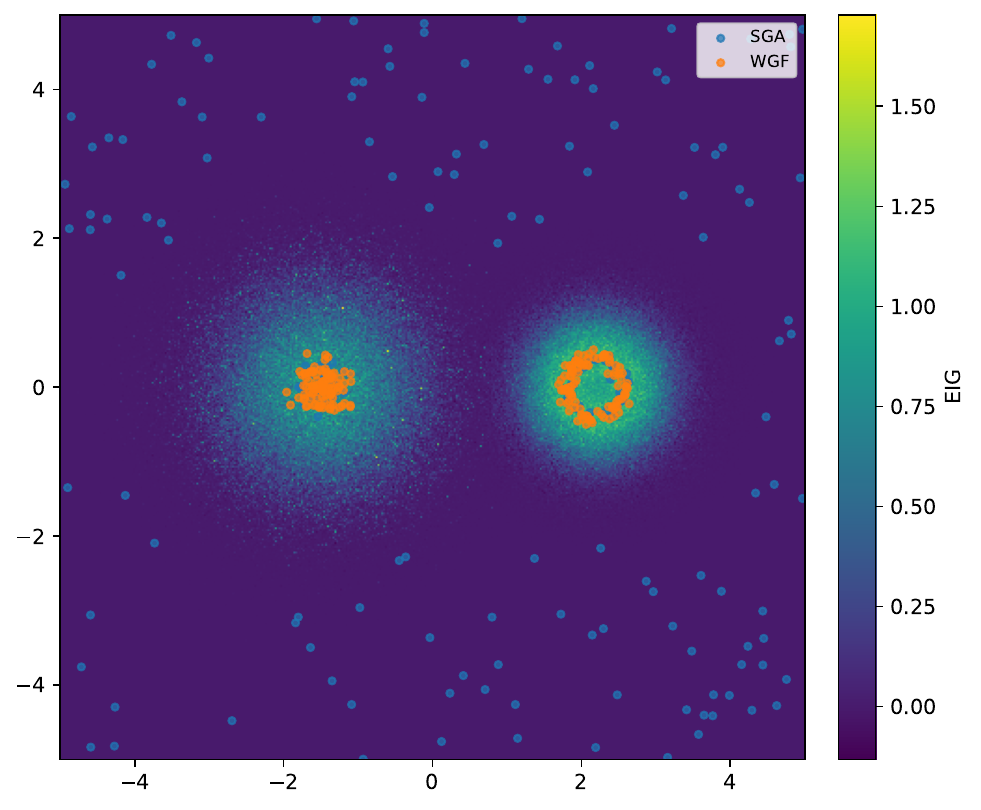}
    \caption{Global Initialisation.}
    \label{fig:5b}
  \end{subfigure}\hfill
  \begin{subfigure}[b]{0.325\textwidth}
    \centering
    \includegraphics[width=\linewidth]{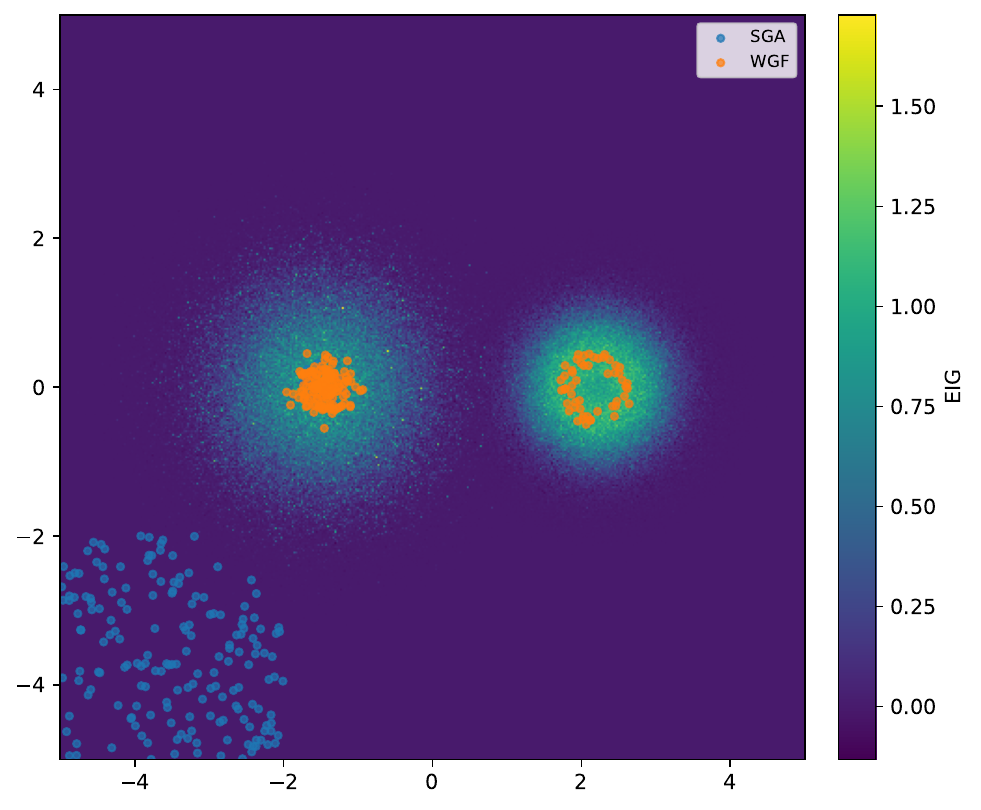}
    \caption{Misinformed Initialisation.}
    \label{fig:5c}
  \end{subfigure}
  \caption{\textbf{Comparison of the designs obtained using stochastic pointwise optimisation (blue) and stochastic distributional optimisation (orange) for a two-dimensional non-linear sensor placement problem, for three different initialisations.} Fig. \ref{fig:5a} displays the designs obtained using SGA with multiple restarts (blue) or i.i.d. copies of the WGF (orange), given a uniform initialisation around the minor mode. Fig. \ref{fig:5b} displays the corresponding results given a uniform initialisation over the entire domain $\Xi = [-5,5]^2$. Fig. \ref{fig:5c} displays the corresponding results given a uniform initialisation far from either mode.}
  \label{fig:5}
\vspace{-2mm}
\end{figure}

In Figure \ref{fig:6} we report the quality of these designs as measured by a high-fidelity EIG estimate. The results confirm our previous observations, with the WGF consistently achieving higher values of the EIG than SGA. This is particularly important when the initialisation is chosen poorly (e.g., Fig. \ref{fig:6a}, Fig. \ref{fig:6c}), in which case SGA entirely fails to discover the global maximum. These effects are expected to become even more pronounced in higher dimensions, where choosing a ``good'' (e.g., space filling) initialisation becomes exponentially more difficult.

\begin{figure}[b!]
  \centering
  \begin{subfigure}[b]{0.325\textwidth}
    \centering
    \includegraphics[width=\linewidth]{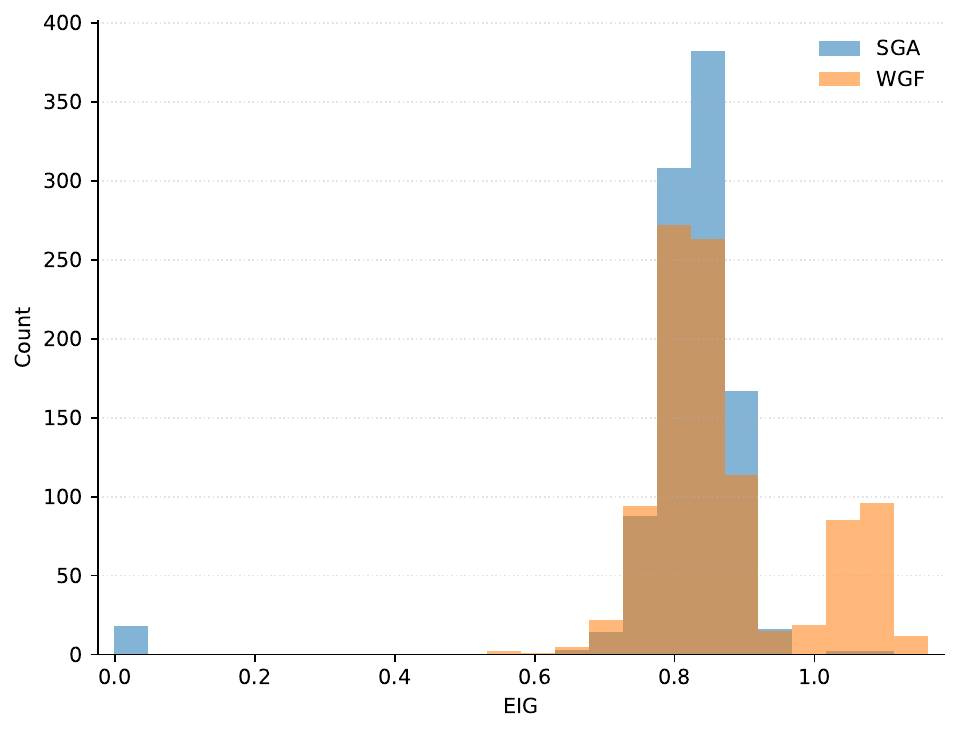}
    \caption{Local Mode Initialisation.}
    \label{fig:6a}
  \end{subfigure}\hfill
  \begin{subfigure}[b]{0.325\textwidth}
    \centering
    \includegraphics[width=\linewidth]{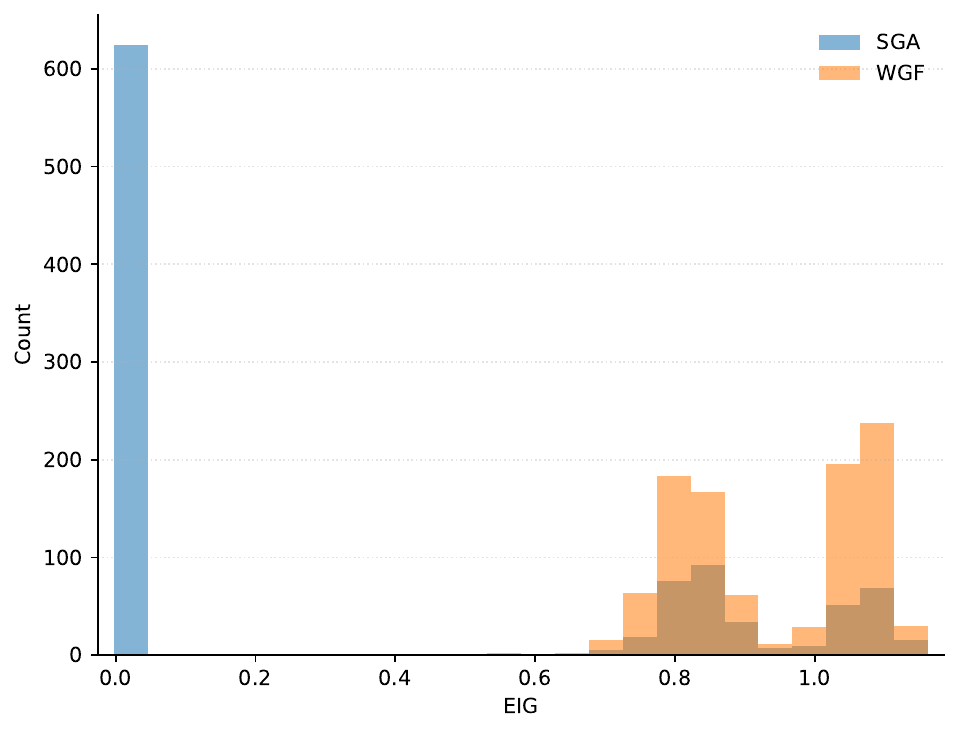}
    \caption{Global Initialisation.}
    \label{fig:6b}
  \end{subfigure}\hfill
  \begin{subfigure}[b]{0.325\textwidth}
    \centering
    \includegraphics[width=\linewidth]{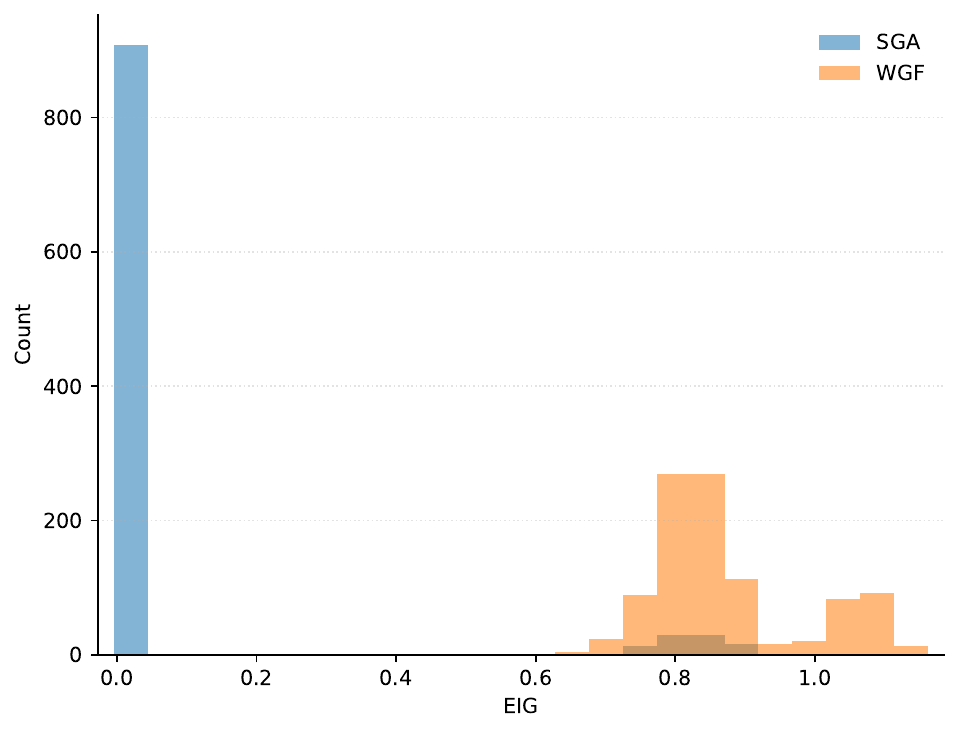}
    \caption{Misinformed Initialisation.}
    \label{fig:6c}
  \end{subfigure}
  \caption{\textbf{A comparison of the EIG (higher is better) achieved by the designs obtained using stochastic pointwise optimisation (blue) and stochastic distributional optimisation (orange) for a two-dimensional non-linear sensor placement problem, for three different initialisations.} In all three cases, the WGF outputs designs corresponding to higher values of the EIG.}
  \label{fig:6}
  \vspace{-2mm}
\end{figure}

 Finally, in Figure \ref{fig:7}, we report a downstream uncertainty proxy to verify that higher EIG corresponds to improved inferential precision. For $\theta^\star\sim\pi$, we draw $y^\star\sim p(\cdot\mid\theta^\star,\xi)$ and approximate the posterior using importance weighting of $M=5000$ prior samples $\{\theta^{(m)}\}_{m=1}^M$ with weights $w^{(m)}\propto p(y^\star\mid\theta^{(m)},\xi)$, normalised so that $\sum_{m=1}^M w^{(m)}=1$. We report $\mathrm{Tr}(\mathrm{Cov}(\theta\mid y^\star,\xi))$ computed from the weighted sample (lower is better). Consistent with the EIG comparison, the WGF yields systematically lower posterior uncertainty than SGA across all three initialisation regimes.

\begin{figure}[t!]
  \centering
  \begin{subfigure}[b]{0.325\textwidth}
    \centering
    \includegraphics[width=\linewidth]{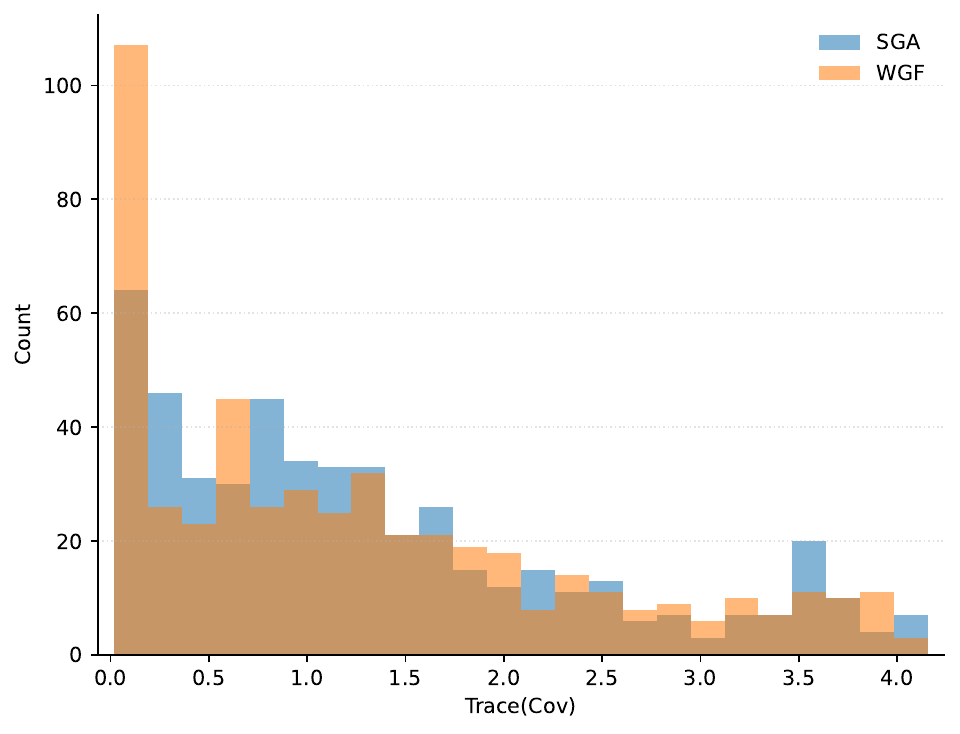}
    \caption{Local Mode Initialisation.}
    \label{fig:7a}
  \end{subfigure}\hfill
  \begin{subfigure}[b]{0.325\textwidth}
    \centering
    \includegraphics[width=\linewidth]{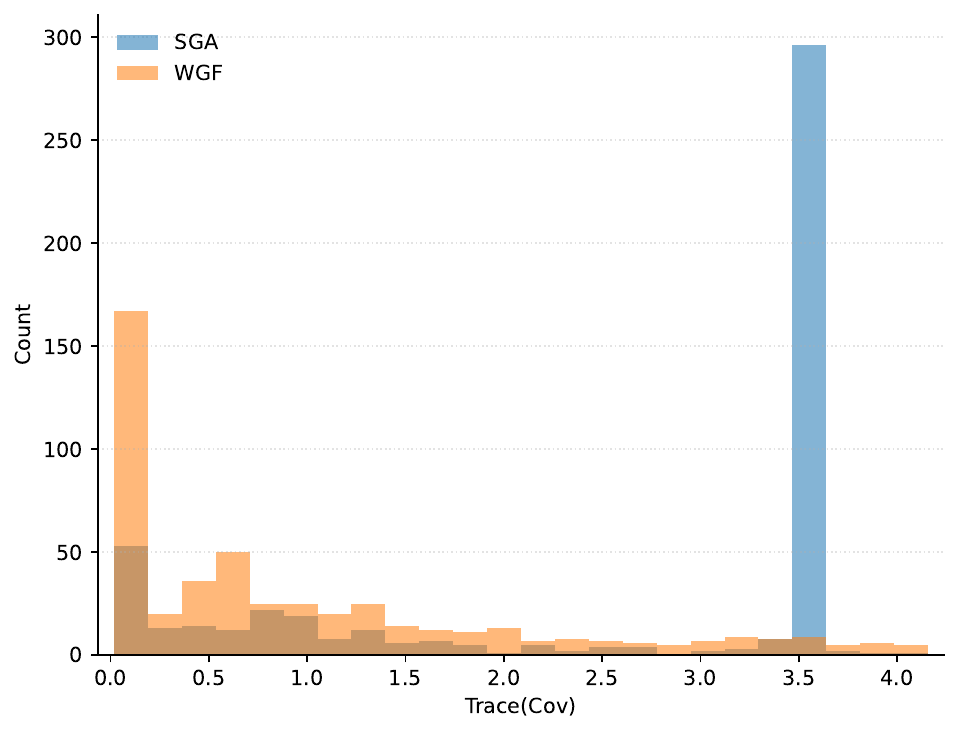}
    \caption{Global Initialisation.}
    \label{fig:7b}
  \end{subfigure}\hfill
  \begin{subfigure}[b]{0.325\textwidth}
    \centering
    \includegraphics[width=\linewidth]{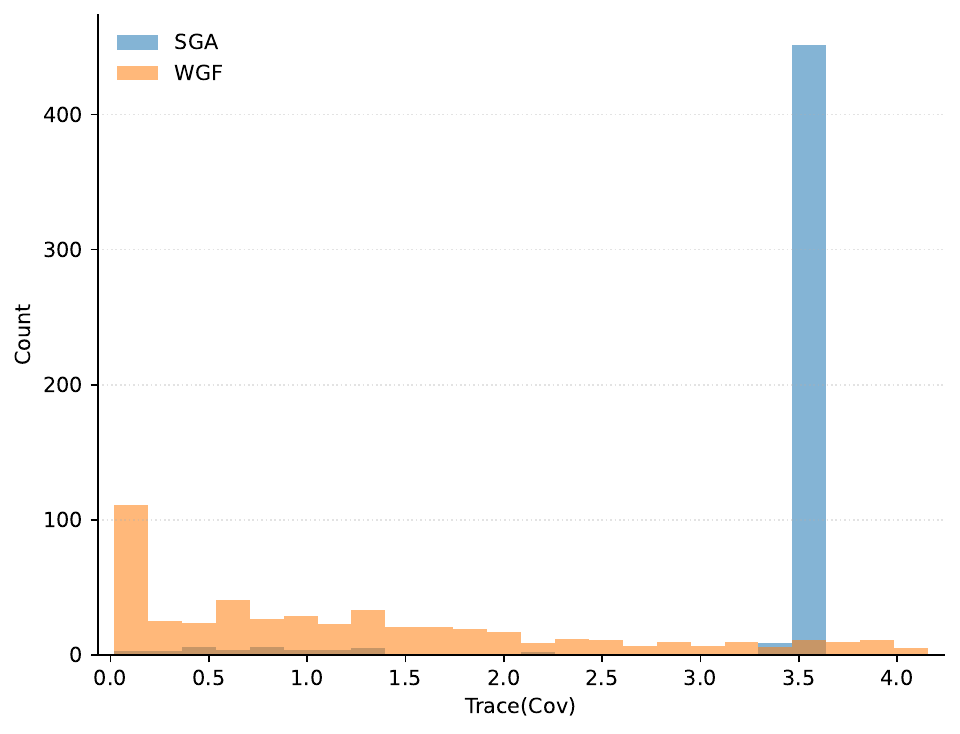}
    \caption{Misinformed Initialisation.}
    \label{fig:7c}
  \end{subfigure}
  \caption{\textbf{A comparison of the posterior uncertainty (lower is better) of the designs obtained using stochastic pointwise optimisation (blue) and stochastic distributional optimisation (orange) for a two-dimensional non-linear sensor placement problem, for three different initialisations.} In all three cases, the WGF outputs designs corresponding to lower values of the posterior uncertainty.}
  \label{fig:7}
\end{figure}

\subsection{Batch Design on the Torus}
\label{sec:toy_batch_design}

\paragraph{Experimental Details.}
We next consider a non-convex batch BOED problem in which each experiment corresponds to observing a noisy scalar response at a location on the circle. The design is therefore a batch of $m$ angles $\boldsymbol{\xi}_m=(\xi_1,\dots,\xi_m)\in[-\pi,\pi)^m$, with angles identified modulo $2\pi$. The parameter of interest is $\theta\in\mathbb{R}^2$ with prior $\theta\sim \mathcal{N}(0,I_2)$. Given a batch design $\boldsymbol{\xi}_m$ and a parameter $\theta$, we observe independent responses according to $y_j \mid \theta, \xi_j \;\sim\; \mathcal{N}\!\big(h(\xi_j)^\top \theta,\;\sigma_y^2\big)$,  where the forward map $h(\xi)\in\mathbb{R}^2$ is  defined in terms of a multimodal, periodic sensitivity profile
\begin{equation}
\label{eq:h_def}
h(\xi) \;=\; a(\xi)\begin{bmatrix}\cos(\xi)\\ \sin(\xi)\end{bmatrix},\qquad
a(\xi)
\;=\;
d_0 + \sum_{k=1}^{4} d_k \exp\!\Big(-\tfrac{1}{2}\big(\tfrac{d(\xi,c_k)}{\ell_0}\big)^2\Big),
\end{equation}
where $d(\xi,c)$ denotes the wrapped (i.e., shortest signed) circular difference in $[-\pi,\pi)$, the centres are $c=(c_1,\dots,c_4)=(0,\pi/2,-\pi/2,\pi)$, the amplitudes are $(d_0,\dots,d_4)=(0.4,2.0,1.9,1.6,1.0)$, and the width parameter is $\ell_0=\ell=0.3$. The noise level is $\sigma_y=0.35$. For this linear Gaussian model, the posterior is Gaussian with $\smash{\Sigma_{\text{post}}(\boldsymbol{\xi}) =(I_2 + \tfrac{1}{\sigma_y^2}H(\boldsymbol{\xi})^\top H(\boldsymbol{\xi}))^{-1}}$,  where $H(\boldsymbol{\xi})\in\mathbb{R}^{m\times 2}$ stacks the rows of $h(\xi_j)^\top$. The EIG thus admits a closed form,  and we can compute $\nabla_{\xi_j}\,\mathrm{EIG}(\boldsymbol\xi)$ exactly. In any case, this construction yields an objective with multiple separated optima and strong within-batch dependencies: repeating a highly informative angle is typically redundant for $m>1$.

We report results for the four methods summarised in Section~\ref{sec:implemented-algorithms}. As in the previous experiment, we report deterministic batches obtained by a common best-of-$n_{\mathrm{eval}}$ extraction rule. We also report results for two additional baselines. The first is the repeated best single design: $\smash{\boldsymbol{\xi}_{\text{rep}}=(\xi_\star^{(1)},\dots,\xi_\star^{(1)})}$, where  $\smash{\xi_\star^{(1)}=\argmax_{\xi\in[-\pi,\pi)}\mathrm{EIG}(\xi)}$. This design is intentionally naive, ignoring correlation among batch elements, but can be competitive when the objective is dominated by a single highly informative region. The second is gradient ascent (GA) on $\mathrm{EIG}(\boldsymbol{\xi})$ in the $m$-dimensional batch space, with multiple random restarts. We provide further experimental details in Appendix~\ref{sec:additional-experimental-details-batch-design}.

\paragraph{Results.} Figure~\ref{fig:9} reports $\mathrm{EIG}_m$ versus batch size under local and global initialisation.
In both regimes, repeating the best single design is increasingly suboptimal as $m$ grows, reflecting strong within-batch redundancy. Under a {global initialisation} (Fig.~\ref{fig:9b}), the remaining methods perform similarly, consistent with a good initial coverage of $\Xi$. On the other hand, under a {local initialisation} (Fig.~\ref{fig:9a}), GA is worse than the WGF-based approaches, for all values of $m$. This is consistent with the multimodal objective: trajectories initialised in local basins tend to remain trapped, with the resulting batches failing to exploit alternative informative regions. In contrast, the WGF-based methods are substantially more robust. Among these, we observe a consistent ordering: \texttt{WGF (Joint)} (green) is outperformed by \texttt{WGF (MF)} (red), \texttt{WGF (MF-IID)} (purple) attains further improvements, and \texttt{WGF (MF-IID-REP)} (brown) performs best overall. 

\begin{figure}[t!]
\vspace{-2mm}
    \centering
    \begin{subfigure}[b]{0.475\textwidth}
    \includegraphics[width=\linewidth]{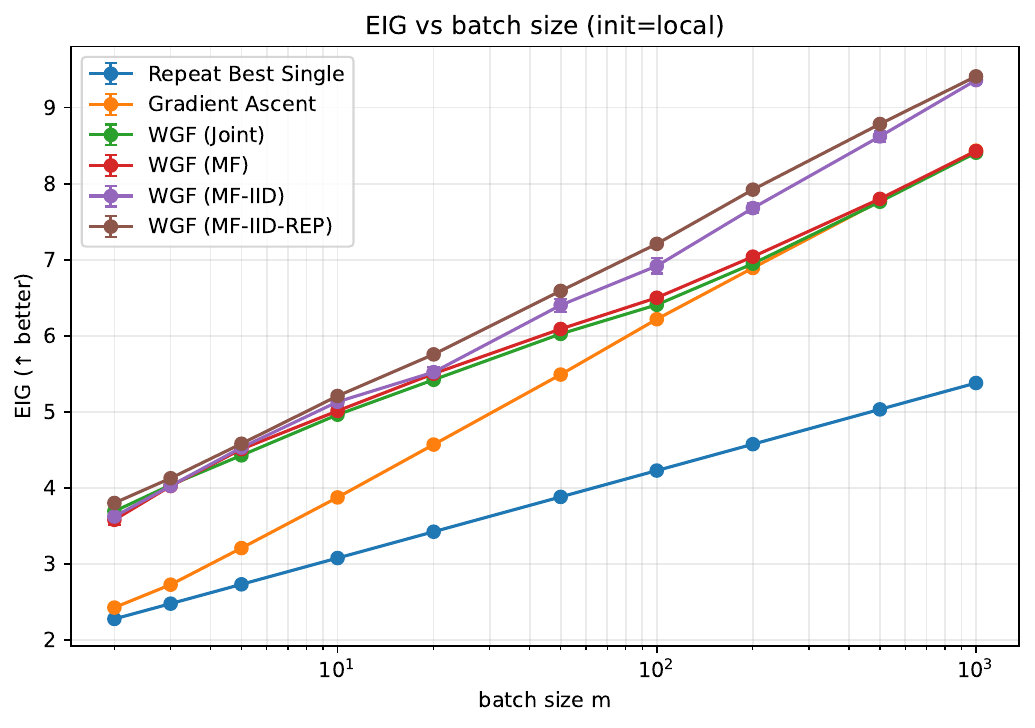}
    \caption{Local Initialisation.}
    \label{fig:9a}
    \end{subfigure}\hfill
    \begin{subfigure}[b]{0.475\textwidth}
    \includegraphics[width=\linewidth]{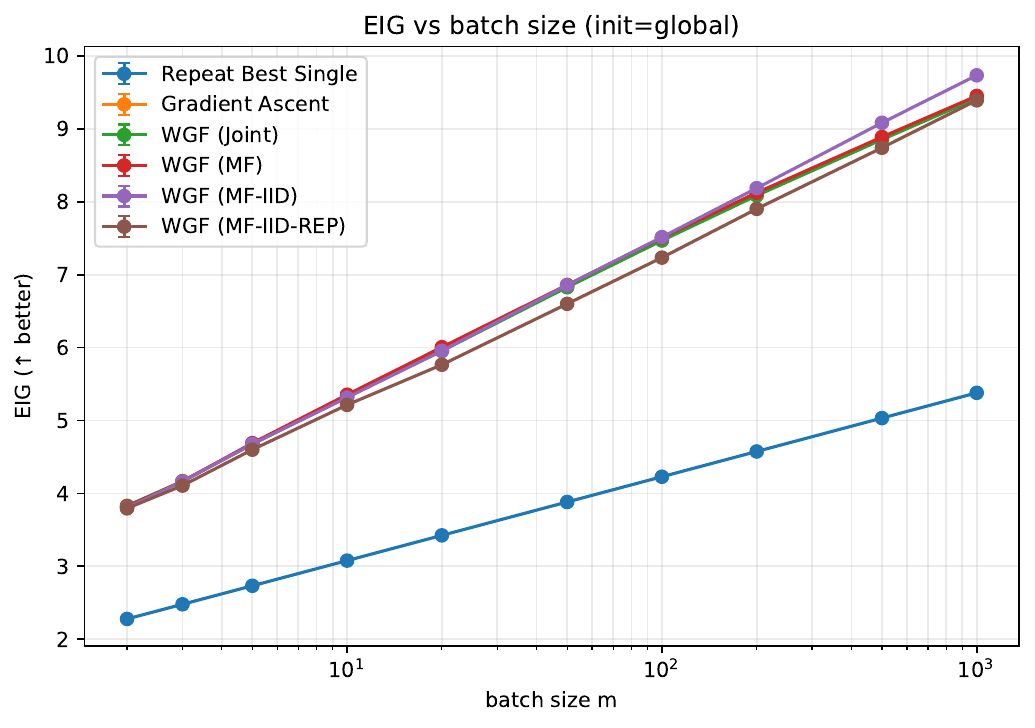}
    \caption{Global Initialisation.}
    \label{fig:9b}
    \end{subfigure}
    \caption{\textbf{A comparison of the EIG obtained using pointwise optimisation and distributional optimisation, including both joint and mean-field approaches, as a function of the batch size.} We plot the achieved EIG as a function of the batch size $m\in\{2,3,5,10,20,50,100,200,500,1000\}$ for six design strategies: repeating the best single design, gradient ascent, \texttt{WGF (Joint)}, \texttt{WGF (MF)},  \texttt{WGF (MF-IID)}, and \texttt{WGF (MF-IID-REP)}. The points show the mean EIG over 5 independent random runs, while the error bars denote $\pm$ one standard error. The ``distributional'' methods (green, red, purple, brown) are superior to the ``pointwise'' method (orange) across all batch sizes; while the repeat-single baseline is clearly suboptimal due to redundancy in repeated measurements (blue). For larger batch sizes, the single-law i.i.d.\ approximations (purple, brown) display an increasing advantage over the joint and coordinate-wise mean-field methods (green, red).}
    \label{fig:9}
\vspace{-3mm}
\end{figure}

\paragraph{Discussion.} It may at first seem counter-intuitive that the i.i.d.\ mean-field methods (purple, brown) can outperform the joint method (green) for large batch sizes. Indeed, by definition, the optimum over the full space $\mathcal{P}_2(\Xi^m)$ is at least as good as the optimum over the restricted space $\mathcal{P}_{2,\mathrm{i.i.d.}}(\Xi^m):=\{\mu^{\otimes m}:\mu\in\mathcal P_2(\Xi)\}$. Thus, our results do not (and cannot) illustrate that a product family can exceed the {true} batch optimum. Rather, they provide evidence that, under fixed iteration budgets and the matched temperature scaling $\lambda_m = \frac{\lambda}{m}$, the restricted design-law formulations can sometimes yield better empirical solutions than the full joint analogue. There are several plausible explanations. First, joint methods operate in the $m$-dimensional space $\Xi^m$, and must explore a high-dimensional landscape with many symmetries, e.g., permutations of design coordinates, and potential energy barriers. As the batch size increases, it becomes increasingly hard to explore $\Xi^m$, as mixing degrades rapidly. In contrast, the mean-field method always operates on $\Xi$, a space which is much easier to explore. Second, once we have learned a design law, we form candidate batches by sampling $\xi_{1:m}\sim \mu^{\otimes m}$, before reporting the design with best utility. This mechanism can be viewed as a global search over combinatorial combinations of the modes of the learned design law. If this law concentrates non-trivial mass on several high-quality regions of $\Xi$, then i.i.d.\ batching in this fashion generates many possible multimodal configurations, and our post-selection procedure can reliably extract a strong batch even when direct joint exploration of $\Xi^m$ fails to locate such configurations within the same budget.

\subsection{Pharmacokinetic Benchmark}
\label{sec:exp:pk-fhn}

\paragraph{Experimental Details.}
We next consider an established benchmark from the BOED literature: pharmacokinetic (PK) sampling-time design \citep[e.g.,][]{overstall2020bayesian}. This is a batch design problem in which $\xi_{1:m}=(t_1,\dots,t_m)$ are observation times over a fixed horizon $[0,T_{\max}]$, to be chosen in order to measure the concentration of a previously administered drug. The concentrations $y_{1:m}=(y(t_1),\dots,y(t_m))$ at these times are assumed to be conditionally independent given $(\theta,\xi_{1:m})$, with
\begin{equation}
\label{eq:pk_mean}
y(t_j)\mid \theta,\xi_{1:m}\ \sim\ \mathcal{N}(a(\theta)\mu_\theta(t_j),\ \sigma^2\,b_\theta(t_j)),
\qquad \sigma^2=0.1,
\end{equation}
where $\theta=(\theta_1,\theta_2,\theta_3)\in\mathbb{R}^3_{+}$ denotes the parameter of interest, $\sigma^2$ is a fixed noise variance, $a(\cdot)$ and $b(\cdot,\cdot)$ are application-dependent functions, and $\mu_\theta(t)=e^{-\theta_1 t}-e^{-\theta_2 t}$. In this case, as in \citet{ryan2014towards,overstall2020bayesian}, we assume that 
\begin{equation}
\label{eq:pk_mu_var}
a(\theta)\;=\;\frac{400\,\theta_2}{\theta_3(\theta_2-\theta_1)},\qquad b_\theta(t)\;=\;1+\frac{a(\theta)^2}{10}(e^{-\theta_1 t}-e^{-\theta_2 t})^2.
\end{equation}
Finally, the parameters are assigned independent log-normal priors: $\log \theta \sim \mathcal{N}(\mu_{\log},\,\sigma^2_{\log} I_3)$, where $\mu_{\log} = (\log 0.1,\ \log 1.0,\ \log 20.0)$ and $\sigma^2_{\log}=0.05$. Following \cite{overstall2020bayesian}, we consider a batch size of $m=15$, a time horizon of $T_{\max}=24$ hours, and enforce ordered sampling times with a minimum spacing of $\Delta_{\min}=0.25$ hours (15 minutes), implemented by clipping to $[0,T_{\max}]$, sorting, and a deterministic minimum-gap repair. We approximate $\mathrm{EIG}_m(\xi_{1:m})$ using a fixed-sample NMC estimator \citep[e.g.,][]{rainforth2018nesting} with $(n_{\mathrm{outer}},n_{\mathrm{inner}})$ samples. In this model, the likelihood is conditionally Gaussian with mean and heteroscedastic variance determined by closed-form model statistics, so both $\mathrm{EIG}_m$ and its gradients can be evaluated efficiently once the fixed randomness is set. 

Similar to before, we consider the four WGF-based methods summarised in Section~\ref{sec:implemented-algorithms}: \texttt{WGF (Joint)}, \texttt{WGF (MF)}, \texttt{WGF (MF-IID)}, and \texttt{WGF (MF-IID-REP)}. We also consider two additional variants: \texttt{WGF (Joint) (FUSE)} is a version of \texttt{WGF (Joint)} which uses the FUSE adaptive step-size schedule introduced in \citet{sharrock2025tuning}, rather than a constant step size. Meanwhile, \texttt{WGF (MF) (Sub)} is a version of \texttt{WGF (MF)} which sub-samples coordinates to update at each time step.  In addition to our own methods, we also consider several natural baselines. First, we include a deterministic \texttt{Uniform} design with evenly spaced times on \([0,T_{\max}]\), together with two dimension-reduction baselines: \texttt{GeometricDRS} and \texttt{BetaDRS} \citep[e.g.,][]{ryan2014towards,overstall2020bayesian}. We also include three coordinate-exchange-type methods \citep{meyer1995coordinate,overstall2020bayesian}: a feasible-grid heuristic \texttt{CE (Feasible Grid)} and two lightweight GP-based variants: \texttt{CE (GP)} and \texttt{CE (GP-G)}. Finally, we compare against \texttt{SGA (Adam)} and an \texttt{Annealed SMC} approach. For all methods, we select the final design using an additional best-of-$n_{\mathrm{eval}}$ extraction rule. We provide further experimental details in Appendix~\ref{sec:additional-experimental-details-pk}.

\begin{figure}[b!]
\vspace{-3mm}
    \centering
    \includegraphics[width=.9\linewidth]{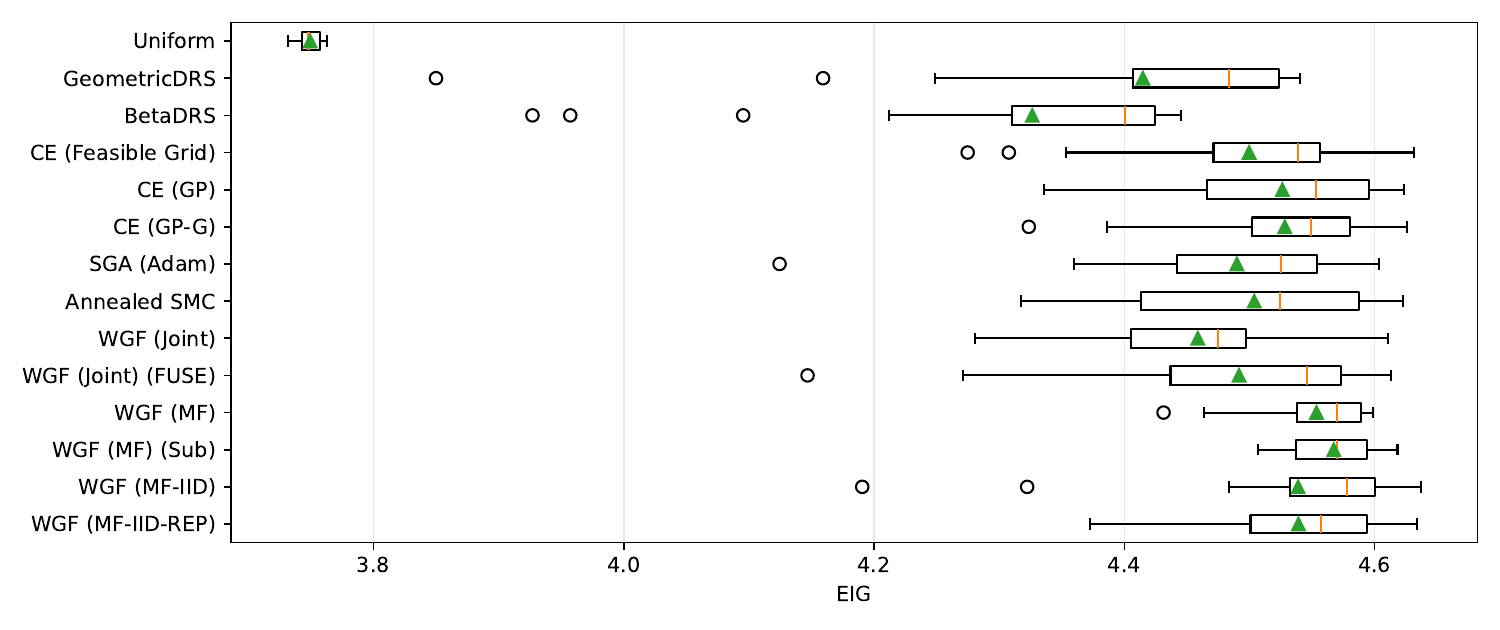}
    \vspace{-4mm}
    \caption{\textbf{EIG summaries for the pharmacokinetic (PK) sampling-time benchmark.} Each boxplot summarises the distribution of independently evaluated $\mathrm{EIG}_m$ values obtained from designs returned over multiple optimisation seeds (higher is better). Boxes show the interquartile range with median (orange line); whiskers extend to $1.5\times$ IQR and circles denote outliers; green triangles indicate the mean.}
    \label{fig:10}
\vspace{-2mm}
\end{figure}

\paragraph{Results} In Figure~\ref{fig:10}, we plot a summary of the EIG attained by the designs output by each method, with the corresponding designs shown in Figure~\ref{fig:11}. The designs attained by the most performant methods recover the characteristic structure reported in the BOED literature: sampling times concentrate in both the early phase (capturing the rapid rise and peak) and the late phase (capturing the elimination tail), with comparatively fewer mid-horizon observations \citep[e.g.,][]{overstall2020bayesian}. Consistent with this observation, the \texttt{Uniform} baseline is clearly suboptimal, while the two-parameter DRS schedules \texttt{GeometricDRS} and \texttt{BetaDRS} return improved but still deficient designs that cannot fully match the other approaches \citep[see also][Section 3.2]{overstall2020bayesian}. Among the optimisation-based methods, the CE style baselines are generally strong, outperforming both \texttt{SGA (Adam)} and \texttt{Annealed SMC}. Meanwhile, amongst the flow-based methods, \texttt{WGF (Joint)} is the weakest performing, although its performance is slightly improved by the adaptive step-size variant \texttt{WGF (Joint) (FUSE)}, or by a more carefully tuned constant step size; see Appendix~\ref{sec:additional-experimental-details-pk}.\footnote{In this experiment, we fixed the step size across all gradient-based methods to enable more direct comparisons. Additional ablations indicate that the performance of \texttt{WGF (Joint)} can be improved by using an increased step size relative to its mean-field analogue; see Appendix~\ref{sec:additional-experimental-details-pk}.} On the other hand, \texttt{WGF (MF)}, \texttt{WGF (MF) (Sub)}, \texttt{WGF (MF-IID)}, and \texttt{WGF (MF-IID-REP)} are consistently strong, attaining the highest or near-highest EIG amongst all methods considered, with a slightly reduced dispersion across seeds. Additional results indicate that this remains true for smaller or larger batch sizes, e.g., $m\in\{5,10,15\}$; see Appendix~\ref{sec:add-numerics-pharma}.

\begin{figure}[t!]
\vspace{-3mm}
\centering
    \includegraphics[width=.9\linewidth]{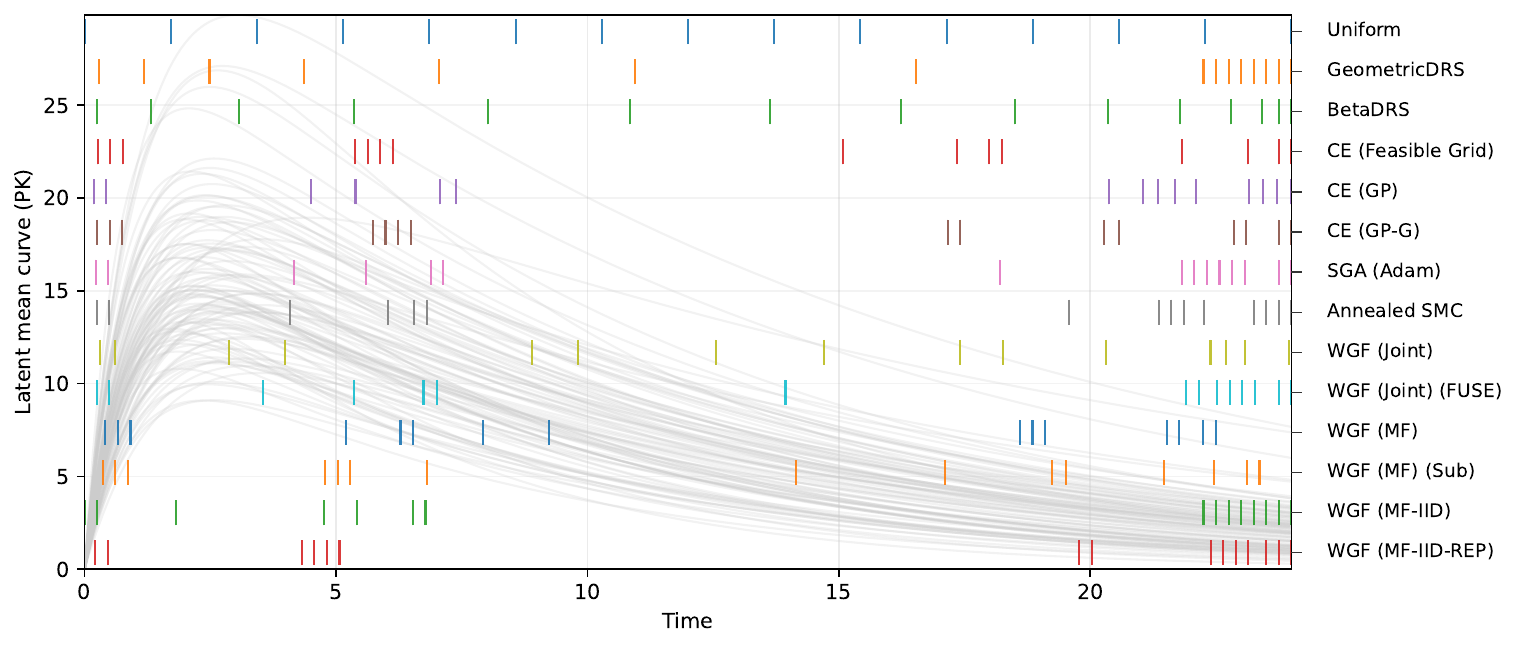}
    \vspace{-4mm}
    \caption{\textbf{Comparison of designs for the pharmacokinetic (PK) sampling-time benchmark.} The grey curves show $100$ latent PK mean trajectories  under prior draws $\theta\sim\pi$. The coloured tick marks indicate the $m$ selected sampling times for each method (one row per method).}
    \label{fig:11}
\vspace{-3mm}
\end{figure}

\subsection{FitzHugh--Nagumo Benchmark}
\label{sec:exp:fhn}

\paragraph{Experimental Details.}
Our final experiment evaluates the proposed distributional optimisation methods on a second established BOED benchmark: sampling-time design for the FitzHugh--Nagumo (FHN) model, which describes the electrical activity of a spiking neuron \citep[e.g.,][]{overstall2020bayesian}. 
Once more, this is a batch design problem in which $\xi_{1:m}=(t_1,\dots,t_m)$ are observation times over a fixed horizon $[0,T_{\max}]$ to be chosen in order to maximise $\mathrm{EIG}_m(\xi_{1:m})$. In this case, the latent state $u(t)=(u_1(t),u_2(t))$ solves the non-linear initial value problem
\begin{equation}
\label{eq:fhn_ode}
\begin{aligned}
\dot u_1(t) = \theta_3\left(u_1(t)-\frac{u_1(t)^3}{3}+u_2(t)\right), \qquad \dot u_2(t) = -\frac{u_1(t)-\theta_1+\theta_2 u_2(t)}{\theta_3},
\qquad u(0)=(-1,1)^\top,
\end{aligned}
\end{equation}
where $u_1(t)$ is the membrane potential (or voltage), $u_2(t)$ is the recovery variable, and $\theta=(\theta_1,\theta_2,\theta_3)$ are unknown model parameters. We assume noisy voltage observations at the chosen times:
\begin{equation}
\label{eq:fhn_obs}
y(t_j)\mid \theta,\sigma,\xi_{1:m}\ \sim\ \mathcal{N}\!\big(u_1(t_j;\theta), \sigma^2\big),
\qquad j=1,\dots,m,
\end{equation}
conditionally independent given $(\theta,\xi_{1:m})$, with $\sigma\sim\mathrm{Unif}[0.5,1.0]$. Following the specification in \cite{overstall2020bayesian}, we assign independent priors $\theta_1,\theta_2\sim\mathrm{Unif}[0,1]$ and $\theta_3\sim\mathrm{Unif}[1,5]$.  We take $m=21$, $T_{\max}=20$, and enforce ordered sampling times with a minimum spacing of $\Delta_{\min}=0.25$, implemented by clipping to $[0,T_{\max}]$, sorting, and a deterministic minimum-gap repair.

We approximate $\mathrm{EIG}_m(\xi_{1:m})$ using the same fixed-sample nested Monte Carlo estimator, low- and high-fidelity budgets, shortlist-and-refine protocol, and matched wall-clock tuning strategy as in the PK benchmark above. Unlike the PK model, the likelihood mean $u_1(t_j;\theta)$ is now defined implicitly via \eqref{eq:fhn_ode}. To make repeated optimisation steps efficient, we precompute forward trajectories for all parameter draws used by the estimator on a dense time grid using an RK4 solver and evaluate $u_1(t)$ at candidate times via linear interpolation; gradients with respect to the sampling times are obtained by differentiating the interpolant. Full experimental details are provided in Appendix~\ref{sec:additional-experimental-details-pk}.

\paragraph{Results.}
In Figure~\ref{fig:12}, we summarise the EIG achieved by each method, with the corresponding sampling-time designs shown in Figure~\ref{fig:13}. Several methods achieve broadly competitive performance, but clear differences emerge both in terms of the attained utility and the robustness across optimisation seeds. Similar to before, the \texttt{Uniform} baseline is consistently suboptimal, while \texttt{GeometricDRS} is competitive only intermittently and exhibits occasional clear failures, reflecting the limitations of a heavily parameterised schedule family. Among the optimisation-based baselines, the CE variants are once again strong: \texttt{CE (GP)} and \texttt{CE (GP-G)} in particular attain designs only marginally worse than the best-performing approaches. The most consistently performant methods are the mean-field WGF variants: \texttt{WGF (MF)} and \texttt{WGF (MF) (Sub)} attain the highest median EIG and exhibit a relatively small seed-to-seed variability. \texttt{WGF (Joint)} dynamics are again less competitive although, similar to before, the adaptive step-size variant \texttt{WGF (Joint) (FUSE)} substantially closes the gap, as does a more careful choice of constant step size; see Appendix~\ref{sec:add-numerics-fitzhugh}. The i.i.d.\ approximations \texttt{WGF (MF-IID)} and \texttt{WGF (MF-IID-REP)} now underperform the mean-field dynamics, and adding repulsion does not yield a significant improvement. As in the previous benchmark, these observations are largely consistent across different batch sizes, e.g., $m\in\{10,15,21\}$; see Appendix \ref{sec:add-numerics-fitzhugh}.

\begin{figure}[t!]
\vspace{-4mm}
    \centering
    \includegraphics[width=.9\linewidth]{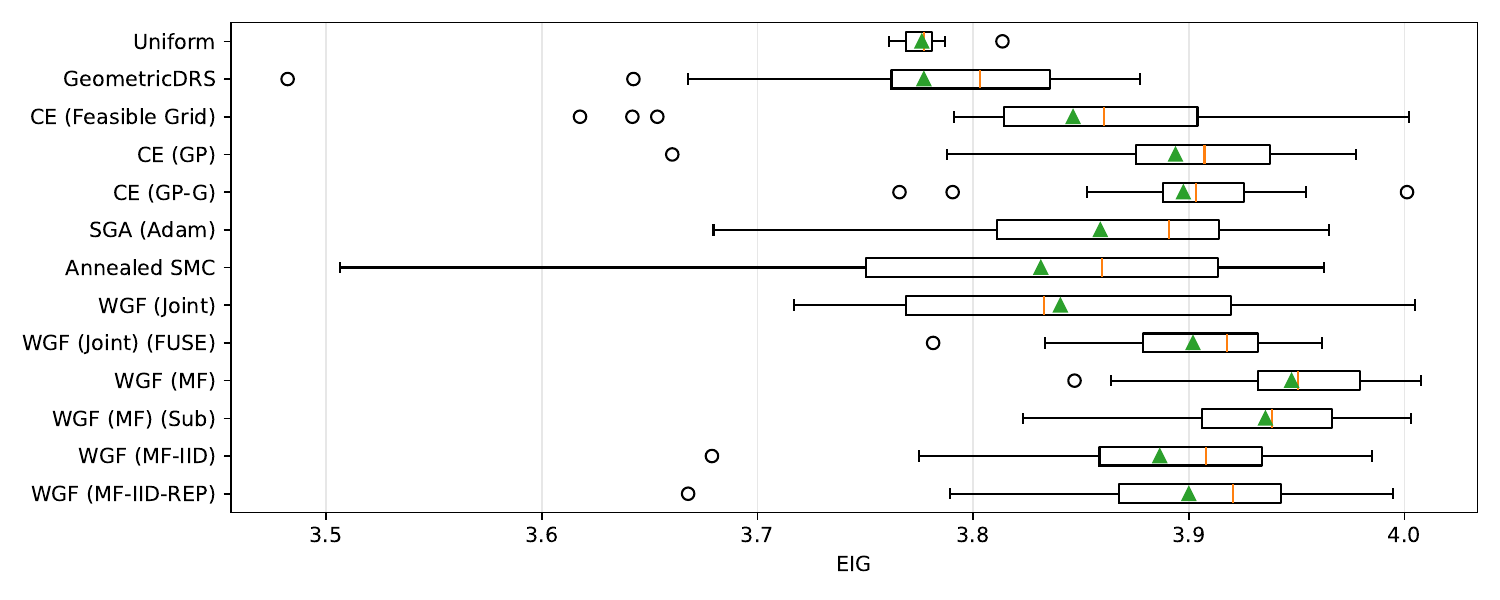}
    \vspace{-4mm}
    \caption{\textbf{EIG summaries for the FitzHugh--Nagumo (FHN) sampling-time benchmark.} Each boxplot summarises the distribution of independently evaluated $\mathrm{EIG}_m$ values obtained from designs returned over multiple optimisation seeds (higher is better). Boxes show the interquartile range with median (orange line); whiskers extend to $1.5\times$ IQR and circles denote outliers; green triangles indicate the mean.}
    \label{fig:12}
\vspace{-4mm}
\end{figure}

\begin{figure}[b!]
\vspace{-3mm}
    \centering
    \includegraphics[width=.9\linewidth]{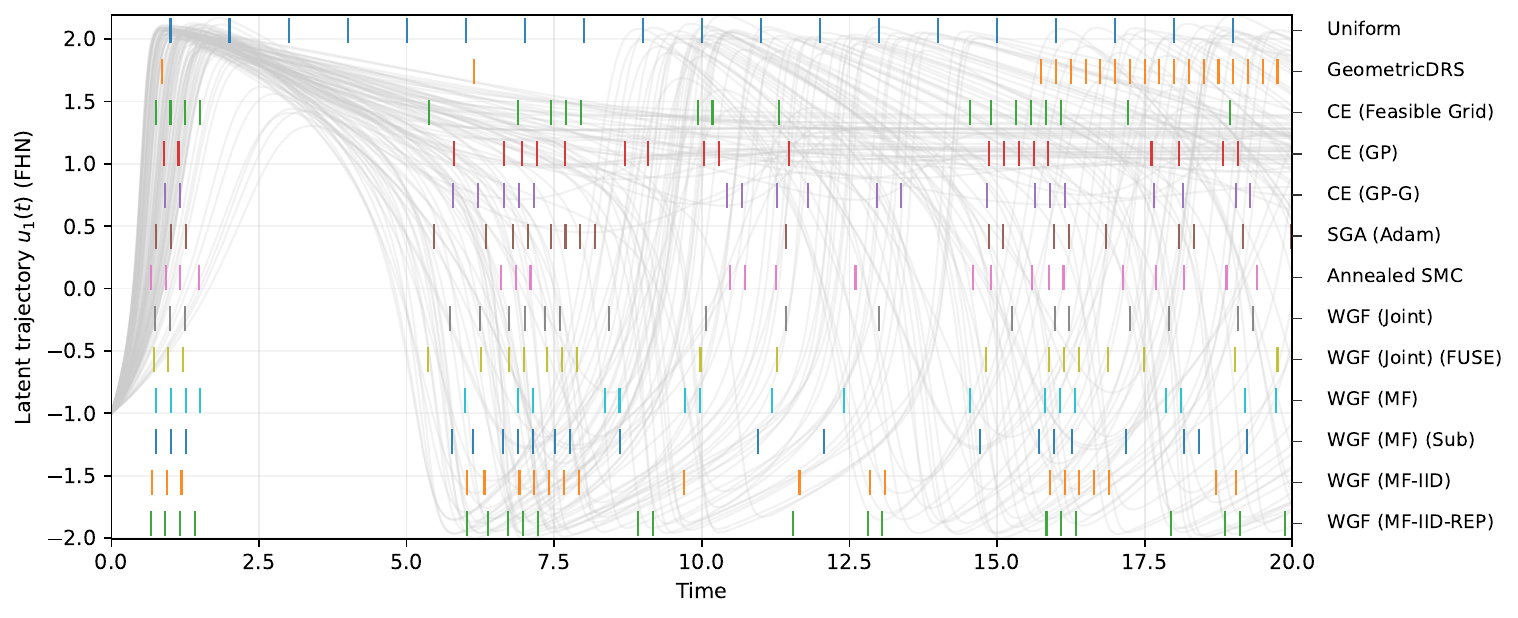}
    \vspace{-4mm}
    \caption{\textbf{Comparison of designs for the FitzHugh--Nagumo (FHN) sampling-time benchmark.} The grey curves show $100$ latent FHN mean trajectories under prior draws $\theta\sim\pi$. The coloured tick marks indicate the $m$ selected sampling times for each method (one row per method).}
    \label{fig:13}
    \vspace{-4mm}
\end{figure}

Qualitatively, Figure~\ref{fig:13} shows that high-performing designs concentrate observations into a small number of informative time windows where prior trajectories exhibit strong curvature or separation, rather than spreading samples uniformly over $[0,T_{\max}]$. In addition, none of the optimal designs places an observation between $t=2$ and $t=5$, as previously observed in \cite{overstall2020bayesian}. In comparison with these designs, \texttt{Uniform} wastes budget in less informative regions.

\section{Conclusions}
\label{sec:conclusions}

In this paper, we considered a distributional reformulation of EIG-based BOED in which pointwise optimisation of a design vector was replaced by optimisation over design measures. For the full joint batch problem on $\Xi^m$, an entropic regularisation yielded a strictly convex free-energy objective with an explicit Gibbs minimiser. For scalability, we then introduced two tractable restrictions of the batch design law, a mean-field product family and an i.i.d.\ product family, and derived the associated WGFs. These flows induced non-linear dynamics, which we approximated using interacting-particle algorithms, including doubly stochastic variants compatible with nested Monte Carlo gradient estimators. Empirically, the proposed methods consistently mitigated common pathologies of pointwise stochastic optimisation in non-convex utility landscapes, including strong basin dependence and mode collapse, and produced high-utility batches in both synthetic examples and established BOED benchmarks. In particular, for the pharmacokinetic and FitzHugh--Nagumo sampling-time problems, our particle-based methods were competitive with existing methods under matched computational budgets.

Several directions merit further study. On the theoretical side, it would be valuable to rigorously characterise the properties of the mean-field algorithm, extending our existing analysis in the i.i.d.\ setting. It would also be of interest to extend the analysis to constrained domains, and to obtain non-asymptotic guarantees for biased inner gradient estimators. On the methodological side, adaptive choices of the temperature parameter, more principled deterministic extraction rules from learned design laws, and sequential or non-myopic extensions are also promising. One could also consider gradient flows under different geometries than $\mathsf W_2$; for example, the Stein geometry would lead to alternative algorithms based on Stein variational gradient descent \citep{liu2016stein,duncan2023geometry} and its nonlinear extension \citep{wang2019nonlinear,chazal2025computable}. Another direction is to enrich the structured approximations beyond i.i.d.\ or product families, which may better capture joint batch dependencies while retaining tractability. Finally, it would be interesting to combine our approach with variational estimators of the EIG and its gradient \citep[e.g.,][]{foster2019variational}. Our framework is modular with respect to this inner approximation, and one could in principle replace the nested Monte Carlo gradient oracle in the doubly stochastic IPS by the gradient of a differentiable variational bound, potentially reducing variance and enabling higher-dimensional applications.

\section*{Acknowledgements}
The author is grateful to Prof. Christopher Nemeth for feedback on an early draft of this manuscript.

\bibliography{references}

\appendix 
\section{Theory: Assumptions, Main Results, and Proofs}
\label{sec:theory}

\subsection{General Notation}
We work on $\Xi=\mathbb{R}^d$. Let $\mathcal P_2(\Xi)$ denote the set of probability measures on $\Xi$ with finite second moment: $\{\mu\in\mathcal{P}(\Xi):M_2(\mu):=\int_{\Xi} \|\xi\|^2\mu(\mathrm{d}\xi)<\infty\}$. In addition, let $\mathcal{P}_{2,\mathrm{ac}}(\Xi)$ denote the subset of $\mathcal{P}_{2}(\Xi)$ consisting of probability measures which are absolutely continuous with respect to the Lebesgue measure, $\mathcal{L}^d$. For any $\mu\in\mathcal{P}_{2}(\Xi)$, let $L^{2}(\mu):=L^2(\mu;\Xi)$ denote the set of measurable functions $f:\Xi\rightarrow\Xi$ such that $\int_{\Xi}\|f(\xi)\|^2\mu(\mathrm{d}\xi)<\infty$. We will write $\|\cdot\|_{L^2(\mu)}$ and $\langle\cdot,\cdot\rangle_{L^2(\mu)}$ to denote, respectively, the norm and the inner product of this space.

Given a probability measure $\mu\in\mathcal{P}_2(\Xi)$ and a measurable function $T:\Xi\rightarrow\Xi$, we write $T_{\#}\mu$ for the pushforward measure of $\mu$ under $T$, that is, the measure such that $T_{\#}\mu(B)=\mu(T^{-1}(B))$ for all Borel measurable $B\in\mathcal{B}(\Xi)$. For every $\mu,\nu\in\mathcal{P}_2(\Xi)$, let $\Gamma(\mu,\nu)$ be the set of couplings (or transport plans) between $\mu$ and $\nu$, defined as $\Gamma(\mu,\nu) = \{\gamma\in\mathcal{P}(\Xi\times\Xi): Q^{1}_{\#}\gamma=\mu, Q^{2}_{\#}\gamma=\nu\}$, where $Q^{1}$ and $Q^{2}$ denote the projections onto the first and second components of $\Xi\times\Xi$. The Wasserstein $2$-distance between $\mu$ and $\nu$ is then defined according to
\begin{equation}
\mathsf{W}_2^2(\mu,\nu) = \inf_{\gamma \in\Gamma(\mu,\nu)}  \int_{\Xi\times\Xi}\|\xi-\eta\|^2 \gamma(\mathrm{d}\xi,\mathrm{d}\eta). \label{eq:wasserstein}
\end{equation}

\subsection{The Joint Batch Objective}
\label{sec:theory:joint}

Let $m\in\mathbb{N}$ be fixed. Let $\rho_m\in\mathcal P_{2,\mathrm{ac}}(\Xi^m)$ be a reference measure on the product space $\Xi^m$ with strictly positive Lebesgue density $\smash{\rho_m(\xi_{1:m}) = Z_{\rho_m}^{-1} e^{-V_m(\xi_{1:m})}}$ for some confining potential $V_m:\Xi^m\to\mathbb{R}$.\footnote{In a slight abuse of notation, we use $\rho_m$ to denote both the measure and its density w.r.t. the Lebesgue measure.}
The KL divergence on $\Xi^m$ is given by
\begin{equation}
\mathrm{KL}(\nu_m\|\rho_m):=
\left\{
\begin{array}{lll}
\int_{\Xi^m}\log\Big(\frac{\mathrm d\nu_m}{\mathrm d\rho_m}\Big) \nu_m(\mathrm d\xi_{1:m})
& , &
\nu_m\ll\rho_m,\\
+\infty & , & \text{otherwise.}
\end{array}
\right.
\end{equation}
Let $G:\Xi^m\to\mathbb R$ denote a permutation-invariant batch utility, e.g.\ $G(\xi_{1:m})=\mathrm{EIG}_m(\xi_{1:m})$. For any batch design law $\nu_m\in\mathcal P_2(\Xi^m)$ we define the
\emph{joint} expected batch utility
\begin{equation}
\mathcal J_m^{\mathrm{joint}}(\nu_m)
:=
\int_{\Xi^m} G(\xi_{1:m}) \nu_m(\mathrm d\xi_{1:m}).
\label{eq:Jm-joint-def}
\end{equation}
We then define, for $\lambda_m>0$, the entropy-regularised joint batch objective (or joint free energy) on $\mathcal P_2(\Xi^m)$ by
\begin{equation}
\mathcal F_m^{\lambda,\mathrm{joint}}(\nu_m)
:=
-\mathcal J_m^{\mathrm{joint}}(\nu_m)+\lambda_m \mathrm{KL}(\nu_m\|\rho_m).
\label{eq:F-joint-def}
\end{equation}

\subsubsection{Basic Results}
\label{sec:joint-batch-basic-results}

\begin{lemma}[Value-preserving lifting on $\mathcal P(\Xi^m)$]\label{lem:lift-value-preserving}
Let $\Xi\subseteq\mathbb R^d$ be a Borel set and let $G:\Xi^m\to\mathbb R$ be measurable and bounded above, with $\sup_{\xi_{1:m}\in\Xi^m}G(\xi_{1:m})<\infty$. Then
\begin{equation}
\sup_{\nu_m\in\mathcal P(\Xi^m)} \mathcal{J}_m^{\mathrm{joint}}(\nu_m)
=
\sup_{\xi_{1:m}\in\Xi^m}G(\xi_{1:m}).
\end{equation}
If, in addition, $G$ attains its maximum on $\Xi^m$, then for any $\nu_m$ supported on $\argmax_{\Xi^m}G$, one has $\int G \mathrm d\nu_m=\sup_{\xi_{1:m}\in\Xi^m}G(\xi_{1:m})$, hence $\nu_m$ is optimal. Conversely, if $\nu_m$ is optimal, then $\nu_m(\argmax_{\Xi^m}G)=1$.
\end{lemma}

\begin{proof}
Let $G^{\star}:=\sup_{\xi_{1:m}\in\Xi^m}G(\xi_{1:m})$. For any $\nu_m\in\mathcal P(\Xi^m)$ we have $G(\xi_{1:m})\le G^\star$ for all $\xi_{1:m}$. It follows immediately that
\begin{equation}
\int_{\Xi^m}G(\xi_{1:m}) \nu_m(\mathrm d\xi_{1:m})\le \int_{\Xi^m}G^\star \nu_m(\mathrm d\xi_{1:m})=G^\star.
\end{equation}
Thus, $\sup_{\nu_m}\int G \mathrm d\nu_m\le G^\star$. For the converse, let $\smash{(\xi^{(n)}_{1:m})_{n\ge1}\subset\Xi^m}$ be a sequence such that $\smash{G(\xi^{(n)}_{1:m})\uparrow G^\star}$ as $n\to\infty$. Then $\smash{\nu_m^{(n)}:=\delta_{\xi^{(n)}_{1:m}}}$ satisfies
$\smash{\int G \mathrm d\nu_m^{(n)}=G(\xi^{(n)}_{1:m})\to G^\star}$. Thus, 
$\smash{\sup_{\nu_m}\int G \mathrm d\nu_m\ge G^\star}$. This proves equality.

Suppose now that $G$ attains its maximum, so $\argmax_{\Xi^m}G\neq\emptyset$ and $G=G^\star$ on $\argmax_{\Xi^m}G$. If $\nu_m(\argmax_{\Xi^m}G)=1$, then $\int G \mathrm d\nu_m=G^\star$, so $\nu_m$ is a maximiser. Conversely, if $\nu_m$ is a maximiser, then $\int (G^\star-G) \mathrm d\nu_m=0$ with $G^\star-G\ge0$ pointwise.
Therefore $G^\star-G=0$ $\nu_m$-a.s., i.e.\ $\nu_m(\argmax_{\Xi^m}G)=1$.
\end{proof}

\begin{proposition}[Joint WGF and trapping in basins of attraction]
\label{prop:joint-wgf-basins}
Let $\Xi=\mathbb R^d$. Let $G:\Xi^m\to\mathbb R$ be $C^2$ with globally Lipschitz gradient. Let $(\Psi_t)_{t\ge 0}$ denote the flow of the ODE
\begin{equation}
\dot \xi_t = \nabla G(\xi_t), \qquad \xi_0=\xi\in \Xi^m.
\label{eq:grad-ascent-ode}
\end{equation}
In addition,  define $\nu_t:=(\Psi_t)_\#\nu_0$ for $\nu_0\in\mathcal P_2(\Xi^m)$. Then $(\nu_t)_{t\ge 0}$ is the Wasserstein gradient flow (WGF) of $-\mathcal{J}_m^{\mathrm{joint}}$, and the weak solution of the continuity equation
\begin{equation}
\partial_t \nu_t + \nabla\cdot(\nu_t \nabla G)=0.
\label{eq:joint-unreg-cont}
\end{equation}
Suppose, in addition, that the ODE in \eqref{eq:grad-ascent-ode} admits finitely many asymptotically stable equilibria
$\xi^{(1)},\dots,\xi^{(K)}\in\Xi^m$, with corresponding basins of attraction
\begin{equation}
\mathcal B_k:=\{\xi\in\Xi^m:\ \Psi_t(\xi)\to \xi^{(k)}\ \text{as }t\to\infty\},\qquad k=1,\dots,K,
\end{equation}
Suppose also that $\nu_0(\Xi^m\setminus\cup_{k=1}^K\mathcal B_k)=0$. Then $\nu_t$ converges weakly as $t\to\infty$ to the following mixture of Dirac measures
\begin{equation}
\nu_t \ \Rightarrow\ \sum_{k=1}^K \nu_0(\mathcal B_k) \delta_{\xi^{(k)}}.
\label{eq:basin-mixture}
\end{equation}
\end{proposition}

\begin{proof}
Since $G\in C^2$ has a globally Lipschitz gradient, the ODE in \eqref{eq:grad-ascent-ode} admits a unique flow $(\Psi_t)_{t\ge 0}$ on $\Xi^m$ \cite[e.g.,][]{hirsch2013differential}. The fact that $(\nu_t)_{t\geq 0}$ is the weak solution of \eqref{eq:joint-unreg-cont} is well known \cite[e.g.,][Chapter~8]{ambrosio2008gradientflows}. Explicitly, fix $\varphi\in C_c^\infty(\Xi^m)$. Then, differentiating in time, using $\frac{\mathrm d}{\mathrm dt}\Psi_t(\xi)=\nabla G(\Psi_t(\xi))$, and the change-of-variables formula, we have
\begin{equation}
\frac{\mathrm d}{\mathrm dt}\int \varphi \mathrm d\nu_t
=
\int_{\Xi^m}\langle \nabla\varphi(\Psi_t(\xi)), \nabla G(\Psi_t(\xi))\rangle \nu_0(\mathrm d\xi)
=
\int_{\Xi^m}\langle \nabla\varphi(\xi), \nabla G(\xi)\rangle \nu_t(\mathrm d\xi),
\end{equation}
which is exactly the weak formulation of \eqref{eq:joint-unreg-cont}. The identification of $(\nu_t)_{t\geq 0}$ with the WGF of $-\mathcal{J}_m^{\mathrm{joint}}$ is also standard \citep[e.g.,][]{ambrosio2008gradientflows,santambrogio2015optimal}.

To prove the second part of the proposition, let $\psi$ be any bounded continuous test function. By the definition of the push-forward, we have that
\begin{equation}
\int_{\Xi^m}\psi(\xi) \nu_t(\mathrm d\xi)
=
\int_{\Xi^m}\psi(\Psi_t(\xi)) \nu_0(\mathrm d\xi).
\end{equation}
By assumption, for $\nu_0$-a.e.\ $\xi$ there exists $k$ such that $\xi\in\mathcal B_k$ and hence $\Psi_t(\xi)\to \xi^{(k)}$. Therefore $\psi(\Psi_t(\xi))\to \sum_{k=1}^K \psi(\xi^{(k)})\mathbf 1_{\mathcal B_k}(\xi)$ pointwise $\nu_0$-a.s. Finally, since $\psi$ is bounded, dominated convergence gives
\begin{equation}
\lim_{t\to\infty}\int \psi \mathrm d\nu_t
=
\sum_{k=1}^K \psi(\xi^{(k)}) \nu_0(\mathcal B_k)
=
\int_{\Xi^m}\psi(\xi) \Big(\sum_{k=1}^K \nu_0(\mathcal B_k)\delta_{\xi^{(k)}}\Big)(\mathrm d\xi),
\end{equation}
This establishes the weak convergence in \eqref{eq:basin-mixture}.
\end{proof}

\begin{proposition}[Joint entropic regularisation: strict convexity and Gibbs minimiser]
\label{prop:joint-gibbs}
Define
\begin{equation}
    \nu_m^{\lambda,\star}(\mathrm d\xi_{1:m}):=\frac{1}{Z_m^\lambda}\exp \Big(\frac{1}{\lambda_m}G(\xi_{1:m})\Big) \rho_m(\mathrm d\xi_{1:m}), \qquad Z_m^\lambda:=\int_{\Xi^m}\exp(\frac{1}{\lambda_m}G(\xi_{1:m})) \rho_m(\mathrm d\xi_{1:m}).
\end{equation}
Suppose that the normalisation constant $\smash{Z_{m}^{\lambda}<\infty}$.\footnote{There are various sufficient conditions under which the normalising constant is finite. For example: Assumption~\ref{ass:G}(i) holds, and the analogue of Assumption~\ref{ass:V}(iii) holds for $\rho_m\in\mathcal{P}_{2,\mathrm{ac}}(\Xi^m)$ (see Appendix~\ref{sec:assumptions}).} In addition, suppose that  $\nu_m^{\lambda,\star}\in\mathcal P_2(\Xi^m)$ and that $\int_{\Xi^m}|G(\xi_{1:m})| \nu_m^{\lambda,\star}(\mathrm d\xi_{1:m})<\infty$. Then:
\begin{enumerate}[leftmargin=*]
\item[(i)] $\mathcal F_m^{\lambda,\mathrm{joint}}$ is proper and strictly convex on its finite-value domain
\begin{equation}
\operatorname{Dom}\big(\mathcal F_m^{\lambda,\mathrm{joint}}\big)
:=
\big\{\nu_m\in\mathcal P_2(\Xi^m):\ \mathcal F_m^{\lambda,\mathrm{joint}}(\nu_m)<\infty\big\}.
\end{equation}
\item[(ii)] $\mathcal F_m^{\lambda,\mathrm{joint}}$ has a unique minimiser on $\mathcal P_2(\Xi^m)$ given by $\nu_m^{\lambda,\star}$.
\end{enumerate}
\end{proposition}

\begin{proof}
For every $\nu_m\ll\rho_m$, $\mathrm{KL}(\nu_m\|\nu_m^{\lambda,\star})
=\mathrm{KL}(\nu_m\|\rho_m)-\frac{1}{\lambda_m}\int_{\Xi^m}G d\nu_m+\log Z_m^\lambda$, since $\log\frac{d\nu_m^{\lambda,\star}}{d\rho_m}=\lambda_m^{-1}G-\log Z_m^\lambda$. This implies that
\begin{equation}
\mathcal F_m^{\lambda,\mathrm{joint}}(\nu_m)
=\lambda_m\mathrm{KL}(\nu_m\|\nu_m^{\lambda,\star})-\lambda_m\log Z_m^\lambda.
\end{equation}
If $\nu_m\not\ll\rho_m$, then $\mathcal F_m^{\lambda,\mathrm{joint}}(\nu_m)=+\infty$ by definition. Since $\nu_m^{\lambda,\star}\in\mathcal P_2(\Xi^m)$ and $\int |G| d\nu_m^{\lambda,\star}<\infty$, the right-hand side is finite at $\nu_m^{\lambda,\star}$, so the functional is proper. The display also shows that $\nu_m^{\lambda,\star}$ is a minimiser and that the minimum value is $-\lambda_m\log Z_m^\lambda$. In addition, the map $\nu_m\mapsto -\int G d\nu_m$ is affine, while $\nu_m\mapsto \mathrm{KL}(\nu_m\|\rho_m)$ is strictly convex on $\{\nu_m:\nu_m\ll\rho_m\}$. Therefore $\mathcal F_m^{\lambda,\mathrm{joint}}$ is strictly convex on its finite-value domain, and the minimiser is unique.
\end{proof}

\subsubsection{Zero-temperature limits}
\label{sec:zero-temp-joint}

\begin{theorem}[Zero-temperature limit of the joint Gibbs design law]
\label{thm:joint-zero-temp}
Assume that the hypotheses of Proposition~\ref{prop:joint-gibbs} hold for all sufficiently small $\lambda_m>0$. Suppose, in addition, that $G$ is continuous and attains its maximum, and write $G_m^\star := \max_{\xi_{1:m}\in\Xi^m} G(\xi_{1:m})$ and $\mathcal M_m := \argmax_{\xi_{1:m}\in\Xi^m} G(\xi_{1:m})$. Then, as $\lambda_m\downarrow 0$, 
\begin{equation}
\mathcal F_m^{\lambda,\mathrm{joint}}(\nu_m^{\lambda,\star})
\longrightarrow
-G_m^\star, \qquad 
\int_{\Xi^m} G(\xi_{1:m}) \nu_m^{\lambda,\star}(\mathrm d\xi_{1:m})
\longrightarrow
G_m^\star.
\end{equation}
Suppose $U\subseteq \Xi^m$ is open, and that $\mathcal M_m\subset U$. In addition, suppose that $G_m^\star-\sup_{\xi_{1:m}\in U^c} G(\xi_{1:m})>0$. Then there exists a constant $c_U<\infty$, independent of $\lambda_m$, such that
\begin{equation}
\nu_m^{\lambda,\star}(U^c)
\le
c_U \exp \Big(-\frac{G_m^\star-\sup_{\xi_{1:m}\in U^c} G(\xi_{1:m})}{2\lambda_m}\Big)
\end{equation}
for all sufficiently small $\lambda_m>0$. In particular, if $\mathcal M_m=\{\xi_{1:m}^\star\}$ and for every $r>0$, $G_m^\star-\sup_{\|\xi_{1:m}-\xi_{1:m}^\star\|\ge r} G(\xi_{1:m})>0$, then
\begin{equation}
\nu_m^{\lambda,\star}\Rightarrow \delta_{\xi_{1:m}^\star}
\qquad\text{as }\lambda_m\downarrow0.
\end{equation}
\end{theorem}

\begin{proof}
The result is standard; see \citet{hwang1980laplace} and
\citet[Chapter~4.3]{dembo1998large}. 
\end{proof}

\subsection{The Mean-Field Batch Objective}
\label{sec:mean-field}
Let $m\in\mathbb N$ be fixed. We now consider the {mean-field} (product-measure) variational family $\mathcal P_{\mathrm{mf}}(\Xi^m)
:=
\{
\mu_1\otimes\cdots\otimes\mu_m:\ \mu_b\in\mathcal P_2(\Xi)  ,b=1,\dots,m
\}$. For $(\mu_1,\dots,\mu_m)\in \mathcal P_2(\Xi)^m$, write $\nu_m:=\mu_1\otimes\cdots\otimes\mu_m\in\mathcal P_{\mathrm{mf}}(\Xi^m)$. We then define
\begin{align}
\mathcal J_m^{\mathrm{mf}}(\mu_1,\dots,\mu_m)
&:=
\int_{\Xi^m} G(\xi_{1:m}) (\otimes_{b=1}^m\mu_b)(\mathrm d\xi_{1:m})
\label{eq:Jm-mf-def} \\
\Phi_b(\xi;\mu_{-b})
&:=
\int_{\Xi^{m-1}}
G(\xi_1,\dots,\xi_{b-1},\xi,\xi_{b+1},\dots,\xi_m) 
\mu_{-b}(\mathrm d\xi_{-b}),
\label{eq:Phi-mf-b-def}
\end{align}
where we use $\mu_{-b}:=\otimes_{j\neq b}\mu_j$ to denote the product of the remaining marginals, and $\xi_{-b}$ to denote the tuple $(\xi_1,\dots,\xi_{b-1},\xi_{b+1},\dots,\xi_m)$. We then define, for $\lambda_m>0$, the entropy-regularised {mean-field} batch objective by
\begin{equation}
\mathcal F_m^{\lambda,\mathrm{mf}}(\mu_1,\dots,\mu_m)
=
-\mathcal J_m^{\mathrm{mf}}(\mu_1,\dots,\mu_m)
+\lambda_m\sum_{b=1}^m \mathrm{KL}(\mu_b\|\rho).
\label{eq:F-mf-def-app}
\end{equation}
This is precisely the restriction of the joint batch objective \eqref{eq:F-joint-def} to $\mathcal P_{\mathrm{mf}}(\Xi^m)$, under the assumption that the reference measure factorises as $\rho_m=\rho^{\otimes m}$ for some $\rho\in\mathcal P_{2,\mathrm{ac}}(\Xi)$, with strictly positive Lebesgue density $\rho(\xi)=Z_\rho^{-1}e^{-V(\xi)}$.

\subsubsection{Basic Results}
\label{sec:mean-field-basic-results}

\begin{proposition}[Mean-field restriction: coordinate-wise Gibbs fixed points]
\label{prop:mf-coordinate-gibbs}
Let $(\mu_1^{\lambda,\star},\dots,\mu_m^{\lambda,\star})$ be a minimiser of $\mathcal F_m^{\lambda,\mathrm{mf}}$ over $\mathcal P_2(\Xi)^m$. For each $b=1,\dots,m$, define
\begin{equation}
\widetilde\mu_b^\lambda(\mathrm d\xi)
:=
\frac{1}{Z_b^\lambda}\exp \Big(\frac{1}{\lambda_m}\Phi_b(\xi;\mu_{-b}^{\lambda,\star})\Big) \rho(\mathrm d\xi), \qquad Z_b^\lambda
:=
\int_{\Xi}\exp \Big(\frac{1}{\lambda_m}\Phi_b(\xi;\mu_{-b}^{\lambda,\star})\Big) \rho(\mathrm d\xi).
\end{equation}
Suppose that, for each $b=1,\dots,m$, the normalising constant $Z_{b}^{\lambda}<\infty$.  Assume moreover that $\widetilde\mu_b^\lambda\in\mathcal P_2(\Xi)$ and
$\smash{\int_{\Xi}\big|\Phi_b(\xi;\mu_{-b}^{\lambda,\star})\big| \widetilde\mu_b^\lambda(\mathrm d\xi)<\infty}$.  Then each marginal satisfies the self-consistency equation
\begin{equation}
\mu_b^{\lambda,\star}(\mathrm d\xi)
=
\frac{1}{Z_b^\lambda}\exp \Big(\frac{1}{\lambda_m}\Phi_b(\xi;\mu_{-b}^{\lambda,\star})\Big) \rho(\mathrm d\xi),
\qquad b=1,\dots,m.
\label{eq:mf-fixedpoint-b-app}
\end{equation}
\end{proposition}

\begin{proof}
Fix $b\in\{1,\dots,m\}$ and freeze the other coordinates at $\mu_{-b}^{\lambda,\star}$. Suppose that we define 
\begin{equation}
\mathcal F_b(\mu_b):=-\int_{\Xi}\Phi_b(\xi;\mu_{-b}^{\lambda,\star}) \mu_b(d\xi)
+\lambda_m\mathrm{KL}(\mu_b\|\rho).
\end{equation}
Then $\mathcal F_m^{\lambda,\mathrm{mf}}(\mu_1^{\lambda,\star},\dots,\mu_{b-1}^{\lambda,\star},\mu_b,\mu_{b+1}^{\lambda,\star},\dots,\mu_m^{\lambda,\star})=C_b+\mathcal F_b(\mu_b)$, where $C_b$ does not depend on $\mu_b$. It follows that $\mu_b^{\lambda,\star}$ minimizes $\mathcal F_b$ over $\mathcal P_2(\Xi)$. But $\mathcal F_b$ is exactly the one-coordinate analogue of Proposition~\ref{prop:joint-gibbs}, with $G$ replaced by $\Phi_b(\cdot;\mu_{-b}^{\lambda,\star})$. Under the stated integrability assumptions, Proposition~\ref{prop:joint-gibbs} applies and yields the unique minimiser $\widetilde\mu_b^\lambda$, which proves \eqref{eq:mf-fixedpoint-b-app}.
\end{proof}

\subsubsection{Zero-temperature limits}
\label{sec:zero-temp-mf}

\begin{proposition}[Zero-temperature limit of the mean-field relaxation]
\label{prop:mf-zero-temp}
Assume that $\rho_m=\rho^{\otimes m}$, where $\rho\in\mathcal P_{2,\mathrm{ac}}(\Xi)$ has a strictly positive Lebesgue density. Suppose, in addition, that $G$ is continuous and attains its maximum at some $\xi_{1:m}^\star=(\xi_1^\star,\dots,\xi_m^\star)\in\Xi^m$. Let
\begin{equation}
G_m^\star:=\max_{\xi_{1:m}\in\Xi^m}G(\xi_{1:m}), \qquad \Psi_m^{\lambda,\star}
:=
\inf_{(\mu_1,\dots,\mu_m)\in\mathcal P_2(\Xi)^m}
\mathcal F_m^{\lambda,\mathrm{mf}}(\mu_1,\dots,\mu_m).
\end{equation}
Then $\Psi_m^{\lambda,\star}\longrightarrow -G_m^\star$ as $\lambda_m\downarrow 0$. Consequently, if for each $\lambda_m>0$ there exists a minimiser
$(\mu_1^{\lambda,\star},\dots,\mu_m^{\lambda,\star})$ of $\mathcal F_m^{\lambda,\mathrm{mf}}$, then
\begin{equation}
\mathcal J_m^{\mathrm{mf}}(\mu_1^{\lambda,\star},\dots,\mu_m^{\lambda,\star})
\longrightarrow
G_m^\star
\qquad\text{as }\lambda_m\downarrow0.
\label{eq:jm-convergence-min}
\end{equation}
If, in addition, $\xi_{1:m}^\star$ is isolated in the sense that for every product neighborhood $U_1\times\cdots\times U_m\ni \xi_{1:m}^\star$, it holds that $G_m^\star-\sup_{\xi_{1:m}\notin U_1\times\cdots\times U_m}G(\xi_{1:m})>0$,  then for $b=1,\dots,m$, 
\begin{equation}
\mu_b^{\lambda,\star}\Rightarrow \delta_{\xi_b^\star}
\qquad \text{as $\lambda_m\downarrow0$.} \label{eq:mu-b-converge}
\end{equation}

\end{proposition}

\begin{proof}
Since $\mathcal J_m^{\mathrm{mf}}\le G_m^\star$, we have $\Psi_m^{\lambda,\star}\ge -G_m^\star$. For the matching upper bound, fix $\varepsilon>0$. By continuity of $G$ at $\xi_{1:m}^\star$, choose neighbourhoods $U_b\ni\xi_b^\star$ such that $G\ge G_m^\star-\varepsilon$ on $U_1\times\cdots\times U_m$. Let
\begin{equation}
\bar\mu_b^\varepsilon(d\xi):=\frac{\mathbf 1_{U_b}(\xi)}{\rho(U_b)} \rho(d\xi),
\qquad b=1,\dots,m.
\end{equation}
Then $\bar\mu_b^\varepsilon\in\mathcal P_2(\Xi)$,
$\mathrm{KL}(\bar\mu_b^\varepsilon\|\rho)=\log(1/\rho(U_b))$, and
\begin{equation}
\mathcal F_m^{\lambda,\mathrm{mf}}(\bar\mu_1^\varepsilon,\dots,\bar\mu_m^\varepsilon)
\le -(G_m^\star-\varepsilon)+\lambda_m\sum_{b=1}^m \log\frac{1}{\rho(U_b)}.
\end{equation}
Hence $\limsup_{\lambda_m\downarrow0}\Psi_m^{\lambda,\star}\le -G_m^\star+\varepsilon$, and since $\varepsilon>0$ is arbitrary,
$\Psi_m^{\lambda,\star}\to -G_m^\star$. If $(\mu_1^{\lambda,\star},\dots,\mu_m^{\lambda,\star})$ is a minimiser, then
\begin{equation}
\Psi_m^{\lambda,\star}
\ge -\mathcal J_m^{\mathrm{mf}}(\mu_1^{\lambda,\star},\dots,\mu_m^{\lambda,\star}),
\end{equation}
because the entropy term is nonnegative. Since the left-hand side tends to
$-G_m^\star$ and the utility is always at most $G_m^\star$, this yields \eqref{eq:jm-convergence-min}.

For the weak convergence, fix product neighbourhoods $U_b\ni\xi_b^\star$ and set $U:=U_1\times\cdots\times U_m$. By the isolation assumption, $\eta_U:=G_m^\star-\sup_{U^c}G>0$. Therefore
\begin{equation}
G_m^\star-\mathcal J_m^{\mathrm{mf}}(\mu_1^{\lambda,\star},\dots,\mu_m^{\lambda,\star})
\ge \eta_U\Bigl(1-\prod_{b=1}^m \mu_b^{\lambda,\star}(U_b)\Bigr).
\end{equation}
The left-hand side tends to $0$, so $\prod_{b=1}^m \mu_b^{\lambda,\star}(U_b)\to1$. Since each factor lies in $[0,1]$, necessarily $\mu_b^{\lambda,\star}(U_b)\to1$ for every $b$. Letting $U_b\downarrow\{\xi_b^\star\}$ and using the Portmanteau characterization of weak convergence gives \eqref{eq:mu-b-converge}.
\end{proof}

\subsection{The i.i.d. Batch Objective}
Let $\rho\in\mathcal P_{2,\mathrm{ac}}(\Xi)$ be a reference measure with strictly positive Lebesgue density $\rho(\xi)= Z_{\rho}^{-1} e^{-V(\xi)}$, for some confining potential $V:\Xi\to\mathbb{R}$, with $Z_{\rho} = \int e^{-V(\xi)}\mathrm{d}\xi<\infty$.\footnote{In a slight abuse of notation, we use $\rho$ to denote both the measure and its density w.r.t. the Lebesgue measure.} The KL divergence (or relative entropy) is given by
\begin{equation}
\mathrm{KL}(\mu\|\rho):=\left\{\begin{array}{lll} \int_{\Xi}\log\Big(\frac{\mathrm d\mu}{\mathrm d\rho}\Big) \mu(\mathrm d\xi) & , & \mu\ll \rho \\
+\infty & , & \text{otherwise} \end{array} \right.
\end{equation}
Let $m\in\mathbb{N}$ be a fixed batch size. In addition, let $G:(\Xi)^m\to\mathbb{R}$ denote any permutation-invariant utility (e.g., the EIG). We then define, for any $\mu\in\mathcal P_2(\Xi)$,
\begin{align}
\mathcal J_m(\mu)
&:=\int_{\Xi^m} G(\xi_{1:m}) \mu^{\otimes m}(\mathrm d\xi_{1:m}),
\label{eq:Jm-def-recall}\\
\Phi_m(\xi;\mu)
&:=\int_{\Xi^{m-1}} G(\xi,\xi_{2:m}) \mu^{\otimes(m-1)}(\mathrm d\xi_{2:m}),
\qquad \xi\in\Xi.
\label{eq:Phi-def-recall}
\end{align}
We thus have, in particular, that $\mathcal J_m(\mu)=\int_{\Xi}\Phi_m(\xi;\mu) \mu(\mathrm d\xi)$.  We also consider an explicit {repulsive} regularisation term. Let $r:\Xi\to\mathbb R$ be a measurable interaction potential.
Define, for $\mu\in\mathcal P_2(\Xi)$,
\begin{equation}
\mathcal R(\mu):=\frac12\int_{\Xi}\int_{\Xi} r(\xi-\chi) \mu(\mathrm d\xi) \mu(\mathrm d\chi),
\qquad
\Psi_r(\xi;\mu):=\int_{\Xi} r(\xi-\chi) \mu(\mathrm d\chi).
\label{eq:R-and-Psi-def-app}
\end{equation}
Finally, for $\lambda>0$ and $\eta\ge 0$, we define the repulsive entropy-regularised free energy
\begin{equation}
\mathcal F_m^{\lambda,\mathrm{rep}}(\mu)
:=
-\mathcal J_m(\mu)+\eta \mathcal R(\mu)+\lambda \mathrm{KL}(\mu\|\rho).
\label{eq:F-rep-def-app}
\end{equation}

\subsubsection{Assumptions}
\label{sec:assumptions}
We will impose the following standing assumptions 

\begin{assumption}[The reference potential]
\label{ass:V}
The potential $V\in C^2(\Xi)$ and there exist constants $\kappa>0$, $K_V<\infty$, $a_V>0$, $b_V\in\mathbb{R}$ such that the following conditions hold for all $\xi\in\Xi$:
\begin{enumerate}[leftmargin=*]
\item[(i)] {Uniform convexity:} $\nabla^2 V(\xi)\succeq \kappa I_d$.
\item[(ii)] {Global smoothness:} $\|\nabla^2 V(\xi)\|_{\mathrm{op}}\le K_V$.
\item[(iii)] {Quadratic confinement:} $V(\xi)\ge a_V\|\xi\|^2+b_V$.
\end{enumerate}
\end{assumption}

Assumption~\ref{ass:V} concerns the reference measure $\rho(\xi)\propto e^{-V(\xi)}$, which can be interpreted as a {design prior} that encodes feasibility and regularity (e.g., penalising extreme sensor locations). It requires that this measure is {strongly log-concave}, and ensures that it is both {normalisable}, and has finite second moment. The uniform convexity condition $\nabla^2V\succeq \kappa I_d$ yields a dissipative drift $\nabla\log\rho=-\nabla V$, which is a standard hypothesis guaranteeing well-posedness and stability of Langevin dynamics. The Hessian bound $\|\nabla^2V\|_{\mathrm{op}}\le K_V$ implies that $\nabla V$ is globally Lipschitz, which is convenient for propagation-of-chaos and discretisation error bounds for the IPS. Finally, the quadratic lower bound $V(\xi)\ge a_V\|\xi\|^2+b_V$ is an explicit tail condition ensuring exponential integrability. This assumption is satisfied, for example, by Gaussian design priors. We note that, under Assumption \ref{ass:V}(i), a quadratic lower bound $V(\xi)\ge \tilde a\|\xi\|^2-\tilde b$ always holds for some $\tilde a>0$, $\tilde b\in\mathbb R$. We state Assumption \ref{ass:V}(iii) separately to allow explicit control of the tail parameter $a_V$ used in exponential integrability estimates.

\begin{assumption}[The utility] 
\label{ass:G}
The utility $G\in C^2((\Xi)^m)$ is permutation-invariant and there exist constants $C_G\ge 0$, $L_G\ge 0$ and an exponent $\varepsilon_G\in(0,2]$ such that the following conditions hold for all $\xi_{1:m}\in(\Xi)^m$:
\begin{enumerate}[leftmargin=*]
\item[(i)] {Subquadratic growth:}\quad
$|G(\xi_{1:m})|\le C_G\big(1+\sum_{j=1}^m\|\xi_j\|^{2-\varepsilon_G}\big)$.
\item[(ii)] {Linear growth of first derivatives:}\quad
$\|\nabla_jG(\xi_{1:m})\|\le C_G\big(1+\sum_{\ell=1}^m\|\xi_\ell\|\big)$ for each $j=1,\dots,m$.
\item[(iii)] {Uniform Hessian bound:}\quad
$\|\nabla^2G(\xi_{1:m})\|_{\mathrm{op}}\le L_G$.
\end{enumerate}
\end{assumption}

Assumption~\ref{ass:G} relates to the (deterministic) batch utility $G(\xi_{1:m})=\mathrm{EIG}_m(\xi_{1:m})$. The subquadratic growth bound ensures integrability of $G$ under product measures $\mu^{\otimes m}$ with $\mu\in\mathcal P_2(\Xi)$, and guarantees finiteness of normalisation constants of the form $\int\exp \big(\frac{1}{\lambda_m}G(\xi_{1:m})\big)\rho_m(\xi_{1:m})\mathrm{d}\xi_{1:m}$ when combined with quadratic confinement of $\rho_m$. The linear growth condition on first derivatives justifies differentiation under the integral sign in the conditional utility $\Phi_m(\cdot;\mu)$, and ensures that the mean-field drift $m\nabla\Phi_m(\xi;\mu)$ is well-defined with finite second moments along the dynamics. Finally, the uniform Hessian bound $\|\nabla^2G\|_{\mathrm{op}}\le L_G$ yields global Lipschitz control of the interaction drift with respect to the state variable, which allows us to establish existence and uniqueness of solutions to the McKean--Vlasov SDE, as well as quantitative stability estimates. This assumption can be verified under standard smoothness and domination conditions on the likelihood $\xi\mapsto \pi_\xi(y\mid\theta)$. In particular, it holds for many smooth parametric models (e.g., linear--Gaussian, smooth additive-noise).

\begin{assumption}[The repulsion]\label{ass:r}
The interaction potential $r:\mathbb R^d\to\mathbb R$ satisfies:
\begin{enumerate}[leftmargin=*]
\item[(i)] {Evenness:} $r(z)=r(-z)$ for all $z\in\mathbb R^d$.
\item[(ii)] {Global smoothness:} $r\in C^2(\mathbb R^d)$ and there exists $L_r<\infty$ such that $\|\nabla^2 r(z)\|_{\mathrm{op}}\le L_r$ for all $z\in\mathbb R^d$.
\item[(iii)] {Quadratic growth:} there exists $C_r<\infty$ such that $|r(z)|\le C_r(1+\|z\|^2)$ for all $z$.
\item[(iv)] {Lower bound:} there exists $\underline r\in\mathbb R$ such that $r(z)\ge \underline r$ for all $z\in\mathbb R^d$.
\end{enumerate}
\end{assumption}

Assumption~\ref{ass:r} relates to the repulsion potential. The evenness condition is standard for symmetric repulsion energies and yields the simple first-variation formula $\delta\mathcal R/\delta\mu=\Psi_r$ (see Lemma~\ref{lem:deltaR}). Assumption~\ref{ass:r}(ii) implies $\nabla r$ is globally Lipschitz with constant $L_r$ and has linear growth. Assumption~\ref{ass:r}(iii) ensures $\mathcal R(\mu)$ is finite for all $\mu\in\mathcal P_2(\Xi)$. Finally, Assumption~\ref{ass:r}(iv) ensures that $\mathcal R$ is bounded below, which is used to obtain existence of minimisers for $\mathcal F_m^{\lambda,\mathrm{rep}}$ uniformly in $\lambda>0$. In many repulsive examples one has $r\ge0$, in which case $\underline r=0$.

\begin{assumption}[The regularisation]
\label{ass:strong}
For $\eta\ge 0$, the constant $\lambda>0$ is chosen such that
\begin{equation}
\alpha
:=
\lambda\kappa - mL_G - \eta L_r
>0,
\label{eq:alpha}
\end{equation}
where $\kappa$ is from Assumption~\ref{ass:V}(i), $L_G$ is from Assumption~\ref{ass:G}(iii), and $L_r$ is from Assumption~\ref{ass:r}(ii).
\end{assumption}

Assumption~\ref{ass:strong} relates to the regularisation parameter $\lambda>0$. It is a {strong-confinement} condition, which requires the entropic regularisation to dominate the curvature of the utility at the scale of the batch size. This hypothesis is key to establishing global contractivity and uniqueness results (e.g. uniqueness of stationary solutions, exponential convergence to equilibrium, and quantitative propagation-of-chaos bounds), since it rules out strong multimodality induced by the utility term. It is worth noting that, since this condition scales with the batch size $m$, it becomes more restrictive for large batch sizes, reflecting the fact that the interaction strength in $m$-batch utilities increases with $m$.

\subsubsection{Well-Posedness}
\begin{lemma}[Well-posedness and basic bounds]\label{lem:Jm-wellposed}
Suppose that Assumption~\ref{ass:G} holds. Then for every $\mu\in\mathcal P_2(\Xi)$, $\mathcal J_m(\mu)$ is finite. Moreover, there exists a constant $C$ depending only on $C_G$, $\varepsilon_G$, and $m$, such that
\begin{align}
|\mathcal J_m(\mu)| &\le C \big(1+mM_2(\mu)\big),
\label{eq:Jm-bound} \\
|\Phi_m(\xi;\mu)| &\le C \Big(1+\|\xi\|^{2-\varepsilon_G}+(m-1)M_2(\mu)\Big), \qquad \forall \xi\in\Xi.
\label{eq:Phi-bound}
\end{align}
\end{lemma}

\begin{proof}
Since $\varepsilon_G>0$, we have that $\|z\|^{2-\varepsilon_G}\le 1+\|z\|^2$ for all $z\in\Xi$. Thus, Assumption~\ref{ass:G}(i) implies that $|G(\xi_{1:m})|\le C_G\big(1+m+\sum_{j=1}^m\|\xi_j\|^2\big)$. Integrating both sides over $\mu^{\otimes m}$, and using Tonelli's theorem, yields \eqref{eq:Jm-bound}. The bound \eqref{eq:Phi-bound}
follows similarly by integrating over $\mu^{\otimes(m-1)}$.
\end{proof}

\subsubsection{First Variations}
\label{sec:first-var-theory}
\begin{lemma}[First variation of $\mathcal J_m$]\label{lem:deltaJ}
Suppose Assumption~\ref{ass:G} holds. Fix $\mu,\nu\in\mathcal P_2(\Xi)$ and define $\mu_\varepsilon=(1-\varepsilon)\mu+\varepsilon\nu$ for $\varepsilon\in[0,1]$. Then $\varepsilon\mapsto \mathcal J_m(\mu_\varepsilon)$ is differentiable at $\varepsilon=0$ and $\smash{\frac{\mathrm d}{\mathrm d\varepsilon}\mathcal J_m(\mu_\varepsilon)|_{\varepsilon=0}=m\int_{\Xi}\Phi_m(\xi;\mu) (\nu-\mu)(\mathrm d\xi)}$. Consequently, the (linear) first variation of $\mathcal{J}_m$ at $\mu$ is given $\mu$-a.e. up to an additive constant by
\begin{equation}
\frac{\delta\mathcal J_m}{\delta\mu}(\mu)(\xi)=m \Phi_m(\xi;\mu).
\end{equation}
\end{lemma}

\begin{proof}
Set $\delta:=\nu-\mu$ and $\mu_\varepsilon:=\mu+\varepsilon\delta$. By multilinearity of the product measure, we can expand $\mu_{\varepsilon}^{\otimes m}$ as 
\begin{equation}
\mu_\varepsilon^{\otimes m}
=\mu^{\otimes m}
+\varepsilon\sum_{j=1}^m \mu^{\otimes(j-1)}\otimes\delta\otimes\mu^{\otimes(m-j)}
+o(\varepsilon)
\end{equation}
against any test function integrable under $\mu^{\otimes m}$ and $\nu\otimes\mu^{\otimes(m-1)}$. Assumption~\ref{ass:G}(i) and $\mu,\nu\in\mathcal P_2(\Xi)$ ensure that $G$ is integrable under these measures. Finally, integrating $G$ against the expansion and using permutation invariance gives
\begin{equation}
\left.\frac{d}{d\varepsilon}\mathcal J_m(\mu_\varepsilon)\right|_{\varepsilon=0}
=m\int_{\Xi^m} G(\xi_{1:m}) (\nu-\mu)(d\xi_1) \mu^{\otimes(m-1)}(d\xi_{2:m})
=m\int_{\Xi}\Phi_m(\xi;\mu) (\nu-\mu)(d\xi).
\end{equation}
The formula for the linear first variation follows from the definition.
\end{proof}

\begin{lemma}[First variation of $\mathcal{R}$]
\label{lem:deltaR}
Suppose Assumption~\ref{ass:r} holds. Fix $\mu,\nu\in\mathcal P_2(\Xi)$ and define the mixture path $\mu_\varepsilon:=(1-\varepsilon)\mu+\varepsilon\nu$ for $\varepsilon\in[0,1]$. Then $\varepsilon\mapsto \mathcal R(\mu_\varepsilon)$ is differentiable at $\varepsilon=0$ and $\frac{\mathrm d}{\mathrm d\varepsilon}\mathcal R(\mu_\varepsilon)|_{\varepsilon=0}
= \int_{\Xi}\Psi_r(\xi;\mu) (\nu-\mu)(\mathrm d\xi)$. Consequently, the (linear) first variation of $\mathcal R$ at $\mu$ is given $\mu$-a.e.\ up to an additive constant by
\begin{equation}
\frac{\delta \mathcal R}{\delta\mu}(\mu)(\xi)=\Psi_r(\xi;\mu).
\label{eq:rep-first-variation}
\end{equation}
\end{lemma}

\begin{proof}
With $\delta:=\nu-\mu$ and $\mu_\varepsilon:=\mu+\varepsilon\delta$. Then $\mathcal R(\mu_\varepsilon) =\frac12\iint r(\xi-\chi) (\mu+\varepsilon\delta)(d\xi)(\mu+\varepsilon\delta)(d\chi)$. The $O(\varepsilon)$ term is given by
\begin{equation}
\frac{\varepsilon}{2}\iint r(\xi-\chi) \delta(d\xi)\mu(d\chi)
+\frac{\varepsilon}{2}\iint r(\xi-\chi) \mu(d\xi)\delta(d\chi).
\end{equation}
Since $r$ is even, the two terms coincide. We thus have that
\begin{equation}
\left.\frac{d}{d\varepsilon}\mathcal R(\mu_\varepsilon)\right|_{\varepsilon=0}
=\int_{\Xi}\Big(\int_{\Xi}r(\xi-\chi) \mu(d\chi)\Big)(\nu-\mu)(d\xi)
=\int_{\Xi}\Psi_r(\xi;\mu) (\nu-\mu)(d\xi),
\end{equation}
The formula for the linear first variation follows from the definition.
\end{proof}

\begin{lemma}[First variation of KL]\label{lem:deltaEnt}
Let $\rho\in\mathcal{P}_{2,\mathrm{ac}}(\Xi)$ be as defined above. Suppose
$\mu\in\mathcal P_{2,\mathrm{ac}}(\Xi)$ satisfies $\mathrm{KL}(\mu\|\rho)<\infty$ and $f:=\frac{\mathrm d\mu}{\mathrm d\rho}>0$ $\rho\text{-a.e.}$. Then, up to an additive constant,  
\begin{equation}
\frac{\delta}{\delta\mu}\mathrm{KL}(\mu\|\rho)(\xi)=\log\Big(\frac{\mathrm d\mu}{\mathrm d\rho}(\xi)\Big)+1
\qquad\mu\text{-a.e.}
\label{eq:dEnt}
\end{equation}
\end{lemma}

\begin{proof}
This result is standard; see, e.g., \citet[Chapter 9]{villani2009optimal}; \citet[Chapter 11]{ambrosio2008gradientflows}.
\end{proof}

\begin{corollary}[First variation of $\mathcal F_m^{\lambda,\mathrm{rep}}$]
\label{cor:deltaF}
Suppose Assumptions~\ref{ass:G} and \ref{ass:r} hold. Let $\rho\in\mathcal{P}_{2,\mathrm{ac}}(\Xi)$ be as defined above. Suppose $\mu\in\mathcal{P}_{2,\mathrm{ac}}(\Xi)$ satisfies $\mathrm{KL}(\mu\|\rho)<\infty$ and $\smash{\frac{\mathrm d\mu}{\mathrm d\rho}>0}$ $\rho$-a.e. Then the (linear) first variation of $\mathcal F_m^{\lambda,\mathrm{rep}}$ at $\mu$ is given $\mu$-a.e.\ up to an additive constant by
\begin{equation}
\frac{\delta \mathcal F_m^{\lambda,\mathrm{rep}}}{\delta\mu}(\mu)(\xi)
=
-m \Phi_m(\xi;\mu)
+\eta \Psi_r(\xi;\mu)
+\lambda\left(\log\Big(\frac{\mathrm d\mu}{\mathrm d\rho}(\xi)\Big)+1\right).
\label{eq:deltaF}
\end{equation}
\end{corollary}

\begin{proof}
The result is an immediate consequence of Lemma~\ref{lem:deltaJ}, Lemma~\ref{lem:deltaR}, and Lemma~\ref{lem:deltaEnt}.
\end{proof}

\subsubsection{Minimisers: Existence, Fixed Point Characterisation, Uniqueness}
\label{sec:exist-unique-gibbs}
\begin{lemma}[Entropy controls second moments]\label{lem:Ent-controls-M2}
Suppose that Assumption~\ref{ass:V} holds. For any $c\in(0,a_V)$ there exists $C_{2,c}<\infty$ such that for all $\mu\ll\rho$,
\begin{equation}
M_2(\mu)\le \frac{1}{c}\mathrm{KL}(\mu\|\rho)+C_{2,c}.
\label{eq:Ent-controls-M2}
\end{equation}
\end{lemma}

\begin{proof}
Apply the Donsker--Varadhan variational formula \citep[e.g.,][]{dembo1998large} with $f(\xi)=c\|\xi\|^2$, where $0<c<a_V$. This implies that
\begin{equation}
\mathrm{KL}(\mu\|\rho)
\ge c M_2(\mu)-\log\int_{\Xi} e^{c\|\xi\|^2} \rho(d\xi).
\end{equation}
Due to Assumption~\ref{ass:V}(iii), we have that $\int e^{c\|\xi\|^2} d\rho<\infty$. The bound in \eqref{eq:Ent-controls-M2} follows immediately upon rearrangement.
\end{proof}

\begin{lemma}[Uniform integrability and continuity of $\mathcal J_m$]\label{lem:UI-Jm-cont}
Fix $\varepsilon\in(0,2]$. Let $(\mu_n)_{n\in\mathbb{N}}$ be a family in $\mathcal P_2(\Xi)$ such that $\sup_n M_2(\mu_n)<\infty$. Then the family of random variables $\|\xi\|^{2-\varepsilon}$ under $(\mu_n)$ is uniformly integrable. Moreover, if $\mu_n\rightharpoonup \mu$ weakly, then the family of functions $\smash{H(\xi_{1:m}) := 1+\sum_{j=1}^m \|\xi_j\|^{2-\varepsilon}}$ is uniformly integrable under $(\mu_n^{\otimes m})$, and for any continuous $G$ satisfying $|G|\le C H$ for some $C<\infty$,
$\smash{\int_{\Xi^m} G \mathrm d\mu_n^{\otimes m}\longrightarrow \int_{\Xi^m} G \mathrm d\mu^{\otimes m}}$. Thus, under Assumption~\ref{ass:G}(i),
\begin{equation}
\mathcal J_m(\mu_n)\to \mathcal J_m(\mu).
\end{equation}
\end{lemma}

\begin{proof}
The elementary bound 
$\|\xi\|^{2-\varepsilon}\mathbf 1_{\{\|\xi\|>R\}}
\le R^{-\varepsilon}\|\xi\|^2$ and the uniform second-moment bound imply that $\|\xi\|^{2-\varepsilon}$ is uniformly integrable under $(\mu_n)$. By the same tail estimate, together with a union bound over the $m$ coordinates, the function
\begin{equation}
H(\xi_{1:m})=1+\sum_{j=1}^m \|\xi_j\|^{2-\varepsilon}
\end{equation}
is uniformly integrable under $(\mu_n^{\otimes m})$. Since
$\mu_n\rightharpoonup\mu$, also $\mu_n^{\otimes m}\rightharpoonup\mu^{\otimes m}$ on $\Xi^m$. Therefore, if $G$ is continuous and $|G|\le CH$, Vitali's theorem yields
\begin{equation}
\int_{\Xi^m} G d\mu_n^{\otimes m}
\longrightarrow
\int_{\Xi^m} G d\mu^{\otimes m}.
\end{equation}
Under Assumption~\ref{ass:G}(i), we have $|G|\le C_G H$ with
$\varepsilon=\varepsilon_G$. It follows, in particular, that
$\mathcal J_m(\mu_n)\to\mathcal J_m(\mu)$.
\end{proof}

\begin{lemma}[Lower semicontinuity of the repulsion energy]
\label{lem:R-lsc}
Suppose Assumption~\ref{ass:r} holds. If $\mu_n\rightharpoonup \mu$ weakly in $\mathcal P_2(\Xi)$, then
\begin{equation}
\mathcal R(\mu)\le \liminf_{n\to\infty}\mathcal R(\mu_n).
\end{equation}
\end{lemma}

\begin{proof}
Since $\mu_n\rightharpoonup \mu$ on $\Xi$, we also have $\mu_n\otimes\mu_n\rightharpoonup \mu\otimes\mu$ on $\Xi\times\Xi$. Define $\smash{ H_r(\xi,\chi):=\frac12 r(\xi-\chi)}$. By Assumption~\ref{ass:r}(ii), $H_r$ is continuous, and by Assumption~\ref{ass:r}(iv), it is bounded below by $\underline r/2$. Therefore, by Portmanteau's theorem, we have
\begin{equation}
\int_{\Xi\times\Xi} H_r(\xi,\chi) (\mu\otimes\mu)(\mathrm d\xi,\mathrm d\chi)
\le
\liminf_{n\to\infty}
\int_{\Xi\times\Xi} H_r(\xi,\chi) (\mu_n\otimes\mu_n)(\mathrm d\xi,\mathrm d\chi),
\end{equation}
which is exactly the claimed lower semicontinuity of $\mathcal R$.
\end{proof}

\begin{theorem}[Existence of a minimiser]
\label{thm:exist}
Suppose that Assumptions~\ref{ass:V}, \ref{ass:G}, and \ref{ass:r} hold. Then $\mathcal F_m^{\lambda,\mathrm{rep}}$ is proper, bounded below on $\mathcal P_2(\Xi)$, and admits at least one minimiser
\begin{equation}
\mu_m^{\lambda,\star}\in\argmin_{\mu\in\mathcal P_2(\Xi)}\mathcal F_m^{\lambda,\mathrm{rep}}(\mu).
\end{equation}
\end{theorem}

\begin{proof}
The fact that $\mathcal F_m^{\lambda,\mathrm{rep}}$ is proper is immediate from $\mathcal F_m^{\lambda,\mathrm{rep}}(\rho)<\infty$, using Lemma~\ref{lem:Jm-wellposed} and Assumption~\ref{ass:r}(iii). For coercivity, Assumption~\ref{ass:G}(i) and the Donsker--Varadhan formula applied to $a\|\xi\|^{2-\varepsilon_G}$ give, for any $a>0$,
\begin{equation}
|\mathcal J_m(\mu)|\le C_a+\frac{mC_G}{a}\mathrm{KL}(\mu\|\rho).
\end{equation}
In addition, we know that $\mathcal R(\mu)\ge \underline r/2$ by Assumption~\ref{ass:r}(iv). It follows from this and the previous display that
\begin{equation}
\mathcal F_m^{\lambda,\mathrm{rep}}(\mu)
\ge -C_a+\frac{\eta\underline r}{2}
+\Bigl(\lambda-\frac{mC_G}{a}\Bigr)\mathrm{KL}(\mu\|\rho).
\end{equation}
Choosing $a>mC_G/\lambda$ shows that $\mathcal F_m^{\lambda,\mathrm{rep}}$ is bounded below and that every minimizing sequence has uniformly bounded entropy. Meanwhile, Lemma~\ref{lem:Ent-controls-M2} gives a uniform second-moment bound.

Finally, let $(\mu_n)$ be a minimizing sequence. By Prokhorov's theorem and the uniform second-moment bound, there exists a subsequence, not relabelled, and a measure $\mu_\infty\in\mathcal P_2(\Xi)$ such that $\mu_n\rightharpoonup\mu_\infty$. It follows from the semicontinuity of $\mathrm{KL}(\cdot\|\rho)$, Lemma~\ref{lem:R-lsc}, and Lemma~\ref{lem:UI-Jm-cont} that $\mu_\infty$ is a minimiser, since
\begin{equation}
\mathcal F_m^{\lambda,\mathrm{rep}}(\mu_\infty)
\le \liminf_{n\to\infty}\mathcal F_m^{\lambda,\mathrm{rep}}(\mu_n).
\end{equation}
\end{proof}

\begin{lemma}[Strict positivity of minimisers]\label{lem:positive-min}
Suppose that Assumptions~\ref{ass:V}, \ref{ass:G}, and \ref{ass:r} hold. Let $\mu_{m}^{\lambda,\star}$ be any minimiser of $\mathcal{F}_m^{\lambda,\mathrm{rep}}$. Then $\mu_{m}^{\lambda,\star}\ll\rho$ and, moreover, $\frac{\mathrm d\mu_{m}^{\lambda,\star}}{\mathrm d\rho}(\xi)>0$ for $\rho$-a.e.\ $\xi\in\Xi$.
\end{lemma}

\begin{proof}
First, absolute continuity with respect to $\rho$ is immediate from finiteness of the entropy term. Let $f^\star:=d\mu_m^{\lambda,\star}/d\rho$ and suppose that $f^\star=0$ on a Borel set $A$ with $\rho(A)>0$. Set $\nu_A:=\rho(\cdot\mid A)$ and $\mu_\varepsilon:=(1-\varepsilon)\mu_m^{\lambda,\star}+\varepsilon\nu_A$.

By Lemmas~\ref{lem:deltaJ} and \ref{lem:deltaR}, the one-sided derivatives of $\mathcal J_m(\mu_\varepsilon)$ and $\mathcal R(\mu_\varepsilon)$ at $\varepsilon=0^+$ are finite. On the other hand, on $A$ we have $\frac{d\mu_\varepsilon}{d\rho}=\varepsilon/\rho(A)$, so
\begin{equation}
\int_A \frac{d\mu_\varepsilon}{d\rho}\log\frac{d\mu_\varepsilon}{d\rho} d\rho
=\varepsilon\log\frac{\varepsilon}{\rho(A)},
\end{equation}
whose right derivative at $0$ is $-\infty$. It follows, in particular, that  $\frac{d}{d\varepsilon}\mathcal F_m^{\lambda,\mathrm{rep}}(\mu_\varepsilon)|_{\varepsilon=0^+}=-\infty$, contradicting minimality. Therefore $f^\star(\xi)>0$ for $\rho$-a.e. $\xi\in\Xi$ or, equivalently, $\rho\ll\mu_m^{\lambda,\star}$.
\end{proof}

\begin{theorem}[Euler--Lagrange condition and Gibbs fixed point]
\label{thm:fixedpoint}
Suppose that Assumptions~\ref{ass:V}, \ref{ass:G}, and \ref{ass:r} hold. Then any minimiser $\mu_m^{\lambda,\star}$ of
$\mathcal F_m^{\lambda,\mathrm{rep}}$ satisfies $\mu_m^{\lambda,\star}\ll\rho$ with strictly positive density $\rho$-a.e., and
\begin{equation}
\mu_m^{\lambda,\star}(\mathrm d\xi)
=
\frac{1}{Z_m^{\lambda,\eta}}\exp \Big(\frac{1}{\lambda}\big(m\Phi_m(\xi;\mu_m^{\lambda,\star})-\eta\Psi_r(\xi;\mu_m^{\lambda,\star})\big)\Big) \rho(\mathrm d\xi),
\label{eq:gibbs-fp}
\end{equation}
where $Z_m^{\lambda,\eta}$ is the normalising constant
\begin{equation}
Z_m^{\lambda,\eta}
:=
\int_{\Xi}\exp \Big(\frac{1}{\lambda}\big(m\Phi_m(\xi;\mu_m^{\lambda,\star})-\eta\Psi_r(\xi;\mu_m^{\lambda,\star})\big)\Big) \rho(\mathrm d\xi)
\in(0,\infty).
\end{equation}
\end{theorem}

\begin{proof}
By Lemma~\ref{lem:positive-min}, $\mu_m^{\lambda,\star}\ll\rho$ with density $f^\star>0$ $\rho$-a.e., and $\mathrm{KL}(\mu_m^{\lambda,\star}\|\rho)<\infty$. Hence Corollary~\ref{cor:deltaF} applies. Let
$\varphi\in L^\infty(\mu_m^{\lambda,\star})$ satisfy $\int \varphi d\mu_m^{\lambda,\star}=0$, and set
$\mu_\varepsilon:=(1+\varepsilon\varphi)\mu_m^{\lambda,\star}$ for $|\varepsilon|$ small. Since $\mu_m^{\lambda,\star}$ is a minimiser,
\begin{equation}
0=\left.\frac{d}{d\varepsilon}\mathcal F_m^{\lambda,\mathrm{rep}}(\mu_\varepsilon)\right|_{\varepsilon=0}
=\int_{\Xi}\frac{\delta \mathcal F_m^{\lambda,\mathrm{rep}}}{\delta\mu}(\mu_m^{\lambda,\star})(\xi)
 \varphi(\xi) \mu_m^{\lambda,\star}(d\xi).
\end{equation}
Since this holds for every bounded mean-zero $\varphi$, the first variation is constant $\mu_m^{\lambda,\star}$-a.e.; since $\rho\ll\mu_m^{\lambda,\star}$, the same holds $\rho$-a.e. Using \eqref{eq:deltaF},
\begin{equation}
-m\Phi_m(\xi;\mu_m^{\lambda,\star})+\eta\Psi_r(\xi;\mu_m^{\lambda,\star})+\lambda\log f^\star(\xi)=C
\qquad \rho\text{-a.e.}
\end{equation}
for some constant $C$. Thus, exponentiating and normalizing, we have that
\begin{equation}
f^\star(\xi)
=\frac{1}{Z_m^{\lambda,\eta}}
\exp \Big(\frac{1}{\lambda}\big(m\Phi_m(\xi;\mu_m^{\lambda,\star})-\eta\Psi_r(\xi;\mu_m^{\lambda,\star})\big)\Big),
\end{equation}
which is exactly \eqref{eq:gibbs-fp}. Finally, since $f^\star$ must integrate to $1$, the normalizing constant is finite and strictly positive.
\end{proof}

\begin{lemma}[Geodesic convexity of $-\mathcal J_m$]\label{lem:J-semi}
Suppose that Assumption~\ref{ass:G} holds. Let $\mu_0,\mu_1\in\mathcal P_2(\Xi)$ and let $\gamma\in\Gamma(\mu_0,\mu_1)$ be an optimal coupling for $\mathsf W_2$. Define the displacement interpolation $\mu_t:=(T_t)_\#\gamma$, $t\in[0,1]$, where $T_t(\xi,\eta):=(1-t)\xi+t\eta$. Then
\begin{equation}
-\mathcal J_m(\mu_t)
\le (1-t)\big(-\mathcal J_m(\mu_0)\big)+t\big(-\mathcal J_m(\mu_1)\big)
+\frac{mL_G}{2} t(1-t) \mathsf{W}_2^2(\mu_0,\mu_1).
\label{eq:J-semi}
\end{equation}
That is, $\mu\mapsto -\mathcal J_m(\mu)$ is $(-mL_G)$-geodesically convex on $(\mathcal P_2(\Xi),\mathsf{W}_2)$.
\end{lemma}

\begin{proof}
Let $\Gamma:=\gamma^{\otimes m}$ and $S_t:=T_t^{\otimes m}$, so that $\mu_t^{\otimes m}=(S_t)_\#\Gamma$. For fixed $(\xi_{1:m},\eta_{1:m})$, set $\psi(t):=G\big((1-t)\xi_{1:m}+t\eta_{1:m}\big)$. Then
\begin{equation}
\psi''(t)=\langle d_m,\nabla^2G((1-t)\xi_{1:m}+t\eta_{1:m})d_m\rangle
\le L_G\|d_m\|^2,
\end{equation}
where $d_m:=(\eta_1-\xi_1,\dots,\eta_m-\xi_m)$. It follows, integrating against $\Gamma$ and using the fact that $\int \|d_m\|^2 d\Gamma = m \mathsf W_2^2(\mu_0,\mu_1)$, that
\begin{equation}
\frac{d^2}{dt^2}\mathcal J_m(\mu_t)
\le mL_G \mathsf W_2^2(\mu_0,\mu_1).
\end{equation}
Hence $t\mapsto \mathcal J_m(\mu_t)+\frac{mL_G}{2}t(1-t)\mathsf W_2^2(\mu_0,\mu_1)$ is concave on $[0,1]$. Evaluating this inequality at time $t$, and multiplying by $-1$, yields \eqref{eq:J-semi}.
\end{proof}

\begin{lemma}[Geodesic semiconvexity of $\mathcal R$]\label{lem:R-semi}
Suppose Assumption~\ref{ass:r}(ii) holds, i.e.\ $\|\nabla^2 r\|_{\mathrm{op}}\le L_r$. Then $\mu\mapsto \mathcal R(\mu)$ is $(-L_r)$-geodesically convex on $(\mathcal P_2(\Xi),\mathsf W_2)$ in the sense that for every $\mu_0,\mu_1\in\mathcal P_2(\Xi)$ and every displacement interpolation $(\mu_t)_{t\in[0,1]}$ between them,
\begin{equation}
\mathcal R(\mu_t)
\le
(1-t)\mathcal R(\mu_0)+t\mathcal R(\mu_1)
+\frac{L_r}{2} t(1-t) \mathsf W_2^2(\mu_0,\mu_1).
\label{eq:R-semi}
\end{equation}
\end{lemma}

\begin{proof}
This is a standard estimate for interaction energies with smooth potentials; see, e.g., \citet[][Section 9.3]{ambrosio2008gradientflows}. It follows by differentiating $t\mapsto \mathbb E[r(X_t-X_t')]$ along the displacement interpolation $X_t=(1-t)X+tY$ under an optimal coupling and using the uniform Hessian bound $\|\nabla^2 r\|_{\mathrm{op}}\le L_r$.
\end{proof}

\begin{theorem}[Uniqueness and quadratic growth in the strongly convex regime]\label{thm:uniq}
Suppose that Assumptions~\ref{ass:V}, \ref{ass:G}, \ref{ass:r}, and \ref{ass:strong} hold. Then $\mathcal F_m^{\lambda,\mathrm{rep}}$ is $\alpha$-geodesically convex on $(\mathcal P_2(\Xi),\mathsf{W}_2)$ with $ \alpha=\lambda\kappa-mL_G-\eta L_r>0.$ In particular, the minimiser $\mu_m^{\lambda,\star}$ is unique and satisfies
\begin{equation}
\mathcal F_m^{\lambda,\mathrm{rep}}(\mu)-\mathcal F_m^{\lambda,\mathrm{rep}}(\mu_m^{\lambda,\star})
\ge \frac{\alpha}{2} \mathsf{W}_2^2(\mu,\mu_m^{\lambda,\star}),
\qquad \forall \mu\in\mathcal P_2(\Xi).
\label{eq:quad-growth}
\end{equation}
\end{theorem}

\begin{proof}
Under Assumption~\ref{ass:V}(i), the functional $\mu\mapsto\mathrm{KL}(\mu\|\rho)$ is $\kappa$-geodesically convex on $(\mathcal P_2,\mathsf{W}_2)$ \citep[e.g.,][Chapter 9]{ambrosio2008gradientflows}. By Lemma~\ref{lem:J-semi}, $\mu\mapsto -\mathcal J_m(\mu)$ is $(-mL_G)$-geodesically convex. By Lemma~\ref{lem:R-semi}, $\mu\mapsto \eta\mathcal R(\mu)$ is $(-\eta L_r)$-geodesically convex. Therefore $\mathcal F_m^{\lambda,\mathrm{rep}}=(-\mathcal J_m)+\eta\mathcal R+\lambda\mathrm{KL}(\cdot\|\rho)$
is $(\lambda\kappa-mL_G-\eta L_r)$-geodesically convex, i.e.\ $\alpha$-convex. Since $\alpha>0$, $\mathcal{F}_m^{\lambda,\mathrm{rep}}$ is strongly geodesically convex, hence the minimiser is unique. The quadratic growth bound \eqref{eq:quad-growth} is a standard consequence of $\alpha$-convexity at a minimiser.
\end{proof}

\begin{corollary}[Uniqueness of the Gibbs fixed point in the strongly convex regime]
\label{cor:unique-fixed-point}
Suppose that Assumptions~\ref{ass:V}, \ref{ass:G}, \ref{ass:r}, and \ref{ass:strong} hold. Then the fixed point equation \eqref{eq:gibbs-fp} admits a unique solution in $\mathcal P_2(\Xi)$, namely the unique minimiser $\mu_m^{\lambda,\star}$ from Theorem~\ref{thm:uniq}.
\end{corollary}

\begin{proof}
By Theorems~\ref{thm:fixedpoint} and \ref{thm:uniq}, the unique minimiser $\mu_m^{\lambda,\star}$ solves \eqref{eq:gibbs-fp}. Conversely, let $\mu\in\mathcal P_2(\Xi)$ solve \eqref{eq:gibbs-fp}. Then
$d\mu/d\rho>0$ $\rho$-a.e. and
\begin{equation}
\log\frac{d\mu}{d\rho}
=\frac1\lambda\bigl(m\Phi_m(\cdot;\mu)-\eta\Psi_r(\cdot;\mu)\bigr)-\log Z_m^{\lambda,\eta}.
\end{equation}
Since $\mathcal J_m(\mu)<\infty$ by Lemma~\ref{lem:Jm-wellposed} and $\mathcal R(\mu)<\infty$ by Assumption~\ref{ass:r}(iii), integrating this identity with respect to $\mu$ shows that $\mathrm{KL}(\mu\|\rho)<\infty$. We can thus apply Corollary~\ref{cor:deltaF} to obtain
\begin{equation}
\frac{\delta \mathcal F_m^{\lambda,\mathrm{rep}}}{\delta\mu}(\mu)
\equiv \lambda\bigl(1-\log Z_m^{\lambda,\eta}\bigr)
\qquad \mu\text{-a.e.}
\end{equation}
Using Lemmas~\ref{lem:Phi-diff-Lip} and \ref{lem:Psi-diff-Lip}, the first variation is $C^1$ in $\xi$; hence its gradient vanishes and $0\in\partial \mathcal F_m^{\lambda,\mathrm{rep}}(\mu)$ in the Wasserstein sense
\citep[Chapter~10]{ambrosio2008gradientflows}. Since
$\mathcal F_m^{\lambda,\mathrm{rep}}$ is $\alpha$-geodesically convex with $\alpha>0$, any such critical point is a minimiser. By uniqueness of the minimiser from Theorem~\ref{thm:uniq}, we conclude that $\mu=\mu_m^{\lambda,\star}$.
\end{proof}

\subsubsection{Wasserstein Gradient Flow}
\label{sec:wgf}
\begin{lemma}[Differentiability and Lipschitz gradient of $\Phi_m$]\label{lem:Phi-diff-Lip}
Suppose that Assumption~\ref{ass:G} holds. Then for each $\mu\in\mathcal P_2(\Xi)$ the map $\xi\mapsto \Phi_m(\xi;\mu)$ is $C^1$, and
\begin{equation}
\nabla_{\xi}\Phi_m(\xi;\mu)=\int_{\Xi^{m-1}}\nabla_1G(\xi,\xi_{2:m}) \mu^{\otimes(m-1)}(\mathrm d\xi_{2:m}).
\label{eq:grad-Phi-recall}
\end{equation}
Moreover, $\nabla_\xi\Phi_m(\cdot;\mu)$ is globally Lipschitz: for all $\xi,\xi'\in\Xi$, 
\begin{equation}
\|\nabla_\xi\Phi_m(\xi;\mu)-\nabla_\xi\Phi_m(\xi';\mu)\|\le L_G\|\xi-\xi'\|.
\end{equation}
\end{lemma}

\begin{proof}
By Assumption~\ref{ass:G}(ii), we have the linear bound $\|\nabla_1G(\xi,\xi_{2:m})\|\le C(1+\|\xi\|+\sum_{j=2}^m\|\xi_j\|)$, which is integrable under $\mu^{\otimes(m-1)}$. Differentiating under the integral sign in \eqref{eq:Phi-def-recall} thus yields \eqref{eq:grad-Phi-recall}. Moreover, by Assumption~\ref{ass:G}(iii), we have that
\begin{equation}
\|\nabla_1G(\xi,\xi_{2:m})-\nabla_1G(\xi',\xi_{2:m})\|
\le L_G\|\xi-\xi'\|.
\end{equation}
Finally, integrating this bound over $\mu^{\otimes(m-1)}$ gives the stated global Lipschitz estimate. Thus, in particular, $\Phi_m(\cdot;\mu)\in C^1(\Xi)$.
\end{proof}

\begin{lemma}[Differentiability and Lipschitz gradient of $\Psi_r$]
\label{lem:Psi-diff-Lip}
Suppose Assumption~\ref{ass:r}(ii) holds. Then for each $\mu\in\mathcal P_2(\Xi)$ the map $\xi\mapsto \Psi_r(\xi;\mu)$ is $C^1$, and
\begin{equation}
\nabla_\xi\Psi_r(\xi;\mu)=\int_{\Xi}\nabla r(\xi-\chi) \mu(\mathrm d\chi).
\label{eq:grad-Psi}
\end{equation}
Moreover, $\nabla_\xi\Psi_r(\cdot;\mu)$ is globally Lipschitz: for all $\xi,\xi'\in\Xi$,
\begin{equation}
\|\nabla_\xi\Psi_r(\xi;\mu)-\nabla_\xi\Psi_r(\xi';\mu)\|\le L_r\|\xi-\xi'\|.
\label{eq:Psi-Lip-x}
\end{equation}
Finally, for all $\mu,\nu\in\mathcal P_2(\Xi)$ and all $\xi\in\Xi$,
\begin{equation}
\|\nabla_\xi\Psi_r(\xi;\mu)-\nabla_\xi\Psi_r(\xi;\nu)\|
\le
L_r \mathsf W_2(\mu,\nu).
\label{eq:Psi-Lip-mu}
\end{equation}
\end{lemma}

\begin{proof}
The gradient representation \eqref{eq:grad-Psi} follows by differentiating under the integral sign, which is justified since $\nabla r$ is globally Lipschitz and hence has linear growth. For \eqref{eq:Psi-Lip-x}, use that $\|\nabla r(\xi-\chi)-\nabla r(\xi'-\chi)\|\le L_r\|\xi-\xi'\|$ and integrate over $\mu$. For \eqref{eq:Psi-Lip-mu}, let $\gamma$ be an optimal coupling of $\mu$ and $\nu$ for $\mathsf W_2$ and let $(U,V)\sim\gamma$. Then
\begin{equation}
\nabla_\xi\Psi_r(\xi;\mu)-\nabla_\xi\Psi_r(\xi;\nu)
=
\mathbb E[\nabla r(\xi-U)-\nabla r(\xi-V)],
\end{equation}
and Lipschitzness of $\nabla r$ gives $\|\nabla r(\xi-U)-\nabla r(\xi-V)\|\le L_r\|U-V\|$. Taking expectations and using Cauchy--Schwarz yields $\|\cdot\|\le L_r(\mathbb E\|U-V\|^2)^{1/2}=L_r \mathsf W_2(\mu,\nu)$.
\end{proof}

\begin{lemma}[Drift Lipschitz bounds]\label{lem:drift-Lip}
Suppose that Assumptions~\ref{ass:V}, ~\ref{ass:G}, and ~\ref{ass:r}(ii) hold. Define
\begin{equation}
b(\xi,\mu)
:=
m \nabla_{\xi}\Phi_m(\xi;\mu)
-\eta \nabla_{\xi}\Psi_r(\xi;\mu)
-\lambda \nabla V(\xi).
\label{eq:b-def}
\end{equation}
Then for all $\xi,\xi'\in\Xi$ and $\mu,\nu\in\mathcal P_2(\Xi)$,
\begin{alignat}{2}
\|b(\xi,\mu)-b(\xi',\mu)\|
&\le L_\xi\|\xi-\xi'\|,
\qquad &&L_\xi:=mL_G+\eta L_r+\lambda K_V,
\label{eq:b-Lip-x}\\
\|b(\xi,\mu)-b(\xi,\nu)\|
&\le L_\mu \mathsf W_2(\mu,\nu),
\qquad &&L_\mu:=mL_G\sqrt{m-1}+\eta L_r.
\label{eq:b-Lip-mu}
\end{alignat}
\end{lemma}

\begin{proof}
The bound in \eqref{eq:b-Lip-x} is immediate from the decomposition \eqref{eq:b-def}, Lemma~\ref{lem:Phi-diff-Lip}, Lemma~\ref{lem:Psi-diff-Lip}, and Assumption~\ref{ass:V}(ii). For \eqref{eq:b-Lip-mu}, let $\gamma$ be an optimal coupling of $\mu$ and $\nu$, and let $(U_j,V_j)_{j=2}^m$ be i.i.d. with law $\gamma$. Then
\begin{equation}
\nabla_{\xi}\Phi_m(\xi;\mu)-\nabla_{\xi}\Phi_m(\xi;\nu)
=\mathbb E\big[\nabla_1G(\xi,U_{2:m})-\nabla_1G(\xi,V_{2:m})\big].
\end{equation}
By Assumption~\ref{ass:G}(iii), we have that $\|\nabla_1G(\xi,U_{2:m})-\nabla_1G(\xi,V_{2:m})\|
\le L_G(\sum_{j=2}^m \|U_j-V_j\|^2)^{1/2}$. Taking expectations and using Jensen's inequality, it follows that
\begin{equation}
\|\nabla_{\xi}\Phi_m(\xi;\mu)-\nabla_{\xi}\Phi_m(\xi;\nu)\|
\le L_G\sqrt{m-1} \mathsf W_2(\mu,\nu).
\end{equation}
Finally, combining this with \eqref{eq:Psi-Lip-mu}, multiplying the $\Phi_m$-bound by $m$ and the $\Psi_r$-bound by $\eta$, and observing that $-\lambda\nabla V$ does not depend on $\mu$, we arrive at \eqref{eq:b-Lip-mu}.
\end{proof}

\begin{lemma}[Dissipativity of the drift]\label{lem:dissip}
Suppose that Assumptions~\ref{ass:V}, \ref{ass:G}, \ref{ass:r}, and \ref{ass:strong} hold. Then for all $\xi,\xi'\in\Xi$ and all $\mu\in\mathcal P_2(\Xi)$,
\begin{equation}
\big\langle \xi-\xi',\ b(\xi,\mu)-b(\xi',\mu)\big\rangle
\le -\alpha \|\xi-\xi'\|^2,
\qquad \alpha=\lambda\kappa-\eta L_r - mL_G>0.
\label{eq:one-sided}
\end{equation}
\end{lemma}

\begin{proof}
Working from the definition in \eqref{eq:b-def}, we have that
\begin{align*}
\langle \xi-\xi', b(\xi,\mu)-b(\xi',\mu)\rangle
&=m\langle \xi-\xi',\nabla_\xi\Phi_m(\xi;\mu)-\nabla_\xi\Phi_m(\xi';\mu)\rangle \\
&\quad -\eta\langle \xi-\xi',\nabla_\xi\Psi_r(\xi;\mu)-\nabla_\xi\Psi_r(\xi';\mu)\rangle \\
&\quad -\lambda\langle \xi-\xi',\nabla V(\xi)-\nabla V(\xi')\rangle.
\end{align*}
By Lemmas~\ref{lem:Phi-diff-Lip} and \ref{lem:Psi-diff-Lip}, together with the Cauchy--Schwarz inequality, the first two terms are bounded above by $mL_G\|\xi-\xi'\|^2$ and $\eta L_r\|\xi-\xi'\|^2$, respectively. Meanwhile, due to Assumption~\ref{ass:V}(i), we have $\langle \xi-\xi',\nabla V(\xi)-\nabla V(\xi')\rangle\ge \kappa\|\xi-\xi'\|^2$. Combining these estimates yields \eqref{eq:one-sided}.
\end{proof}

\begin{theorem}[McKean--Vlasov SDE, PDE]\label{thm:MV-wellposed}
Suppose that Assumptions~\ref{ass:V}, \ref{ass:G}, and \ref{ass:r} hold. Let $b$ be defined by \eqref{eq:b-def}. Then, for each initial law $\mu_0\in\mathcal P_2(\Xi)$, there exists a unique strong solution to the McKean--Vlasov SDE
\begin{equation}
\mathrm d\xi_t=b(\xi_t,\mu_t) \mathrm dt+\sqrt{2\lambda} \mathrm dw_t,
\qquad \mu_t:=\mathrm{Law}(\xi_t),
\label{eq:MV-SDE}
\end{equation}
and the curve $t\mapsto \mu_t$ belongs to $\mathcal C([0,\infty);\mathcal P_2(\Xi))$ with respect to $\mathsf W_2$. Moreover, $(\mu_t)$ solves the nonlinear Fokker--Planck equation
\begin{equation}
\partial_t\mu_t=-\nabla\cdot\big(b(\cdot,\mu_t)\mu_t\big)+\lambda\Delta\mu_t
\label{eq:MV-PDE}
\end{equation}
in the weak sense: for all $\varphi\in C_c^\infty(\Xi)$ and all $t\ge 0$,
\begin{equation}
\int_{\Xi} \varphi(\xi) \mu_t(\mathrm d\xi)
=
\int_{\Xi} \varphi(\xi) \mu_0(\mathrm d\xi)
+\int_0^t\int_{\Xi}\Big(\langle\nabla\varphi(\xi),b(\xi,\mu_s)\rangle+\lambda\Delta\varphi(\xi)\Big) 
\mu_s(\mathrm d\xi) \mathrm ds.
\label{eq:MV-weak}
\end{equation}
\end{theorem}

\begin{proof}
By Lemma~\ref{lem:drift-Lip}, the drift $b(\xi,\mu)$ is globally Lipschitz in $\xi$ and Lipschitz in $\mu$ with respect to $\mathsf W_2$. Standard fixed-point arguments on $\mathcal C([0,T];\mathcal P_2(\Xi))$ therefore yield existence and uniqueness of a strong solution to \eqref{eq:MV-SDE} on each finite horizon $[0,T]$, and hence globally in time; see, e.g., \citet[Sec.~I.2]{sznitman1991topics} or \citet{carmona2018probabilisticI,carmona2018probabilisticII}.  Continuity of $t\mapsto\mu_t$ in $\mathsf W_2$ follows from standard SDE stability. Finally, applying It\^o's formula to $\varphi(\xi_t)$ and taking expectations (the martingale term has mean zero) yields \eqref{eq:MV-weak}, which is equivalent to \eqref{eq:MV-PDE} in the weak sense.
\end{proof}

\begin{proposition}[Wasserstein gradient flow structure and energy dissipation]
\label{prop:EDI-smooth}
Suppose that Assumptions~\ref{ass:V}, \ref{ass:G}, and \ref{ass:r} hold. Let $(\mu_t)$ be a sufficiently regular solution to \eqref{eq:MV-PDE} with strictly positive smooth density and sufficient decay at infinity so that all differentiations under the integral sign and integrations by parts below are justified. Then \eqref{eq:MV-PDE} can be written in the form
\begin{equation}
\partial_t\mu_t+\nabla\cdot(\mu_t v_t)=0,
\qquad
v_t(\xi)=-\nabla_{\xi}\frac{\delta\mathcal{F}_m^{\lambda,\mathrm{rep}}}{\delta\mu}(\mu_t)(\xi),
\label{eq:cont-v}
\end{equation}
Moreover, the free energy $\mathcal{F}_m^{\lambda,\mathrm{rep}}$ satisfies the energy dissipation identity: for a.e.\ $t>0$,
\begin{equation}
\frac{\mathrm d}{\mathrm dt}\mathcal{F}_m^{\lambda,\mathrm{rep}}(\mu_t)
=
-\int_{\Xi}\left\|\nabla_{\xi}\frac{\delta\mathcal{F}_m^{\lambda,\mathrm{rep}}}{\delta\mu}(\mu_t)(\xi)\right\|^2 \mu_t(\mathrm d\xi)
\le 0.
\label{eq:EDI}
\end{equation}
\end{proposition}

\begin{proof}
Using the expression for the first-variation obtained in Corollary~\ref{cor:deltaF}, cf. \eqref{eq:deltaF}, together with the fact that $\rho(\xi)\propto e^{-V(\xi)}$, it is straightforward to verify that
\begin{equation}
\nabla_\xi\frac{\delta\mathcal{F}_m^{\lambda,\mathrm{rep}}}{\delta\mu}(\mu_t)(\xi)
=
-m\nabla_\xi\Phi_m(\xi;\mu_t)+\eta \nabla_{\xi} \Psi_{r}(\xi;\mu_t)+\lambda\nabla\log\mu_t(\xi)+\lambda\nabla V(\xi).
\end{equation}
Thus, the velocity field is given by
\begin{equation}
v_t(\xi)= m\nabla_\xi\Phi_m(\xi;\mu_t) -\eta \nabla_{\xi} \Psi_{r}(\xi;\mu_t)-\lambda\nabla V(\xi)-\lambda\nabla\log\mu_t(\xi),
\end{equation}
and \eqref{eq:cont-v} is equivalent to \eqref{eq:MV-PDE} since $\nabla\cdot(\mu_t\nabla\log\mu_t)=\Delta\mu_t$. Under the stated smoothness and decay assumptions, the chain rule and integration by parts yield \eqref{eq:EDI}. See \citet[Chapters~10--11]{ambrosio2008gradientflows} and \citet{jordan1998variational}.
\end{proof}

\begin{lemma}[Wasserstein contractivity of the McKean--Vlasov dynamics in the dissipative regime]
\label{lem:MV-contractivity}
Suppose that Assumptions~\ref{ass:V}, \ref{ass:G}, \ref{ass:r}, and \ref{ass:strong} hold, and suppose moreover that
\begin{equation}
\delta:=\alpha-L_\mu>0,
\qquad \alpha=\lambda\kappa-\eta L_r - mL_G,\quad L_\mu \text{ is as in Lemma~\ref{lem:drift-Lip}}.
\label{eq:delta-contract}
\end{equation}
Let $(\mu_t)_{t\ge0}$ and $(\tilde\mu_t)_{t\ge0}$ be two solution laws of \eqref{eq:MV-SDE} with initial laws $\mu_0,\tilde\mu_0\in\mathcal P_2(\Xi)$. Then for all $t\ge0$,
\begin{equation}
\mathsf W_2(\mu_t,\tilde\mu_t)\le e^{-\delta t} \mathsf W_2(\mu_0,\tilde\mu_0).
\label{eq:MV-contract}
\end{equation}
In particular, if $\mu_m^{\lambda,\star}$ is a stationary solution of \eqref{eq:MV-PDE} (e.g.\ the Gibbs minimiser from Corollary~\ref{cor:stationary-minimiser} in the $\alpha>0$ regime), then
\begin{equation}
\mathsf W_2(\mu_t,\mu_m^{\lambda,\star})\le e^{-\delta t} \mathsf W_2(\mu_0,\mu_m^{\lambda,\star}),
\qquad t\ge0.
\label{eq:MV-to-stationary}
\end{equation}
\end{lemma}

\begin{remark}
The rate $\delta=\alpha-L_\mu$ arises from a synchronous-coupling bound and is generally conservative. In the $\alpha$-convex regime, the corresponding $\mathrm{EVI}_\alpha$ gradient flow contracts at rate $\alpha$ (see Theorem~\ref{thm:EVI}); rigorously identifying the McKean--Vlasov law with that flow would sharpen \eqref{eq:MV-contract}.
\end{remark}

\begin{proof}
Let $(\xi_t,\tilde\xi_t)$ be the synchronous coupling of the two McKean--Vlasov solutions, driven by the same Brownian motion, and set $\Delta_t:=\xi_t-\tilde\xi_t$. Then
\begin{equation}
\frac{d}{dt}\mathbb E\|\Delta_t\|^2
=2 \mathbb E\langle \Delta_t, b(\xi_t,\mu_t)-b(\tilde\xi_t,\tilde\mu_t)\rangle.
\end{equation}
We decompose the difference in the drift as $[b(\xi_t,\mu_t)-b(\tilde\xi_t,\mu_t)]
+[b(\tilde\xi_t,\mu_t)-b(\tilde\xi_t,\tilde\mu_t)]$. By Lemma~\ref{lem:dissip}, the first term contributes at most $-2\alpha\mathbb E\|\Delta_t\|^2$. By \eqref{eq:b-Lip-mu} and Cauchy--Schwarz, the second contributes at most
\begin{equation}
2L_\mu \mathbb E\big[\|\Delta_t\| \mathsf W_2(\mu_t,\tilde\mu_t)\big].
\end{equation}
Since the law of $(\xi_t,\tilde\xi_t)$ is a coupling of $(\mu_t,\tilde\mu_t)$, $\mathsf W_2^2(\mu_t,\tilde\mu_t)\le \mathbb E\|\Delta_t\|^2$, and therefore $\mathbb E[\|\Delta_t\| \mathsf W_2(\mu_t,\tilde\mu_t)] \le \mathbb E\|\Delta_t\|^2$. It follows that
\begin{equation}
\frac{d}{dt}\mathbb E\|\Delta_t\|^2
\le -2(\alpha-L_\mu) \mathbb E\|\Delta_t\|^2
=-2\delta \mathbb E\|\Delta_t\|^2.
\end{equation}
Gronwall yields
$\mathbb E\|\Delta_t\|^2\le e^{-2\delta t}\mathbb E\|\Delta_0\|^2$, and since
$\mathsf W_2^2(\mu_t,\tilde\mu_t)\le \mathbb E\|\Delta_t\|^2$, this proves
\eqref{eq:MV-contract}. Taking $\tilde\mu_t\equiv \mu_m^{\lambda,\star}$ gives
\eqref{eq:MV-to-stationary}.
\end{proof}

\begin{corollary}[Stationarity of Gibbs minimisers for the McKean--Vlasov dynamics]
\label{cor:stationary-minimiser}
Suppose that Assumptions~\ref{ass:V}, \ref{ass:G}, and \ref{ass:r} hold. Let $\mu_m^{\lambda,\star}$ be any minimiser of
$\mathcal{F}_m^{\lambda,\mathrm{rep}}$, thus a solution of the Gibbs fixed point \eqref{eq:gibbs-fp} (by Theorem~\ref{thm:fixedpoint}). Then $\mu_m^{\lambda,\star}$ is a stationary weak solution of the nonlinear Fokker--Planck equation \eqref{eq:MV-PDE}. Consequently, if $\xi_0\sim \mu_m^{\lambda,\star}$ and $(\xi_t)$ solves \eqref{eq:MV-SDE}, then
$\mu_t=\mathrm{Law}(\xi_t)=\mu_m^{\lambda,\star}$ for all $t\ge 0$.
\end{corollary}

\begin{proof}
Let $q^\star$ denote the Lebesgue density of $\mu_m^{\lambda,\star}$, so that $\mu_m^{\lambda,\star}(\mathrm d\xi)=q^\star(\xi) \mathrm d\xi$. From \eqref{eq:gibbs-fp} and $\rho(\xi)\propto e^{-V(\xi)}$, we have $q^\star(\xi)\propto \exp \big(\frac{1}{\lambda}(m\Phi_m(\xi;\mu_m^{\lambda,\star}) - \eta\Psi_{r}(\xi;\mu_m^{\lambda,\star}))-V(\xi)\big)$, hence
\begin{equation}
\nabla \log q^\star(\xi)
=
\frac{m}{\lambda}\nabla_\xi\Phi_m(\xi;\mu_m^{\lambda,\star})
-\frac{\eta}{\lambda}\nabla_\xi\Psi_r(\xi;\mu_m^{\lambda,\star})
-\nabla V(\xi).
\label{eq:grad-log-qstar}
\end{equation}
Multiplying \eqref{eq:grad-log-qstar} by $\lambda$ gives $b(\xi,\mu_m^{\lambda,\star})=m\nabla_\xi\Phi_m(\xi;\mu_m^{\lambda,\star})-\eta\nabla_\xi\Psi_r(\xi;\mu_m^{\lambda,\star})-\lambda\nabla V(\xi)=\lambda\nabla\log q^\star(\xi)$. Thus, $b(\cdot,\mu_m^{\lambda,\star}) \mu_m^{\lambda,\star}=\lambda\nabla q^\star$, and so
\begin{equation}
-\nabla\cdot\big(b(\cdot,\mu_m^{\lambda,\star})\mu_m^{\lambda,\star}\big)+\lambda\Delta\mu_m^{\lambda,\star}
=
-\lambda\nabla\cdot(\nabla q^\star)+\lambda\Delta q^\star
=0
\end{equation}
That is, $\mu_m^{\lambda,\star}$ is a stationary weak solution of the nonlinear Fokker-Planck equation. The final claim follows by applying Theorem~\ref{thm:MV-wellposed}, i.e., uniqueness of the MV solution curve, with initial law $\mu_0=\mu_m^{\lambda,\star}$.
\end{proof}

\subsubsection{EVI Gradient Flow}
\label{sec:evi}
\begin{definition}[EVI$_\alpha$ gradient flow]\label{def:EVI}
Let $\alpha\in\mathbb R$ and let $\mathcal F:\mathcal P_2(\Xi)\to(-\infty,+\infty]$ be a proper functional. A locally absolutely continuous curve $(\mu_t)_{t\ge 0}\subset\mathcal P_2(\Xi)$ is called an
$\mathrm{EVI}_\alpha$ (evolution variational inequality) gradient flow of $\mathcal F$ starting from $\mu_0$ if for every $\nu\in\mathcal P_2(\Xi)$ with $\mathcal F(\nu)<\infty$ and for a.e.\ $t>0$,
\begin{equation}
\frac{1}{2}\frac{\mathrm d}{\mathrm dt}\mathsf W_2^2(\mu_t,\nu)
+\frac{\alpha}{2}\mathsf W_2^2(\mu_t,\nu)
\le \mathcal F(\nu)-\mathcal F(\mu_t).
\label{eq:EVI-def}
\end{equation}
\end{definition}

\begin{theorem}[EVI$_\alpha$ gradient flow and sharp contractivity]\label{thm:EVI}
Suppose that Assumptions~\ref{ass:V}, \ref{ass:G}, \ref{ass:r}, and \ref{ass:strong} hold. Let $\alpha=\lambda\kappa-\eta L_r - mL_G>0$. Then $\mathcal F_m^{\lambda,\mathrm{rep}}$ admits a unique $\mathrm{EVI}_\alpha$ Wasserstein gradient flow $(\mu_t)_{t\ge0}$ from every $\mu_0\in\mathcal P_2(\Xi)$. Moreover, for any two $\mathrm{EVI}_\alpha$ solutions $(\mu_t)$ and $(\tilde\mu_t)$,
\begin{equation}
\mathsf W_2(\mu_t,\tilde\mu_t)\le e^{-\alpha t} \mathsf W_2(\mu_0,\tilde\mu_0),
\qquad t\ge 0.
\label{eq:EVI-contract}
\end{equation}
In particular, letting $\mu_m^{\lambda,\star}$ denote the unique minimiser of $\mathcal F_m^{\lambda,\mathrm{rep}}$,
\begin{equation}
\mathsf W_2(\mu_t,\mu_m^{\lambda,\star})\le e^{-\alpha t} \mathsf W_2(\mu_0,\mu_m^{\lambda,\star}),
\qquad t\ge 0,
\label{eq:EVI-to-min}
\end{equation}
and the free energy converges exponentially:
\begin{equation}
\mathcal F_m^{\lambda,\mathrm{rep}}(\mu_t)-\mathcal F_m^{\lambda,\mathrm{rep}}(\mu_m^{\lambda,\star})
\le e^{-2\alpha t}\Big(\mathcal F_m^{\lambda,\mathrm{rep}}(\mu_0)-\mathcal F_m^{\lambda,\mathrm{rep}}(\mu_m^{\lambda,\star})\Big),
\qquad t\ge 0.
\label{eq:EVI-energy}
\end{equation}
\end{theorem}

\begin{proof}
By Theorems~\ref{thm:exist} and \ref{thm:uniq}, $\mathcal F_m^{\lambda,\mathrm{rep}}$ is proper, lower semicontinuous, and $\alpha$-geodesically convex on $(\mathcal P_2(\Xi),\mathsf W_2)$. Existence and uniqueness of the $\mathrm{EVI}_\alpha$ flow, together with the contractivity estimate \eqref{eq:EVI-contract}, are standard consequences of the general theory in \citet[Thms.~11.1.4, 11.2.1, 11.2.4]{ambrosio2008gradientflows}. The convergence bounds \eqref{eq:EVI-to-min}--\eqref{eq:EVI-energy} follow by taking $\tilde\mu_t\equiv\mu_m^{\lambda,\star}$
and using standard $\alpha$-convexity consequences.
\end{proof}

\begin{remark}[On identification of MV dynamics and EVI gradient flows]
\label{rem:MV-vs-EVI-revised}
Equation \eqref{eq:MV-PDE} is formally the Wasserstein gradient-flow associated with $\mathcal F_m^{\lambda,\mathrm{rep}}$
(cf.\ Proposition~\ref{prop:EDI-smooth}). In the $\alpha$-convex regime, the $\mathrm{EVI}_\alpha$ gradient flow is unique. A full identification of the McKean--Vlasov law $(\mu_t)$ from Theorem~\ref{thm:MV-wellposed} with the $\mathrm{EVI}_\alpha$ flow requires an argument showing that $(\mu_t)$ is a curve of maximal slope for $\mathcal F_m^{\lambda,\mathrm{rep}}$ (e.g.\ via an energy dissipation inequality). In this appendix, whenever we require \emph{quantitative} convergence of the \emph{McKean--Vlasov} law to equilibrium, we use instead
a direct synchronous-coupling contractivity estimate (Lemma~\ref{lem:MV-contractivity}), which holds under the dissipative condition
$\delta=\alpha-L_\mu>0$.
\end{remark}

\subsubsection{Propagation of Chaos and Empirical Measure Error}
\label{sec:poc}
We first recall the definition of the interacting particle system (IPS) from Section~\ref{sec:method:ips}. For $i\in[N]$, let $(\xi_t^{i,N})_{t\ge 0}$ solve
\begin{equation}
\mathrm d\xi_t^{i,N}=b(\xi_t^{i,N},\mu_t^N) \mathrm dt+\sqrt{2\lambda} \mathrm dw_t^i,
\qquad
\xi_0^{i,N}\sim \mu_0\ \text{i.i.d.}
\label{eq:IPS-recall}
\end{equation}
where $\mu_t^N:=\frac1N\sum_{j=1}^N\delta_{\xi_t^{j,N}}$. In addition, for each $i\in[N]$, let $(\bar \xi_t^i)_{t\ge 0}$ be an i.i.d.\ family of \emph{nonlinear copies} with common law $\mu_t$, defined as the unique strong solutions to
\begin{equation}
\mathrm d\bar \xi_t^{i}=b(\bar \xi_t^{i},\mu_t) \mathrm dt+\sqrt{2\lambda} \mathrm dw_t^i,
\qquad \bar \xi_0^i\sim \mu_0,\ \text{i.i.d.}.
\label{eq:nonlinear-copies}
\end{equation}
and define $\bar\mu_t^N:=\frac1N\sum_{i=1}^N\delta_{\bar \xi_t^i}$.  

\begin{lemma}[Propagation of moments for the McKean--Vlasov solution]\label{lem:MV-moment-q}
Suppose that Assumptions~\ref{ass:V}, \ref{ass:G}, and \ref{ass:r} hold. Let $(\mu_t)_{t\ge 0}$ be the law of the unique strong solution to \eqref{eq:MV-SDE}. If $M_q(\mu_0)<\infty$ for some $q>2$, then for every $T>0$,
\begin{equation}
\sup_{t\in[0,T]} M_q(\mu_t) < \infty.
\label{eq:MV-moment-q}
\end{equation}
\end{lemma}

\begin{proof}
This is standard for SDEs with globally Lipschitz drift and at most linear growth: apply It\^o's formula to $\|\xi_t\|^q$, use the linear-growth bounds implied by Assumptions~\ref{ass:V}--\ref{ass:G}, and conclude by Gr\"onwall. See, e.g., \citet[Sec.~I.2]{sznitman1991topics} or \citet{carmona2018probabilisticI,carmona2018probabilisticII}.
\end{proof}

\begin{theorem}[Propagation of chaos in $\mathsf{W}_2$ (finite horizon)]\label{thm:PoC}
Suppose that Assumptions~\ref{ass:V}, \ref{ass:G}, and \ref{ass:r} hold. Fix $T>0$ and assume $M_q(\mu_0)<\infty$ for some $q>2$. Then there exists $C_T<\infty$ such that
\begin{equation}
\sup_{t\in[0,T]}\mathbb{E}\big[\mathsf{W}_2^2(\mu_t^N,\mu_t)\big]
\le C_T \beta_{d,q}(N),
\label{eq:PoC-W2}
\end{equation}
where $\beta_{d,q}(N)$ is as in \citet[Theorem~1]{fournier2015rate}.
\end{theorem}

\begin{proof}
Our proof follows a standard synchronous-coupling argument; see, e.g., \citet{sznitman1991topics}. Couple the IPS \eqref{eq:IPS-recall} with the nonlinear copies \eqref{eq:nonlinear-copies} using the same Brownian motions and the same initial data. In addition, set $\Delta_t^i:=\xi_t^{i,N}-\bar\xi_t^i$. By It\^o's formula and Lemma~\ref{lem:drift-Lip},
\begin{equation}
\frac{d}{dt}\mathbb E\|\Delta_t^i\|^2
\le C \mathbb E\|\Delta_t^i\|^2
+C \mathbb E\big[\mathsf W_2^2(\mu_t^N,\mu_t)\big].
\end{equation}
Let $r_t:=\mathbb E\|\Delta_t^1\|^2$. Exchangeability and the coupling
$\frac1N\sum_{i=1}^N\delta_{(\xi_t^{i,N},\bar\xi_t^i)}$ give $\mathbb E\big[\mathsf W_2^2(\mu_t^N,\bar\mu_t^N)\big]\le r_t$. Hence, by the triangle inequality, we have that
\begin{equation}
\mathbb E\big[\mathsf W_2^2(\mu_t^N,\mu_t)\big]
\le 2r_t+2\mathbb E\big[\mathsf W_2^2(\bar\mu_t^N,\mu_t)\big].
\end{equation}
Conditionally on $\mu_t$, the particles $(\bar\xi_t^i)$ are i.i.d. with law $\mu_t$. Therefore, using
\citet[Theorem~1]{fournier2015rate} and Lemma~\ref{lem:MV-moment-q}, we have that
\begin{equation}
\sup_{t\in[0,T]}\mathbb E\big[\mathsf W_2^2(\bar\mu_t^N,\mu_t)\big]
\le C_T \beta_{d,q}(N).
\end{equation}
It follows, in particular, that $r_t$ satisfies
\begin{equation}
r_t'\le Cr_t+C_T\beta_{d,q}(N),\qquad r_0=0,
\end{equation}
Thus, via Gr\"onwall's inequality, we have $\sup_{t\le T}r_t\le C_T\beta_{d,q}(N)$. Substituting this back into the previous triangle bound proves \eqref{eq:PoC-W2}.
\end{proof}

\begin{lemma}[Uniform propagation of moments for the McKean--Vlasov solution]
\label{lem:MV-moment-q-uniform}
Suppose that Assumptions~\ref{ass:V}, \ref{ass:G}, \ref{ass:r}, and \ref{ass:strong} hold. Suppose moreover that
\begin{equation}
\delta:=\alpha-L_\mu>0,
\qquad\text{where}\qquad
\alpha=\lambda\kappa-\eta L_r - mL_G,\quad L_\mu \text{ is as in Lemma~\ref{lem:drift-Lip}},
\label{eq:delta-moment}
\end{equation}
Let $(\mu_t)_{t\ge0}$ be the law of the unique strong solution to \eqref{eq:MV-SDE}. If $M_q(\mu_0)<\infty$ for some $q>2$, then
\begin{equation}
\sup_{t\ge0} M_q(\mu_t) < \infty.
\label{eq:MV-moment-q-uniform}
\end{equation}
\end{lemma}

\begin{proof}
Let $(\xi_t)_{t\ge0}$ denote a solution to the McKean-Vlasov SDE \eqref{eq:MV-SDE}, with $\mathrm{Law}(\xi_t)=\mu_t$. Applying It\^o's formula to $f(x)=\|x\|^q$, we have
\begin{equation}
\frac{\mathrm d}{\mathrm dt}\mathbb{E}\big[\|\xi_t\|^q\big]
=
q \mathbb E \left[\|\xi_t\|^{q-2} \langle \xi_t, b(\xi_t,\mu_t)\rangle\right]
+\lambda q(d+q-2) \mathbb E\big[\|\xi_t\|^{q-2}\big].
\label{eq:ito-moment}
\end{equation}
We begin by bounding the drift term. By Lemma~\ref{lem:dissip} (i.e., the dissipativity of the drift), with $\xi'=0$, we have that
\begin{equation}
\langle \xi_t, b(\xi_t,\mu_t)\rangle
=
\langle \xi_t, b(\xi_t,\mu_t)-b(0,\mu_t)\rangle+\langle \xi_t, b(0,\mu_t)\rangle
\le -\alpha\|\xi_t\|^2+\|\xi_t\| \|b(0,\mu_t)\|.
\label{eq:drift-split}
\end{equation}
Meanwhile, by Lemma~\ref{lem:drift-Lip}, we have $\|b(0,\mu)-b(0,\delta_0)\| \le L_\mu \mathsf W_2(\mu,\delta_0) = L_\mu M_2(\mu)^{\frac{1}{2}}$. Thus, with the constant $B_0:=\|b(0,\delta_0)\|<\infty$, we have $\|b(0,\mu_t)\|\le B_0+L_\mu M_2(\mu_t)^{1/2}$. Substituting this into \eqref{eq:drift-split}, multiplying  by $\|\xi_t\|^{q-2}$, and taking expectations, yields
\begin{align}
\mathbb E \left[\|\xi_t\|^{q-2} \langle \xi_t, b(\xi_t,\mu_t)\rangle\right]
&\le
-\alpha \mathbb E\|\xi_t\|^q
+\big(B_0+L_\mu M_2(\mu_t)^{1/2}\big) \mathbb E\|\xi_t\|^{q-1}.
\label{eq:drift-moment-bound}
\end{align}
By H\"older's inequality, $\mathbb E\|\xi_t\|^{q-1}\le [\mathbb{E}\big[\|\xi_t\|^q\big]]^{(q-1)/q}$, $\mathbb E\|\xi_t\|^{q-2}\le [\mathbb{E}\big[\|\xi_t\|^q\big]]^{(q-2)/q}$, and $M_2(\mu_t)^{1/2}\le \mathbb{E}\big[\|\xi_t\|^q\big]^{1/q}$. Substituting these bounds into \eqref{eq:drift-moment-bound}, and then \eqref{eq:drift-moment-bound} into \eqref{eq:ito-moment}, yields
\begin{equation}
\frac{\mathrm{d}}{\mathrm{d}t}\mathbb{E}\big[\|\xi_t\|^q\big]
\le
-q(\alpha-L_\mu) \mathbb{E}\big[\|\xi_t\|^q\big]
+ qB_0 \mathbb{E}\big[\|\xi_t\|^q\big]^{\frac{q-1}{q}}
+ \lambda q(d+q-2) \mathbb{E}\big[\|\xi_t\|^q\big]^{\frac{q-2}{q}}.
\label{eq:yt-ode}
\end{equation}
We now absorb the sublinear terms, recalling that $\delta=\alpha-L_\mu>0$. For any $\varepsilon>0$ and any $r\in(0,1)$, there exists $C_{\varepsilon,r}<\infty$ such that $z^r\le \varepsilon z + C_{\varepsilon,r}$ for all $z\ge0$. Applying this with $r_1=(q-1)/q$ and $r_2=(q-2)/q$, and choosing $\varepsilon>0$ small enough that $\smash{q\varepsilon(B_0+\lambda(d+q-2))\le {q\delta}/{2}}$,  we obtain from \eqref{eq:yt-ode} an inequality of the form
\begin{equation}
\frac{\mathrm{d}}{\mathrm{d}t}\mathbb{E}\big[\|\xi_t\|^q\big] \le -\frac{q\delta}{2} \mathbb{E}\big[\|\xi_t\|^q\big] + C_q,
\label{eq:yt-linear}
\end{equation}
for a constant $C_q<\infty$ depending only on $q,d,\lambda,\delta,B_0$, i.e., only on the standing model constants. Finally, solving \eqref{eq:yt-linear} gives
\begin{equation}
\mathbb{E}\big[\|\xi_t\|^q\big] \le e^{-\frac{q\delta}{2}t}\mathbb{E}\big[\|\xi_0\|^q\big] + \frac{2C_q}{q\delta},
\qquad \forall t\ge0.
\end{equation}
This immediately implies that $\sup_{t\geq 0}M_{q}(\mu_t):=\sup_{t\geq 0}\mathbb{E}[\|\xi_t\|^q]<\infty$, and thus completes the proof.
\end{proof}

\begin{theorem}[Propagation of chaos in $\mathsf{W}_2$ (uniform in time)]
\label{thm:PoC-uniform}
Suppose that Assumptions~\ref{ass:V}, \ref{ass:G}, \ref{ass:r}, and \ref{ass:strong} hold. Assume $M_q(\mu_0)<\infty$ for some $q>2$. Suppose moreover that
\begin{equation}
\delta:=\alpha-L_\mu>0,
\qquad\text{where}\qquad
\alpha=\lambda\kappa-\eta L_r - mL_G,\quad L_\mu \text{ is as in Lemma~\ref{lem:drift-Lip}},
\label{eq:delta}
\end{equation}
Then there exists $C<\infty$ (independent of $t$ and $N$) such that
\begin{equation}
\sup_{t\ge 0}\mathbb E\big[\mathsf W_2^2(\mu_t^N,\mu_t)\big]
\le C \beta_{d,q}(N),
\label{eq:PoC-uniform-W2}
\end{equation}
where $\beta_{d,q}(N)$ is as in \citet[Theorem~1]{fournier2015rate}.
\end{theorem}

\begin{proof}
Once again, we follow a standard synchronous coupling argument, now suitable for the dissipative regime; see \citet{malrieu2001logarithmic}. Consider the same synchronous coupling as in Theorem~\ref{thm:PoC}, and set $r_t:=\mathbb E\|\Delta_t^1\|^2$. Due to Lemma~\ref{lem:dissip}, we have that
\begin{equation}
r_t'\le -2\alpha r_t+2L_\mu \mathbb E\big[\|\Delta_t^1\| \mathsf W_2(\mu_t^N,\mu_t)\big].
\end{equation}
By the triangle inequality, we have that $\mathsf W_2(\mu_t^N,\mu_t)\le \mathsf W_2(\mu_t^N,\bar\mu_t^N)+\mathsf W_2(\bar\mu_t^N,\mu_t)$. For the first term, as in the proof of Theorem~\ref{thm:PoC},
\begin{equation}
\mathbb E\big[\|\Delta_t^1\| \mathsf W_2(\mu_t^N,\bar\mu_t^N)\big]\le r_t.
\end{equation}
Thus, due to Cauchy--Schwarz, we have $\mathbb E[\|\Delta_t^1\| \mathsf W_2(\bar\mu_t^N,\mu_t)]
\le r_t^{1/2}(\mathbb E\mathsf W_2^2(\bar\mu_t^N,\mu_t))^{1/2}$. Substituting this into the previous differential inequality, and using Young's inequality, we have that
\begin{equation}
r_t'\le -\delta r_t+\frac{L_\mu^2}{\delta} \mathbb E\mathsf W_2^2(\bar\mu_t^N,\mu_t),
\qquad \delta=\alpha-L_\mu>0.
\end{equation}
By \citet[Theorem~1]{fournier2015rate} and Lemma~\ref{lem:MV-moment-q-uniform}, we have  $\sup_{t\ge0}\mathbb E\mathsf W_2^2(\bar\mu_t^N,\mu_t)\le C \beta_{d,q}(N)$. Using this and the fact that $r_0=0$, Gr\"onwall's inequality then yields $\sup_{t\ge0}r_t\le C\beta_{d,q}(N)$. Finally, applying the triangle inequality, we have \eqref{eq:PoC-uniform-W2}.
\end{proof}

\subsubsection{Time Discretisation}
\label{sec:Euler}
Fix $h>0$ and $t_n:=nh$. Recall, from Section~\ref{sec:method:ips}, the Euler--Maruyama discretisation of the continuous-time dynamics in \eqref{eq:IPS-recall}:
\begin{equation}
\xi_{n+1}^{i,N,h}
=
\xi_n^{i,N,h}
+h b(\xi_n^{i,N,h},\mu_n^{N,h})
+\sqrt{2\lambda h} Z_{n+1}^i,
\qquad
\mu_n^{N,h}:=\frac1N\sum_{j=1}^N\delta_{\xi_n^{j,N,h}}, \qquad (Z_n^i)_{n\ge1}^{i\in[N]} \stackrel{\text{i.i.d.}}{\sim} \mathcal N(0,I_d).
\label{eq:Euler-IPS}
\end{equation}
For notational consistency with the continuous-time laws, we will occasionally also write $\mu_{t_n}^{N,h}:=\mu_n^{N,h}$.

\begin{theorem}[Strong Euler error for the IPS at grid times (finite horizon)]\label{thm:Euler}
Suppose that Assumptions~\ref{ass:V}, \ref{ass:G}, and \ref{ass:r} hold. Suppose also that $M_2(\mu_0)<\infty$. Fix $T>0$ and let $n_T=\lfloor T/h\rfloor$. Let \eqref{eq:IPS-recall} and \eqref{eq:Euler-IPS} be coupled by taking $\smash{Z_{n+1}^i=(w_{t_{n+1}}^i-w_{t_n}^i)/\sqrt{h}}$ and $\smash{\xi_0^{i,N,h}=\xi_0^{i,N}}$. Then there exists $C_T<\infty$ (independent of $N$ and $h$) such that
\begin{equation}
\max_{0\le n\le n_T}\ \mathbb{E}\Big[\frac1N\sum_{i=1}^N\|\xi_{t_n}^{i,N}-\xi_n^{i,N,h}\|^2\Big]
\le C_T h.
\label{eq:Euler-avg}
\end{equation}
Consequently, for all $0\le n\le n_T$,
\begin{equation}
\mathbb{E}\big[\mathsf{W}_2^2(\mu_{t_n}^N,\mu_n^{N,h})\big]
\le \mathbb{E}\Big[\frac1N\sum_{i=1}^N\|\xi_{t_n}^{i,N}-\xi_n^{i,N,h}\|^2\Big]
\le C_T h.
\label{eq:Euler-W2}
\end{equation}
\end{theorem}

\begin{proof}
Let $\xi_t^N:=(\xi_t^{1,N},\dots,\xi_t^{N,N})\in\mathbb R^{dN}$ and $\xi_n^{N,h}:=(\xi_n^{1,N,h},\dots,\xi_n^{N,N,h})\in\mathbb R^{dN}$. Then \eqref{eq:IPS-recall} can be written as a $dN$-dimensional SDE
\begin{equation}
\mathrm d\xi_t^N = B_N(\xi_t^N) \mathrm dt + \sqrt{2\lambda} \mathrm dw_t^N,
\label{eq:stacked-SDE}
\end{equation}
where $w_t^N:=(w_t^1,\dots,w_t^N)$ and $B_N(\xi_1,\dots,\xi_N) := \big(b(\xi_i,\tfrac1N\sum_{j=1}^N\delta_{\xi_j})\big)_{i=1}^N$. Similarly, one can view \eqref{eq:Euler-IPS} as the Euler--Maruyama scheme for \eqref{eq:stacked-SDE}. By Lemma~\ref{lem:drift-Lip} and the coupling inequality
\begin{equation}
\mathsf{W}_2^2 \Big(\tfrac1N\sum_{j=1}^N\delta_{x_j},\tfrac1N\sum_{j=1}^N\delta_{y_j}\Big)
\le \tfrac1N\sum_{j=1}^N\|x_j-y_j\|^2,
\label{eq:empirical-W2-coupling}
\end{equation}
the map $B_N$ is globally Lipschitz on $\mathbb R^{dN}$ with a Lipschitz constant depending only on the constants $L_\xi,L_\mu$ from Lemma~\ref{lem:drift-Lip}. In particular, these constants are independent of $N$. Indeed, writing $\mu_x:=\tfrac1N\sum_{j=1}^N\delta_{x_j}$ and $\mu_y:=\tfrac1N\sum_{j=1}^N\delta_{y_j}$, we have
\begin{align}
\|B_N(x)-B_N(y)\|^2
&=\sum_{i=1}^N \|b(x_i,\mu_x)-b(y_i,\mu_y)\|^2 \\
&\le 2L_\xi^2\sum_{i=1}^N\|x_i-y_i\|^2 + 2N L_\mu^2 \mathsf W_2^2(\mu_x,\mu_y)
\le 2(L_\xi^2+L_\mu^2) \|x-y\|^2,
\end{align}
where the last inequality uses \eqref{eq:empirical-W2-coupling}. Therefore the standard strong mean-square Euler--Maruyama estimate for globally Lipschitz SDEs applies to \eqref{eq:stacked-SDE}; see, e.g., \citet{mao2007sde,kloedenplaten1992,higham2002strong}. This yields
\begin{equation}
\max_{0\le n\le n_T}\ \mathbb{E}\big[\|\xi_{t_n}^N-\xi_n^{N,h}\|^2\big]\le C_T N h,
\end{equation}
for some $C_T<\infty$, independent of both $N$ and $h$. Dividing both sides of this bound by $N$ gives \eqref{eq:Euler-avg}. Finally, \eqref{eq:Euler-W2} follows by using the coupling $\frac1N\sum_{i=1}^N\delta_{(\xi_{t_n}^{i,N},\xi_n^{i,N,h})}$ between $\mu_{t_n}^N$ and $\mu_n^{N,h}$ in the definition of $\mathsf W_2$.
\end{proof}

\begin{remark}[On strong order for additive noise]
Since the diffusion coefficient in \eqref{eq:IPS-recall} is constant, the Euler--Maruyama scheme coincides with the Milstein scheme. In the standard setting with sufficiently smooth drift, this yields strong order $1$ at fixed times, i.e.\ mean-square error $O(h^2)$, whereas \eqref{eq:Euler-avg} gives a conservative $O(h)$ bound under global Lipschitzness \citep[e.g.,][]{kloedenplaten1992}.
\end{remark}

\begin{remark}[Uniform-in-time Euler discretisation in the contractive regime]
\label{rem:uniform-euler}
The finite-horizon estimate \eqref{eq:Euler-avg} is a direct consequence of global Lipschitzness of the drift and yields a constant $C_T$ that typically grows (at least exponentially) with $T$. In the contractive regime where $\delta:=\alpha-L_\mu>0$, under suitable conditions on the step size, it is often possible to upgrade this bound to a uniform-in-time discretisation estimate of the form
\begin{equation}
\sup_{n\ge 0}\ \mathbb{E}\Big[\frac1N\sum_{i=1}^N\|\xi_{t_n}^{i,N}-\xi_n^{i,N,h}\|^2\Big]\le C h.
\label{eq:Euler-uniform}
\end{equation}
Indeed, let $\xi,\xi'\in\mathbb R^{dN}$. Define $\Delta_i=\xi_i-\xi'_i$, and $\mu_{\xi}=\frac1N\sum_{j=1}^N\delta_{\xi_j}$, $\mu_{\xi'}=\frac1N\sum_{j=1}^N\delta_{\xi'_j}$. Then, using Lemma~\ref{lem:drift-Lip}, Lemma~\ref{lem:dissip}, together with the bound $\mathsf W_2^2(\mu_{\xi},\mu_{\xi'})\le \frac1N\sum_{j=1}^N\|\Delta_j\|^2$, one can show that
\begin{equation}
\label{eq:BN-contract}
\langle \xi-\xi', B_N(\xi)-B_N(\xi')\rangle
=
\sum_{i=1}^N \langle \Delta_i,  b(\xi_i,\mu_{\xi})-b(\xi'_i,\mu_{\xi'})\rangle
\le -\delta \|\xi-\xi'\|^2,
\end{equation}
In addition, as shown in the proof of Theorem~\ref{thm:Euler}, $B_N$ is globally Lipschitz with constant $\smash{L_B:=\sqrt{2(L_\xi^2+L_\mu^2)}}$ independent of $N$. Together, these two properties imply that the \emph{explicit Euler map} $\xi\mapsto \xi+hB_N(\xi)$ is contractive for sufficiently small steps:
\begin{equation}
\label{eq:euler-map-contract}
\|\xi+hB_N(\xi)-(\xi'+hB_N(\xi'))\|^2
\le \big(1-2\delta h+L_B^2 h^2\big) \|\xi-\xi'\|^2
\le (1-\delta h) \|\xi-\xi'\|^2,
\end{equation}
whenever $0<h\le \delta/L_B^2$. Combining \eqref{eq:euler-map-contract} with a one-step local truncation error estimate for Euler--Maruyama applied to \eqref{eq:stacked-SDE} yields a recursion of the form
\begin{equation}
e_{n+1}\le (1-\delta h) e_n + C h^2,\qquad
e_n:=\mathbb E\Big[\frac1N\sum_{i=1}^N\|\xi_{t_n}^{i,N}-\xi_n^{i,N,h}\|^2\Big],
\end{equation}
for some constant $C<\infty$ which depends only on $\lambda$, $d$, $L_B$ and suitable uniform-in-time moment bounds for the IPS $(\xi_t^N)_{t\ge0}$, which can be obtained in the contractive regime via the same arguments as used in the proof of Lemma~\ref{lem:MV-moment-q-uniform}. Solving this recursion gives the uniform-in-time strong error bound in \eqref{eq:Euler-uniform} with $C=O(\delta^{-1})$. A general framework making this principle explicit is developed in \citet{schuhsouttar2024conditions}.
\end{remark}

\subsubsection{Doubly-Stochastic Approximation}
\label{sec:stoch}
We now formalise the doubly-stochastic approximation used by the algorithm, allowing both the utility interaction and the repulsion drift to be estimated by Monte Carlo. Fix integers $K\in\mathbb N$, $K_{\mathrm{rep}}\in\mathbb N$, and $m\ge2$. For each step $n\ge0$ and particle $i\in[N]$, sample $K$ i.i.d.\ index tuples
\begin{equation}
(I_{2,n,k}^i,\dots,I_{m,n,k}^i)\ \sim\ \mathcal U([N]^{m-1}) \qquad (k=1,\dots,K),
\end{equation}
with replacement, and independently sample $K_{\mathrm{rep}}$ i.i.d.\ indices
\begin{equation}
J_{n,\ell}^i \sim \mathcal U([N]) \qquad (\ell=1,\dots,K_{\mathrm{rep}}),
\end{equation}
again with replacement. Let $(U_{n,k}^i)_{n\ge0,k\in[K]}^{i\in[N]}$ be i.i.d.\ auxiliary randomness, independent of $(Z_{n}^i)_{n\ge1,i\in[N]}$ and independent of the past. More precisely, if $\mathcal F_n$ denotes the $\sigma$-field generated by the particle system up to time $n$ and the Gaussian noises up to time $n$, viz
\begin{equation}
\mathcal F_n:=\sigma\big(\{\xi_\ell^{j,N,h,\mathrm{st}}:0\le \ell\le n,\ 1\le j\le N\},\ \{Z_\ell^j:1\le \ell\le n,\ 1\le j\le N\}\big),
\end{equation}
then for each $n$ the collections $\{(I_{2,n,k}^i,\dots,I_{m,n,k}^i), U_{n,k}^i\}_{i,k}$ and $\{J_{n,\ell}^i\}_{i,\ell}$ are independent of $\mathcal F_n$ and mutually independent. We assume access to an unbiased oracle $\smash{\widehat{\nabla_1G}(\xi_{1:m};U)}$, measurable in the auxiliary randomness $U$, with conditional second-moment control. In particular, for all $\xi_{1:m}\in(\Xi)^m$,
\begin{align}
\mathbb{E}\big[\widehat{\nabla_1G}(\xi_{1:m};U)\mid \xi_{1:m}\big]&=\nabla_1G(\xi_{1:m}) \label{eq:oracle-mean} \\
\mathbb{E}\big[\|\widehat{\nabla_1G}(\xi_{1:m};U)-\nabla_1G(\xi_{1:m})\|^2\mid \xi_{1:m}\big]
&\le \sigma_G^2\Big(1+\sum_{j=1}^m\|\xi_j\|^2\Big).
\label{eq:oracle-var}
\end{align}
We can now define the doubly stochastic drift estimators. In particular, given the particle positions at time $n$, for each particle $i$ we define
\begin{equation}
\widehat{\nabla_\xi\Phi}_m\big(\xi_n^{i,N,h,\mathrm{st}};\mu_n^{N,h,\mathrm{st}}\big)
:=
\frac1K\sum_{k=1}^K
\widehat{\nabla_1G}\Big(\xi_n^{i,N,h,\mathrm{st}},\xi_n^{I_{2,n,k}^i,N,h,\mathrm{st}},\dots,\xi_n^{I_{m,n,k}^i,N,h,\mathrm{st}}; U_{n,k}^i\Big),
\end{equation}
and
\begin{equation}
\widehat{\nabla_\xi\Psi}_r\big(\xi_n^{i,N,h,\mathrm{st}};\mu_n^{N,h,\mathrm{st}}\big)
:=
\frac1{K_{\mathrm{rep}}}\sum_{\ell=1}^{K_{\mathrm{rep}}}
\nabla r \Big(\xi_n^{i,N,h,\mathrm{st}}-\xi_n^{J_{n,\ell}^i,N,h,\mathrm{st}}\Big),
\label{eq:rep-est}
\end{equation}
where $\mu_n^{N,h,\mathrm{st}}:=\frac1N\sum_{j=1}^N\delta_{\xi_n^{j,N,h,\mathrm{st}}}$ denotes the empirical measure of the particles. The doubly-stochastic Euler scheme is then given by
\begin{equation}
\xi_{n+1}^{i,N,h,\mathrm{st}}
=
\xi_n^{i,N,h,\mathrm{st}}
+h\Big(
m \widehat{\nabla_\xi\Phi}_m(\xi_n^{i,N,h,\mathrm{st}};\mu_n^{N,h,\mathrm{st}})
-\eta \widehat{\nabla_\xi\Psi}_r(\xi_n^{i,N,h,\mathrm{st}};\mu_n^{N,h,\mathrm{st}})
-\lambda\nabla V(\xi_n^{i,N,h,\mathrm{st}})
\Big)
+\sqrt{2\lambda h} Z_{n+1}^i.
\label{eq:stoch-Euler}
\end{equation}

\begin{remark}[Relation to random batch methods and stochastic-gradient mean-field Langevin]
Sampling interaction tuples $(I_{2,n,k}^i,\dots,I_{m,n,k}^i)$ with replacement is a multi-body analogue of {random batch methods (RBM)}
for pairwise interacting particle systems, introduced in \citet{jinliliu2020rbm}. Likewise, the repulsion estimator \eqref{eq:rep-est} is the natural with-replacement Monte Carlo approximation of the empirical repulsion drift. Convergence of with-replacement variants (RBM-r) has been analysed recently in the pairwise setting; see, e.g., \citet{cai2024convergence}. The additional oracle noise in \eqref{eq:oracle-mean} - \eqref{eq:oracle-var}, together with the repulsion subsampling, yields a stochastic approximation of the full drift. Uniform-in-time convergence guarantees for related mean-field Langevin dynamics accounting simultaneously for finite-$N$ effects, time discretisation, and stochastic gradient noise have been developed in \citet{suzuki2023convergence}.
\end{remark}

\begin{lemma}[Unbiasedness and conditional variance of the interaction and repulsion estimators]
\label{lem:est}
Suppose that Assumption~\ref{ass:G}(ii) and Assumption~\ref{ass:r}(ii) hold. Suppose also that \eqref{eq:oracle-mean} and \eqref{eq:oracle-var} hold. Then, conditional on $\mathcal F_n$, the estimators satisfy
\begin{align}
\mathbb{E}\big[\widehat{\nabla_\xi\Phi}_m(\xi_n^{i,N,h,\mathrm{st}};\mu_n^{N,h,\mathrm{st}})\mid \mathcal F_n\big]
&=\nabla_\xi\Phi_m(\xi_n^{i,N,h,\mathrm{st}};\mu_n^{N,h,\mathrm{st}}),
\label{eq:est-bias}
\\
\mathbb{E}\Big[\big\|\widehat{\nabla_\xi\Phi}_m(\xi_n^{i,N,h,\mathrm{st}};\mu_n^{N,h,\mathrm{st}})
-\nabla_\xi\Phi_m(\xi_n^{i,N,h,\mathrm{st}};\mu_n^{N,h,\mathrm{st}})\big\|^2\ \Big|\ \mathcal F_n\Big]
&\le \frac{C_{\mathrm{var},G}}{K}\Big(1+\|\xi_n^{i,N,h,\mathrm{st}}\|^2+\frac1N\sum_{j=1}^N\|\xi_n^{j,N,h,\mathrm{st}}\|^2\Big),
\label{eq:est-var}
\\
\mathbb{E}\big[\widehat{\nabla_\xi\Psi}_r(\xi_n^{i,N,h,\mathrm{st}};\mu_n^{N,h,\mathrm{st}})\mid \mathcal F_n\big]
&=\nabla_\xi\Psi_r(\xi_n^{i,N,h,\mathrm{st}};\mu_n^{N,h,\mathrm{st}}),
\label{eq:est-bias-r}
\\
\mathbb{E}\Big[\big\|\widehat{\nabla_\xi\Psi}_r(\xi_n^{i,N,h,\mathrm{st}};\mu_n^{N,h,\mathrm{st}})
-\nabla_\xi\Psi_r(\xi_n^{i,N,h,\mathrm{st}};\mu_n^{N,h,\mathrm{st}})\big\|^2\ \Big|\ \mathcal F_n\Big]
&\le \frac{C_{\mathrm{var},r}}{K_{\mathrm{rep}}}\Big(1+\|\xi_n^{i,N,h,\mathrm{st}}\|^2+\frac1N\sum_{j=1}^N\|\xi_n^{j,N,h,\mathrm{st}}\|^2\Big),
\label{eq:est-var-r}
\end{align}
for constants $C_{\mathrm{var},G}<\infty$ and $C_{\mathrm{var},r}<\infty$ depending only on $(m,C_G,\sigma_G^2)$ and $(L_r,\|\nabla r(0)\|)$, respectively.
\end{lemma}

\begin{proof}
Conditional on $\mathcal F_n$, the interaction summands are i.i.d. Let $Y_k$ denote one such summand and let $Y:=\mathbb E[Y_1\mid\mathcal F_n]$. Then the unbiasedness result in \eqref{eq:est-bias} follows immediately from \eqref{eq:oracle-mean} and the uniform sampling of indices:
\begin{equation}
Y=
\int \nabla_1G(\xi_n^{i,N,h,\mathrm{st}},\xi_{2:m}) \mu_n^{\otimes(m-1)}(d\xi_{2:m})
=\nabla_\xi\Phi_m(\xi_n^{i,N,h,\mathrm{st}};\mu_n^{N,h,\mathrm{st}}).
\end{equation}
We next establish \eqref{eq:est-var}. In this case, the $Y_k$ are i.i.d., we have that
\begin{equation}
\mathbb E\Big[\Big\|\frac1K\sum_{k=1}^K(Y_k-Y)\Big\|^2\Bigm|\mathcal F_n\Big]
=\frac1K\mathrm{Var}(Y_1\mid\mathcal F_n)
\le \frac1K\mathbb E[\|Y_1\|^2\mid\mathcal F_n].
\end{equation}
Meanwhile, due to Assumption~\ref{ass:G}(ii) and \eqref{eq:oracle-var}, we can bound the RHS as
\begin{equation}
\mathbb E[\|Y_1\|^2\mid\mathcal F_n]
\le C_{\mathrm{var},G}\Bigl(1+\|\xi_n^{i,N,h,\mathrm{st}}\|^2
+\frac1N\sum_{j=1}^N\|\xi_n^{j,N,h,\mathrm{st}}\|^2\Bigr).
\end{equation}
This proves \eqref{eq:est-var}. The proof of \eqref{eq:est-bias-r}--\eqref{eq:est-var-r} is identical, now using $W_\ell:=\nabla r(\xi_n^{i,N,h,\mathrm{st}}-\xi_n^{J_{n,\ell}^i,N,h,\mathrm{st}})$, and the linear-growth bound $\|\nabla r(z)\|^2\le C_{\mathrm{var},r}(1+\|z\|^2)$ implied by Assumption~\ref{ass:r}(ii).
\end{proof}

\begin{lemma}[Second-moment stability of the doubly-stochastic Euler scheme (finite horizon)]
\label{lem:stoch-moment}
Suppose that Assumptions~\ref{ass:V}, \ref{ass:G}(ii), and \ref{ass:r}(ii) hold. In addition, suppose that \eqref{eq:oracle-mean}--\eqref{eq:oracle-var} hold. Suppose also that $M_2(\mu_0)<\infty$. Fix $T>0$ and let $n_T=\lfloor T/h\rfloor$. Then there exists $C_T<\infty$ (independent of $N,h,K,K_{\mathrm{rep}}$) such that
\begin{equation}
\max_{0\le n\le n_T}\ \mathbb E\Big[\frac1N\sum_{i=1}^N \|\xi_n^{i,N,h,\mathrm{st}}\|^2\Big]\le C_T\Big(1+M_2(\mu_0)\Big).
\end{equation}
\end{lemma}

\begin{proof}
Let us suppress $(N,h,\mathrm{st})$ from the notation. We can then write the update in the form
\begin{equation}
\xi_{n+1}^i=\xi_n^i+hD_n^i+\sqrt{2\lambda h} Z_{n+1}^i,
\end{equation}
where $D_n^i$ is the stochastic drift. Conditioning on $\mathcal F_n$ and using $\mathbb E[Z_{n+1}^i\mid\mathcal F_n]=0$, $\mathbb E\|Z_{n+1}^i\|^2=d$, and $\|a+b\|^2\le (1+h)\|a\|^2+(1+h^{-1})\|b\|^2$, we obtain
\begin{equation}
\mathbb E[\|\xi_{n+1}^i\|^2\mid\mathcal F_n]
\le (1+h)\|\xi_n^i\|^2+Ch \mathbb E[\|D_n^i\|^2\mid\mathcal F_n]+2\lambda hd.
\end{equation}
By Assumption~\ref{ass:V}(ii), Lemma~\ref{lem:est}, and the linear-growth bounds on $\nabla_1G$ and $\nabla r$, there exists $C<\infty$ such that
\begin{equation}
\mathbb E[\|D_n^i\|^2\mid\mathcal F_n]
\le C\Bigl(1+\|\xi_n^i\|^2+\frac1N\sum_{j=1}^N\|\xi_n^j\|^2\Bigr).
\end{equation}
Define  $A_n:=\mathbb E[\frac1N\sum_{i=1}^N\|\xi_n^i\|^2]$. Then, averaging the previous display over $i$, and taking expectations, we arrive at
\begin{equation}
A_{n+1}\le (1+Ch)A_n+Ch.
\end{equation}
Finally, a discrete Gr\"onwall argument for $n\le n_T$ gives the claimed result: 
\begin{equation}
\max_{0\le n\le n_T}A_n\le C_T(1+A_0)=C_T(1+M_2(\mu_0)).
\end{equation}
\end{proof}

\begin{theorem}[Error from the doubly-stochastic approximation (finite horizon)]
\label{thm:stoch}
Suppose that Assumptions~\ref{ass:V}, \ref{ass:G}, and \ref{ass:r} hold. In addition, suppose that the oracle conditions \eqref{eq:oracle-mean}--\eqref{eq:oracle-var} hold. Suppose also that $M_2(\mu_0)<\infty$. Fix $T>0$ and $n_T=\lfloor T/h\rfloor$. Let $(\xi_n^{i,N,h})$ be the IPS in \eqref{eq:Euler-IPS}, and $(\xi_n^{i,N,h,\mathrm{st}})$ the doubly-stochastic approximation in \eqref{eq:stoch-Euler}, coupled with the same Gaussians $(Z_{n}^i)$ and the same initial conditions. Then there exists $C_T<\infty$ (independent of $N,h,K,K_{\mathrm{rep}}$) such that
\begin{equation}
\max_{0\le n\le n_T}\ \mathbb{E}\Big[\frac1N\sum_{i=1}^N\|\xi_n^{i,N,h,\mathrm{st}}-\xi_n^{i,N,h}\|^2\Big]
\le C_T h\Big(\frac{1}{K}+\frac{1}{K_{\mathrm{rep}}}\Big).
\label{eq:stoch-bound}
\end{equation}
\end{theorem}

\begin{proof}
Write $\tilde\xi_n^i:=\xi_n^{i,N,h,\mathrm{st}}$, $\xi_n^i:=\xi_n^{i,N,h}$, $\tilde\mu_n:=\frac1N\sum_j\delta_{\tilde\xi_n^j}$, and $\mu_n:=\frac1N\sum_j\delta_{\xi_n^j}$. Set $\Delta_n^i:=\tilde\xi_n^i-\xi_n^i$. Define
\begin{equation}
\varepsilon_n^i
:=
\widehat{\nabla_\xi\Phi}_m(\tilde\xi_n^i;\tilde\mu_n)
-\nabla_\xi\Phi_m(\tilde\xi_n^i;\tilde\mu_n),
\qquad
\zeta_n^i
:=
\widehat{\nabla_\xi\Psi}_r(\tilde\xi_n^i;\tilde\mu_n)
-\nabla_\xi\Psi_r(\tilde\xi_n^i;\tilde\mu_n).
\end{equation}
Let $\mathcal G_n$ be the $\sigma$-field generated by $\{(\tilde\xi_\ell^j,\xi_\ell^j):0\le \ell\le n,\ 1\le j\le N\}$, the Gaussian variables $\{Z_\ell^j:1\le \ell\le n,\ 1\le j\le N\}$, and all interaction-index, repulsion-index, and oracle variables up to time $n-1$. Then $\tilde\xi_n,\xi_n,\tilde\mu_n,\mu_n$ are $\mathcal G_n$-measurable, while the new randomness at step $n$ is independent of $\mathcal G_n$. Then
\begin{equation}
\Delta_{n+1}^i
=\Delta_n^i+h\bigl(d_n^i+m\varepsilon_n^i-\eta\zeta_n^i\bigr),
\qquad d_n^i:=b(\tilde\xi_n^i,\tilde\mu_n)-b(\xi_n^i,\mu_n),
\end{equation}
where $\varepsilon_n^i$ and $\zeta_n^i$ are the centered interaction and repulsion estimator errors. Condition on $\mathcal G_n$. By Lemma~\ref{lem:est}, $\mathbb E[\varepsilon_n^i\mid\mathcal G_n]=\mathbb E[\zeta_n^i\mid\mathcal G_n]=0$ and
\begin{align*}
\mathbb E[\|\varepsilon_n^i\|^2\mid\mathcal G_n]
&\le \frac{C}{K}\Bigl(1+\|\tilde\xi_n^i\|^2+\frac1N\sum_j\|\tilde\xi_n^j\|^2\Bigr),\\
\mathbb E[\|\zeta_n^i\|^2\mid\mathcal G_n]
&\le \frac{C}{K_{\mathrm{rep}}}\Bigl(1+\|\tilde\xi_n^i\|^2+\frac1N\sum_j\|\tilde\xi_n^j\|^2\Bigr).
\end{align*}
Thus, the cross term with $m\varepsilon_n^i-\eta\zeta_n^i$ vanishes and
\begin{equation}
\mathbb E[\|\Delta_{n+1}^i\|^2\mid\mathcal G_n]
\le (1+Ch)\|\Delta_n^i\|^2+Ch\|d_n^i\|^2
+Ch^2\Bigl(\frac1K+\frac1{K_{\mathrm{rep}}}\Bigr)
\Bigl(1+\|\tilde\xi_n^i\|^2+\frac1N\sum_j\|\tilde\xi_n^j\|^2\Bigr).
\end{equation}
By Lemma~\ref{lem:drift-Lip} and the empirical coupling inequality, we have that
\begin{equation}
\|d_n^i\|^2
\le C\Bigl(\|\Delta_n^i\|^2+\frac1N\sum_{j=1}^N\|\Delta_n^j\|^2\Bigr).
\end{equation}
Let $e_n:=\mathbb E\Big[\frac1N\sum_{i=1}^N\|\Delta_n^i\|^2\Big]$. Averaging over $i$, taking expectations, and using Lemma~\ref{lem:stoch-moment}, we thus have that
\begin{equation}
e_{n+1}\le (1+Ch)e_n+C_T h^2\Bigl(\frac1K+\frac1{K_{\mathrm{rep}}}\Bigr),
\qquad e_0=0.
\end{equation}
Finally, a discrete version of Gr\"onwall's inequality yields \eqref{eq:stoch-bound}.
\end{proof}

\begin{remark}[Uniform-in-time doubly-stochastic approximation in the contractive regime]
\label{rem:uniform-stoch}
The finite-horizon bound in Theorem~\ref{thm:stoch} yields a constant $C_T$ that typically grows with $T$. In the contractive regime where $\delta:=\alpha-L_\mu>0$, under suitable conditions on the step size, it is often possible to upgrade \eqref{eq:stoch-bound} to a uniform-in-time estimate of the form
\begin{equation}
\sup_{n\ge 0}\ \mathbb{E}\Big[\frac1N\sum_{i=1}^N\|\xi_n^{i,N,h,\mathrm{st}}-\xi_n^{i,N,h}\|^2\Big]
\le C h\Big(\frac1K+\frac1{K_{\mathrm{rep}}}\Big).
\label{eq:stoch-uniform}
\end{equation}
A sufficient step-size condition is the same contractivity requirement as in Remark~\ref{rem:uniform-euler}. Writing the stacked drift $B_N$ as in \eqref{eq:stacked-SDE} and letting $L_B:=\sqrt{2(L_\xi^2+L_\mu^2)}$ (cf.~the proof of Theorem~\ref{thm:Euler}), the explicit Euler map $\xi\mapsto \xi+hB_N(\xi)$ is contractive whenever $0<h\le \delta/L_B^2$; see \eqref{eq:euler-map-contract}. Under synchronous coupling, the mean-square error then satisfies a recursion of the form
\begin{equation}
e_{n+1}\le (1-\delta h) e_n + C h^2\Big(\frac1K+\frac1{K_{\mathrm{rep}}}\Big),
\label{eq:stoch-recursion}
\end{equation}
where the additive term comes from the conditional variance bounds for both the utility and repulsion estimators in Lemma~\ref{lem:est},  together with suitable uniform-in-time second-moment bounds for the doubly-stochastic scheme (obtained via a discrete Lyapunov argument in the contractive/small-step regime, analogously to the finite-horizon stability Lemma~\ref{lem:stoch-moment}). Solving \eqref{eq:stoch-recursion} yields \eqref{eq:stoch-uniform} with $C=O(\delta^{-1})$. A general framework making this principle explicit for numerical discretisations, including mean-field systems and stochastic perturbations, is developed in \citet{schuhsouttar2024conditions}.
\end{remark}

\subsubsection{End-to-End Error Decompositions}
\label{sec:end2end}

\begin{theorem}[End-to-end bound at time horizon $T$]
\label{thm:end2end}
Suppose that Assumptions~\ref{ass:V}, \ref{ass:G}, and \ref{ass:r} hold. In addition, suppose that the oracle conditions \eqref{eq:oracle-mean}--\eqref{eq:oracle-var} hold. Suppose also that $M_q(\mu_0)<\infty$ for some $q>2$. Fix $T>0$ and let $n_T=\lfloor T/h\rfloor$ and $T_h:=t_{n_T}$. Let $\mu_{T_h}$ be the McKean--Vlasov law at time $T_h$, $\smash{\mu_{T_h}^{N}}$ the IPS empirical law at time $T_h$, $\smash{\mu_{T_h}^{N,h}}$ the Euler IPS empirical law at time $T_h$, and $\smash{\mu_{T_h}^{N,h,\mathrm{st}}}$ the empirical law of the doubly-stochastic scheme at time $T_h$. Then there exists $C_T<\infty$ (independent of $N,h,K,K_{\mathrm{rep}}$) such that
\begin{equation}
\mathbb{E}\Big[\mathsf{W}_2^2\big(\mu_{T_h}^{N,h,\mathrm{st}},\mu_{T_h}\big)\Big]
\le
C_T\Big(\beta_{d,q}(N)+h+\frac{h}{K}+\frac{h}{K_{\mathrm{rep}}}\Big),
\label{eq:end2end}
\end{equation}
where $\beta_{d,q}(N)$ is the empirical-measure rate from \citet[Theorem~1]{fournier2015rate} (with $p=2$). If, in addition, Assumption~\ref{ass:strong} holds and $\delta=\alpha-L_\mu>0$, then
\begin{equation}
\mathbb{E}\Big[\mathsf{W}_2^2\big(\mu_{T_h}^{N,h,\mathrm{st}},\mu_m^{\lambda,\star}\big)\Big]
\le
C_T\Big(\beta_{d,q}(N)+h+\frac{h}{K}+\frac{h}{K_{\mathrm{rep}}}\Big)
+2e^{-2\delta T_h} \mathsf W_2^2(\mu_0,\mu_m^{\lambda,\star}),
\qquad \delta=\alpha-L_\mu.
\label{eq:end2end-stationary-T-fixed}
\end{equation}
\end{theorem}

\begin{proof}
We begin by establishing the bound in \eqref{eq:end2end}. Using the triangle inequality, we can decompose $\mathsf{W}_2(\mu_{T_h}^{N,h,\mathrm{st}},\mu_{T_h})$ into three terms, viz
\begin{equation}
\mathsf{W}_2(\mu_{T_h}^{N,h,\mathrm{st}},\mu_{T_h})
\le \mathsf{W}_2(\mu_{T_h}^{N,h,\mathrm{st}},\mu_{T_h}^{N,h})
+\mathsf{W}_2(\mu_{T_h}^{N,h},\mu_{T_h}^N)
+\mathsf{W}_2(\mu_{T_h}^N,\mu_{T_h}).
\end{equation}
Thus, applying the standard algebraic inequality $(a+b+c)^2\le 3(a^2+b^2+c^2)$, and taking expectations, we have that
\begin{equation}
\mathbb E\big[\mathsf W_2^2(\mu_{T_h}^{N,h,\mathrm{st}},\mu_{T_h})\big]
\le 3\Big(
\mathbb E\big[\mathsf W_2^2(\mu_{T_h}^{N,h,\mathrm{st}},\mu_{T_h}^{N,h})\big]
+\mathbb E\big[\mathsf W_2^2(\mu_{T_h}^{N,h},\mu_{T_h}^N)\big]
+\mathbb E\big[\mathsf W_2^2(\mu_{T_h}^N,\mu_{T_h})\big]\Big).
\end{equation}
Each of the three terms can be bounded using our existing results. For the first term, we use the coupling $\frac1N\sum_{i=1}^N\delta_{(\xi_{n_T}^{i,N,h,\mathrm{st}},\xi_{n_T}^{i,N,h})}$ between $\mu_{T_h}^{N,h,\mathrm{st}}$ and $\mu_{T_h}^{N,h}$ to obtain
\begin{equation}
\mathsf W_2^2(\mu_{T_h}^{N,h,\mathrm{st}},\mu_{T_h}^{N,h})
\le \frac1N\sum_{i=1}^N\|\xi_{n_T}^{i,N,h,\mathrm{st}}-\xi_{n_T}^{i,N,h}\|^2,
\end{equation}
and then apply Theorem~\ref{thm:stoch}. Meanwhile, the second term is bounded by Theorem~\ref{thm:Euler} and the third by Theorem~\ref{thm:PoC}. This yields \eqref{eq:end2end}.

We now establish the stationary-target bound in \eqref{eq:end2end-stationary-T-fixed}. Using the triangle inequality, we have
\begin{equation}
\mathsf W_2(\mu_{T_h}^{N,h,\mathrm{st}},\mu_m^{\lambda,\star})
\le \mathsf W_2(\mu_{T_h}^{N,h,\mathrm{st}},\mu_{T_h})+\mathsf W_2(\mu_{T_h},\mu_m^{\lambda,\star}).
\end{equation}
Using $(a+b)^2\le 2a^2+2b^2$ and taking expectations yields
\begin{equation}
\mathbb{E}\big[\mathsf{W}_2^2\big(\mu_{T_h}^{N,h,\mathrm{st}},\mu_m^{\lambda,\star}\big)\big]
\le
2\mathbb E\big[\mathsf W_2^2(\mu_{T_h}^{N,h,\mathrm{st}},\mu_{T_h})\big]
+2 \mathbb{E}\big[\mathsf W_2^2(\mu_{T_h},\mu_m^{\lambda,\star})\big].
\end{equation}
We have just established a bound for the first term. Meanwhile, by Lemma~\ref{lem:MV-contractivity}, we have that $\mathsf W_2(\mu_{T_h},\mu_m^{\lambda,\star})\le e^{-\delta T_h}\mathsf W_2(\mu_0,\mu_m^{\lambda,\star})$. This completes the proof.
\end{proof}

\begin{theorem}[Stationary-accuracy bound in the contractive regime]
\label{thm:end2end-stationary}
Suppose that Assumptions~\ref{ass:V}, \ref{ass:G}, \ref{ass:r}, and \ref{ass:strong} hold. In addition, suppose that the oracle conditions \eqref{eq:oracle-mean}--\eqref{eq:oracle-var} hold. Suppose also that $M_q(\mu_0)<\infty$ for some $q>2$. Assume moreover that $\delta:=\alpha-L_\mu>0$, where $\alpha=\lambda\kappa-\eta L_r - mL_G$ and $L_\mu$ is as in Lemma~\ref{lem:drift-Lip}. Finally, assume that the following uniform-in-time discretisation bounds hold:
\begin{align}
\sup_{n\ge 0}\ \mathbb{E}\Big[\mathsf W_2^2(\mu_{t_n}^{N,h},\mu_{t_n}^{N})\Big] &\le C_{\mathrm{Eul}} h,
\\
\sup_{n\ge 0}\ \mathbb{E}\Big[\mathsf W_2^2(\mu_{t_n}^{N,h,\mathrm{st}},\mu_{t_n}^{N,h})\Big]
&\le C_{\mathrm{st}} h\Big(\frac1K+\frac1{K_{\mathrm{rep}}}\Big),
\label{eq:assume-unif}
\end{align}
for some constants $C_{\mathrm{Eul}},C_{\mathrm{st}}<\infty$ independent of $N,h,K,K_{\mathrm{rep}}$. Then there exists $C<\infty$ (independent of $n,N,h,K,K_{\mathrm{rep}}$) such that for all $n\ge 0$,
\begin{equation}
\mathbb{E}\Big[\mathsf{W}_2^2\big(\mu_{t_n}^{N,h,\mathrm{st}},\mu_m^{\lambda,\star}\big)\Big]
\le
C\Big(e^{-2\delta t_n}\mathsf{W}_2^2(\mu_0,\mu_m^{\lambda,\star})+\beta_{d,q}(N)+h+\frac{h}{K}+\frac{h}{K_{\mathrm{rep}}}\Big),
\label{eq:end2end-uniform-explicit}
\end{equation}
where $\beta_{d,q}(N)$ is as in \citet[Theorem~1]{fournier2015rate}. In particular,
\begin{equation}
\limsup_{n\to\infty}\ \mathbb{E}\Big[\mathsf{W}_2^2\big(\mu_{t_n}^{N,h,\mathrm{st}},\mu_m^{\lambda,\star}\big)\Big]
\le
C\Big(\beta_{d,q}(N)+h+\frac{h}{K}+\frac{h}{K_{\mathrm{rep}}}\Big).
\label{eq:uniform-lim-sup}
\end{equation}
\end{theorem}

\begin{proof}
Fix $n\ge0$. By the triangle inequality and the fact that $(a+b+c+d)^2\le 4(a^2+b^2+c^2+d^2)$, we have
\begin{align*}
\mathbb{E}\big[\mathsf W_2^2(\mu_{t_n}^{N,h,\mathrm{st}},\mu_m^{\lambda,\star})\big]
\le 4\Big(&
\mathbb{E}\big[\mathsf W_2^2(\mu_{t_n}^{N,h,\mathrm{st}},\mu_{t_n}^{N,h})\big]
+\mathbb{E}\big[\mathsf W_2^2(\mu_{t_n}^{N,h},\mu_{t_n}^{N})\big]
+\mathbb{E}\big[\mathsf W_2^2(\mu_{t_n}^{N},\mu_{t_n})\big]
+\mathsf W_2^2(\mu_{t_n},\mu_m^{\lambda,\star})
\Big).
\end{align*}
The first two terms are controlled by \eqref{eq:assume-unif}. The third term is controlled by Theorem~\ref{thm:PoC-uniform}. For the fourth term, Lemma~\ref{lem:MV-contractivity} (with $\tilde\mu_t\equiv\mu_m^{\lambda,\star}$) gives $\mathsf W_2(\mu_{t_n},\mu_m^{\lambda,\star})\le e^{-\delta t_n}\mathsf W_2(\mu_0,\mu_m^{\lambda,\star})$. Combining these bounds yields \eqref{eq:end2end-uniform-explicit}. Taking $\limsup_{n\to\infty}$ gives
\eqref{eq:uniform-lim-sup}.
\end{proof}

\subsubsection{From \texorpdfstring{$\mathsf W_2$}{} to the Utility}
\label{app:eig:utility}

The main results in the previous section provided bounds on the $\mathsf{W}_2$ distance between design measures (e.g., between algorithmic iterates and the target design law). To connect this to BOED performance, we now derive the corresponding bounds on the expected batch utility $\mathcal{J}_m(\mu)=\int G(\xi_{1:m}) \mu^{\otimes m}(\mathrm{d}\xi_{1:m})$.

\begin{proposition}[Lipschitz-type control of $\mathcal{J}_m$ by $\mathsf W_2$]\label{prop:app:Jm-W2}
Suppose that Assumption~\ref{ass:G}(ii) holds. Then for any $\mu,\nu\in\mathcal{P}_2(\Xi)$,
\begin{equation}
\big|\mathcal{J}_m(\mu)-\mathcal{J}_m(\nu)\big|
\le
C_m\Big(1+\sqrt{M_2(\mu)}+\sqrt{M_2(\nu)}\Big) \mathsf W_2(\mu,\nu),
\label{eq:app:Jm-W2}
\end{equation}
for a constant $C_m<\infty$ depending only on $(C_G,m)$.
\end{proposition}

\begin{proof}
Let $\gamma\in\Gamma(\mu,\nu)$ be optimal for $\mathsf W_2(\mu,\nu)$ and let $(X,Y)\sim\gamma$. Take i.i.d. copies $(X_j,Y_j)_{j=1}^m$. We then have that
\begin{equation}
\mathcal J_m(\mu)-\mathcal J_m(\nu)
=\mathbb E\big[G(X_{1:m})-G(Y_{1:m})\big].
\end{equation}
By the fundamental theorem of calculus along the segment joining $Y_{1:m}$ to $X_{1:m}$ and Assumption~\ref{ass:G}(ii), we can bound
\begin{equation}
|G(X_{1:m})-G(Y_{1:m})|
\le C_m\Bigl(1+\sum_{j=1}^m(\|X_j\|+\|Y_j\|)\Bigr) \|X_{1:m}-Y_{1:m}\|,
\end{equation}
for some constant $C_m$ depending only on $(C_G,m)$. Taking expectations and applying the Cauchy--Schwarz inequality then gives
\begin{equation}
|\mathcal J_m(\mu)-\mathcal J_m(\nu)|
\le C_m\Bigl(1+\sqrt{M_2(\mu)}+\sqrt{M_2(\nu)}\Bigr)
\sqrt{\mathbb E\|X_{1:m}-Y_{1:m}\|^2}.
\end{equation}
Finally, identify $\mathbb E\|X_{1:m}-Y_{1:m}\|^2=m \mathsf W_2^2(\mu,\nu)$, and absorbing the factor $\sqrt m$ into the constant $C_m$, we arrive at \eqref{eq:app:Jm-W2}.
\end{proof}

\begin{corollary}[Utility suboptimality from a $\mathsf W_2$ end-to-end bound]
\label{cor:app:utility-gap}
Suppose that Assumption~\ref{ass:G}(ii) holds. Then, if an algorithm produces a random design measure $\widehat{\mu}$ satisfying $\mathbb{E}[\mathsf W_2^2(\widehat{\mu},\mu_m^{\lambda,\star})]\le \varepsilon^2$ and $\mathbb{E}[M_2(\widehat{\mu})]\le M$ for some $M<\infty$, then
\begin{equation}
\label{eq:jm-mu-hat-mu-star-bound}
\mathbb{E}\big[|\mathcal{J}_m(\widehat{\mu})-\mathcal{J}_m(\mu_m^{\lambda,\star})|\big]
\le
C_m\Big(1+\sqrt{M}+\sqrt{M_2(\mu_m^{\lambda,\star})}\Big) \varepsilon.
\end{equation}
In particular, under the assumptions of Theorem~\ref{thm:end2end}, applying the above bound with $\widehat{\mu}=\mu_{T_h}^{N,h,\mathrm{st}}$ and target measure $\mu_{T_h}$ yields
\begin{equation}
\label{eq:jm-mu-th-n-mu-th-bound}
\mathbb{E}\big[|\mathcal{J}_m(\mu_{T_h}^{N,h,\mathrm{st}})-\mathcal{J}_m(\mu_{T_h})|\big]
\le
C_m\Big(1+\sqrt{M}+\sqrt{M_2(\mu_{T_h})}\Big) 
\sqrt{C_T\Big(\beta_{d,q}(N)+h+\tfrac{h}{K}+\tfrac{h}{K_{\mathrm{rep}}}\Big)},
\end{equation}
where $C_T$ is the constant given in \eqref{eq:end2end}, $M:=\sup_{0\le n\le n_T}\mathbb{E}[M_2(\mu_{t_n}^{N,h,\mathrm{st}})]<\infty$ (e.g., by Lemma~\ref{lem:stoch-moment}), and $\beta_{d,q}(N)$ is the empirical-measure rate from \citet[Theorem~1]{fournier2015rate} (with $p=2$). If, in addition, Assumption~\ref{ass:strong} holds and $\delta=\alpha-L_\mu>0$, then applying the above bound with $\smash{\widehat{\mu}=\mu_{T_h}^{N,h,\mathrm{st}}}$ and target measure $\mu_m^{\lambda,\star}$ gives
\begin{align}
\label{eq:jm-mu-th-n-mu-m-lambda-bound}
\mathbb{E}\big[|\mathcal{J}_m(\mu_{T_h}^{N,h,\mathrm{st}})-\mathcal{J}_m(\mu_m^{\lambda,\star})|\big]
&\le
C_m\left(1+\sqrt{M}+\sqrt{M_2(\mu_m^{\lambda,\star})}\right)
\\
&\qquad \times
\sqrt{C_T\Big(\beta_{d,q}(N)+h+\tfrac{h}{K}+\tfrac{h}{K_{\mathrm{rep}}}\Big)+2e^{-2\delta T_h} \mathsf W_2^2(\mu_0,\mu_m^{\lambda,\star})}.
\notag
\end{align}
\end{corollary}

\begin{proof}
From Proposition~\ref{prop:app:Jm-W2}, for each realisation of $\widehat{\mu}$, it holds that
\begin{equation}
|\mathcal{J}_m(\widehat{\mu})-\mathcal{J}_m(\mu_m^{\lambda,\star})|
\le
C_m\Big(1+\sqrt{M_2(\widehat{\mu})}+\sqrt{M_2(\mu_m^{\lambda,\star})}\Big) \mathsf W_2(\widehat{\mu},\mu_m^{\lambda,\star}).
\label{eq:jm-mu-hat-mu-star}
\end{equation} 
Taking expectations and using Jensen's inequality gives $\mathbb{E}[\mathsf W_2(\widehat{\mu},\mu_m^{\lambda,\star})]\le \sqrt{\mathbb{E}[\mathsf W_2^2(\widehat{\mu},\mu_m^{\lambda,\star})]}\le \varepsilon$.
Moreover, by the Cauchy--Schwarz inequality,
\begin{equation}
\mathbb{E} \left[\sqrt{M_2(\widehat{\mu})} \mathsf W_2(\widehat{\mu},\mu_m^{\lambda,\star})\right]
\le
\sqrt{\mathbb{E}[M_2(\widehat{\mu})]} \sqrt{\mathbb{E}[\mathsf W_2^2(\widehat{\mu},\mu_m^{\lambda,\star})]}
\le
\sqrt{M} \varepsilon.
\label{eq:m2-w2}
\end{equation}
Combining \eqref{eq:jm-mu-hat-mu-star} and \eqref{eq:m2-w2} yields the bound  in \eqref{eq:jm-mu-hat-mu-star-bound}. Finally, the bounds in \eqref{eq:jm-mu-th-n-mu-th-bound} and \eqref{eq:jm-mu-th-n-mu-m-lambda-bound} follow directly from \eqref{eq:jm-mu-hat-mu-star-bound}, and the bounds in Theorem~\ref{thm:end2end}.
\end{proof}

\subsubsection{Best-of-\texorpdfstring{$n$}{} extraction from a learned design law}
\label{sec:bon}

The distributional formulation yields a design law on $\Xi^m$, which may be used directly as a randomised policy, or converted into a deterministic batch via a best-of-$n$ (BoN) extraction step. The next results quantify this extraction procedure and connect it to the i.i.d.\ Wasserstein gradient flow developed above.

\begin{proposition}[Best-of-$n$ extraction from an arbitrary design law]
\label{prop:bon-general}
Let $\Xi\subseteq \mathbb R^d$ be Borel, let $G:\Xi^m\to\mathbb R$ be measurable and bounded above, and define $G^\star:=\sup_{\xi_{1:m}\in\Xi^m} G(\xi_{1:m})<\infty$. Let $\nu_m\in\mathcal P(\Xi^m)$, and draw 
$\smash{\xi_{1:m}^{(1)},\dots,\xi_{1:m}^{(n)}\stackrel{\mathrm{i.i.d.}}{\sim}\nu_m}$. Define the best-of-$n$ extractor by $\smash{I_n:=\min\argmax_{1\le r\le n} G(\xi_{1:m}^{(r)})}$ with $\smash{\widehat{\xi}_{1:m}^{(n)}:=\xi_{1:m}^{(I_n)}}$. Finally, for $\varepsilon>0$, define the $\varepsilon$-optimal set $\smash{A_\varepsilon
:=
\{\xi_{1:m}\in\Xi^m:\ G(\xi_{1:m})\ge G^\star-\varepsilon\}}$. Then, for every $\varepsilon>0$, it holds that 
\begin{equation}
\label{eq:bon-exact}
\mathbb P\big(G(\widehat{\xi}_{1:m}^{(n)})\ge G^\star-\varepsilon\big)
=
1-\bigl(1-\nu_m(A_\varepsilon)\bigr)^n.
\end{equation}
Thus, if $\nu_m(A_\varepsilon)>0$, then for all $\delta\in(0,1)$,
\begin{equation}
\label{eq:bon-n-sufficient}
n\ge \frac{\log(1/\delta)}{\nu_m(A_\varepsilon)}
\qquad\Longrightarrow\qquad
\mathbb P\big(G(\widehat{\xi}_{1:m}^{(n)})\ge G^\star-\varepsilon\big)\ge 1-\delta.
\end{equation}
Suppose, in addition, that $\gamma(\nu_m):= G^\star-\int_{\Xi^m} G(\xi_{1:m}) \nu_m(\mathrm d\xi_{1:m})<\infty$. Then
\begin{equation}
\label{eq:bon-gap-identity}
\mathbb E \left[G^\star-G(\widehat{\xi}_{1:m}^{(n)})\right]
=
\int_0^\infty \bigl(1-\nu_m(A_\varepsilon)\bigr)^n \mathrm d\varepsilon.
\end{equation}
Moreover, for every $\varepsilon>0$, it holds that $\nu_m(A_\varepsilon)\ge 1-\min\{1,\frac{\gamma(\nu_m)}{\varepsilon}\}$, and therefore
\begin{equation}
\label{eq:bon-gap-prob}
\mathbb P\big(G(\widehat{\xi}_{1:m}^{(n)})\ge G^\star-\varepsilon\big)
\ge
1-
\left(
\min \left\{1,\frac{\gamma(\nu_m)}{\varepsilon}\right\}
\right)^n.
\end{equation}
\end{proposition}

\begin{proof}
Since $G(\widehat\xi_{1:m}^{(n)})=\max_{1\le r\le n}G(\xi_{1:m}^{(r)})$, we have that $\{G(\widehat\xi_{1:m}^{(n)})<G^\star-\varepsilon\}
=\cap_{r=1}^n\{\xi_{1:m}^{(r)}\notin A_\varepsilon\}$. The samples are i.i.d., and thus
\begin{equation}
\mathbb P\big(G(\widehat\xi_{1:m}^{(n)})<G^\star-\varepsilon\big)
=(1-\nu_m(A_\varepsilon))^n,
\end{equation}
which is precisely \eqref{eq:bon-exact}. The sufficient condition in \eqref{eq:bon-n-sufficient} follows immediately from the standard inequality $1-u\le e^{-u}$.

For the expectation identity, let $\Delta_n:=G^\star-G(\widehat\xi_{1:m}^{(n)})\ge0$. Since $\Delta_n\le G^\star-G(\xi_{1:m}^{(1)})$ and the latter is integrable,
\begin{equation}
\mathbb E[\Delta_n]
=\int_0^\infty \mathbb P(\Delta_n>\varepsilon) d\varepsilon
=\int_0^\infty (1-\nu_m(A_\varepsilon))^n d\varepsilon,
\end{equation}
which is the result in \eqref{eq:bon-gap-identity}. Finally, let $Y:=G^\star-G(\xi_{1:m})$ and $\xi_{1:m}\sim\nu_m$. We then have $\mathbb E[Y]=\gamma(\nu_m)$ and, by Markov's inequality,
\begin{equation}
1-\nu_m(A_\varepsilon)=\mathbb P(Y>\varepsilon)
\le \min\Bigl\{1,\frac{\gamma(\nu_m)}{\varepsilon}\Bigr\}
\end{equation}
Substituting this into \eqref{eq:bon-exact} yields \eqref{eq:bon-gap-prob}.
\end{proof}

\begin{corollary}[Best-of-$n$ extraction with approximate utility scores]
\label{cor:bon-noisy}
Suppose that the assumptions of Proposition~\ref{prop:bon-general} hold. Let $\smash{\widehat G_1,\dots,\widehat G_n}$ be measurable scores, and define $\smash{\widetilde I_n:=\min\argmax_{1\le r\le n}\widehat G_r}$ with $\smash{\widetilde{\xi}_{1:m}^{(n)}:=\xi_{1:m}^{(\widetilde I_n)}}$. Suppose that for some $\eta>0$ and $\delta_{\mathrm{score}}\in[0,1]$, $\smash{\mathbb P(\max_{1\le r\le n}\big|\widehat G_r-G(\xi_{1:m}^{(r)})\big|\le \eta)\ge 1-\delta_{\mathrm{score}}}$. Then, for every $\varepsilon>0$, it holds that
\begin{equation}
\label{eq:bon-noisy-main}
\mathbb P \left(
G(\widetilde{\xi}_{1:m}^{(n)})\ge G^\star-\varepsilon-2\eta
\right)
\ge
1-\delta_{\mathrm{score}}-\bigl(1-\nu_m(A_\varepsilon)\bigr)^n.
\end{equation}
\end{corollary}

\begin{proof}
Define the events $\smash{
E_\eta
:=
\{
\max_{1\le r\le n}\big|\widehat G_r-G(\xi_{1:m}^{(r)})\big|\le \eta
\}}$ and $\smash{
S_\varepsilon
:=
\{
\max_{1\le r\le n} G(\xi_{1:m}^{(r)})\ge G^\star-\varepsilon
\}}$. By \eqref{eq:bon-exact}, we have
\begin{equation}
\mathbb P(S_\varepsilon)=1-\bigl(1-\nu_m(A_\varepsilon)\bigr)^n.
\end{equation}
Suppose that both $E_\eta$ and $S_\varepsilon$ occur. Then there exists $\smash{r_\varepsilon\in\{1,\dots,n\}}$ such that $\smash{G(\xi_{1:m}^{(r_\varepsilon)})\ge G^\star-\varepsilon}$. Since $\smash{\widetilde I_n}$ maximises the approximate scores, we have $\smash{\widehat G_{\widetilde I_n}\ge \widehat G_{r_\varepsilon}}$. On the event $E_\eta$, this implies that
\begin{equation}
\widehat G_{\widetilde I_n}
\ge
\widehat G_{r_\varepsilon}
\ge
G(\xi_{1:m}^{(r_\varepsilon)})-\eta
\ge
G^\star-\varepsilon-\eta.
\end{equation}
It follows, once more conditioning on the event $E_\eta$, that 
\begin{equation}
G(\widetilde{\xi}_{1:m}^{(n)})
\ge
\widehat G_{\widetilde I_n}-\eta
\ge
G^\star-\varepsilon-2\eta.
\end{equation}
We thus have that $\smash{E_\eta\cap S_\varepsilon
\subseteq \{ G(\widetilde{\xi}_{1:m}^{(n)})\ge G^\star-\varepsilon-2\eta\}}$. Taking probabilities and using the union bound, we arrive at
\begin{equation}
\mathbb P\big(
G(\widetilde{\xi}_{1:m}^{(n)})\ge G^\star-\varepsilon-2\eta
\big)
\ge
\mathbb P(E_\eta\cap S_\varepsilon)
\ge
1-\mathbb P(E_\eta^c)-\mathbb P(S_\varepsilon^c),
\end{equation}
which is exactly \eqref{eq:bon-noisy-main}.
\end{proof}

\begin{proposition}[Best-of-$n$ extraction along the i.i.d.\ Wasserstein gradient flow]
\label{prop:bon-wgf}
Suppose that Assumptions~\ref{ass:V}, \ref{ass:G}, \ref{ass:r}, and \ref{ass:strong} hold. Let $\smash{(\mu_t)_{t\ge 0}}$ be the law of the unique strong solution to the McKean--Vlasov SDE \eqref{eq:MV-SDE}, started from some initial law $\smash{\mu_0\in\mathcal P_2(\Xi)}$. Let $\smash{\mu_m^{\lambda,\star}}$ denote the unique minimiser of $\mathcal F_m^{\lambda,\mathrm{rep}}$; by Corollary~\ref{cor:stationary-minimiser}, this law is stationary for \eqref{eq:MV-PDE}. Suppose moreover that
\begin{equation}
\delta:=\alpha-L_\mu>0,
\qquad\text{where}\qquad
\alpha=\lambda\kappa-\eta L_r - mL_G,\quad L_\mu \text{ is as in Lemma~\ref{lem:drift-Lip}},
\label{eq:delta-recall}
\end{equation}
Define $G^\star:=\sup_{\xi_{1:m}\in\Xi^m}G(\xi_{1:m})<\infty$ and $\Delta_{\mathrm{rel}}
:=
G^\star-\mathcal J_m(\mu_m^{\lambda,\star})\ge 0$. In addition, define $B_{\mathrm{BoN}}
:=
L_{\mathrm{BoN}} \mathsf W_2(\mu_0,\mu_m^{\lambda,\star})$, where
\begin{equation}
\label{eq:l-bon-def}
L_{\mathrm{BoN}}
:=
C_m\Big(
1+\sqrt{M_2(\mu_m^{\lambda,\star})}
+\sqrt{2M_2(\mu_m^{\lambda,\star})+2\mathsf W_2^2(\mu_0,\mu_m^{\lambda,\star})}
\Big),
\end{equation}
and where $C_m$ is the constant from Proposition~\ref{prop:app:Jm-W2}. Fix $t\ge 0$, draw ${\xi_{1:m}^{(1)},\dots,\xi_{1:m}^{(n)}\stackrel{\mathrm{i.i.d.}}{\sim}\mu_t^{\otimes m}}$, and define 
\begin{equation}
I_n(t):=\min\argmax_{1\le r\le n} G(\xi_{1:m}^{(r)}), \qquad \widehat{\xi}_{1:m,t}^{(n)}:=\xi_{1:m}^{(I_n(t))}. 
\end{equation}
Then, for every $\varepsilon>0$, it holds that
\begin{equation}
\label{eq:bon-wgf-main}
\mathbb P \left(
G(\widehat{\xi}_{1:m,t}^{(n)})\ge G^\star-\varepsilon
\right)
\ge
1-
\left(
\min \left\{
1, 
\frac{
\Delta_{\mathrm{rel}}
+
L_{\mathrm{BoN}} \mathsf W_2(\mu_t,\mu_m^{\lambda,\star})
}{\varepsilon}
\right\}
\right)^n.
\end{equation}
Consequently, by Lemma~\ref{lem:MV-contractivity},
\begin{equation}
\label{eq:bon-wgf-explicit}
\mathbb P \left(
G(\widehat{\xi}_{1:m,t}^{(n)})\ge G^\star-\varepsilon
\right)
\ge
1-
\left(
\min \left\{
1, 
\frac{\Delta_{\mathrm{rel}}+B_{\mathrm{BoN}}e^{-\delta t}}{\varepsilon}
\right\}
\right)^n.
\end{equation}
\end{proposition}

\begin{proof}
We begin by applying Proposition~\ref{prop:bon-general} to $\nu_t:=\mu_t^{\otimes m}$. In particular, from the bound in \eqref{eq:bon-gap-prob}, we have that
\begin{equation}
\mathbb P (G(\widehat\xi_{1:m,t}^{(n)})\ge G^\star-\varepsilon)
\ge 1-\big(\min\{1,\gamma_t/\varepsilon\}\big)^n, \qquad \gamma_t:=G^\star-\mathcal J_m(\mu_t).
\end{equation}
We now seek an upper bound for $\gamma_t$. First observe that
\begin{equation}
\gamma_t
=\Delta_{\mathrm{rel}}+\mathcal J_m(\mu_m^{\lambda,\star})-\mathcal J_m(\mu_t)
\le \Delta_{\mathrm{rel}}+|\mathcal J_m(\mu_t)-\mathcal J_m(\mu_m^{\lambda,\star})|.
\end{equation}
In addition, by Proposition~\ref{prop:app:Jm-W2}, we have that
\begin{equation}
\big|\mathcal J_m(\mu_t)-\mathcal J_m(\mu_m^{\lambda,\star})\big|
\le C_m\Bigl(1+\sqrt{M_2(\mu_t)}+\sqrt{M_2(\mu_m^{\lambda,\star})}\Bigr)
\mathsf W_2(\mu_t,\mu_m^{\lambda,\star}).
\end{equation}
Meanwhile, via the triangle inequality and Lemma~\ref{lem:MV-contractivity}, 
\begin{equation}
M_2(\mu_t)
\le 2M_2(\mu_m^{\lambda,\star})+2\mathsf W_2^2(\mu_t,\mu_m^{\lambda,\star})
\le 2M_2(\mu_m^{\lambda,\star})+2\mathsf W_2^2(\mu_0,\mu_m^{\lambda,\star}),
\end{equation}
Combining these three displays, we thus have that
\begin{equation}
\gamma_t\le \Delta_{\mathrm{rel}}+L_{\mathrm{BoN}} \mathsf W_2(\mu_t,\mu_m^{\lambda,\star}),
\end{equation}
with $L_{\mathrm{BoN}}$ is the constant defined in \eqref{eq:l-bon-def}. This proves \eqref{eq:bon-wgf-main}. Finally, \eqref{eq:bon-wgf-explicit} follows directly from Lemma~\ref{lem:MV-contractivity}.
\end{proof}

\begin{corollary}[Best-of-$n$ extraction with approximate scores along the i.i.d.\ Wasserstein gradient flow]
\label{cor:bon-noisy-wgf}
Suppose that the assumptions of Proposition~\ref{prop:bon-wgf} hold. Fix $t\ge 0$, draw $\smash{\xi_{1:m}^{(1)},\dots,\xi_{1:m}^{(n)}\stackrel{\mathrm{i.i.d.}}{\sim}\mu_t^{\otimes m}}$, let ${\widehat G_1,\dots,\widehat G_n}$ be measurable scores, and define 
\begin{equation}
\widetilde I_n(t):=\min\argmax_{1\le r\le n}\widehat G_r, \qquad \widetilde{\xi}_{1:m,t}^{(n)}:=\xi_{1:m}^{(\widetilde I_n(t))}.
\end{equation}
Suppose that for some $\eta>0$ and $\delta_{\mathrm{score}}\in[0,1]$, we have $\mathbb P (\max_{1\le r\le n}\big|\widehat G_r-G(\xi_{1:m}^{(r)})\big|\le \eta)\ge 1-\delta_{\mathrm{score}}$. Then, for every $\varepsilon>0$, it holds that
\begin{equation}
\label{eq:bon-noisy-wgf-main}
\mathbb P \left(
G(\widetilde{\xi}_{1:m,t}^{(n)})\ge G^\star-\varepsilon-2\eta
\right)
\ge
1-\delta_{\mathrm{score}}
-
\left(
\min \left\{
1, 
\frac{
\Delta_{\mathrm{rel}}
+
L_{\mathrm{BoN}} \mathsf W_2(\mu_t,\mu_m^{\lambda,\star})
}{\varepsilon}
\right\}
\right)^n.
\end{equation}
Consequently,
\begin{equation}
\label{eq:bon-noisy-wgf-explicit}
\mathbb P \left(
G(\widetilde{\xi}_{1:m,t}^{(n)})\ge G^\star-\varepsilon-2\eta
\right)
\ge
1-\delta_{\mathrm{score}}
-
\left(
\min \left\{
1, 
\frac{\Delta_{\mathrm{rel}}+B_{\mathrm{BoN}}e^{-\delta t}}{\varepsilon}
\right\}
\right)^n.
\end{equation}
\end{corollary}

\begin{proof}
The proof is an immediate consequence of Corollary~\ref{cor:bon-noisy} and Proposition~\ref{prop:bon-wgf}. In particular, we apply Corollary~\ref{cor:bon-noisy} with $\smash{\nu_m=\mu_t^{\otimes m}}$, and then bound $\smash{(1-\mu_t^{\otimes m}(A_\varepsilon))^n}$ using Proposition~\ref{prop:bon-wgf}.
\end{proof}

\begin{remark}[EVI formulation]
\label{rem:bon-evi}
For the $\mathrm{EVI}_\alpha$ gradient flow in Theorem~\ref{thm:EVI}, an identical argument to the one used in the proof of Proposition~\ref{prop:bon-wgf}, together with the quadratic-growth bound \eqref{eq:quad-growth} and the exponential free-energy decay from Theorem~\ref{thm:EVI}, yields in place of \eqref{eq:bon-wgf-main}
\begin{equation}
\mathbb P \left(
G(\widehat{\xi}_{1:m,t}^{(n)})\ge G^\star-\varepsilon
\right)
\ge
1-
\left(
\min \left\{
1, 
\frac{
\Delta_{\mathrm{rel}}
+
L_{\mathrm{BoN}}
\sqrt{\frac{2}{\alpha}\bigl(\mathcal F_m^{\lambda,\mathrm{rep}}(\mu_t)- \mathcal F_m^{\lambda,\mathrm{rep}}(\mu_m^{\lambda,\star})\bigr)}
}{\varepsilon}
\right\}
\right)^n,
\end{equation}
under the weaker condition $\alpha=\lambda\kappa-\eta L_r - mL_G>0$, and the assumption that $\smash{\mathcal F_m^{\lambda,\mathrm{rep}}(\mu_0)<\infty}$. Consequently, defining $\widetilde B_{\mathrm{BoN}}
:=
L_{\mathrm{BoN}}
\sqrt{\frac{2}{\alpha}\bigl(\mathcal F_m^{\lambda,\mathrm{rep}}(\mu_0)-\mathcal  F_m^{\lambda,\mathrm{rep}}(\mu_m^{\lambda,\star}) \bigr)}$, one obtains in place of \eqref{eq:bon-wgf-explicit}
\begin{equation}
\mathbb P \left(
G(\widehat{\xi}_{1:m,t}^{(n)})\ge G^\star-\varepsilon
\right)
\ge
1-
\Big(
\min \Big\{
1, 
\frac{\Delta_{\mathrm{rel}}+\widetilde B_{\mathrm{BoN}}e^{-\alpha t}}{\varepsilon}
\Big\}
\Big)^n.
\end{equation}
Similarly, an analogous modification yields the corresponding version of \eqref{eq:bon-noisy-wgf-main} and \eqref{eq:bon-noisy-wgf-explicit} for approximate utility scores.
\end{remark}

\subsubsection{Zero-temperature limits}
\label{sec:zero-temp-iid}

For the i.i.d.\ family, convergence to the \emph{true} joint maximiser can only be expected in the non-repulsive case $\eta=0$. In this subsection we therefore consider
\begin{equation}
\mathcal F_m^\lambda(\mu):=-\mathcal J_m(\mu)+\lambda \mathrm{KL}(\mu\|\rho),
\qquad \mu\in\mathcal P_2(\Xi).
\end{equation}

\begin{proposition}[Zero-temperature limit in the i.i.d.\ family under a diagonal maximiser]
\label{prop:iid-zero-temp}
Suppose that Assumptions~\ref{ass:V} and \ref{ass:G} hold. Suppose, in addition, that there exists $\bar\xi^\star\in\Xi$ such that $G(\bar\xi^\star,\dots,\bar\xi^\star)
=
G_m^\star
:=
\max_{\xi_{1:m}\in\Xi^m}G(\xi_{1:m})$. Then, for any $\mu_m^{\lambda,\star}\in\argmin_{\mu\in\mathcal{P}_2(\Xi)}\mathcal{F}_m^{\lambda}$, as $\lambda\downarrow 0$, 
\begin{equation}
\label{eq:limit-1-iid}
\mathcal F_m^\lambda(\mu_m^{\lambda,\star})
\longrightarrow
-G_m^\star, \qquad
\mathcal J_m(\mu_m^{\lambda,\star})
\longrightarrow
G_m^\star.
\end{equation}
If, in addition, $\bar\xi^{\star}$ is isolated in the sense that for every neighbourhood $U\ni\bar\xi^\star$, it holds that $G_m^\star-\sup_{\xi_{1:m}\notin U^m}G(\xi_{1:m})>0$,  then
\begin{equation}
\label{eq:limit-2-iid}
\mu_m^{\lambda,\star}\Rightarrow \delta_{\bar\xi^\star}
\qquad\text{as }\lambda\downarrow0.
\end{equation}
\end{proposition}

\begin{proof}
The argument is the same as in Proposition~\ref{prop:mf-zero-temp}, but now restricted to the i.i.d.
family $\{\mu^{\otimes m}:\mu\in\mathcal P_2(\Xi)\}$. For completeness, we sketch the two steps.

First, $\mathcal J_m(\mu)\le G_m^\star$ for every $\mu$, and hence $\mathcal F_m^\lambda(\mu)\ge -G_m^\star$. Conversely, for any $\varepsilon>0$, continuity of $G$ at $(\bar\xi^\star,\dots,\bar\xi^\star)$ yields an open set $U\ni\bar\xi^\star$ such that $G\ge G_m^\star-\varepsilon$ on $U^m$. Define
\begin{equation}
\bar\mu^\varepsilon(d\xi):=\frac{\mathbf 1_U(\xi)}{\rho(U)} \rho(d\xi).
\end{equation}
We then have that
\begin{equation}
\mathcal F_m^\lambda(\bar\mu^\varepsilon)
\le -(G_m^\star-\varepsilon)+\lambda\log\frac{1}{\rho(U)}.
\end{equation}
It follows that $\mathcal F_m^\lambda(\mu_m^{\lambda,\star})\to -G_m^\star$ and, since the entropy term is nonnegative, also that $\mathcal J_m(\mu_m^{\lambda,\star})\to G_m^\star$. This proves \eqref{eq:limit-1-iid}.

For the weak convergence result, fix a neighbourhood $U\ni\bar\xi^\star$. In addition, define $\eta_U:=G_m^\star-\sup_{\Xi^m\setminus U^m}G>0$. We then have that
\begin{equation}
G_m^\star-\mathcal J_m(\mu_m^{\lambda,\star})
\ge \eta_U\bigl(1-\mu_m^{\lambda,\star}(U)^m\bigr).
\end{equation}
The left-hand side tends to $0$, so $\mu_m^{\lambda,\star}(U)\to1$ for every neighbourhood $U$ of $\bar\xi^\star$. This is equivalent to $\mu_m^{\lambda,\star}\Rightarrow\delta_{\bar\xi^\star}$, proving
\eqref{eq:limit-2-iid}.
\end{proof}

\begin{remark}[Why the diagonal assumption is needed in the i.i.d.\ case]
\label{rem:iid-zero-temp-caveat}
The diagonality assumption in Proposition~\ref{prop:iid-zero-temp} is not merely technical. In general, the i.i.d.\ family $\{\mu^{\otimes m}:\mu\in\mathcal P_2(\Xi)\}$ cannot represent an arbitrary deterministic batch $\delta_{\xi_1^\star}\otimes\cdots\otimes\delta_{\xi_m^\star}$ unless $\xi_1^\star=\cdots=\xi_m^\star$. Therefore, without a diagonal optimal batch, one cannot expect the zero-temperature limit to recover the \emph{true} joint maximiser of $G$; at best it can recover the restricted optimum $\sup_{\mu\in\mathcal P_2(\Xi)}\mathcal J_m(\mu)$. This suggests that the additional ``extraction'' step used in our numerical experiments is particularly important for the i.i.d. methods.
\end{remark}
\section{Additional Experimental Details}
\label{sec:add-experimental-details}

\subsection{1D Benchmark with Multimodal Observation Model}
\label{sec:add-experimental-details-1d-benchmark}

\paragraph{Methods.} We compared GA with multiple restarts against the WGF with multiple independent chains. We used $200$ restarts and $200$ chains, respectively. For both methods, we used $10{,}000$ iterations and a constant step size of $\gamma=0.1$. For the WGF, we set the temperature $\lambda=0.05$. 

\paragraph{Initialisation.} We considered two initialisations for both methods. For the global initialisation, we drew initial designs uniformly from $[-3.5,3.5]$. For the local initialisation, we used designs uniformly distributed on $[\xi_{\mathrm{loc},2}-0.2,\xi_{\mathrm{loc},2}+0.2]$. 

\paragraph{Evaluation.} For the downstream validation, we computed results using $5{,}000$ simulated trials.

 \subsection{2D Non-Linear Sensor Placement with Multimodal Priors}
 \label{sec:additional-experimental-details-2d-non-linear}
 \paragraph{Methods.} We compared SGA with multiple restarts against the WGF with multiple independent chains. In both cases, we used a projection to ensure that particles remained in the constraint set $\Xi$. We used $200$ restarts and $200$ chains, respectively. For both methods, we used $5{,}000$ iterations and a constant step size of $\gamma=0.05$. For the WGF, we also set $\lambda=0.05$, and used a Gaussian reference law $\rho=\mathcal N(0,\sigma_\rho^2 I_2)$ with $\sigma_\rho=1.0$.  

 \paragraph{Initialisation.} We considered three initialisation regimes for both methods. The first was a local-box initialisation, with $\xi_1\sim\mathrm{Unif}[-4,0]$ and $\xi_2\sim\mathrm{Unif}[-2,2]$, so that the initial designs were sampled near the minor prior mode. The second was a global initialisation, with designs initialised uniformly over the full domain $[-5,5]^2$. The last was an uninformative initialisation, whereby designs were initially sampled from $[-5,-2]^2$, a corner of the domain far from both the local and global maxima of the EIG. 
 
\paragraph{EIG estimation.} We estimated the EIG and its gradient using a nested Monte Carlo (NMC) estimator. During optimisation, both methods used a low-fidelity estimator with $(n_{\mathrm{outer}},n_{\mathrm{inner}})=(20,50)$. Meanwhile, all reported utilities were computed using a high-fidelity estimator with $(n_{\mathrm{outer}},n_{\mathrm{inner}})=(500,1000)$. 

\paragraph{Extraction Step.} For both methods, we reported the best design visited during the final portion of the run, as selected via a common best-of-$n_{\mathrm{eval}}$ extraction procedure. We first screened the last $n_{\mathrm{eval}}^{+}=500$ iterates using the low-fidelity estimator, retaining the top $n_{\mathrm{eval}}^{-}=5$ candidates per chain, and thereby resulting in a total of $n_{\mathrm{eval}}= 200 n_{\mathrm{eval}}^{-} =1000$ total candidates. We then re-evaluated these candidates using the high-fidelity estimator, and chose the maximiser. This approach mitigates selection bias due to Monte Carlo noise, while remaining computationally tractable.

\paragraph{Evaluation.} We computed posterior-uncertainty comparisons using $100$ repetitions, with $5{,}000$ prior samples to approximate the posterior covariance after one observation. Finally, for visualisation, the EIG landscape was approximated on a $400\times 400$ grid using a Monte Carlo estimator with $(n_{\mathrm{outer}},n_{\mathrm{inner}})=(50,100)$. All experiments were run over $5$ random seeds.

\subsection{Batch Design on the Torus}
 \label{sec:additional-experimental-details-batch-design}
\paragraph{Methods.} We compared GA with multiple restarts against our four WGF-based methods (see Section~\ref{sec:implemented-algorithms}). For all methods, we used $5{,}000$ iterations and a constant step size $\gamma=0.05$. For GA, we used $20$ random restarts. For WGF (Joint), we used $20$ chains and temperature $\smash{\lambda_m=\frac{\lambda}{m}}$ with $\lambda=0.1$. For WGF (MF), we used $20$ particles per coordinate, temperature $\smash{\lambda_m=\frac{\lambda}{m}}$, with $\lambda=0.1$, and $K=2$ Monte Carlo partner draws per iteration. For the WGF (MF-IID), we used $20$ particles, temperature $\lambda=0.1$, and $K=2$ partner draws. For WGF (MF-IID-REP), we used the same settings together with
$\smash{r_{\mathbb T}(z)=({d_{\mathbb T}(z)^2+\delta^2})^{-1}}$, where $d_{\mathbb T}(z)$ denotes the wrapped angular difference on the circle, with repulsion strength $\eta=0.2$, repulsion scale $\delta=0.2$, and $K_{\mathrm{rep}}=2$ repulsion samples per particle per iteration. In all particle methods, the reference measure was uniform on the circle, so no additional confining drift term was used.

\paragraph{Initialisations.} We considered two initialisation regimes for all methods. In the global regime, all methods were initialised uniformly on $[-\pi,\pi)$. In the local regime, all methods were initialised from a wrapped Gaussian distribution centred at $-\pi/2$ with standard deviation $0.2$.

\paragraph{Extraction Step.} We extracted deterministic batches using a common best-of-$n$ extraction step with $n_{\mathrm{eval}}=500$ candidate batches. For \texttt{WGF (Joint)}, candidates were selected from post-burn-in chain states, using a burn-in fraction of $0.8$. For \texttt{WGF (MF)}, candidate batches were formed by sampling one particle independently from each coordinate-wise empirical marginal. For \texttt{WGF (MF-IID)} and \texttt{WGF (MF-IID-REP)}, candidate batches were formed by i.i.d.\ sampling from the final empirical design law. The repeated-best-single baseline was constructed by first maximising the single-design EIG over a grid of $4{,}000$ equally spaced angles and then repeating the resulting design $m$ times. Finally, we reported results averaged over 5 random seeds.

\subsection{Pharmacokinetic and FitzHugh--Nagumo Benchmarks}
\label{sec:additional-experimental-details-pk}

\paragraph{Methods.} We compared our methods to the following baselines: \texttt{Uniform}, \texttt{GeometricDRS}, \texttt{BetaDRS}, \texttt{CE (Feasible Grid)}, \texttt{CE (GP)}, \texttt{CE (GP-G)}, \texttt{SGA (Adam)}, and \texttt{Annealed SMC}. These methods are defined as follows:

\begin{itemize}
    \item \texttt{Uniform} consists of evenly spaced times on $[0,T_{\max}]$.
    \item \texttt{GeometricDRS} is a dimension reduction scheme (DRS) \citep{ryan2014towards,overstall2020bayesian} defined by a geometric schedule $t_j=t_0 r^j$, with $t_0\in[0,T_{\max}]$ and $r\in[1,r_{\max}]$. In our implementation, the parameters were selected by dense random search. 
    \item \texttt{BetaDRS} is a dimension reduction scheme (DRS) \citep{ryan2014towards,overstall2020bayesian} defined by a Beta schedule $t_j=T_{\max}F^{-1}(q_j;\alpha_1,\alpha_2)$ at $q_j=j/(m+1)$. In our implementation, the parameters were selected by dense random search. 
    \item \texttt{CE (Feasible Grid)} is a grid-based coordinate-exchange (CE) type method \citep[e.g.,][]{meyer1995coordinate} that maximises the low-fidelity EIG over a one-dimensional grid within the feasible interval for each coordinate $t_i$, namely, $[t_{i-1}+\Delta_{\min},\,t_{i+1}-\Delta_{\min}]$. In our implementation, we used a feasible one-dimensional grid of size $500$.
    \item \texttt{CE (GP)} is a lightweight approximate coordinate exchange (ACE) style algorithm \citep{overstall2020bayesian} which, at each iteration, fits a 1D GP emulator per coordinate, proposes the maximiser of the emulator mean, and always accepts it.  
    \item \texttt{CE (GP-G)} is a lightweight approximate coordinate exchange (ACE) style algorithm \citep{overstall2020bayesian} which, at each iteration, fits a 1D GP emulator per coordinate, proposes the maximiser of the emulator mean, and accepts it only if the corresponding estimate of the EIG increases.\footnote{We also tested the standard ACE algorithm \citep{overstall2020bayesian}, but found that it was outperformed by these heuristics at low simulation budgets.} 
    \item \texttt{SGA (Adam)} obtains a design by stochastically optimising the batch EIG using the Adam optimiser \citep{kingma2015adam}.
    \item \texttt{Annealed SMC} is a sequential Monte Carlo (SMC) scheme \citep[e.g.,][]{delmoral2006smc,chopin2002sequential} that targets $\smash{p_{\beta}(\xi_{1:m})\propto\exp(\beta\,\widehat{\mathrm{EIG}}_m(\xi_{1:m}))}$, with a linear temperature ladder $\smash{\beta\in[0,\frac{m}{\lambda}]}$, ESS-based resampling, and projected Gaussian random-walk mutation steps.
\end{itemize}

To approximately match the computational cost across methods, we tuned a single parameter for each method to target a common wall-clock budget of $7$ seconds, while holding all other hyperparameters fixed. Specifically, we tuned $n_{\mathrm{random}}$ for \texttt{GeometricDRS} and \texttt{BetaDRS}, $n_{\mathrm{sweeps}}$ for \texttt{CE (Feasible Grid)}, \texttt{CE (GP)}, and \texttt{CE (GP-G)}, $n_{\mathrm{steps}}$ for \texttt{SGA (Adam)} and all \texttt{WGF} variants, and $n_{\mathrm{mcmc}}$ for \texttt{Annealed SMC}. 

For the \texttt{WGF}-based methods, the common particle counts were $50$ chains for \texttt{WGF (Joint)} and \texttt{WGF (Joint) (FUSE)}, $10$ particles per coordinate for \texttt{WGF (MF)} and \texttt{WGF (MF) (Sub)}, and $50$ particles for both \texttt{WGF (MF-IID)} and \texttt{WGF (MF-IID-REP)}. We used $K=1$ for mean-field and i.i.d.\ methods. For \texttt{WGF (MF-IID-REP)}, we used inverse quadratic potential
$r(z)=(z^2+\delta_{\mathrm{rep}}^2)^{-1}$, with $\eta_{\mathrm{rep}}=0.01$, $\delta_{\mathrm{rep}}=1.0$, and $K_{\mathrm{rep}}=2$. For all of the constant-step \texttt{WGF}-based methods, as well as for \texttt{SGA (Adam)}, we use a constant step size of $\gamma=0.01$. For \texttt{WGF (Joint) (FUSE)}, we used the adaptive FUSE schedule with $r_{\varepsilon}=10^{-8}$ \citep{sharrock2025tuning}.\footnote{The guarantees in \cite{sharrock2025tuning} only hold under the assumptions that the target measure is log-concave, and that one has access to an unbiased stochastic gradient oracle. In our setting, where neither of these assumptions typically hold, this approach should thus be used with caution, despite its impressive empirical performance in these examples.}

For the baseline methods, the following additional settings were used. \texttt{GeometricDRS} used a logistic random-search parameterisation over $t_0\in[0,T_{\max}]$ and $r\in[1,r_{\max}]$, with $r_{\max}=2.5$. \texttt{CE (Feasible Grid)} used a feasible one-dimensional grid of size $500$. Both \texttt{CE (GP)} and \texttt{CE (GP-G)} used $n_{\mathrm{starts}}=2$ random starts and $R_{\mathrm{train}}=20$ training points when fitting the GP surrogate, a lengthscale of $2.0$, and a one-dimensional candidate grid of size $200$.\footnote{We additionally tested fitting the GP lengthscale adaptively via marginal-likelihood estimation, but observed only negligible changes in performance relative to fixed lengthscales.} For \texttt{Annealed SMC}, we used $20$ particles, a linear temperature ladder with $20$ temperatures, an ESS threshold of $0.7$, and Gaussian random-walk mutation steps with a scale of $0.5$.

\paragraph{Initialisation.}
We used the following initialisations. \texttt{GeometricDRS} sampled latents uniformly on $[-4,4]^2$, before mapping them to $(t_0,r)$. \texttt{BetaDRS} sampled $(\log\alpha_1,\log\alpha_2)$ uniformly over $[\log(0.1),\log(10.0)]^2$, before exponentiating them to $(\alpha_1,\alpha_2)$. The remaining  methods used a uniform initialisation over $[0,T_{\mathrm{max}}]$, before projecting to the feasible region. In particular, \texttt{CE (Feasible Grid)} started from one random design sampled uniformly over $[0,T_{\max}]^m$. \texttt{CE (GP)} and \texttt{CE (GP-G)} sampled points uniformly over $[0,T_{\max}]^m$ at each restart. \texttt{Annealed SMC} initialized $N_{\text{particles}}$ designs uniformly over $[0,T_{\max}]^m$. \texttt{SGA (Adam)} initialized restart designs by sampling uniformly over $[0,T_{\max}]^m$. \texttt{WGF (Joint)} and \texttt{WGF (Joint) (FUSE)} initialized chains uniformly over $[0,T_{\max}]^m$ (projected). \texttt{WGF (MF)} and \texttt{WGF (MF) (Sub)} initialize coordinate-wise particles uniformly over $[0,T_{\max}]$. Finally, \texttt{WGF (MF-IID)} and \texttt{WGF (MF-IID-REP)} initialized scalar particles uniformly over $[0,T_{\max}]$.

\paragraph{EIG Estimation} For all methods, we estimated the EIG and, if required, its gradient, using an NMC estimator. During optimisation, all methods used a low-fidelity estimator with $(n_{\mathrm{outer}},n_{\mathrm{inner}})=(20,50)$. Meanwhile, in-run high-fidelity scoring used $(n_{\mathrm{outer}},n_{\mathrm{inner}})=(500,1000)$. Finally, the reported EIG values were computed using $(n_{\mathrm{outer}},n_{\mathrm{inner}})=(1000,2000)$ and averaged over $20$ independent replications.

\paragraph{Extraction Step} We extracted deterministic batches using a common best-of-$n_{\mathrm{eval}}$ step with $n_{\mathrm{eval}}=50$ candidate batches. For \texttt{WGF (Joint)} and \texttt{WGF (MF)} variants, candidates were selected from the final $2{,}000$ states. For \texttt{WGF (MF-IID)} and \texttt{WGF (MF-IID-REP)}, candidate batches were formed by i.i.d.\ sampling from the final empirical design law. For \texttt{GeometricDRS} and \texttt{BetaDRS}, the candidate set consisted of the random proposals generated during the low-dimensional parameter search. For the coordinate-exchange methods (\texttt{CE (Feasible Grid)}, \texttt{CE (GP)}, and \texttt{CE (GP-G)}), candidates were taken from the designs visited during the coordinate-wise optimisation sweeps. For \texttt{SGA (Adam)}, candidates were selected from the final $2{,}000$ iterates across all restarts. For \texttt{Annealed SMC}, candidates were taken from the final particle population. The \texttt{Uniform} baseline returned its deterministic design directly and therefore did not require an additional extraction step.

\section{Additional Numerical Results}
\label{sec:add-numerics}

\subsection{Pharmacokinetic Benchmark}
\label{sec:add-numerics-pharma}

We here provide additional results for the pharmacokinetic (PK) benchmark. Unless otherwise specified, the experimental setup is identical to that used to obtain the results in Section~\ref{sec:exp:pk-fhn}. 

\subsubsection{Additional Results for Different Batch Sizes}
~ 
\begin{figure}[h!]
\vspace{-8mm}
    \centering
    \includegraphics[width=.82\linewidth]{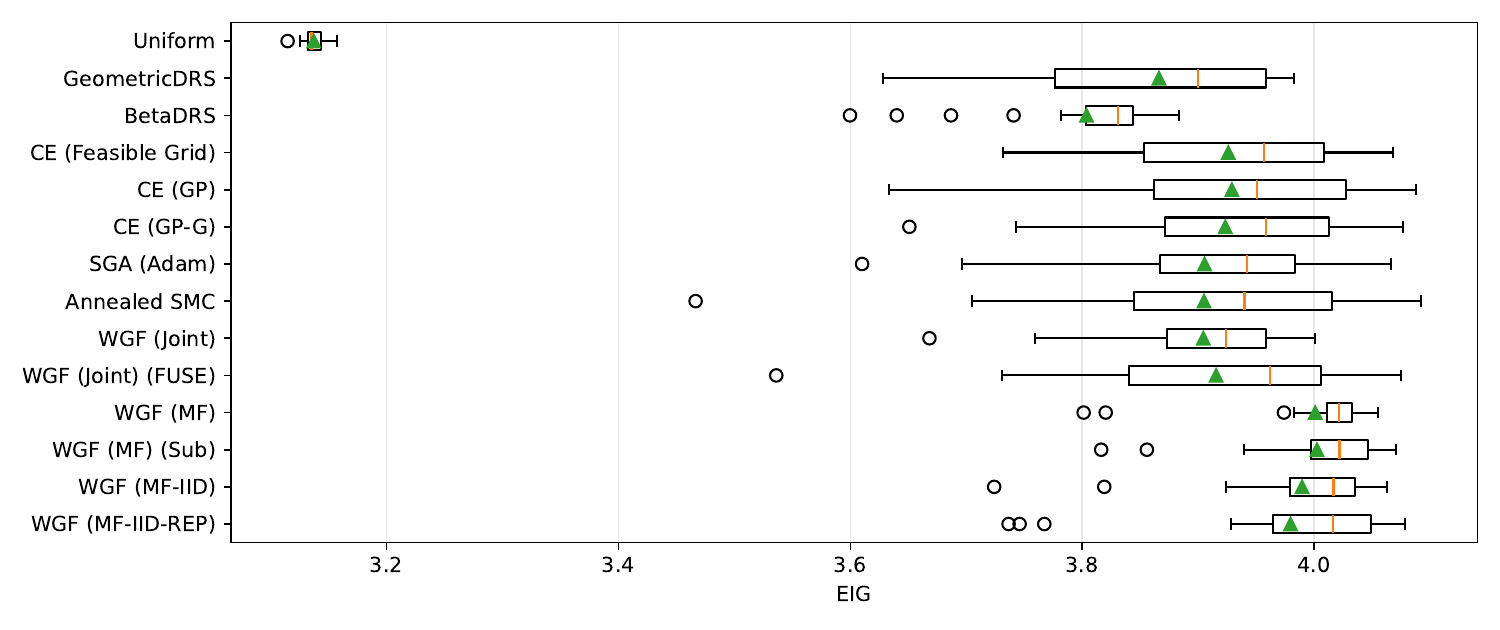}
    \vspace{-4mm}
    \caption{\textbf{EIG summaries for the pharmacokinetic sampling-time benchmark for $m=10$}.}
    \label{fig:app-1}
\vspace{-4mm}
\end{figure}

\begin{figure}[h!]
    \centering
    \includegraphics[width=.82\linewidth]{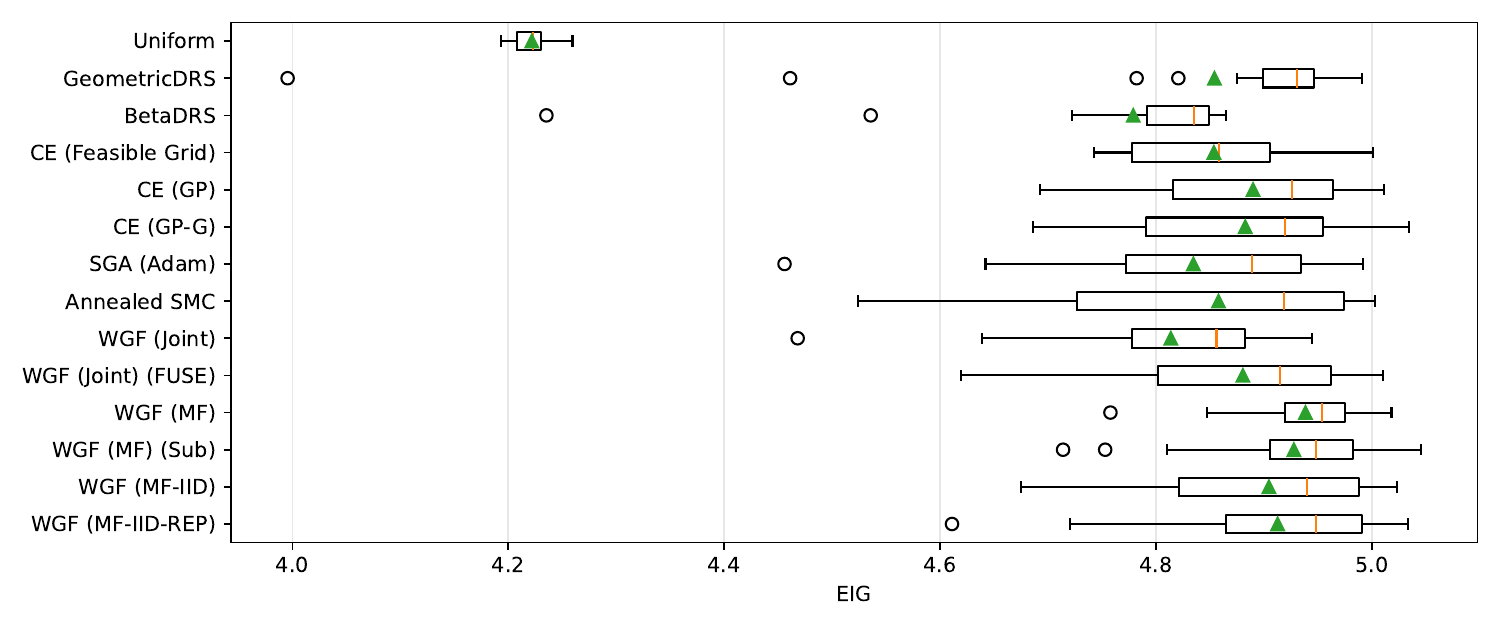}
    \vspace{-4mm}
    \caption{\textbf{EIG summaries for the pharmacokinetic sampling-time benchmark for $m=20$}.}
    \label{fig:app-2}
\vspace{-4mm}
\end{figure}

\subsubsection{Additional Results for Different Step Sizes}

\begin{figure}[h!]
\vspace{-4mm}
    \centering
    \includegraphics[width=.82\linewidth]{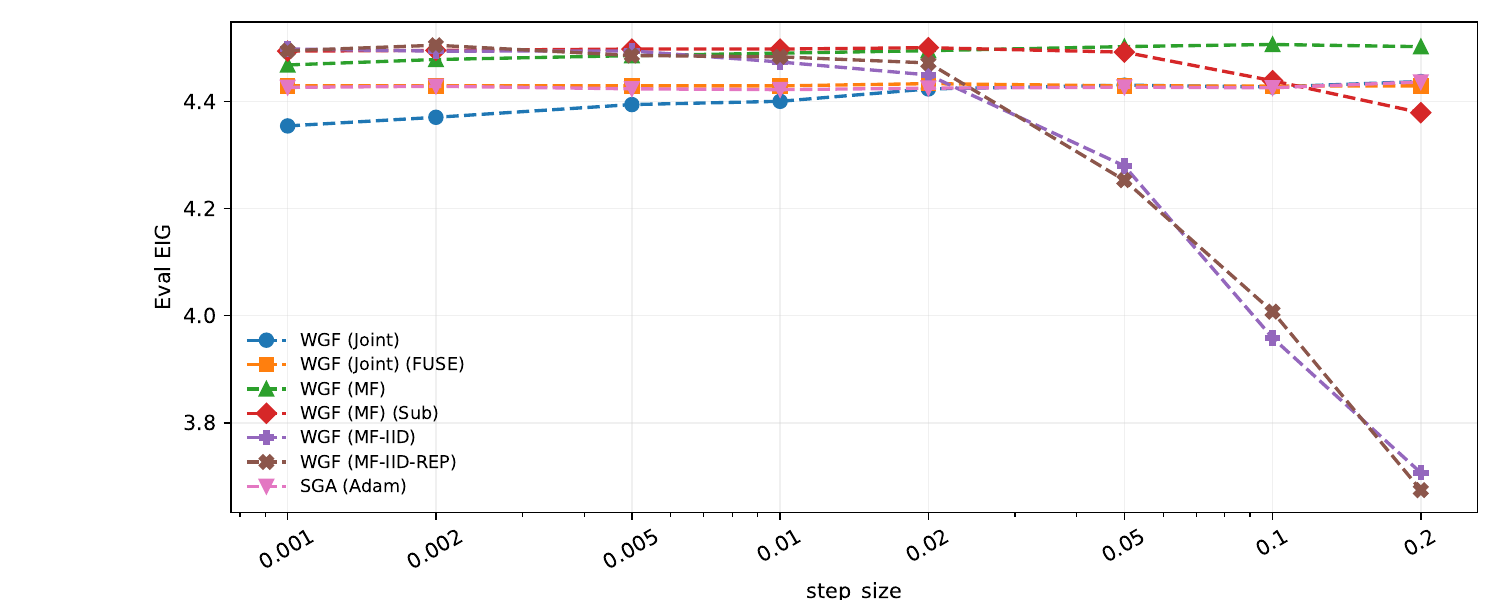}
    \vspace{-2mm}
    \caption{\textbf{EIG versus step-size for the pharmacokinetic sampling-time benchmark}.}
    \label{fig:app-3}
\vspace{-4mm}
\end{figure}

\subsubsection{Additional Results for Different Numbers of Particles}

\begin{figure}[h!]
    \centering
    \includegraphics[width=.84\linewidth]{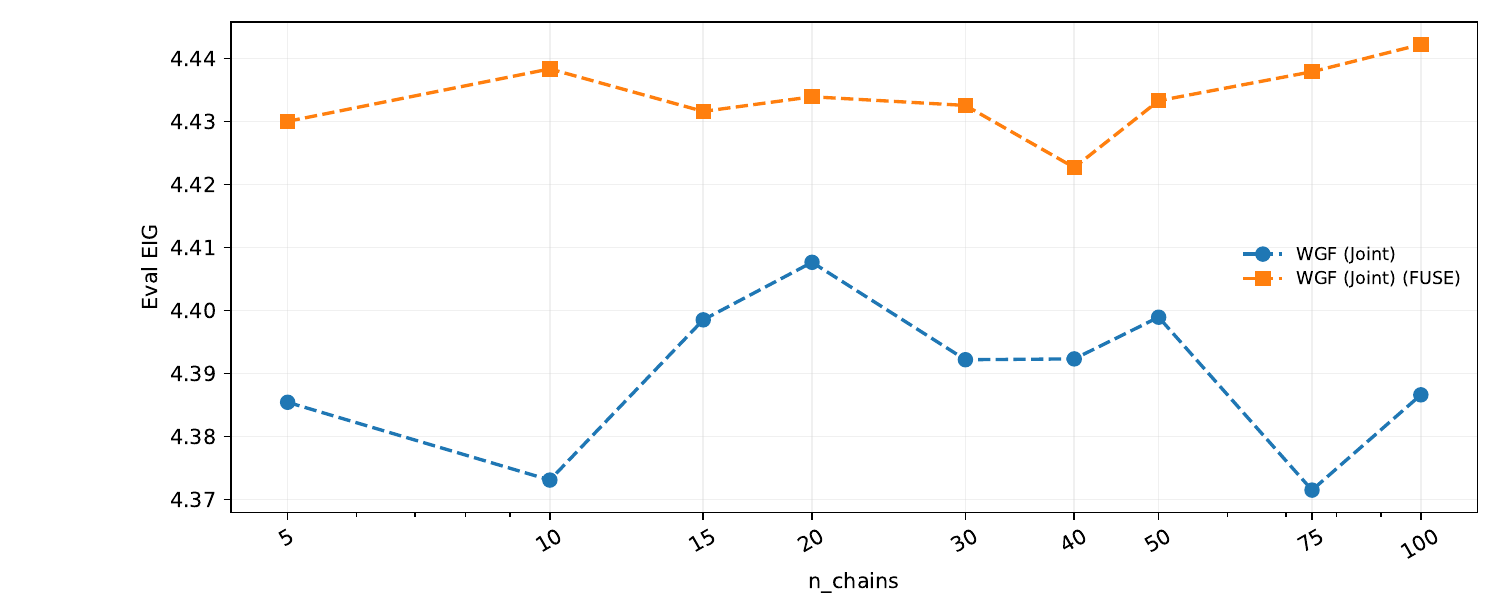}
    \vspace{-2mm}
    \caption{\textbf{EIG versus number of chains used by \texttt{WGF (Joint)} and \texttt{WGF (Joint) (FUSE)} for the pharmacokinetic sampling-time benchmark}.}
    \label{fig:app-4}
\end{figure}

\begin{figure}[h!]
\vspace{-4mm}
    \centering
    \includegraphics[width=.84\linewidth]{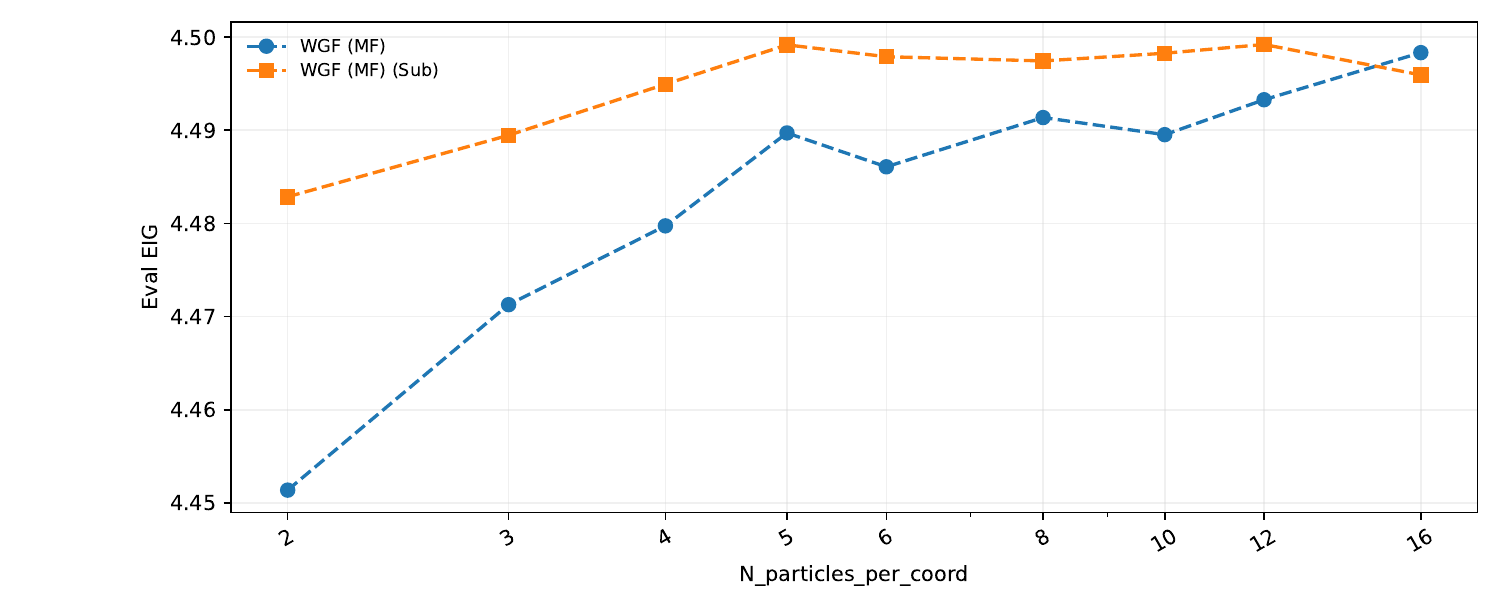}
    \vspace{-2mm}
    \caption{\textbf{EIG versus number of particles per coordinate used by \texttt{WGF (MF)} and \texttt{WGF (MF) (Sub)} for the pharmacokinetic sampling-time benchmark}.}
    \label{fig:app-5}
\end{figure}

\begin{figure}[h!]
\vspace{-4mm}
    \centering
    \includegraphics[width=.84\linewidth]{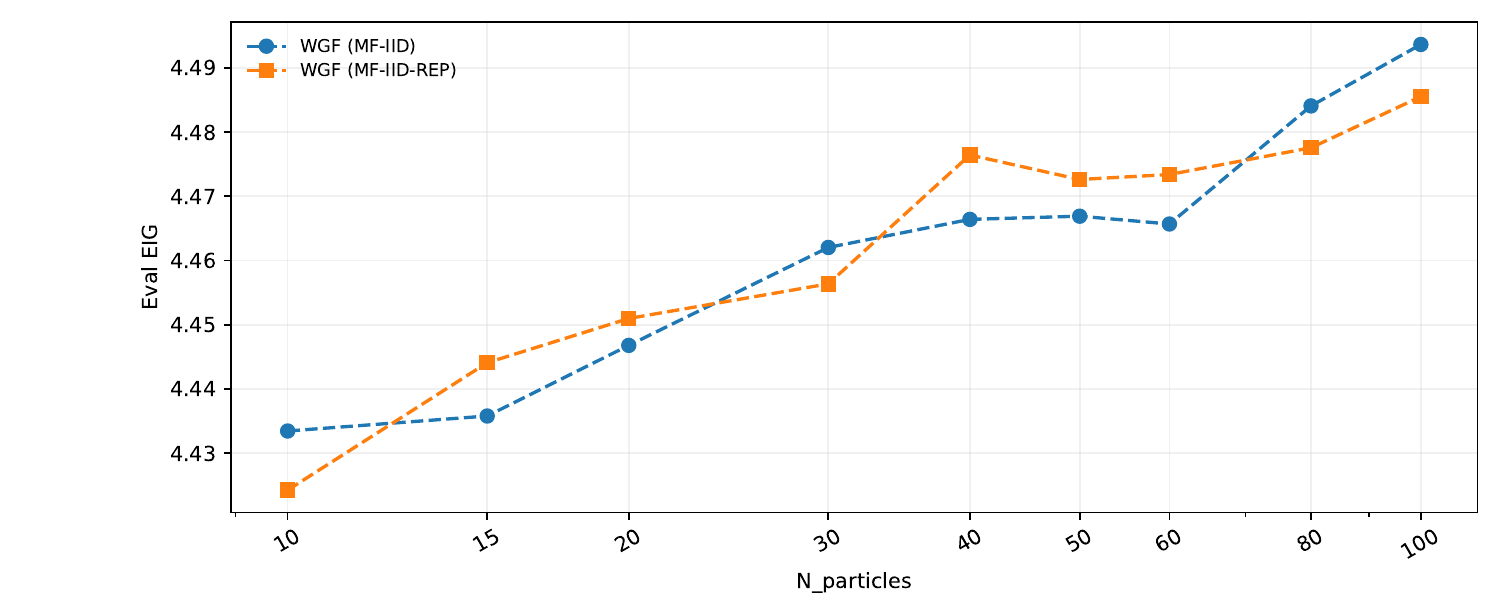}
    \vspace{-2mm}
    \caption{\textbf{EIG versus number of particles used by \texttt{WGF (MF-IID)} and \texttt{WGF (MF-IID-REP)} for the pharmacokinetic sampling-time benchmark}.}
    \label{fig:app-6}
\vspace{-5mm}
\end{figure}

\clearpage
\subsection{FitzHugh--Nagumo Benchmark}
\label{sec:add-numerics-fitzhugh}

We here provide additional results for the pharmacokinetic (PK) benchmark. Unless otherwise specified, the experimental setup is identical to that used to obtain the results in Section~\ref{sec:exp:fhn}. 

\subsubsection{Additional Results for Different Batch Sizes}

\begin{figure}[h!]
\vspace{-2mm}
    \centering
    \includegraphics[width=.84\linewidth]{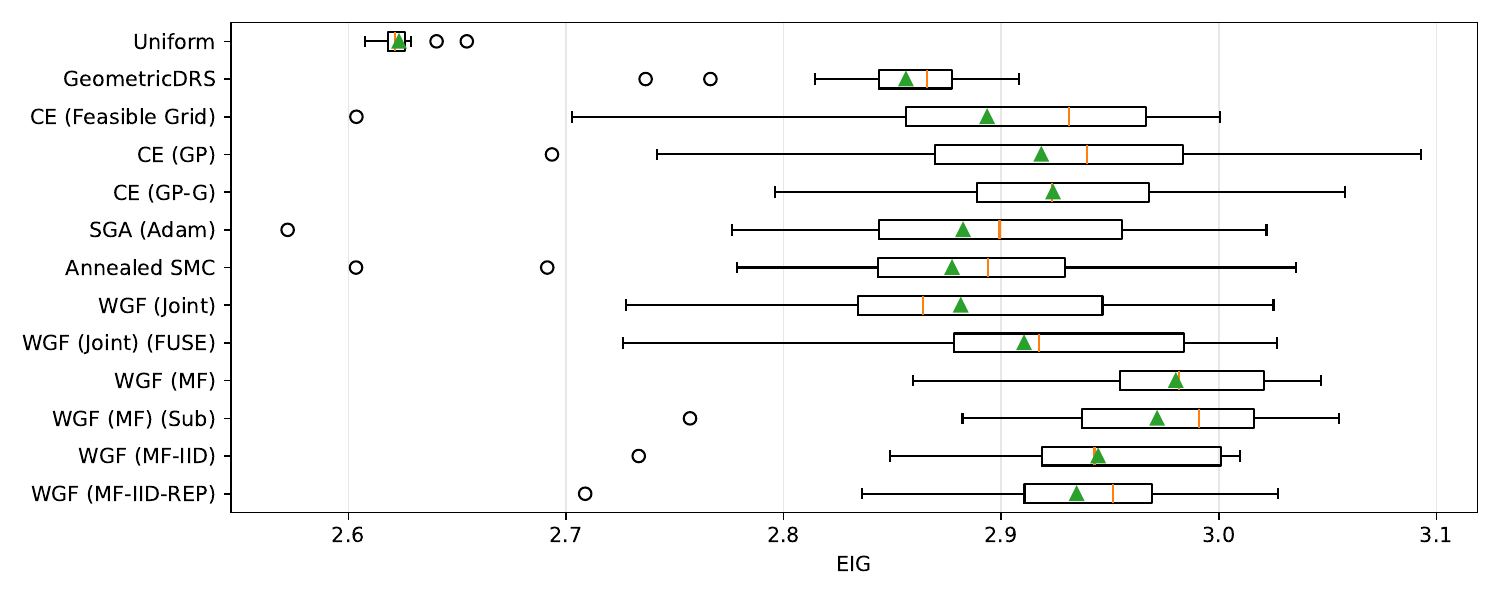}
    \vspace{-4mm}
    \caption{\textbf{EIG summaries for the Fitzhugh--Nagumo sampling-time benchmark for $m=10$}.}
    \label{fig:app-7}
    \vspace{-2mm}
\end{figure}

\begin{figure}[h!]
    \centering
    \includegraphics[width=.84\linewidth]{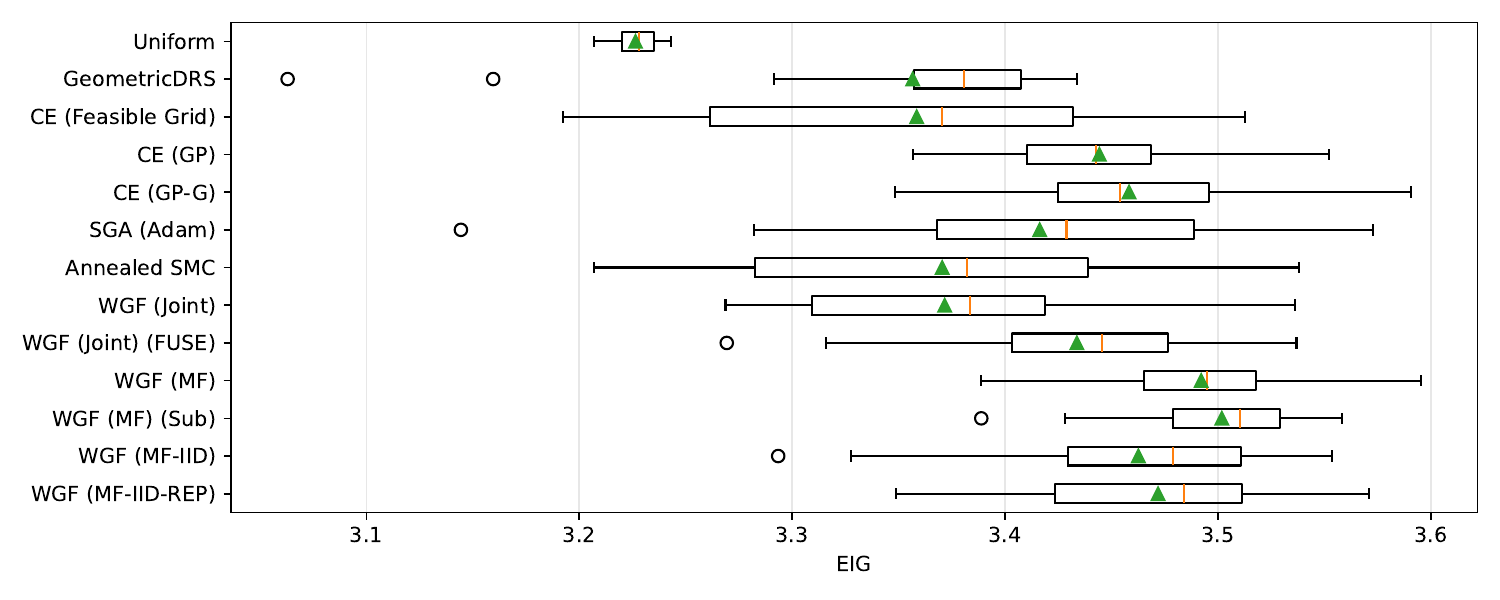}
    \vspace{-4mm}
    \caption{\textbf{EIG summaries for the Fitzhugh--Nagumo sampling-time benchmark for $m=15$}.}
    \label{fig:app-8}
    \vspace{-4mm}
\end{figure}

\subsubsection{Additional Results for Different Step Sizes}

\begin{figure}[h!]
\vspace{-4mm}
    \centering
    \includegraphics[width=.84\linewidth]{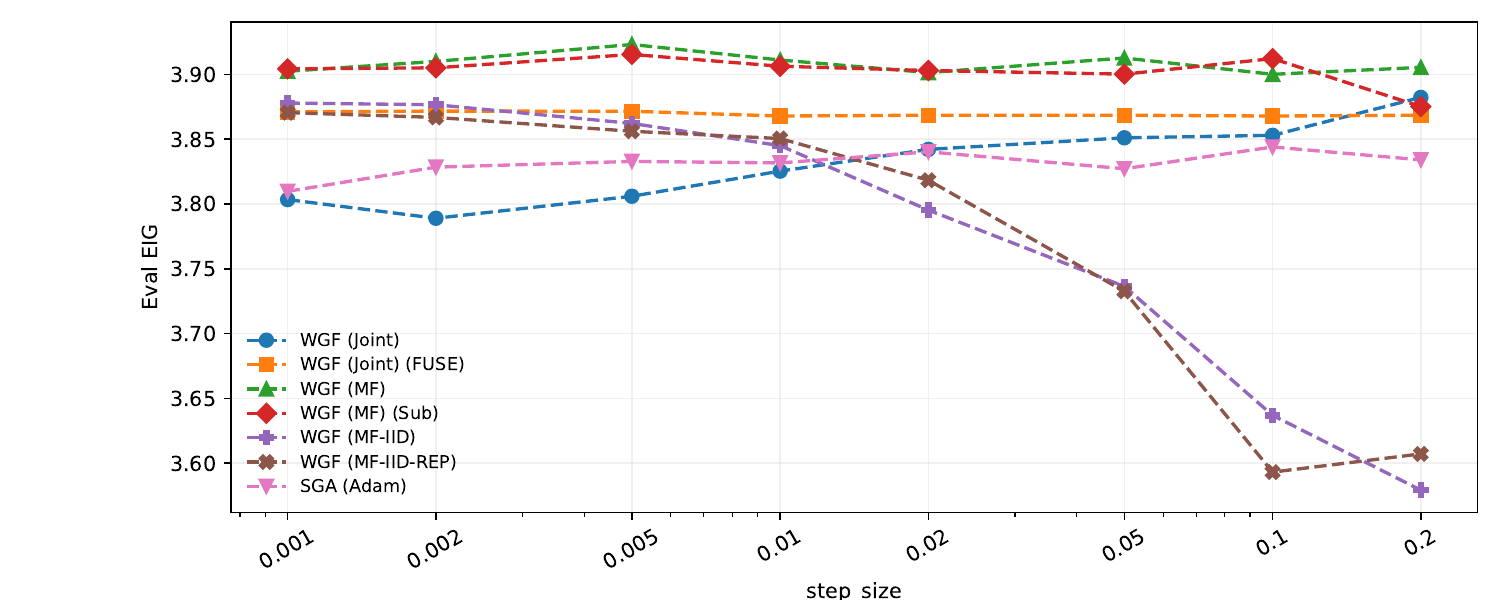}
    \vspace{-2mm}
    \caption{\textbf{EIG versus step-size for the Fitzhugh--Nagumo sampling-time benchmark}. \texttt{WGF (Joint) (FUSE)} does not have a tunable step size parameter, but is included for comparative purposes. }
    \label{fig:app-9}
\vspace{-5mm}
\end{figure}

\clearpage
\subsubsection{Additional Results for Different Numbers of Particles}

\begin{figure}[h!]
    \centering
    \includegraphics[width=.84\linewidth]{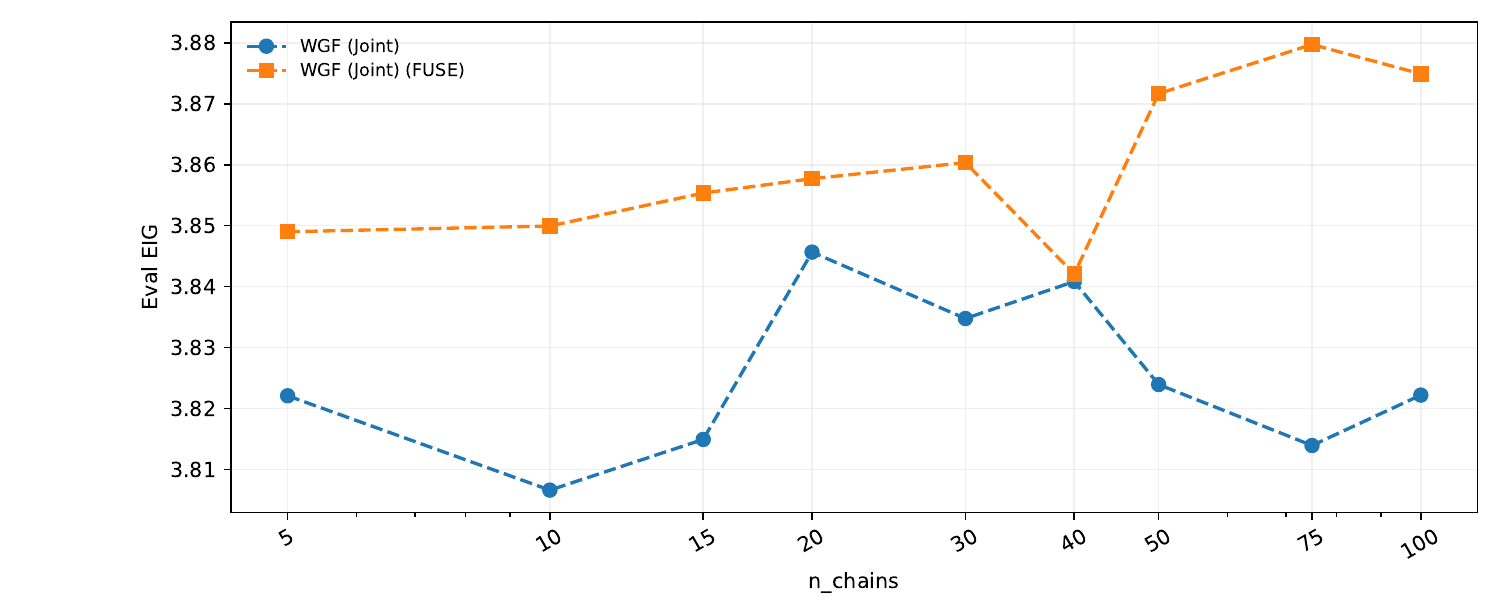}
    \vspace{-2mm}
    \caption{\textbf{EIG versus number of chains used by \texttt{WGF (Joint)} and \texttt{WGF (Joint) (FUSE)} for the Fitzhugh--Nagumo sampling-time benchmark}.}
    \label{fig:app-10}
\end{figure}

\begin{figure}[h!]
\vspace{-4mm}
    \centering
    \includegraphics[width=.84\linewidth]{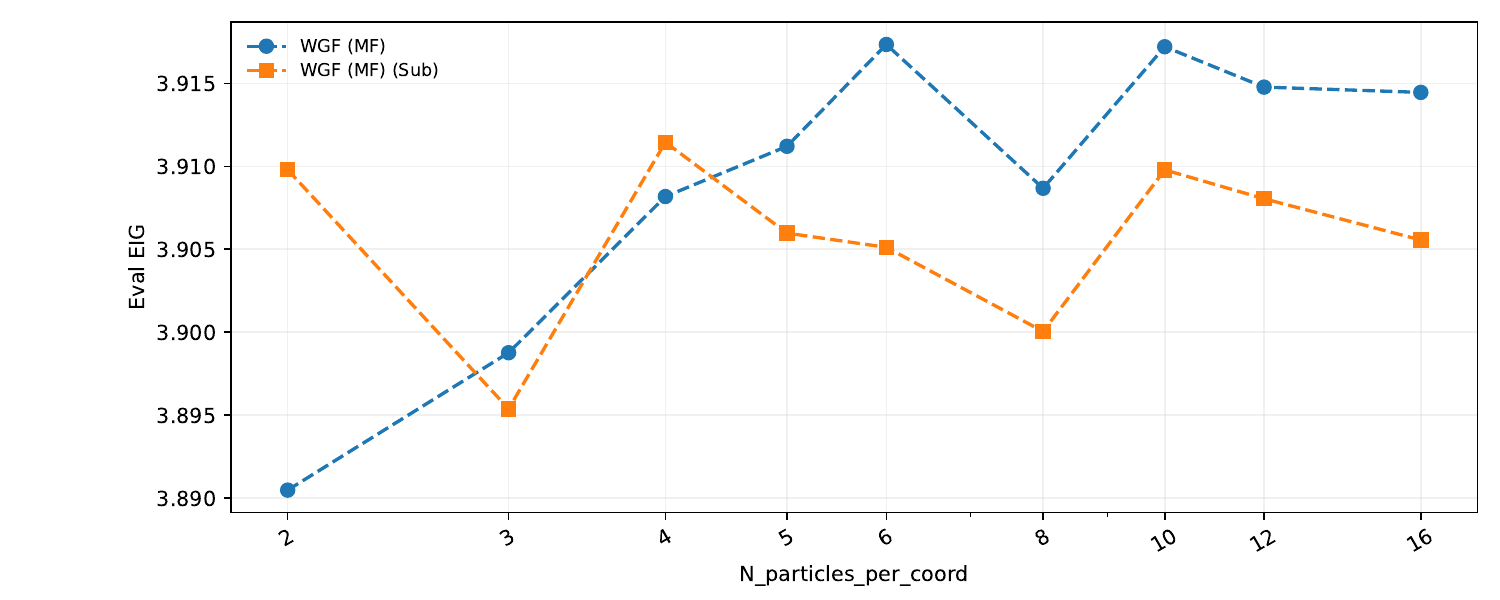}
    \vspace{-2mm}
    \caption{\textbf{EIG versus number of particles per coordinate used by \texttt{WGF (MF)} and \texttt{WGF (MF) (Sub)} for the Fitzhugh--Nagumo sampling-time benchmark}.}
    \label{fig:app-11}
\end{figure}

\begin{figure}[h!]
\vspace{-4mm}
    \centering
    \includegraphics[width=.84\linewidth]{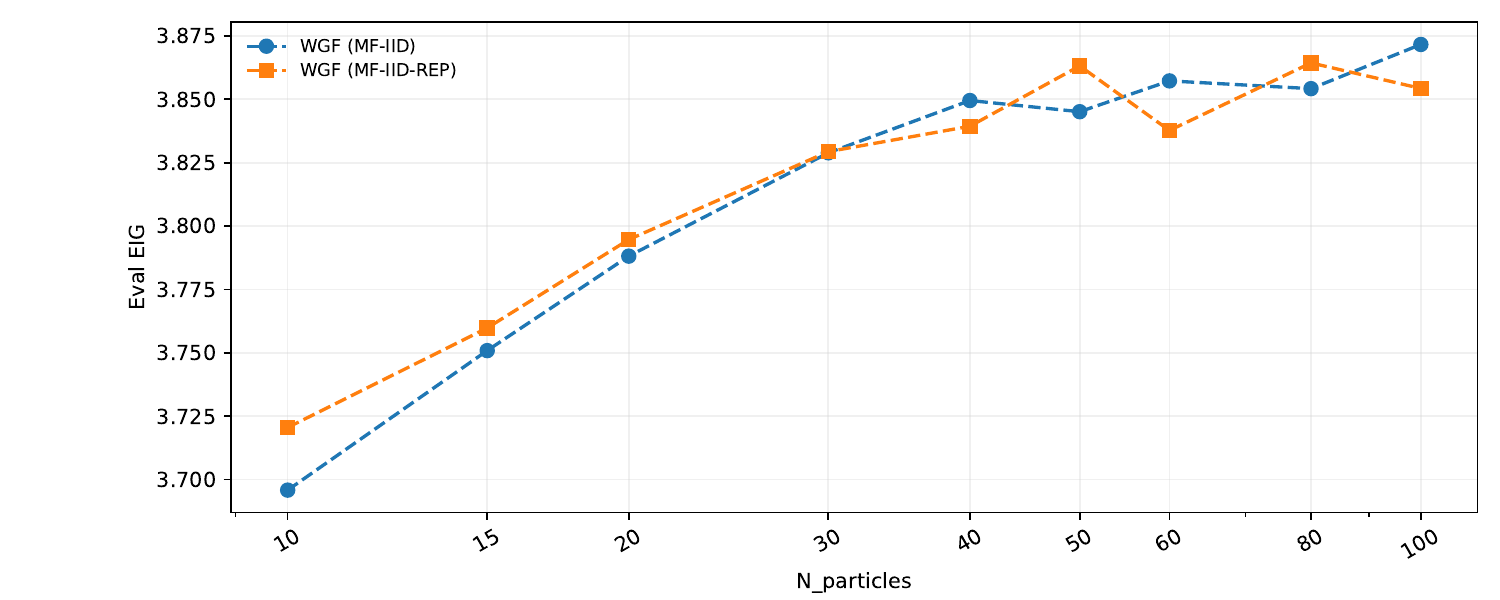}
    \vspace{-2mm}
    \caption{\textbf{EIG versus number of particles used by \texttt{WGF (MF-IID)} and \texttt{WGF (MF-IID-REP)} for the Fitzhugh--Nagumo sampling-time benchmark}.}
    \label{fig:app-12}
\vspace{-5mm}
\end{figure}

\end{document}